\documentclass[10pt]{article}

\usepackage{amsmath,amssymb,amsfonts,mathtools,amsthm}
\usepackage{mathrsfs}
\usepackage{bm}

\usepackage[
    top=1in,
    bottom=1.05in,
    left=0.95in,
    right=0.95in,
    heightrounded
]{geometry}
\usepackage[T1]{fontenc}
\usepackage[utf8]{inputenc}
\usepackage[protrusion=true,expansion=false]{microtype} 
\usepackage{float}
\usepackage[table]{xcolor}
\usepackage{graphicx}
\usepackage{booktabs}
\usepackage{adjustbox}
\usepackage{xcolor}
\usepackage{amssymb}
\usepackage{wrapfig}
\usepackage{subcaption}
\usepackage{enumitem}

\definecolor{darkcerulean}{rgb}{0.03, 0.27, 0.49}
\definecolor{DarkGreenCell}{RGB}{221,239,226}
\definecolor{LightGreenCell}{RGB}{238,247,240}
\definecolor{DarkOrangeCell}{RGB}{247,227,214}
\definecolor{NeutralCell}{RGB}{248,248,248}
\definecolor{DarkRedCell}{rgb}{0.85, 0.30, 0.30}
\definecolor{SoftRedCell}{rgb}{0.98, 0.68, 0.68}
\definecolor{britishracinggreen}{rgb}{0.0, 0.26, 0.15}
\definecolor{burgundy}{rgb}{0.5, 0.0, 0.13}
\definecolor{CurveRed}{RGB}{204,55,74}    
\definecolor{CurveGreen}{RGB}{61,119,47} 

\newcommand{\samepower}[1]{\cellcolor{DarkGreenCell}\textcolor{black}{#1}}
\newcommand{\sameclass}[1]{\cellcolor{LightGreenCell}\textcolor{black}{#1}}
\newcommand{\badjump}[1]{\cellcolor{DarkOrangeCell}\textcolor{black}{#1}}
\newcommand{\neutral}[1]{\cellcolor{NeutralCell}\textcolor{black}{#1}}

\providecommand{\yes}{}

\providecommand{\no}{}

\renewcommand{\yes}{\cellcolor{DarkGreenCell!80}$\checkmark$}

\renewcommand{\no}{\cellcolor{SoftRedCell!50}\textcolor{black}{$\times$}}

\usepackage[
  colorlinks=true,
  linkcolor=blue,
  citecolor=blue,
  urlcolor=blue
]{hyperref}

\usepackage[numbers]{natbib}
\usepackage{xcolor}

\usepackage[linesnumbered,lined,boxed,commentsnumbered,ruled,longend]{algorithm2e}

\definecolor{darkcyan}{rgb}{0.0, 0.55, 0.55}
\definecolor{MidnightBlue}{RGB}{25,25,112}
\definecolor{MidnightBlueComplementingGreen}{RGB}{25,112,25}
\definecolor{MidnightBlueComplementingPurple}{RGB}{112,25,112}
\definecolor{MidnightBlueComplementingRed}{RGB}{112,25,69}
\definecolor{WowColor}{rgb}{.75,0,.75}
\definecolor{MildlyAlarming}{rgb}{0.85,0.25,0.1}
\definecolor{SubtleColor}{rgb}{0,0,.50}
\definecolor{antiquefuchsia}{rgb}{0.57, 0.36, 0.51}
\definecolor{fashionfuchsia}{rgb}{0.96, 0.0, 0.63}
\definecolor{jade}{rgb}{0.0, 0.66, 0.42}
\definecolor{caribbeangreen}{rgb}{0.0, 0.8, 0.6}
\definecolor{aquamarine}{rgb}{0.5, 0.8, 0.85}
\definecolor{lightseagreen}{rgb}{0.13, 0.7, 0.67}
\definecolor{darkgreen}{rgb}{0.0, 0.2, 0.13}
\definecolor{darkspringgreen}{rgb}{0.09, 0.45, 0.27}
\definecolor{attentioncolor}{RGB}{152,90,81}
\definecolor{burgred}{RGB}{40,3,22}
\definecolor{AnnieGreen}{RGB}{17,123,92}
\definecolor{Turquoise}{RGB}{64,224,208}

\definecolor{darkjade}{RGB}{0,122,84}
\definecolor{Window1}{RGB}{92,150,31}%
    \definecolor{Window1dark}{RGB}{41,67,13}%
\definecolor{Window2}{RGB}{255,168,28}
    \definecolor{Window2dark}{RGB}{114,75,12}
\definecolor{Window3}{RGB}{255,96,33}
    \definecolor{Window3dark}{RGB}{97,36,12}
\definecolor{InputColor}{RGB}{20,255,177}
    \definecolor{InputColorlight}{RGB}{222,237,229}

\definecolor{RedAlizarin}{rgb}{0.82, 0.1, 0.26}

\usepackage{mathtools}
\usepackage{hyperref}
\makeatletter
\newcommand{\mytag}[2]{%
  \text{#1}%
  \@bsphack
  \begingroup
    \@onelevel@sanitize\@currentlabelname
    \edef\@currentlabelname{%
      \expandafter\strip@period\@currentlabelname\relax.\relax\@@@%
    }%
    \protected@write\@auxout{}{%
      \string\newlabel{#2}{%
        {#1}%
        {\thepage}%
        {\@currentlabelname}%
        {\@currentHref}{}%
      }%
    }%
  \endgroup
  \@esphack
}
\makeatother

\usepackage{soul}

\usepackage[colorinlistoftodos]{todonotes}

\NewDocumentCommand{\AK}{mo}{
    \IfValueF{#2}{
                        {{
                            \textcolor{magenta}{ 
                            \textbf{AK:}
                            \textit{{#1}}
                            }
                        }}
        }
    \IfValueT{#2}{
                        \marginnote{{\scriptsize
                            \textcolor{magenta}{ 
                            \textbf{AK:}
                            \textit{{#1}}
                            }
                        }}
        }
                    }

\NewDocumentCommand{\TF}{mo}{
    \IfValueF{#2}{
                        {{
                            \textcolor{blue}{ 
                            \textbf{Takashi:}
                            \textit{{#1}}
                            }
                        }}
        }
    \IfValueT{#2}{
                        \marginnote{{\scriptsize
                            \textcolor{blue}{ 
                            \textbf{Takashi:}
                            \textit{{#1}}
                            }
                        }}
        }
                    }

\NewDocumentCommand{\Bum}{mo}{
    \IfValueF{#2}{\textcolor{teal}{\textbf{Bum: }\textit{#1}}}
    \IfValueT{#2}{\marginnote{{\scriptsize\textcolor{teal}{\textbf{Bum: }\textit{#1}}}}}
}

\definecolor{cobalt}{rgb}{0.0, 0.28, 0.67}

\usepackage{graphicx}

\usepackage[T1]{fontenc}    
\usepackage{booktabs}       

\usepackage{amsfonts,amsmath,amssymb,bbm}
\usepackage{enumitem}
\usepackage{graphicx}
\usepackage{adjustbox}
\usepackage[UKenglish]{isodate}
\usepackage{graphicx}
\usepackage[font=small,labelfont=it]{caption}
\usepackage{subcaption}
\usepackage{hyperref} 
    \hypersetup{
    	colorlinks = true,
    	linkcolor = blue,
    	anchorcolor = blue,
    	citecolor = blue,
    	filecolor = blue,
    	urlcolor = blue
    }
\usepackage{cleveref}
\usepackage{float}
\usepackage{enumitem}
\usepackage{multirow}
\usepackage{fancyvrb}
\usepackage{rotating}
    \usepackage{tikz}
    \usepackage{tikz-cd} 
    \pgfdeclarelayer{edgelayer}
    \pgfdeclarelayer{nodelayer}
    \pgfsetlayers{edgelayer,nodelayer,main}
    \tikzstyle{new style 0}=[fill={rgb,255: red,255; green,94; blue,247}, draw=black, shape=circle]
    \tikzstyle{pointy}=[fill=white, draw=black, shape=circle]
    \tikzstyle{pointy}=[->]
\usepackage{fontawesome}

\usepackage{lscape}

\DeclareMathOperator{\supp}{supp}

\renewcommand{\phi}{\varphi}

\let\emptyset\varnothing

\newcommand{\eqdef}{\ensuremath{\stackrel{\mbox{\upshape\tiny def.}}{=}}}

\NewDocumentCommand{\luca}{mo}{
    \IfValueF{#2}{
                        {{\scriptsize
                            \textcolor{green}{ 
                            \textbf{L:}
                            \textit{{#1}}
                            }
                        }}
        }
    \IfValueT{#2}{
                        \marginnote{{\scriptsize
                            \textcolor{green}{ 
                            \textbf{L:}
                            \textit{{#1}}
                            }
                        }}
        }
                    }
\NewDocumentCommand{\giulia}{mo}{
    \IfValueF{#2}{
                        {{\scriptsize
                            \textcolor{red}{ 
                            \textbf{GL:}
                            \textit{{#1}}
                            }
                        }}
        }
    \IfValueT{#2}{
                        \marginnote{{\scriptsize
                            \textcolor{red}{ 
                            \textbf{GL:}
                            \textit{{#1}}
                            }
                        }}
        }
}
\NewDocumentCommand{\anastasis}{mo}{
    \IfValueF{#2}{
                        {{\scriptsize
                            \textcolor{violet}{ 
                            \textbf{A:}
                            \textit{{#1}}
                            }
                        }}
        }
    \IfValueT{#2}{
                        \marginnote{{\scriptsize
                            \textcolor{violet}{ 
                            \textbf{A:}
                            \textit{{#1}}
                            }
                        }}
        }
                    }
                    
\NewDocumentCommand{\cody}{mo}{
    \IfValueF{#2}{
                        {{\scriptsize
                            \textcolor{orange}{ 
                            \textbf{A:}
                            \textit{{#1}}
                            }
                        }}
        }
    \IfValueT{#2}{
                        \marginnote{{\scriptsize
                            \textcolor{orange}{ 
                            \textbf{A:}
                            \textit{{#1}}
                            }
                        }}
        }
                    }

\NewDocumentCommand{\Greg}{mo}{
    \IfValueF{#2}{
                        {{\scriptsize
                            \textcolor{cyan}{ 
                            \textbf{Y:}
                            \textit{{#1}}
                            }
                        }}
        }
    \IfValueT{#2}{
                        \marginnote{{\scriptsize
                            \textcolor{cyan}{ 
                            \textbf{Y:}
                            \textit{{#1}}
                            }
                        }}
        }
                    }

\usepackage{appendix}
\usepackage{comment}

\NewDocumentCommand{\NN}{oo}{
    \ensuremath{
        \mathcal{NN}
        \IfValueT{#1}{_{#1}}\IfValueF{#1}{_{[d]}}
        \IfValueT{#2}{^{#2}}\IfValueF{#2}{^{\sigma}}
    }
}

\usepackage{cleveref}
\newcounter{termcounter}
\renewcommand{\thetermcounter}{\Roman{termcounter}}
\crefname{term}{term}{terms}
\creflabelformat{term}{#2\textup{(#1)}#3}

\makeatletter
\def\term{\@ifnextchar[\term@optarg\term@noarg}
\def\term@optarg[#1]#2{%
  \textup{#1}%
  \def\@currentlabel{#1}%
  \def\cref@currentlabel{[][2147483647][]#1}%
  \cref@label[term]{#2}}
\def\term@noarg#1{%
  \refstepcounter{termcounter}%
  \textup{(\thetermcounter)}%
  \cref@label[term]{#1}}
\makeatother

\NewDocumentCommand{\pa}{m}{\operatorname{pa}_{#1}}
\NewDocumentCommand{\ch}{m}{\operatorname{ch}_{#1}}
\NewDocumentCommand{\comp}{o}{\operatorname{Comp}
    {\IfValueT{#1}{({#1}})}
}
\NewDocumentCommand{\DAG}{o}{
\operatorname{DAG}\IfValueT{#1}{
        _{{#1}}
    }
    \IfValueF{#1}{_{d,D}}
}

\NewDocumentCommand{\Rep}{o}{
{f}
    {\IfValueT{#1}{
        _{{#1}}
    }}
    {\IfValueF{#1}{
        (_{\mathcal{C}})
    }}
}

\newcommand{\bigtimes}{\mathop{\scalebox{2}{$\times$}}}
\usepackage{amsthm}

\theoremstyle{plain}
\newtheorem{theorem}{Theorem}[section]
\newtheorem{proposition}[theorem]{Proposition}
\newtheorem{lemma}[theorem]{Lemma}
\newtheorem{corollary}[theorem]{Corollary}
\theoremstyle{definition}
\newtheorem{definition}[theorem]{Definition}
\newtheorem{assumption}[theorem]{Assumption}
\theoremstyle{remark}
\newtheorem{remark}[theorem]{Remark}

\newcommand{\reals}{\mathbb{R}}

\theoremstyle{plain}
\newtheorem{thm}{Theorem}[section]
\newtheorem{thrminf}{Informal Theorem}[section]
\newtheorem*{thm*}{Theorem}
\newtheorem{cor}[thm]{Corollary}
\newtheorem*{cor*}{Corollary}

\newtheorem*{prop*}{Proposition}

\newtheorem*{fact*}{Fact}
\newtheorem{lem}[thm]{Lemma}
\newtheorem*{lem*}{Lemma}

\newtheorem*{ex*}{Exercise}

\newtheorem*{claim*}{Claim}

\newtheorem*{conj*}{Conjecture}

\newtheorem*{question*}{Question}
\newtheorem*{notation*}{Notation}

\theoremstyle{definition}

\newtheorem*{Def*}{Definition}

\newtheorem*{rmk*}{Remark}

\newtheorem*{soln*}{Solution}

\newtheorem*{note*}{Note}
\newtheorem{eg}[thm]{Example}
\newtheorem*{eg*}{Example}

\newtheorem*{construction*}{Construction}

\newtheorem*{warning*}{Warning}

\newtheorem*{obs*}{Observation}

\newtheorem*{recall*}{Recollection}

\numberwithin{equation}{section}

\definecolor{darkgray}{rgb}{0.66, 0.66, 0.66}
\definecolor{darkchampagne}{rgb}{0.76, 0.7, 0.5}
\definecolor{slighlylesswarmblack}{rgb}{0.0, 0.15, 0.15}
\let\oldproof\proof
\renewcommand{\proof}{\color{slighlylesswarmblack}\oldproof}

\title{Algorithmic Foundations of Deep Learning:
Complexity-Theoretic Rates and a Characterization of Universal Approximation}

\author{
Anastasis Kratsios\thanks{McMaster University \& Vector Institute, Toronto, ON, Canada. 
Email: \texttt{kratsioa@mcmaster.ca}}
\and
Simone Brugiapaglia\thanks{Concordia University, Montr\'eal, QC, Canada. 
Email: \texttt{simone.brugiapaglia@concordia.ca}}
\and
Bum Jun Kim\thanks{University of Tokyo, Tokyo, Japan. Email: \texttt{bumjun.kim@weblab.t.u-tokyo.ac.jp}}
\and
Gregory Cousins\thanks{McMaster University, Hamilton, ON, Canada. Email: \texttt{gcousins@alumni.nd.edu}}
\and
Haitz S\'aez de Oc\'ariz Borde\thanks{University of Cambridge 
\& University of Oxford, United Kingdom; Email: \texttt{chri6704@ox.ac.uk}}
}

\date{}

\begin{document}

\maketitle

\begin{abstract}
Feedforward neural network (NN) expressivity is typically studied by emulating optimal basis-expansion schemes. While powerful, this perspective is incomplete: it primarily captures complexity through regularity, and therefore does not distinguish intuitively simple and complicated objects with comparable regularity, such as the square-root function and a typical Brownian path.

We address the first issue through a quantitative circuit-to-neural-network compilation theorem.  
The guiding message is that \textit{neural networks should be viewed not only as flexible basis functions, but also as models of computation}.
If a function $f:[-1,1]^d\to \mathbb{R}$ is computable to accuracy $\varepsilon>0$ by a real-valued circuit over a prescribed elementary gate language, then $f$ can be computed to comparable accuracy by an NN with explicit depth, width, and non-zero-parameter bounds controlled by the depth, width, gate count, and gate language of the original circuit.  Thus, neural-network complexity is not governed by regularity alone, but also by the complexity of an elementary computation producing $f$.

We then show that any definable feedforward NN model satisfying a natural parallelization condition, allowing possibly multivariate non-linearities such as attention or layer normalization, is a universal approximator if and only if it contains a non-affine nonlinearity.  This gives a general characterization of universality beyond the classical shallow multilayer-perceptron setting, where universality was characterized in terms of the non-polynomiality of the componentwise activation.

The scope of our theory is illustrated by deducing universal approximation guarantees for continuous functions, minimax-optimal approximation guarantees for Besov classes, logarithmic-error complexity for holomorphic functions, and by showing that NNs can emulate numerical algorithms such as Newton-Raphson root finding and power iteration without architecture-specific arguments.  Its precision is illustrated by shortest-path computation on $k$-vertex graphs: compiling the tropical dynamic-programming circuit yields NNs with $\mathcal{O}(\log(1/\varepsilon))$ non-zero parameters, exponentially improving in $\varepsilon^{-1}$ over the generic $\widetilde{\mathcal{O}}(\varepsilon^{-c k^2})$ Lipschitz-approximation scale, for an absolute constant $c>0$.
\end{abstract}

Real-Valued Computation, Deep Learning Theory, O-Minimality, Neural Network Expressivity.

\section{The Gap: Classical Approximation Theory is not Rich Enough}
\label{s:Intro}
Neural-network approximation theory has been deeply influenced by classical ideas in constructive approximation theory organized around regularity: a target function is placed in a H\"{o}lder~\cite{yarotsky2018optimal,kratsios2022universal}, Sobolev, Besov~\cite{suzuki2018adaptivity,petersen2018optimal,LuShenYangZhang_SIAM_2021_OptimalApproxSmooth}, or holomorphic class~\cite{adcock2022deep}, and networks are shown to emulate classical schemes such as splines~\cite{devore1988interpolation,devore1993besov}, wavelets~\cite{wojtaszczyk1997mathematical,cohen2000wavelet,Daubechies}, polynomials~\cite{lorentz1986bernstein,bustamante2017bernstein} and their Stone-Weierstrass-Nachbin generalizations~\cite{nachbin1944extension,stone1948generalized,nachbin1965weighted,prolla1977approximation,cuchiero2026global}, optimal linear combinations using optimally sampled function values~\cite{krieg2021function,dolbeault2023sharp}, sparsity via orthogonal polynomial approximators~\cite{adcock2022sparse}, holomorphy when available~\cite{CohenDeVoreSchwab2011,ChkifaCohenSchwab2014,opschoor2022exponential,doi:10.1137/21M1465718}, sparse basis expansions~\cite{iwen2010combinatorial,choi2021sparseI,choi2021sparseII,cohen2009compressed,rauhut2012sparse}, or other specialized basis expansions~\cite{gonon2023approximation,mao2023rates,siegel2024sharpB,kulbatov2025bases,aftab2025quantum,schneider2025nonlocal}.  
A key advance showcasing the power of depth, namely approximation-rate accelerations that bypass the lower bounds of linear-width theory~\cite{pinkus1985nwidths}
and metric-entropy lower bounds\footnote{Cf.~\cite[page 129]{wainwright2019high}.  This gives a similar result for the Lipschitz widths by the Carl's type inequality in~\citep[Theorem 4.7]{petrova2023lipschitz}.}~%
was pioneered in~\cite{yarotsky2017error} and subsequently optimized in~\cite{shen2022optimal}, by bringing into neural-network approximation the bit-encoding technique of~\cite{bartlett1998almost} 
inspired by the adaptive two-scale cache-channel ideas of~\cite{shannon1949synthesis}.  
This innovation was decisive because it fundamentally relies on networks whose depth diverges with the target accuracy, unlike earlier $\mathcal{O}(1)$-depth methods, which were essentially dictionary/basis-expansion schemes encoded into neural networks.  Thus, bit-encoding showed that depth can exceed the maximal expressivity of any linear method; nevertheless, the arguments in~\cite{yarotsky2017error,shen2022optimal} remained rooted in classical H\"{o}lder-type approximation classes.

While this viewpoint has proven indispensable for orienting questions of the form ``what can neural networks represent?'', these tools remain fundamentally tied to a non-compositional approximation theory designed for basis-expansion methods.  Although the limitations of non-compositional approximation theory have already been observed in~\cite{mhaskar2016deep,poggio2024compositional,danhofer2025position}, and although developments in ``fully non-linear'' approximation theory have begun~\cite{cohen2022optimal,petrova2023lipschitz,dahmen2025compositional}, the resulting abstract approximation classes are still largely rooted in regularity.  The limitation of this regularity lens is clearly illustrated in Figure~\ref{fig:intuitive_low_regularity}, which displays two functions with the same $\tfrac{1}{2}$-H\"{o}lder regularity.  There, we contrast the objectively elementary square-root function $x\mapsto \sqrt{x}$, taught at the pre-university level, with an objectively complicated object: a typical Brownian motion path, which requires advanced probability even to define.  An analogous illustration is given in Figure~\ref{fig:intuitive_high_regularity} in the holomorphic case.

\begin{figure}[htp!]
    \centering
    \begin{subfigure}[t]{0.48\textwidth}
        \centering
        \includegraphics[width=\linewidth]{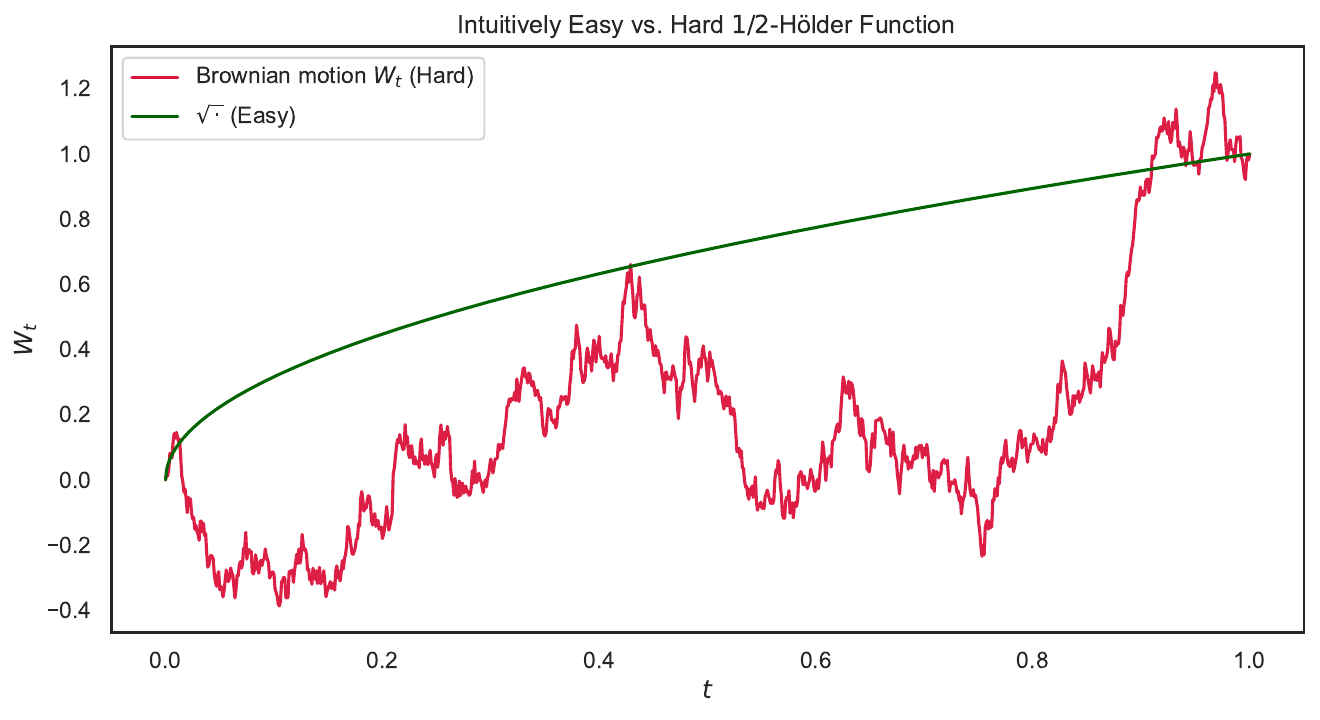}
        \caption{Low regularity: $\tfrac{1}{2}$-H\"{o}lder -- $\sqrt{x}$ function (green) vs.\ typical Brownian motion path (red).
        \hfill\\
        Classical theory predicts that $\mathcal{O}(\varepsilon^{-2})$ neurons are needed for both, while \textit{our theory} predicts $\mathcal{O}(\varepsilon^{-1})$ for $\sqrt{x}$.
        }
        \label{fig:intuitive_low_regularity}
    \end{subfigure}
    ~
    \begin{subfigure}[t]{0.48\textwidth}
        \centering
        \includegraphics[width=\linewidth]{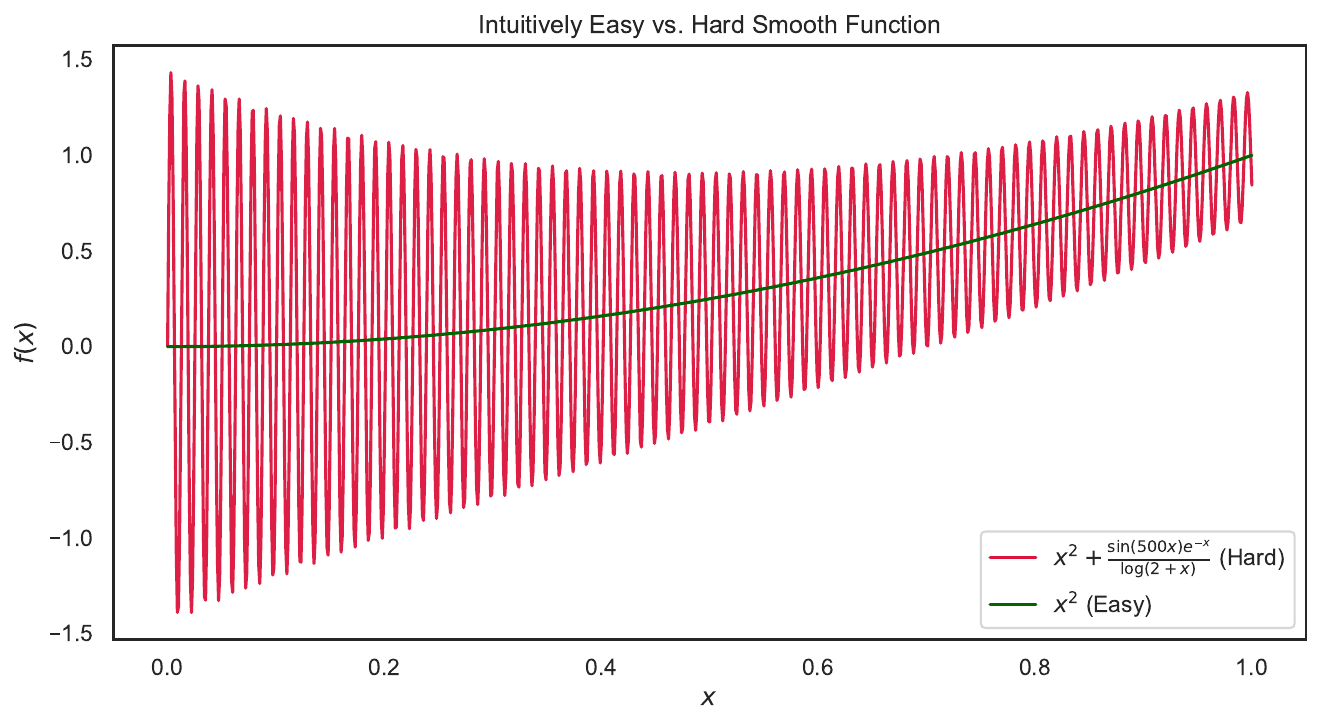}
        \caption{High regularity: Analytic -- $x^2$ function (green) vs.\ complicated analytic function (red) $x^2+\tfrac{\sin(500x)e^{-x}}{\log(2+x)}$.
        \hfill\\
        Classical theory predicts that $\mathcal{O}(\log(\varepsilon^{-1}))$ neurons are needed for both, while \textit{our theory} predicts $\mathcal{O}(1)$ for $x^2$.
        }
        \label{fig:intuitive_high_regularity}
    \end{subfigure}
    \caption{\textbf{The Gap:}
    In both panels, the red function is intuitively more difficult to describe, and hence should be harder to approximate, than its green counterpart.  Nevertheless, the two functions in each panel have the same classical regularity: low regularity in~\subref{fig:intuitive_low_regularity} and high regularity in~\subref{fig:intuitive_high_regularity}.  Thus, the regularity-based lens of classical constructive approximation theory, cf.~\cite{devore1988interpolation,DeVoreLorentz__CABook___Vol1_1993}, or compressed sensing~\cite{adcock2022sparse} (e.g., 
    Besov or holomorphic), can detect neither the simplicity of the {\color{CurveGreen}{green}} curves relative to the hardness of the {\color{CurveRed}{red}} curves, nor whether the approximators under study can exploit this distinction.  
    In contrast, our complexity-theoretic lens (Theorem~\ref{thrm:concrete_surgery}) predicts the gap in network size. 
    \hfill\\
    Here, $\varepsilon>0$ is the uniform approximation error over $[0,1]$.
    }
    \label{fig:intuitive}
\end{figure}

Neural networks can easily approximate the elementary examples but naturally struggle with the more complicated ones; yet the conservative upper bounds of classical neural-network approximation theory fail to predict this gap.  This raises the question: \textit{What structure, accessible through compositionality and depth, is invisible to the classical regularity-based lens?}

Rather than measuring the \textit{hardness} of a function only through its smoothness or regularity class, e.g., Besov or holomorphic regularity, and then approximating it by truncating a basis expansion which is subsequently emulated by a neural network, we measure \textit{hardness} by the \textit{complexity} of a \textit{non-recursive algorithm computing it}.  
  Such an algorithm is built from elementary computations, called gates, drawn from a fixed language.  Examples include the $\mathbb{R}$-algebra operations $\times$, $k\cdot$ for $k\in \mathbb{R}$, and $+$, as well as more expressive operations such as division $\div$, exponentiation $\exp$, radicals $\sqrt{\cdot}$, and tropical operations $\max\{\cdot,\cdot\}$.  These operations capture many of the basic building blocks appearing in numerical analysis, scientific computing, and computer science, and recover the usual Boolean circuit viewpoint when restricted to the Boolean cube $\{0,1\}^B$.  A function is thus ``hard'' precisely when it requires a large elementary recipe to compute uniformly from scratch. 
In short, we propose to augment the traditional purely analytic perspective on ``hardness'' by incorporating combinatorial structures from classical theoretical computer science (TCS), e.g.,~\cite{jukna2012boolean,MR3308677,margulies2016polynomial,jukna2023tropical}, and mathematical logic, e.g.,~\cite{cook2010logical}.

\begin{figure}[H]
\centering
\captionsetup[subfigure]{justification=centering,singlelinecheck=false}
\begin{subfigure}[t]{0.32\linewidth}
    \vspace{0pt}
    \centering
    \begin{minipage}[t][0.16\textheight][c]{\linewidth}
        \centering
        \includegraphics[width=\linewidth,height=0.145\textheight,keepaspectratio]{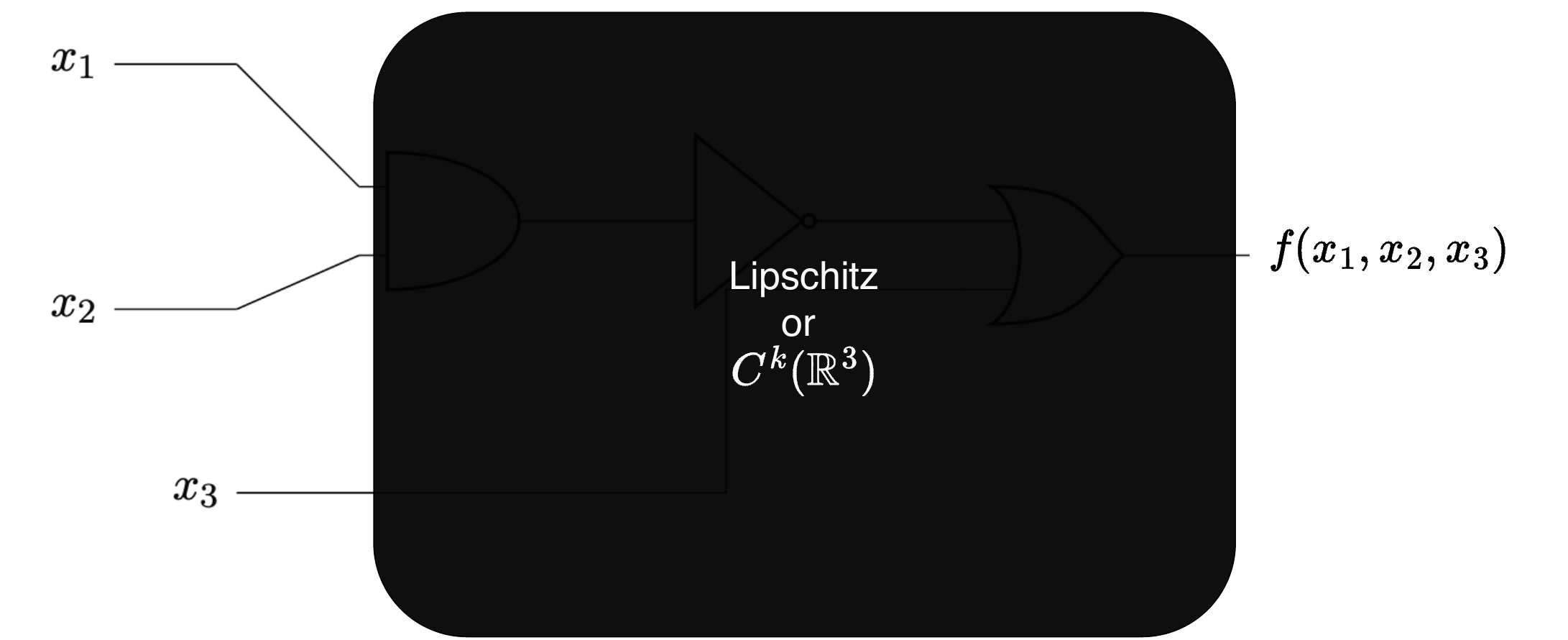}
    \end{minipage}
    \caption{\textbf{\textbf{{\color{slighlylesswarmblack}{Black}}} box: regularity.}\\
    The target function is seen only through input-output regularity.}
    \label{fig:BlackBox}
\end{subfigure}
\hfill
\begin{subfigure}[t]{0.32\linewidth}
    \vspace{0pt}
    \centering
    \begin{minipage}[t][0.16\textheight][c]{\linewidth}
        \centering
        \includegraphics[width=\linewidth,height=0.145\textheight,keepaspectratio]{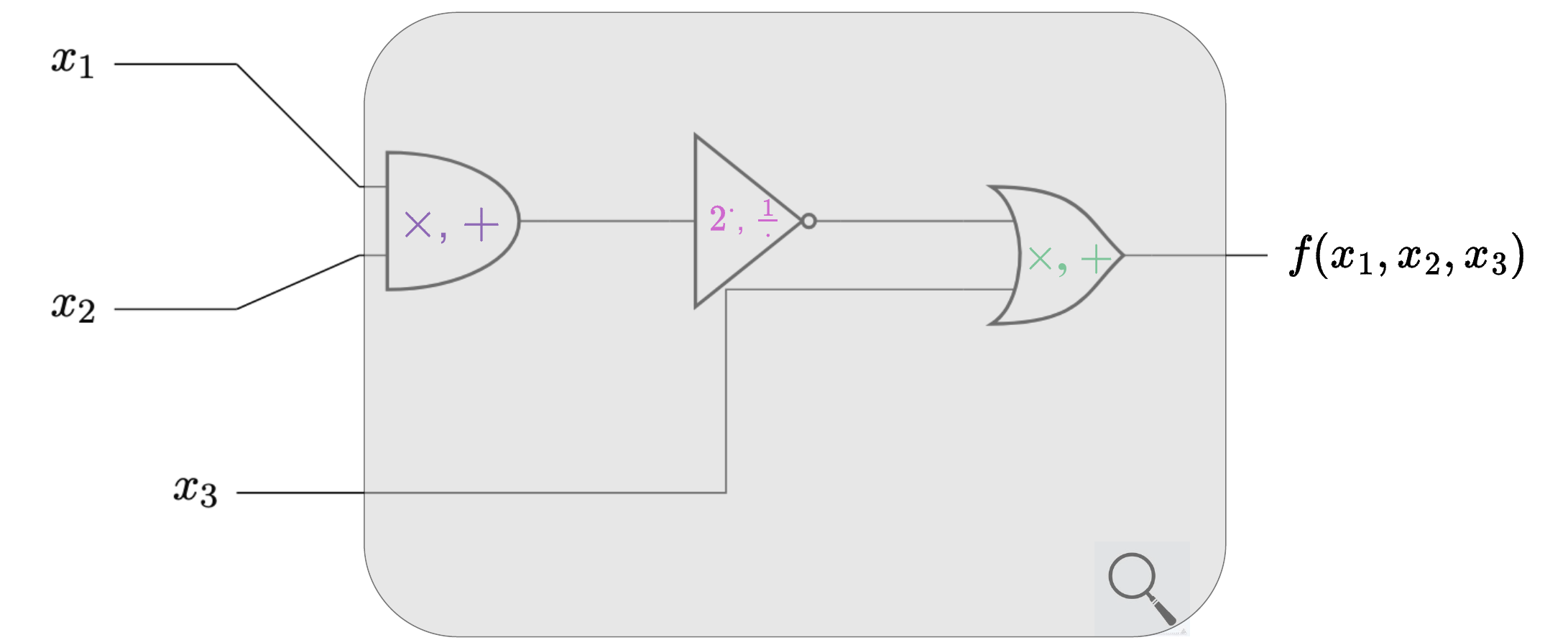}
    \end{minipage}
    \caption{\textbf{{\color{darkgray}{Grey}} box: circuit complexity.}\\
    The target function is seen through circuits that approximate it.}
    \label{fig:GreyBox}
\end{subfigure}
\hfill
\begin{subfigure}[t]{0.32\linewidth}
    \vspace{0pt}
    \centering
    \begin{minipage}[t][0.16\textheight][c]{\linewidth}
        \centering
        \includegraphics[width=\linewidth,height=0.145\textheight,keepaspectratio]{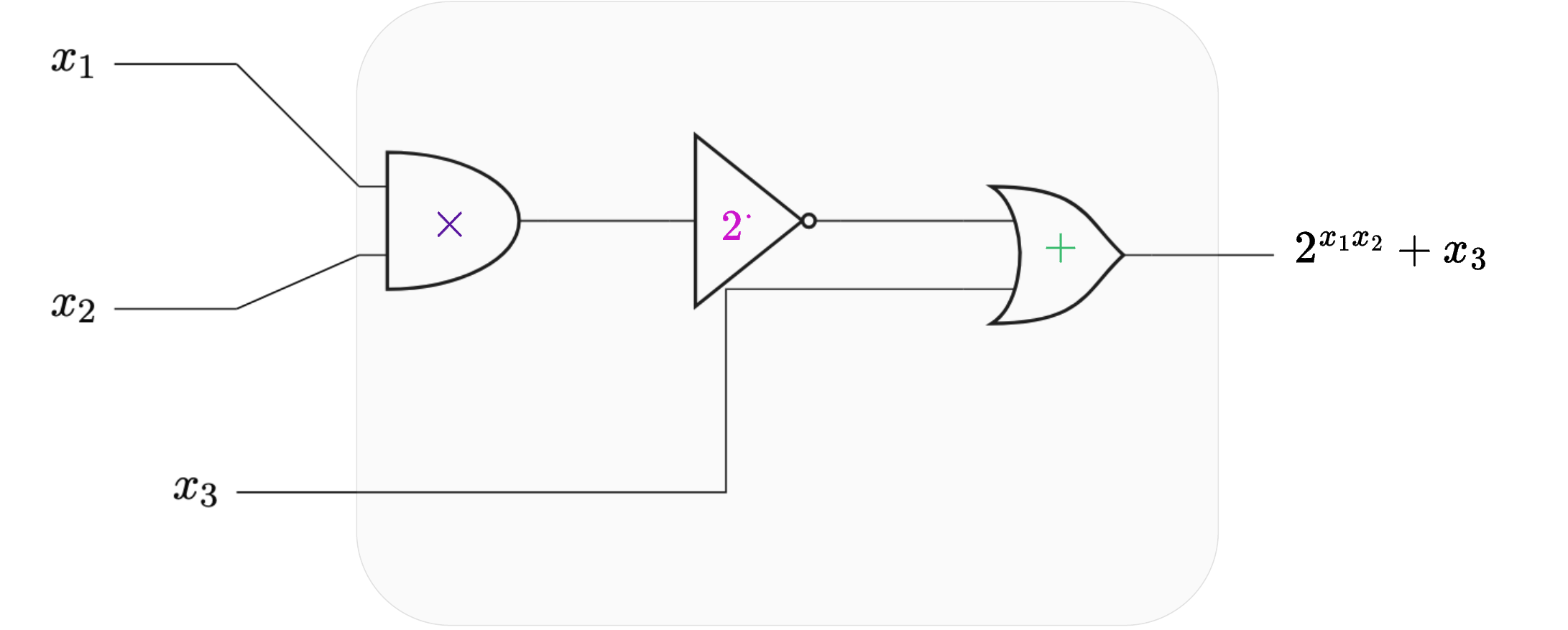}
    \end{minipage}
    \caption{\textbf{{\color{darkchampagne}{White}} box: explicit algorithms.}\\
    The full algorithm is compiled gate-by-gate into a neural network.}
    \label{fig:WhiteBox}
\end{subfigure}
\caption{\textbf{The missing piece: algorithmic complexity.}
Classical approximation theory is black-box: it sees the regularity of the input-output map, but not the algorithmic structure producing it.  Our framework interpolates between this regularity-based viewpoint and a white-box compilation of explicit algorithms.  The grey-box theory forms the bridge: it converts circuit-level descriptions of a computation into neural-network approximation guarantees.
\hfill\\
\textit{\textbf{{\color{slighlylesswarmblack}{Black}}}-box setting:}
This recovers standard black-box approximation results for low-, moderate-, and high-regularity classes, including Theorems~\ref{thrm:UAT__continouusLowreg}, \ref{thrm:Besov_cube_Lp_ANN}, and~\ref{thrm:HolomorphicApproximation}.
\hfill\\
\textit{{\color{darkgray}{Grey}}-box setting:}
When the broad features of the algorithmic complexity are known, such as the growth rates of the circuit depth and number of gates as functions of the target accuracy, together with the elementary gate language used by the computation, Theorem~\ref{thrm:concrete_surgery} yields automatic neural-network complexity bounds.  This is often the relevant regime for practical algorithms; see the lookup tables in Section~\ref{s:Approximation__ss:Quantitative__ss:LookupTables}, Corollary~\ref{cor:AllpairsShortestpaths}, and Figure~\ref{fig:APSP}.
\hfill\\
\textit{{\color{darkchampagne}{White}}-box setting:}
When the algorithm is explicit, Proposition~\ref{prop:Step1_GateEmulation} and Table~\ref{tab_thrm:gate_approximation} give a fine-grained gate-emulation route, used for numerical algorithms; e.g., Newton-Raphson root finding, power iteration, and ODE solvers (Section~\ref{s:Implication__BeyondMinimaxALGORITHIMS}).}

\label{fig:InputOutputRelations_put_in_Boxes}
\end{figure}

\subsection{Universal Computation: Qualitative and Quantitative Guarantees}

With this notion of algorithmic complexity in mind, our main quantitative result can be summarized as follows.  If a continuous function $f:[0,1]^d\to \mathbb{R}$ can be computed uniformly to accuracy $\varepsilon>0$ by a non-recursive algorithm in the above language, then any neural network architecture with a non-piecewise-linear nonlinearity can compute $f$ to comparable accuracy using roughly the same number of neurons as the original algorithm has elementary gates.  
This is in the spirit of, and extends the scope of, the classical characterization in~\cite{PinkusOGpaper_1999}, which treated the shallow one-hidden-layer multilayer perceptron (MLP) model with a univariate componentwise nonlinearity; i.e., the shallow MLP model is a universal approximator if and only if its componentwise non-linearity is non-polynomial.
Our informal universal computation theorem characterizes universality for general deep models, without imposing any MLP restriction.

\begin{thrminf}[Non-Affineness Characterizes Universal Approximation: cf.\ Theorem~\ref{thrm:Universal_Computation}]
\label{inf_thrm:Universal_Computation}
Any definable class of feedforward neural networks is universal in $C([0,1]^d,\mathbb{R})$, for every $d\in \mathbb{N}_+$, \underline{if and only if} it contains a non-affine non-linearity.
\end{thrminf}

\subsection{Complexity Guarantees for Neural Computation: Quantitative Guarantees}

Theorem~\ref{thrm:Universal_Computation}, informally sketched in Informal Theorem~\ref{inf_thrm:Universal_Computation}, states that any computable function, in the relevant circuit model, can be computed by a neural network with a non-affine non-linearity.  However, this qualitative statement does not quantify the size of the resulting network.  It therefore leaves open the original \textit{hardness} question posed in the thought experiment of Figure~\ref{fig:intuitive}:
\begin{equation}
\label{eq:ta_Q}
\tag{Q}
\textit{How large must a neural network be in order to emulate a given computation?}
\end{equation}

For reference, Table~\ref{tab:reference_algorithms} gives representative examples from \textit{theoretical computer science} (TCS), \textit{numerical analysis}, and \textit{constructive approximation theory}.  For each algorithm, it records the elementary-operation language in which the computation is naturally expressed, its depth, i.e., the number of sequential computational stages, and its size, i.e., the total number of elementary operations.  The main message is that many standard algorithms across mathematics already have a circuit form; our theorem explains how such circuit descriptions can be systematically converted into neural-network complexity guarantees.

\begin{table}[H]
    \centering
    \newcommand{\yesell}{\cellcolor{DarkGreenCell!80}$\checkmark_{\ell=2}$}
    \begin{adjustbox}{width=\textwidth, center}
    \begin{tabular}{l|cccccccc|cc|l}
    \multicolumn{1}{c|}{\textit{Algorithm}}
    &
    \multicolumn{8}{c|}{\textit{Elementary operations (gates/language)}}
    &
    \multicolumn{2}{c|}{\textit{Complexity}}
    &
    \multicolumn{1}{c}{\textit{Reference}}
    \\
     & $1$\footnotemark[2] & $+$ & $\times$ & $\div$ & $\cdot^{\ell}$ & $\sqrt[\ell]{\cdot}$ & $2^{\cdot}$ & $\max\{\cdot,\cdot\}$ & Depth & Size & Reference \\
     \midrule

    \multicolumn{12}{l}{\textit{Classical TCS Algorithms}}\\
    \arrayrulecolor{gray}\hline\arrayrulecolor{black}
    All-pairs shortest paths on a weighted $k$-node graph via Floyd-Warshall-Roy DPP\footnotemark[3] 
    &
    \yes & \yes & \yes & \no & \no & \no & \no & \yes
    &
    $\mathcal{O}(k)$
    &
    $\mathcal{O}(k^3)$
    &
    \cite{warshall1962theorem,roy1959transitivite,jukna2023tropical}
    \\

    Travelling Salesman Problem on a $k$-node graph via Bellman-Held-Karp DPP 
    &
    \yes & \yes & \yes & \no & \no & \no & \no & \yes
    &
    $\mathcal{O}(k\log k)$
    &
    $\mathcal{O}(k^2 2^k)$
    &
    \cite{bellman1962dynamic,held1962dynamic,jerrum1982some}
    \\

    \midrule
    \multicolumn{12}{l}{\textit{Classical Numerical Analysis Algorithms}}\\
    \arrayrulecolor{gray}\hline\arrayrulecolor{black}
    Fixed point $y\mapsto \operatorname{sol}$ where $\operatorname{sol}^{\ell}=y$ via Newton-Raphson
    &
    \yes
    &
    \yes
    &
    \yes
    &
    \yes
    &
    \yes
    &
    \no
    &
    \no
    &
    \no
    &
    $\mathcal{O}(\varepsilon^{-1})$
    &
    $\mathcal{O}(\varepsilon^{-1})$
    &
    \cite{Raphson1690,Ypma1995NewtonRaphson}
    \\

    $A\mapsto \lambda_{\max}(A)$ via Power Iteration
    &
    \yes
    &
    \yes
    &
    \yes
    &
    \yes
    &
    \yesell
    &
    \yesell
    &
    \no
    &
    \no
    &
    $\mathcal{O}(\varepsilon^{-1})$
    &
    $\mathcal{O}(\varepsilon^{-1})$
    &
    \cite{MisesPollaczekGeiringer1929}
    \\

    \midrule
    \multicolumn{12}{l}{\textit{Classical Constructive Approximation Theory}}\\
    \arrayrulecolor{gray}\hline\arrayrulecolor{black}
    Best polynomial approximation of continuous functions via Bernstein-Weierstrass polynomials
    &
    \yes & \yes & \yes & \no & \no & \no & \no & \no
    &
    $\mathcal{O}(1)$
    &
    $\mathcal{O}(\varepsilon^{-2d})$
    &
    \cite{bernstein1912weierstrass,bustamante2017bernstein}
    \\

    Best $N$-term approximation in the Besov ball $B_{p,q}^s([0,1]^d)$ via spline-wavelets
    &
    \yes & \yes & \yes & \no & \no & \no & \no & \yes
    &
    $\mathcal{O}(1)$
    &
    $\mathcal{O}(\varepsilon^{-d/s})$
    &
    \cite{devore1988interpolation,devore1993besov}
    \\

    Best $N$-term approximation of holomorphic functions via orthogonal polynomials
    &
    \yes & \yes & \yes & \no & \no & \no & \no & \no
    &
    $\mathcal{O}(1)$
    &
    $\mathcal{O}((\log(\varepsilon^{-1}))^{2d})$
    &
    \cite{adcock2022sparse,opschoor2022exponential}
    \\
    \bottomrule
    \end{tabular}
    \end{adjustbox}
    \caption{\textbf{Complexity of Common Algorithms - TCS, Numerical Analysis, and Approximation Theory:}
    Algorithms, their elementary-operation language, circuit depth (sequential computation), and circuit size (total number of elementary operations), where $\varepsilon>0$ is the approximation error.
    }
    \label{tab:reference_algorithms}
\end{table}
\footnotetext[2]{$1$ denotes the constant function $x\mapsto 1$.}
\footnotetext[3]{DPP stands for Bellman's dynamic programming principle; cf.~\cite{bellman1962dynamic}.}


\begin{figure}[H]
	\centering
	
	\begin{subfigure}[t]{0.48\linewidth}
		\centering
		\includegraphics[width=0.85\linewidth]{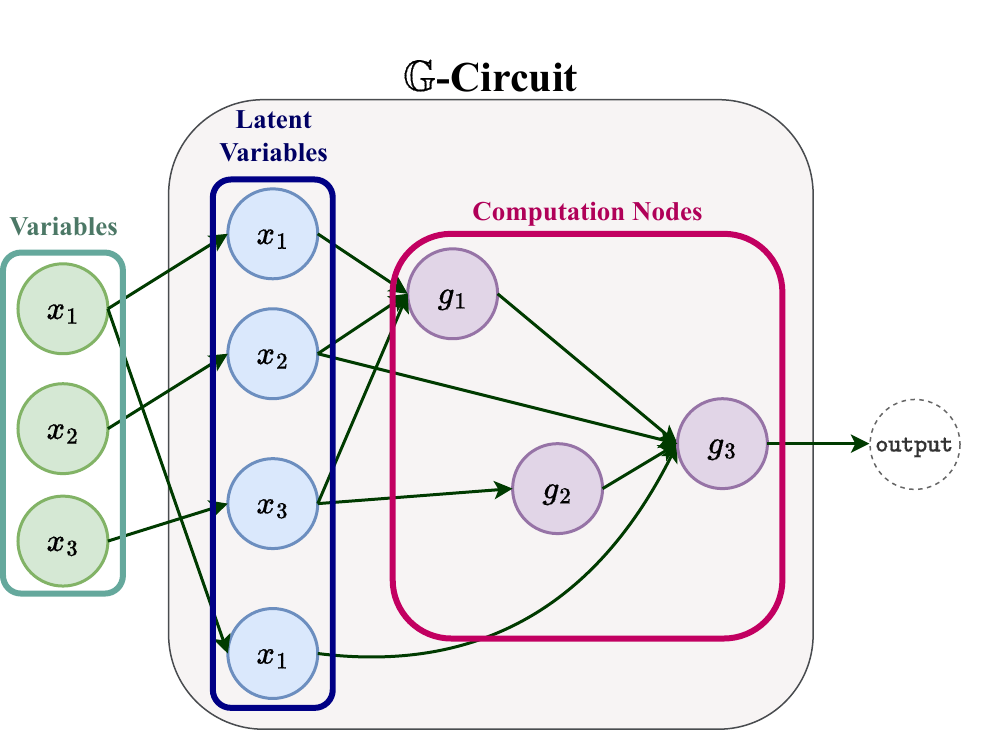}
		\caption{\textbf{$\mathbb{G}$-circuit.}
			A non-recursive real-valued circuit
            whose gates are elementary operations from a dictionary $\mathbb{G}$, containing the field, vector-space, radical, and tropical operations.
            The circuit computes $f:[-1,1]^d\to\mathbb{R}$ up to error $\varepsilon>0$.}
		\label{fig:Computation}
	\end{subfigure}
	\hfill
	\begin{subfigure}[t]{0.48\linewidth}
		\centering
		\includegraphics[width=0.85\linewidth]{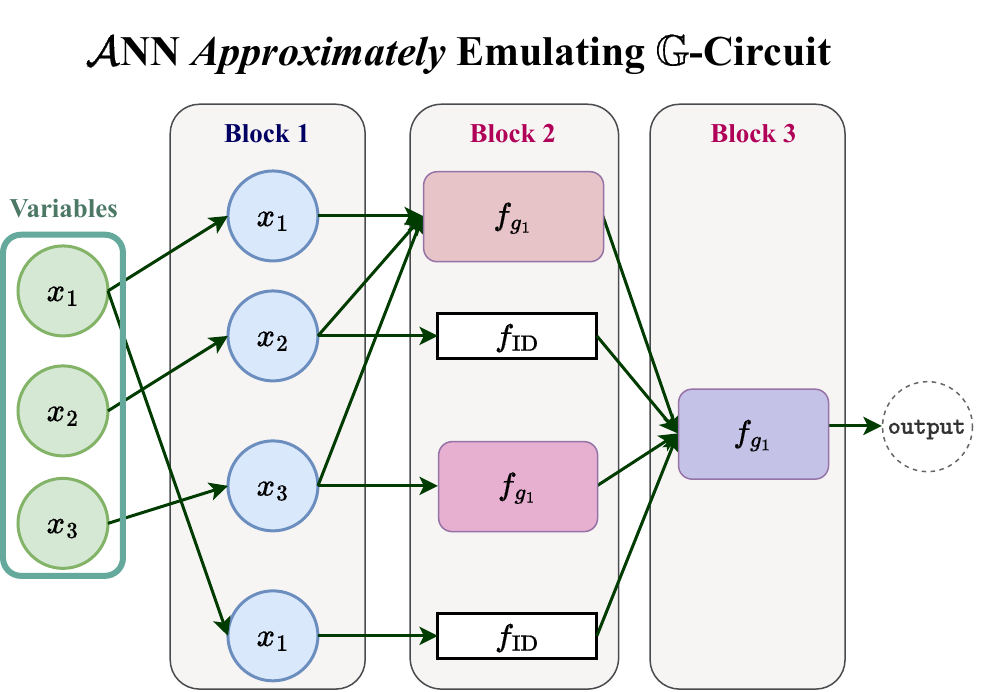}
		\caption{\textbf{$\mathcal{A}$NN emulator.}
			A feedforward neural architecture, such as an MLP, ResNet, Transformer, or more general $\mathcal{A}$NN, emulates the same circuit by replacing each $\mathbb{G}$-gate with a neural block and preserving the circuit wiring.}
		\label{fig:Emulation}
	\end{subfigure}
	
	\caption{\textbf{Informal Summary of the Quantitative Neural Compilation Theorem~\ref{thrm:concrete_surgery}.}
		Given a $\mathbb{G}$-circuit computing $f:[-1,1]^d\to\mathbb{R}$ up to error $\varepsilon>0$, we construct a neural network computing $f$ up to error $2\varepsilon$, with explicit depth, width, and non-zero-parameter bounds inherited from the circuit depth, size, and elementary-operation language.  Each elementary gate is replaced by a neural emulator using the abstract surgery technique of Proposition~\ref{prop:abstract_surgery}.
	}
	\label{fig:Summary}
\end{figure}


Our quantitative result (Theorem~\ref{thrm:concrete_surgery}), summarized in Figure~\ref{fig:Summary}, answers Question~\eqref{eq:ta_Q} by controlling the complexity of converting classical algorithms computing functions $f:[-1,1]^d\to\mathbb{R}$, written using the usual elementary operations, into neural networks computing the same function up to a prescribed error.  As in most real-valued computation, the target is not exact equality but approximation to accuracy $\varepsilon>0$, typically uniformly over $[-1,1]^d$ or in another relevant norm.

\subsection{Implications and Direct Consequences}
\label{s:Implications}

\begin{enumerate}
\item[(i)] \textbf{A Unified Theory of Approximation and Computation.}
Our results recast neural approximation as neural computation: constructive approximation schemes, numerical algorithms, and TCS-style circuits are all finite recipes built from elementary gates.  The Quantitative Neural Compilation Theorem (Theorem~\ref{thrm:concrete_surgery}) shows that, once such a recipe is written as a circuit, it can be converted into an $\mathcal{A}$NN with explicit depth, width, and size bounds.

\item[(ii)] \textbf{A Phase Diagram of Neural Complexity Classes.}
This yields a phase diagram for neural complexity (Tables~\ref{tab:log_depth} and~\ref{tab:poly_depth}): the cost of representing a function is governed not only by its regularity class, but also by the complexity of the best circuit or algorithm computing it.  Thus, computable functions may be far simpler than the worst-case elements of the same Besov or holomorphic class, provided that the operations and complexity of the algorithm computing the target are known (white-box) or can be estimated (grey-box).

\item[(iii)] \textbf{All Non-Affine Models are Minimax-Optimal Approximators.}
\label{s:Implication_Minimax}
Our results provide a self-contained, unified guarantee: every definable feedforward architecture with at least one non-affine nonlinearity achieves the standard minimax approximation rates for Besov (Theorem~\ref{thrm:Besov_cube_Lp_ANN}) and holomorphic classes (Theorem~\ref{thrm:HolomorphicApproximation}). 
These rates hold in $L^p_{\mu}(\mathcal{X})$, $1\le p<\infty$, for Ahlfors-regular measures on compact domains $\mathcal{X}\subseteq\mathbb{R}^d$, including domains of non-integer effective dimension (Theorem~\ref{thrm:UAT__continouusLowreg}).

\begin{figure}[H]
	\centering
	\includegraphics[width=0.75\linewidth]{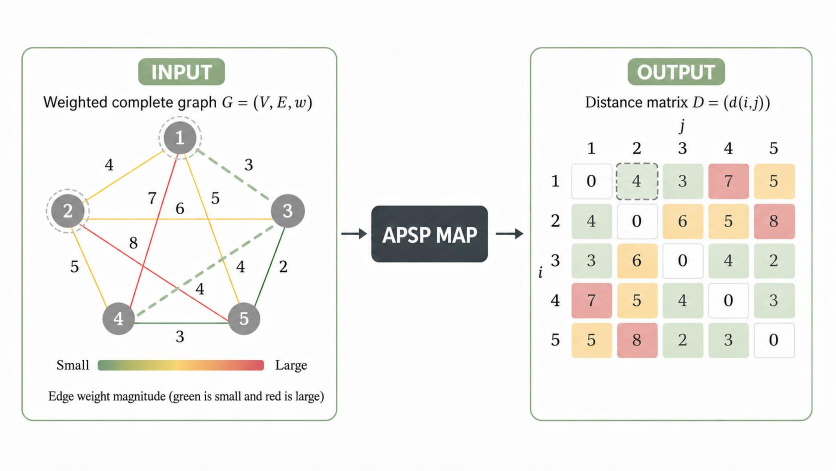}
		\caption{\textbf{Realistic network sizes for structured problems, beyond worst-case regularity.}
		The all-pairs shortest paths (APSP) problem on a $k$-vertex graph defines a map $f:\mathbb{R}^{k(k-1)/2}\to \mathbb{R}^{k(k-1)/2}$ sending positive edge weights to the corresponding shortest-path distance matrix.  Although this APSP map is $1$-Lipschitz, classical best-approximation rates for the Lipschitz class, cf.~\cite{yarotsky2017error,shen2022optimal}, lead to networks of size $\mathcal{O}(\varepsilon^{-ck^2})$, for an absolute constant $c>0$.  In contrast, by incorporating the tropical-circuit structure, i.e., the algorithmic information in the APSP computation, our theory predicts the more realistic size $\widetilde{\mathcal{O}}_k(\log(1/\varepsilon))$, yielding an exponential improvement in $\varepsilon^{-1}$.  The queried pair $(1,2)$ is marked on the input side, and its shortest-path distance is the $(1,2)$-entry of the distance matrix on the output side.}
	\label{fig:APSP}
\end{figure}

\item[(iv)] \textbf{``Small'' Neural Networks Can Solve Standard TCS and Numerical Analysis Problems.}
\label{s:Implication_Numerical}
The same principle shows that generic feedforward architectures can emulate standard non-recursive numerical algorithms at essentially the cost of their pseudocode, once the elementary-operation language is fixed.  This includes quadrature, ODE flow solvers, Newton-Raphson root finding, and power iteration from numerical analysis; see Propositions~\ref{prop:integral_by_quadrature}, \ref{prop:flowemulating}, \ref{prop:radicals_ANNs}, and~\ref{prop:ANN_emulation_power_iteration}.  It also includes all-pairs shortest paths (APSP), illustrated in Figure~\ref{fig:APSP}; see Corollary~\ref{cor:AllpairsShortestpaths}.

\item[(v)] \textbf{Message to End-Users:} \textit{Approximation Is Not the Differentiator.}
Under our hypotheses, universality, minimax-optimal approximation, and the ability to emulate many classical algorithms are generic consequences of non-affine feedforward computation.  Thus approximation power is a baseline rather than a differentiator; architectural comparison should focus instead on inductive bias, symmetry, geometry, stability, scalability, and optimization behaviour.
\end{enumerate}

\subsection{Related Work}
\label{s:Intro__ss:RelatedWork}

Our results connect two largely independent branches of deep learning theory: one rooted in contemporary approximation theory, and one rooted in theoretical computer science (TCS).  The novelty of our contribution is not the observation that depth, compositionality, or algorithmic structure matter in isolation; each of these themes has already played a central role in deep learning theory.  Rather, the main conceptual contribution is to make circuit complexity and real-valued computability into approximation-theoretic objects: explicit computations are organized through primitive operation dictionaries and execution graphs, and then compiled into neural networks with quantitative, architecture-agnostic depth, width, and size guarantees.  In this sense, the algorithmic viewpoint of TCS enters our framework not merely as analogy or motivation, but as the mechanism producing the approximation bounds.

\subsubsection{Compositional Sparsity: Contemporary Ideas in Approximation Theory}
\label{s:Intro__ss:RelatedWork___sss:Compositional_sparisty}

Our perspective is complementary to the compositional sparsity ideas developed in the approximation-theoretic foundations of deep learning~\cite{mhaskar2016deep,mhaskar2017and,schmidthieber2020nonparametric,poggio2024compositional,cheridito2022efficient,dahmen2025compositional}. 
In this line of work, the target is typically assumed to be well-described by a hierarchical composition of low-dimensional smooth, Lipschitz, H\"older, or Besov-type constituents; related depth benefits also arise for piecewise smooth functions with low-dimensional feature structure~\cite{petersen2018optimal}.  These results give a powerful regularity-and-composition-based explanation for why deep networks can overcome some of the limitations of shallow or linear approximation schemes.

Our viewpoint differs in the object being compiled.  Compositional sparsity studies targets whose input-output map factors through a favourable hidden compositional structure.  Here, the structure is instead supplied by an explicit real-valued computation, built from a small and interpretable dictionary of elementary operations: addition, multiplication, division, radicals, exponentials, logarithms, maximization, and related primitives.  This algorithmic viewpoint can yield smaller networks for realistic computations by avoiding worst-case overestimates for arbitrary constituent functions, even in one dimension; cf.~Figures~\ref{fig:intuitive_low_regularity} and~\ref{fig:intuitive_high_regularity}.  It also explains depth through the execution graph of the algorithm, gives interpretable neural emulators whose structure is inherited from the computation, and applies architecture-agnostically to a broad class of feedforward-type models, including modern MLPs and transformers with realistic normalization layers.  Thus, the relevant complexity is not only the smoothness of the target map, but also the complexity of the best elementary computation producing it.

\subsubsection{Boolean Algorithmic Complexity Theory and Restricted Neural Network Models}
\label{s:Intro__ss:RelatedWork___sss:BooleanTCS}

The connections between neural networks and algorithms date back to some of the earliest theoretical results on neural networks~\cite{mcculloch1943logical}.  These connections have largely focused on the ability of neural networks to compute Boolean functions, i.e., maps $f:\{0,1\}^{B_{\mathrm{in}}}\to \{0,1\}^{B_{\mathrm{out}}}$, for some numbers of bits $B_{\mathrm{in}},\,B_{\mathrm{out}}\in \mathbb{N}_+$.  This Boolean emphasis is natural: the associated complexity classes, such as $\operatorname{NC}^0$, $\operatorname{AC}^0$, and $\operatorname{TC}^0$, are comparatively well-developed, and it is classical that a single threshold neuron computes the majority function when restricted to Boolean inputs $\{0,1\}^B$, for some $B\in\mathbb{N}_+$~\cite{muroga1971threshold,rosenblatt1958perceptron}.  More generally, constant-depth polynomial-size MLPs with threshold gates realize Boolean threshold circuits, and hence capture $\operatorname{TC}^0$-type computation; for example, parity is computed by a depth-two threshold network with linearly many hidden units~\cite{siu1991depth,paturi1991threshold}.  Analogous comparisons between smooth sigmoidal gates and Boolean threshold circuits were developed in~\cite{maass1991sigmoid,maass1995sigmoid}.  Thus, neural networks restricted to Boolean inputs can be strictly more powerful than $\operatorname{AC}^0$, by the classical lower bounds separating $\operatorname{AC}^0$ from parity and majority~\cite{Hastad86,haastad2014correlation}; related $\operatorname{TC}^0$-type characterizations have also recently been developed for transformer architectures~\cite{merrill2022saturated,chiang2024transformerstc0}.
The connections between recursive models of computation, e.g., Turing machines, and recursive neural-network models, e.g., recurrent neural networks (RNNs)~\cite{siegelmann1991turing,siegelmann1995computational} and looped transformers~\cite{perez2018on,chung2021turing,mali2021neural}, are also well studied.  These results describe qualitative limits of neural-network computation in recursive or Boolean settings.  By contrast, the language of circuit complexity theory, cf.~\cite{jukna2012boolean,jukna2023tropical}, is particularly well-suited to quantitative statements about depth, size, and gate structure.  Our work follows this quantitative circuit-theoretic spirit, but moves from Boolean computation to real-valued uniform approximation, which is otherwise comparably underdeveloped.

The connection between Boolean algorithms and neural networks is also central to the expressivity theory of graph neural networks.  Standard message-passing GNNs are closely tied to the distinguishing power of the $1$-dimensional Weisfeiler-Leman test, also known as colour refinement~\cite{xu2019powerful,morris2019weisfeiler}, while higher-order architectures emulate stronger $k$-dimensional Weisfeiler--Leman tests and thereby yield finer expressivity hierarchies~\cite{morris2019weisfeiler,maron2019provably,chen2019equivalence}.  These works show that algorithmic tests can be used to compare neural architectures, but the algorithms being emulated remain discrete or graph-combinatorial.  The present paper is complementary to this tradition: we use circuit complexity as a constructive approximation-theoretic language for real-valued computation.  The primitive gates are arithmetic, radicals, exponentials, logarithms, maximization, and related real-valued operations, and our main theorems quantify how efficiently circuits built from these primitives can be compiled into broad feedforward neural architectures.  In this sense, our framework is closer in spirit to continuous logic, where truth-values may be $[0,1]$-valued and structures are often metric or real-valued, than to the purely Boolean logical foundations of classical TCS; cf.~\cite{cook2010logical,goldbring2021almost,hart2023introduction}.  We use this only as a conceptual analogy, but it helps explain why real-valued algorithmic primitives call for a complexity theory distinct from the Boolean case.

\subsubsection{Neural Algorithmic Reasoning}
\label{s:Intro__ss:RelatedWork___sss:NAR}

A modern and largely empirical evolution of these ideas is \textit{neural algorithmic reasoning} (NAR), cf.~\cite{velivckovic2021neural}, which trains neural networks from algorithmic traces, i.e., intermediate states produced during the execution of an algorithm, with the aim of executing classical algorithms, improving interpretability and out-of-distribution generalization on combinatorial tasks, or transferring algorithmic reasoning to related tasks~\cite{xhonneux2021transfer}.  Representative examples include graph neural executors for breadth-first search, Bellman--Ford, Prim's algorithm, dynamic graph connectivity, augmenting-path computations, and graph search~\cite{velickovic2020neural,velickovic2020pointer,georgiev2020neuralmatching,georgiev2022algorithmic,pandy2022learning}; CLRS-style neural processors trained across sorting, searching, dynamic-programming, graph, string, and geometric algorithms~\cite{velickovic2022clrs,ibarz2022generalist}; and structured neural executors based on priority queues, algorithmic pretraining, fixed-point reasoning, proximal-splitting emulation, multiple-solution supervision, and parallel algorithmic traces~\cite{jain2023neuralpq,georgiev2024narco,georgiev2024deep,kratsios2025generative,kujawa2024multiple,engelmayer2024parallel}.

These works provide compelling empirical evidence that trained neural architectures can emulate classical combinatorial algorithms.  Theoretical support has begun to emerge; for example, in~\cite{dudzik2022gnns} it is shown that certain GNNs align with dynamic-programming principles.  Such results are important, but remain architecture-specific and tied to particular algorithmic paradigms.  By contrast, our goal is a circuit-level, architecture-agnostic compilation theorem.  
The closest theoretical result to our line of work is likely~\cite{kratsios2025quantifying}, which analyzes $\operatorname{ReLU}$ feedforward networks on digital computers within finite-precision machine learning; cf.~\cite{kratsios2026tighter}.  There, inputs and outputs are discretized and networks may exploit additional bit precision, so the framework is intrinsically finite-state and does not apply directly to real-valued inputs and outputs.  Our contribution is instead a real-valued circuit-to-network compilation theorem: primitive operation dictionaries and execution graphs yield quantitative, architecture-agnostic complexity bounds.

\subsubsection{Algorithmic Unrolling onto Neural Network Models}
\label{s:Intro__ss:AlgoUnrolling}
Primarily in the numerical analysis and compressed sensing communities, NAR-adjacent ideas have also emerged independently, typically under the name of algorithmic \textit{unrolling}.  In this setting, a specialized or customized neural-network architecture is designed to approximately \textit{emulate} a specific numerical-analytic algorithm, often one known to be efficient or near-optimal for a particular computational physics or engineering problem.  A first prominent line unrolls sparse-coding, compressed-sensing, and inverse-problem algorithms: for instance, learned sparse-coding iterations are unrolled in~\cite{gregor2010learning}; LASSO, ISTA/FISTA-type, and related sparse-recovery procedures are unrolled or emulated by specialized networks in~\cite{monga2021algorithm,wang2016learning,xin2016maximal,scarlett2022theoretical}; and this viewpoint is further specialized to stable and accurate networks for compressed sensing and analysis-sparse inverse problems in~\cite{colbrook2022difficulty,neyranesterenko2023nestanets,mohammadtaheri2025deep,mohammadtaheri2025greedy}, or graph sampling algorithms~\cite{neuman2023transferability}.  A second line emulates proximal forward-backward splitting algorithms, in the sense of~\cite{combettes2011proximal}, by customized neural-operator models for solving specialized convex optimization problems~\cite{kratsios2025generative}.  A third line emulates fixed-point and iterative solution procedures for PDEs: this includes fixed-point schemes for semi-linear elliptic PDEs and their FBSDE/2FBSDE counterparts~\cite{furuya2024simultaneouslysolvingfbsdesneural,furuya2025one}, learned iterative PDE solvers with convergence guarantees~\cite{hsieh2019learning}, neural or learned multigrid methods~\cite{greenfeld2019learning,chen2022metamgnet,han2024ugrid,he2023mgno}, neural-operator corrections or warm-starts for relaxation and  solvers~\cite{zhang2024blending,li2023learning,rudikov2024neural,kopanicakova2025deeponet}, and train-and-unroll neural-operator constructions~\cite{he2025selfcomposing}.

\subsubsection{o-Minimal Definability: A Natural General Setting for Deep Learning}
Our main theorem applies to feedforward neural networks definable in an o-minimal expansion of the real field $(\mathbb{R},\times,+,\le)$, as defined below.  This allows possibly multivariate and non-piecewise-affine non-linearities, including attention, $\operatorname{SwiGLU}$, layer normalization, and standard bounded-domain positional encodings of the kind used in modern transformers.  O-minimality provides a rigorous language for geometric and combinatorial ``tameness''~\cite{denef1988padic,vddriesMiller1996geometric,vddries1998tame}; it has received significant attention in mathematical logic and number theory~\cite{pila2006rational}, and many useful structures are known to be o-minimal~\cite{WilkieCompletnessResp_Model_1996Pfaffian_JAMS,Speissegger1999Pfaffian,BinyaminiNovikovZak2024WilkiePfaffian,BinyaminiNovikov2017WilkieRestricted,BinyaminiNovikov2019ComplexCells}.  In learning and optimization, related tameness ideas appear in statistical learning theory~\cite{Greggy2026}, foreshadowed by Pfaffian VC-dimension bounds~\cite{karpinski1997polynomial,PardoMontana_VCDimensionPfaffian_AMAI_2009}, and in non-convex optimization through the Kurdyka-\L{}ojasiewicz framework~\cite{kurdyka1998gradients,bolte2007lojasiewicz,bolte2007clarke,bolte2009tame}.

Thus, o-minimal definability gives a broad but mathematically tame setting: many standard feedforward architectures used in practice, from classical MLPs and ResNets to transformers with realistic pooling and normalization layers, are definable in suitable o-minimal expansions of the real field and therefore fall within the scope of this paper.  For a detailed catalogue of such examples, we refer the reader to the accompanying statistical paper~\cite{Greggy2026}, where it is shown that definable neural networks have finite sample complexity for regression and classification whenever the loss is definable.  That work also contains an extensive list of definable architectures and loss functions, illustrating the breadth of the framework.

\subsection{Organization of the Paper}
\label{s:Intro__ss:Paperorg}
The paper is organized as follows.  In \hyperref[s:Intro]{Section~\ref{s:Intro}} we motivate the shift from regularity-based approximation theory to algorithmic complexity.  \hyperref[s:Prelims]{Section~\ref{s:Prelims}} introduces $\mathcal{A}$NNs, $\mathbb{G}$-circuits, and the o-minimal definability framework.  \hyperref[s:Approximation]{Section~\ref{s:Approximation}} states the main qualitative and quantitative results.  \hyperref[s:Implication__Minimax]{Section~\ref{s:Implication__Minimax}} derives the unified minimax-approximation consequences for uniformly continuous, Besov, and holomorphic functions.  \hyperref[s:Implication__BeyondMinimaxALGORITHIMS]{Section~\ref{s:Implication__BeyondMinimaxALGORITHIMS}} shows that $\mathcal{A}$NNs emulate algorithms from TCS and numerical analysis, including shortest-path computation, power iteration, root finding, and structured ODE solvers.  \hyperref[s:complexity_classes__ss:whyitworks]{Section~\ref{s:complexity_classes__ss:whyitworks}} explains the abstract-surgery mechanism underlying the quantitative compilation theorem.  \hyperref[s:Conclusion]{Section~\ref{s:Conclusion}} concludes.

\section{Preliminaries}
\label{s:Prelims}

An exhaustive list of the notation used in the paper can be found in Appendix~\ref{a:Notation_full}.  We now recall the background and terminology needed to state our results.

\subsection{Model-Theoretic Geometry: o-Minimality}
\label{s:ModGeo}
We use o-minimal geometry to formalize the tameness of the non-linearities appearing in our networks.  Throughout, an interval means one of $(a,b)$, $(-\infty,a)$, or $(b,\infty)$, with $a,b\in\mathbb{R}$, and each $\mathbb{R}^n$ is equipped with its usual (norm) topology.

\begin{definition}[o-Minimal Structure]
\label{def:o-minimal-structure}
A \emph{structure expanding the real field $(\mathbb{R}, +, \times, 0,1)$} is a collection $\mathcal S=(S_n)_{n\in\mathbb N}$, where each $S_n$ is a collection of subsets of $\mathbb{R}^n$, satisfying the following conditions:
\begin{enumerate}
\item[(i)] every algebraic subset of $\mathbb{R}^n$ belongs to $S_n$;
\item[(ii)] each $S_n$ is a Boolean sub-algebra of $\mathcal P(\mathbb{R}^n)$;
\item[(iii)] if $A\in S_m$ and $B\in S_n$, then $A\times B\in S_{m+n}$;
\item[(iv)] if $A\in S_{n+m}$ and $\pi:\mathbb{R}^{n+m}\to\mathbb{R}^n$ is the projection onto the first $n$ coordinates, then $p(A)\in S_n$. 
\end{enumerate}
The elements of $S_n$ are called the \emph{definable subsets} of $\mathbb{R}^n$.  The structure $(\reals, +, \times, 0, 1, \mathcal S)$ is \emph{o-minimal} if and only if every definable subsets of $\mathbb{R}$ is a finite union of intervals and points.
\end{definition}

It follows inductively that the collection of definable sets is closed under any coordinate projections, equivalently under existential quantification. Since definable sets are also closed under complementation, they are likewise closed under universal quantification, by taking complements of coordinate projections of complements. See \cite{vanDenDries1998TameTopology} for the standard foundational reference on o-minimality.

\begin{eg}
Examples, and non-examples, of o-minimal structures include:
\begin{enumerate}
    \item[(i)] The real field $(\reals, +,\times, 0, 1)$ with no extra structure is o-minimal. The definable sets are precisely the semi-algebraic sets, since the standard ordering on $\reals$ (considered as a relation on $\reals^2$) is the projection of the algebraic set $\{(x,y,t):t^2=y-x\}$ onto the first two coordinates. 
    \item[(ii)] The structure $(\reals, +, \times, 0, 1, \Gamma_{\sin(x)})$ is \emph{not} o-minimal. If the graph $\Gamma_{\sin(x)}=\{(x,y):y=\sin(x)\}$ is definable, then so too is the set $\{x\in \reals: \sin(x)=0\}$, since it is the projection onto the first coordinate of the definable set $\Gamma_{\sin(x)}\cap\{(x,y):y=0\}$. This is an infinite union of isolated points, and so is strictly forbidden by the definition of o-minimality. 
    \item[(iii)] $\mathcal{R}_{\operatorname{an},\operatorname{exp}}$, the real field expanded by the graphs of all restricted analytic functions, and the full exponential function, is o-minimal, by~\cite{vanDenDriesMacintyreMarker1994RestrictedAnalyticExp} and by~\cite{vanDenDriesMiller1994RealExpRestrictedAnalytic} using different techniques; extending the results of~\cite{WilkieCompletnessResp_Model_1996Pfaffian_JAMS,vanDenDriesMacintyreMarker1994RestrictedAnalyticExp}.
\end{enumerate}
\end{eg}

In this setting, a well-behaved function is thought of as being \emph{tame} when its graph is definable in \emph{some} o-minimal expansion of the real field. 

\begin{definition}[Definable map]
\label{def:definable-map}
Let $A\subseteq\mathbb{R}^n$.  A map $f:A\to\mathbb{R}^p$ is called \emph{definable} in $\mathcal S$ if its graph $\Gamma_f\eqdef \{(x,f(x))\in\mathbb{R}^{n+p}: x\in A\}$ is a definable subset of $\mathbb{R}^{n+p}$. 
We will say that a function $f$ is definable when we mean that it is definable in \emph{some} fixed o-minimal expansion $\mathcal{S}$ of the real field, and will say that $f$ is \emph{not} definable when we mean that it cannot be definable in \emph{any} o-minimal expansion of the real field. 
\end{definition}

Observe that if $f:A\rightarrow \reals^p$ is definable for $A\subseteq \reals^n,$ then so too are its domain $A$ and range, since they are both projections of the graph $f$.  Note that semi-algebraic functions are always definable.

As shown in detail in~\cite{Greggy2026}, a broad class of modern feedforward deep-learning models used in practice is definable in suitable o-minimal expansions of the real field; for instance, in the Pfaffian closure of $\mathbb{R}_{an,\exp}$, cf.~\cite{Speissegger1999Pfaffian}.  We will repeatedly use the standard closure properties of definable maps: linear combinations, coordinate-wise concatenations, and compositions remain definable in the same o-minimal structure; cf.~\citep[Chapter~1.3]{coste1999introduction}. Furthermore, it is easy to see that o-minimal structures are closed under \emph{reducts}; that is, if $(\reals, +, \times, 0,1, \mathcal{S})$ is an o-minimal structure expanding the real ordered field, and $\mathcal{S}'\subseteq \mathcal{S}$ is a subset which still generates a structure in the sense of Definition \ref{def:o-minimal-structure}, then $(\reals, +,\times, 0, 1, \mathcal{S}')$ is also necessarily o-minimal. For this reason, a given network that is definable in \emph{some} o-minimal structure can also be assumed to be definable in a ``smallest'' o-minimal structure in which only the functions necessary to define the network are added to the standard real field. 

Lastly, we will make use of piecewise linear functions in this paper.  We follow the convention that piecewise linear functions are not required to have polyhedral pieces over which they are linear.
\begin{definition}[Piecewise Linear]
\label{defn:PWLlinear}
A definable function $f:\mathbb{R}^d\to \mathbb{R}^D$ is called \textit{piecewise linear} if there exists a partition%
\setcounter{footnote}{2}\footnote{A pairwise disjoint cover.}~%
$\{\Pi^{(p)}\}_{p=1}^P$ of $\mathbb{R}^d$ into definable sets such that, for each $p\in [P]_+$, there exist a $D\times d$ matrix $A^{(p)}$ and a vector $b^{(p)}\in \mathbb{R}^D$ satisfying $f(x)=A^{(p)}x+b^{(p)}$ for all $x\in \Pi^{(p)}$.
\end{definition}

\subsection{Deep Learning}
\label{s:Background__ss:Neural_Networks}

We fix a family of scalar-valued functions on Euclidean spaces of arbitrary finite dimension, which serve as the non-linearities in our models.  We call this family the \textit{dictionary of non-linearities} $\mathcal{A}\subseteq \bigcup_{n=1}^{\infty} C(\mathbb{R}^n,\mathbb{R})$.\footnote{The non-linearities may depend on trainable parameters.  Since our theorems do not require this additional structure, the parameterized setting is covered by adjoining the corresponding parameter instances to the dictionary.  Likewise, vector-valued non-linearities reduce to the scalar-valued case by coordinate projection; hence we do not treat them separately.}

The main content of our only assumption is that the dictionary $\mathcal{A}$ satisfies minimal tameness and closure conditions.  The definability requirement excludes pathological activation functions that would not be used in practice, such as typical Brownian paths, or functions that are not statistically tame.  For instance, such functions may fail to have finite $\gamma$-fat-shattering dimension, cf.~\cite{yarotsky2021elementary,shen2022deep_JMLR}, and hence may fail to generalize in the uniform convergence model of learning theory~\cite{shalev2010learnability}.  It is also unclear whether they generalize in the PAC model, since their $\gamma$-graph dimension need not be finite~\cite{attias2023optimal}; moreover, they violate the broad sufficient condition for PAC learning in~\cite{Greggy2026}.

The remaining assumptions are satisfied by the standard feed-forward models we know of in practice, and mainly encode closure under concatenation.  We prohibit piecewise linearity only for our main quantitative result, since the piecewise-linear case is effectively the well-understood $\operatorname{ReLU}$-MLP case, cf.\ Lemma~\ref{lem:piecewiseLinear_ominimalsetting}, under the o-minimal umbrella.  The piecewise-linear case is treated in the qualitative Theorem~\ref{thrm:Universal_Computation}.

\begin{assumption}
\label{ass:THEASSUMPTIONS}
Fix an o-minimal structure $(\reals, +, \times, ,0,1,\mathcal{S})$ expanding the real field $\mathbb{R}$.  We assume the following.
\begin{enumerate}
\item[(i)] \textbf{Definability:} For every $f\in \mathcal{A}$, $f$ is definable in $\mathcal{S}$; cf.~Definition~\ref{def:definable-map}.

\item[(ii)] \textbf{Parallelizability, a.k.a.\ Closure under Concatenation:} For every $K\in \mathbb{N}_+$ and every $f_1,\dots,f_K\in \mathcal{A}$, with $f_k:\mathbb{R}^{n_k}\to \mathbb{R}$, there exist functions $f^{\parallel:1},\dots,f^{\parallel:K}\in \mathcal{A}$, with $f^{\parallel:j}:\mathbb{R}^{\sum_{k=1}^K n_k}\to \mathbb{R}$, such that, for every $x=(x_1,\dots,x_K)\in \prod_{k=1}^K\mathbb{R}^{n_k}$,
\[
    \big(f^{\parallel:j}(x)\big)_{j=1}^K
    =
    \big(f_1(x_1),\dots,f_K(x_K)\big).
\]

\item[(iii)] \textbf{Non-Piecewise Linear:} There exists some $f\in \mathcal{A}$ which is not piecewise linear.
\end{enumerate}
\end{assumption}
\begin{eg}[MLP Case]
\label{ex:MLP}
For instance, for MLPs with activation function $\sigma \in C(\mathbb{R})$, Assumption~\ref{ass:THEASSUMPTIONS} (ii) says that, for every $K\in\mathbb{N}_+$ and $j\in [K]_+$, the coordinate map $\sigma^{\parallel:K,j}:\mathbb{R}^K\to\mathbb{R}$ defined by $\sigma^{\parallel:K,j}(x_1,\dots,x_K)\eqdef \sigma(x_j)$ belongs to the dictionary of non-linearities $\mathcal{A}$.
\hfill\\
In this sense, $\sigma$ generates the dictionary $\mathcal{A}=\{\sigma^{\parallel:K,j}:K\in\mathbb{N}_+,\ j\in [K]_+\}$ by closure under repeated concatenation.
\end{eg}

We now define a (feedforward) neural network.   We emphasize the dependence on the dictionary of non-linearities, which will ultimately be immaterial for our asymptotic computational results.

\begin{definition}[$\mathcal{A}$-Neural Network]
\label{defn:ANN}
Let $\mathcal{A}$ be a dictionary of non-linearities.  For any $d_{in},d_{out}\in \mathbb{N}_+$, a map $f:\mathbb{R}^{d_{in}}\to \mathbb{R}^{d_{out}}$ is called an $\mathcal{A}$-(feedforward) neural network, abbreviated $\mathcal{A}$NN, if it admits a layerwise representation of the form
\begin{equation}
\label{eq:recursive_ANN}
\begin{aligned}
f(x) & \eqdef A^{(L+1)}x^{(L+1)} + b^{(L+1)},
\\
x^{(l+1)} & \eqdef
\big[
\varsigma^{(l,i)}
\big(
A^{(l)}x^{(l)} + b^{(l)}
\big)
\big]_{i=1}^{d_{l+1}},
\qquad l=0,\dots,L,
\\
x^{(0)} & \eqdef x,
\end{aligned}
\end{equation}
where $L\in \mathbb{N}$, $d_0=d_{in}$, $d_1,\dots,d_{L+1}\in \mathbb{N}_+$, $A^{(L+1)}\in \mathbb{R}^{d_{out}\times d_{L+1}}$, and, for every $l=0,\dots,L$, $A^{(l)}\in \mathbb{R}^{d_{l+1}\times d_l}$ and $b^{(l)}\in \mathbb{R}^{d_{l+1}}$.  For every $l=0,\dots,L$ and every $i\in [d_{l+1}]_+$, the map $\varsigma^{(l,i)}:\mathbb{R}^{d_{l+1}}\to \mathbb{R}$ belongs to $\mathcal{A}$.
\end{definition}

We also regard every affine map $x\mapsto Ax+b$ as a depth-zero $\mathcal{A}$NN; in particular, $x\mapsto x$ is an exact $\mathcal{A}$NN.

The feedforward structure above is uniform: its computational graph is multipartite, with possibly inactive connections between layers due to zeroed-out weights, whereas general circuits may have more complicated graph structure.  The combination of multivariate non-linearities and zeroed-out weights allows a broad class of modern deep learning models to be captured by the definition of an $\mathcal{A}$NN; cf.\ the extensive list, from practical transformers to MLPs, in the companion paper~\cite{Greggy2026}.  However, CNNs and GNNs impose additional constraints on their weight matrices and therefore form restricted subclasses; without further assumptions, one should not expect such restricted architecture classes to be universal at this level of generality.

\subsection{Non-Recursive Algorithms: Real-Valued Circuit Complexity}
\label{s:Background__ss:Circuits}

We study neural networks as non-uniform real-valued models of computation by showing that they can approximately implement broad classes of circuits built from prescribed elementary gates.  Throughout this section, complexity means real gate complexity: we count elementary operations, circuit depth, and later non-zero real parameters in the compiled neural network, but not bit complexity, parameter precision, conditioning, memory traffic, or training time.  The circuit viewpoint includes finite Boolean~\cite{MR2895965}, arithmetic~\cite{MR3308677,MR4238568}, probabilistic Boolean, and certain Pfaffian circuits~\cite{karpinski1997polynomial,morton2015generalized} as motivating special cases.  We follow closely the presentation in~\cite{kratsios2025quantifying,li2026certifiable}.

\subsubsection{Language: Elementary Computations as Gates}
\label{s:Background__ss:Circuits___sss:Gates}

The complexity of a model or algorithm is measured relative to the language specifying which operations are treated as elementary, i.e., as having $\mathcal{O}(1)$ gate complexity.  Thus, the expressive power of a computational language is largely determined by its \textit{gates}: the operations it declares to be primitive.  We denote the collection of admissible gates by $\mathbb{G}$, and assume that $\mathbb{G}$ is a non-empty subset of $\bigcup_{n=1}^{\infty}[\mathbb{R}^n:\mathbb{R}]$, where $[\mathbb{R}^n:\mathbb{R}]$ denotes the set of all functions from $\mathbb{R}^n$ to $\mathbb{R}$.

The main gate languages used in this paper are the following.  For $k\in\mathbb{N}_+$ and $c>0$, the \textit{algebraic gate language} $\mathbb{G}_{alg}^{k,c}$ consists of constant gates $x\mapsto \tilde c$ with $|\tilde c|\le c$, coordinate projections, and $\tilde k$-ary addition and multiplication gates, for $1\le \tilde k\le k$.  The \textit{algebro-tropical gate language} $\mathbb{G}_{t\text{-}alg}^{k,c}$ augments $\mathbb{G}_{alg}^{k,c}$ by the absolute-value gate and $\tilde k$-ary maximization gates, for $1\le \tilde k\le k$.  Binary minimum is therefore also available, since $\min\{x,y\}=-\max\{-x,-y\}$; higher-arity minima are obtained analogously.
For $r^{\star}\in\mathbb{N}_+$, the \textit{radical algebro-tropical gate language} $\mathbb{G}_{rt\text{-}alg}^{k,c,r^{\star}}$ augments $\mathbb{G}_{t\text{-}alg}^{k,c}$ by the power gates $x\mapsto x^r$, for $1\le r\le r^{\star}$, and the radical gates $x\mapsto x^{1/\ell}$ when this real-valued branch is defined, respectively $x\mapsto |x|^{1/\ell}$ when $\ell$ is even, for $1\le \ell\le r^{\star}$.
Finally, the \textit{rational-tropical gate language} $\mathbb{G}_{Rat}^{k,c,r^{\star}}$ augments $\mathbb{G}_{rt\text{-}alg}^{k,c,r^{\star}}$ by the reciprocal gate $x\mapsto 1/x$, used only on domains where the denominator is non-zero.  Exponential and logarithmic gates are not part of these four core languages unless explicitly added in an extended gate language.  The hierarchy of gate languages, and representative objects computable by circuits over these languages, is summarized in Table~\ref{tab:gate_languages}.

\begin{table}[t]
\centering
\begingroup
\renewcommand{\arraystretch}{1.18}
\begin{adjustbox}{max width=\textwidth,center}
\begin{tabular}{@{}lcccccccp{0.52\textwidth}@{}}
\toprule
Language
& Const. $x\mapsto \tilde c$
& $+$
& $\times$
& $|\cdot|$
& $\max$
& Powers $x^{\ell}$ / radicals $\sqrt[\ell]{x}$
& $1/x$
& Representative computable objects
\\
\midrule
$\mathbb{G}_{alg}^{k,c}$
& \yes & \yes & \yes & \no & \no & \no & \no
& Sparse and best polynomial approximators~\cite{adcock2022sparse}; Bernstein approximators~\cite{lorentz1986bernstein}; discrete FFT~\cite{morgenstern1973note}.
\\
\midrule
$\mathbb{G}_{t\text{-}alg}^{k,c}$
& \yes & \yes & \yes & \yes & \yes & \no & \no
& Spline approximators in Besov spaces~\cite{de1983approximation,devore1993besov,devore1996convex}; tropical dynamic programming~\cite{jukna2023tropical}, including Floyd-Warshall~\cite{floyd1962shortest,warshall1962theorem}; first-order splitting schemes for linear programs.
\\
\midrule
$\mathbb{G}_{rt\text{-}alg}^{k,c,r^{\star}}$
& \yes & \yes & \yes & \yes & \yes & \yes & \no
& Algebraic and tropical computations augmented by bounded-degree powers and radicals.
\\
\midrule
$\mathbb{G}_{Rat}^{k,c,r^{\star}}$
& \yes & \yes & \yes & \yes & \yes & \yes & \yes
& Finite-step Newton-Raphson and secant iterations on domains where the required divisions are well-defined~\cite{press1992numericalrecipes}; matrix inversion via the adjugate formula on nonsingular matrices.
\\
\bottomrule
\end{tabular}
\end{adjustbox}
\endgroup
\caption{Gate languages used in this paper.  Green entries indicate primitive gates in the corresponding language, while red entries indicate operations not included as primitive gates.  The final column lists representative objects and algorithms computable by circuits over the corresponding language, on the relevant domains; the list is illustrative rather than exhaustive.}
\label{tab:gate_languages}
\end{table}

Classical sparse polynomial approximator classes, such as best $m$-term approximation~\cite{adcock2022sparse}, are already computable in the algebraic language $\mathbb{G}_{alg}^{k,c}$.  Pure dynamic-programming algorithms use only a proper subset of the algebro-tropical gates, cf.\ tropical complexity theory~\cite{jukna2023tropical}.  Piecewise-linear neural networks, such as $\operatorname{ReLU}$-MLPs, also fit naturally into this hierarchy: they use affine operations together with piecewise-linear gates, and in particular implement binary maximization exactly by $\max\{x,y\}=x+\operatorname{ReLU}(y-x)$.  Binary multiplication is not primitive in such networks, but can be approximately implemented by standard logarithmic-depth constructions; cf.~\cite{petersen2024mathematical}.
Having specified the basic \textit{units of computation}, we now explain how these units are assembled into an \textit{algorithm}.  We adopt the non-uniform viewpoint of circuit complexity theory, where elementary operations are organized as a finite computational graph rather than generated recursively, as in uniform models of computation such as Turing machines.

\subsubsection{Directed Acyclic Graphs (DAGs)}
\label{s:Prelims__ss:Algos___sss:DAGs}
In what follows, we consider connected \textit{directed acyclic graphs} (DAGs), where connected means connected after forgetting edge orientations.  A DAG is a pair $D=(V,E)$ of a (possibly infinite) non-empty set $V\subseteq \mathbb{N}$ and directed edges $E\subseteq \{(v_1,v_2)\in V^2:\, v_1\neq v_2\}$ with no directed cycles.  A finite sequence $(v_n)_{n=1}^N\subseteq V$ satisfying $(v_n,v_{n+1})\in E$ for each $n=1,\dots,N-1$ is called a \textit{directed path}; it is a \textit{cycle} if $N\ge 2$ and $v_1=v_N$.
A parent of a vertex $v\in V$ is a vertex $w\in V$ for which $(w,v)\in E$, in this case, we say that $v$ is the \textit{child} of $w$.  
The set of parents of a vertex $v$ is denoted by $\pa{v}\eqdef \{w\in V:\, (w,v)\in E\}$ and the set of children of $v$ is denoted by $\ch{v}\eqdef \{w\in V:\, (v,w)\in E\}$.

Following standard circuit complexity theory, we will use DAGs to encode the computational graphs of algorithms, as in contemporary automatic-differentiation software such as $\operatorname{TensorFlow}$~\cite{abadi2015tensorflow} and $\operatorname{PyTorch}$~\cite{paszke2019pytorch}.
A vertex $v\in V$ is called an \textit{input node} if it has no parents; input nodes represent inputs, or components of vector-valued inputs, of the algorithm encoded by $D$.  A vertex $v\in V$ is called an \textit{output node} if it has no children; output nodes represent outputs, or components of vector-valued outputs.
The sets of input nodes and output nodes are thus, respectively,
\[
        V_{\operatorname{in}}\eqdef \{v\in V:\, \pa{v}
    =
        \emptyset \}
    \mbox{ and }
        V_{\operatorname{out}}\eqdef \{v\in V:\, \ch{v}
    =
        \emptyset \}.
\]
All nodes in $V$ which are not input nodes are called \textit{computation nodes}; these are precisely the points in the algorithm where a computation or readout occurs.  The set of all computation nodes is
\[
    V_{\operatorname{comp}}\eqdef \{v\in V:\, \pa{v}\neq \emptyset\}
    =
    V\setminus V_{\operatorname{in}}.
\]
Let $d,D \in \mathbb{N}_+$.  We use $\DAG$ to denote the set of connected DAGs with exactly $d$ input nodes, exactly $D$ output nodes, and finitely many computation nodes, $\#V_{\operatorname{comp}}<\infty$.  These DAGs will be used to encode circuits computing functions from $\mathbb{R}^d$ to $\mathbb{R}^D$.

\subsubsection{$\mathbb{G}$-Computability}
\label{s:Prelims__ss:Algos___sss:Circuits}
A \textit{$\mathbb{G}$-circuit} $\mathcal{C}$ is a triple $\mathcal{C}\eqdef (V,E,\mathcal{G})$ consisting of a DAG $D_{\mathcal{C}}\eqdef (V,E)$ in $\DAG[d_0,D]$, with $V\subseteq \mathbb{N}$, and a collection of node decorations $\mathcal{G}\eqdef \{g_v\}_{v\in V}$.  Input nodes are identified with the identity map on $\mathbb{R}$, while each non-input node $v\in V\setminus V_{\operatorname{in}}$ is decorated by a gate $g_v\in\mathbb{G}$ whose arity is $\#\pa{v}$.  These decorations satisfy the \textit{compatibility condition} for each non-input node $v\in V\setminus V_{\operatorname{in}}$
\begin{equation}
\tag{Comp}
\label{eq:compatability}
        \big(
            x^{(v_1)},
            \dots,
            x^{(v_n)}
        \big)
    \in 
        \operatorname{dom}(g_v)
\end{equation}
whenever the parent values $x^{(v_1)},\dots,x^{(v_n)}$ have been defined in the recursive evaluation below, where $\{v_1<\dots<v_n\}=\pa{v}$ is the naturally ordered set of parent nodes feeding into $v$.
A $\mathbb{G}$-circuit $\mathcal{C}$ \textit{computes} a function $\Rep[\mathcal{C}]:\mathbb{R}^d\to \mathbb{R}^D$ if there are $d_0\in \mathbb{N}_+$ possibly duplicated lifted input variables and a row-stochastic lifting matrix $\Pi\in\{0,1\}^{d_0\times d}$ such that $\Rep[\mathcal{C}]$ is defined recursively for each $x\in \mathbb{R}^d$ as follows:
\hfill\\
Let $V_{\operatorname{in}}=\{v_1<\dots<v_{d_0}\}$.  Define the input-node values by
\begin{equation}
\label{eq:base_variables}
    (x^{(v_1)},\dots,x^{(v_{d_0})})
    \eqdef
    \Pi x.
\end{equation}
Next, evaluate the non-input nodes in any order compatible with the partial order of the DAG.  For each $w\in V\setminus V_{\operatorname{in}}$, write $\pa{w}=\{w_1<\dots<w_{D_w}\}$ and define $x^{(w)}\in \mathbb{R}$ by\footnote{We emphasize that $x^{(w)}$ is a function of the input data $x\in \mathbb{R}^d$.}
\begin{equation}
\label{eq:computation_recursive}
    x^{(w)}
        \eqdef 
    g_w\big(
        x^{(w_1)},\dots,x^{(w_{D_w})}
    \big).
\end{equation}
Finally, using the natural ordering $V_{\operatorname{out}}=\{u_1<\dots<u_D\}$, we output
\begin{equation}
\label{eq:computation_finished}
        \Rep[\mathcal{C}](x)
    \eqdef
        (x^{(u_i)})_{i=1}^D.
\end{equation}
As in~\cite{li2026certifiable}, we call $\Pi$ a \textit{lifting channel}.
In this way, we distinguish an algorithm from the function it computes, since a single function may admit many algorithmic representations.  This is analogous to the distinction between a parameterized MLP and the function it realizes in~\cite{petersen2018optimal}.
We close this section with a precise description of the complexity of computing a $\mathbb{G}$-computable function, relative to the gate language $\mathbb{G}$.
\begin{definition}
\label{defn:complexity_class__exact}
Fix input and output dimensions $d,D\in \mathbb{N}_+$ and a function $f:\mathbb{R}^d\to \mathbb{R}^D$.
\hfill\\
\noindent
Fix a gate language $\mathbb{G}$ and parameters $d_0,\Delta,\Upsilon\in\mathbb{N}_+$.  
We say that $f$ belongs to the complexity class $\operatorname{G}^{\Delta,\Upsilon}_{d_0}$, relative to $\mathbb{G}$, if there exists a lifting channel $\Pi\in \{0,1\}^{d_0\times d}$ and a $\mathbb{G}$-circuit $\mathcal{C}$ of depth at most $\Delta$ and width at most $\Upsilon$ computing $f$, i.e., satisfying \eqref{eq:base_variables},~\eqref{eq:computation_recursive}, and~\eqref{eq:computation_finished}.
\end{definition}

\section{Main Results}
\label{s:Approximation}

\subsection{Main Qualitative Result: Characterization of Universal Approximation}
\label{s:Approximation__ss:Qualitative}

We begin with our qualitative universality result, the formal version of Informal Theorem~\ref{inf_thrm:Universal_Computation}, before turning to quantitative circuit-to-network bounds below.  The quantitative statements are cleanest for gate languages without reciprocal gates.\footnote{The reciprocal gate can introduce large quantitative overheads because of the singularity of $1/x$ at the origin.  Quantitative results with reciprocal gates are also available, but are more technical; see Theorem~\ref{thrm:concrete_surgery}.}
Informally, the definability hypotheses restrict the admissible non-linearities to tame operations admitting finite, though arbitrarily long, rule-based descriptions.  In particular, they exclude typical Brownian paths, highly oscillatory activations of the kind considered in~\cite{yarotsky2021elementary,shen2022deep_JMLR}, and unrestricted sine waves on unbounded intervals~\cite{Khovanskii_Fewnomials} as primitive non-linearities.
In this tame setting, our main qualitative result shows that a definable feedforward model of the form of Definition~\ref{defn:ANN} is a universal approximator if and only if its dictionary contains a non-affine non-linearity.

\begin{theorem}[Characterization of Universal Approximation]
\label{thrm:Universal_Computation}
Under Assumption~\ref{ass:THEASSUMPTIONS} (i) and (ii), the following are equivalent:
\begin{enumerate}
        \item \textbf{Non-Affineness:} $\mathcal{A}$ contains a non-affine non-linearity.
    \    \item \textbf{Universal Approximator:} For any uniformly continuous $f:[0,1]^d\to \mathbb{R}$ and every $\epsilon>0$, there exists a $\mathcal{A}$NN $\hat{f}_{\epsilon}:\mathbb{R}^d\to \mathbb{R}$ satisfying
\begin{equation}
\label{eq:universality}
    \sup_{x\in [0,1]^d}
    \,
        \big|
            f(x)
            -
            \hat{f}_{\epsilon}(x)
        \big|
    <
        \epsilon
.
\end{equation}
\end{enumerate}
\end{theorem}
\begin{proof}
See Section~\ref{s:Proof_of_thrm:Universal_Computation}.
\end{proof}

Theorem~\ref{thrm:Universal_Computation} gives a definable analogue of the classical activation-function characterizations of universal approximation, replacing the shallow non-polynomial condition by the precise non-affine condition for the present definable feedforward model.  This should be contrasted with the sufficient but not necessary conditions for universality developed in~\cite{kidger2020universal,zhang2024deep}; which we now discuss before moving to our quantitative guarantee.
\subsubsection{Discussion: The Role of Definability}
\label{s:Approximation__ss:Qualitative___sss:RoleOfDefinabilyOMinimal}
The requirement that our \textit{networks} are definable (within a fixed, o-minimal expansion of $\reals$), while the target function need not be, plays a precise structural role. 
The classes introduced in~\cite{zhang2024deep} are designed to cover a broad range of commonly used activation functions, including piecewise-smooth activations, smooth ReLU-type activations, and S-shaped activations.  From the perspective of the present paper, however, these conditions are not tameness assumptions: they do not impose definability in an o-minimal structure, and therefore do not rule out artificial perturbations by sets that are not ``tame'' enough to be definable in an o-minimal setting.  This is not a limitation of~\cite{zhang2024deep}, whose aim is a broad class of sufficient conditions, but it explains how the o-minimality hypothesis is doing additional structural maintenance.

Fix a shifted copy $\mathfrak{C}\subseteq [-2,-1]$ of the classical Cantor set.  Since every one-dimensional definable set in an o-minimal structure is a finite union of points and intervals, $\mathfrak{C}$ is not definable in any o-minimal structure.  Let
\[
        \xi(u)
    \eqdef
        \max\big\{0,
            1-\big|u+\tfrac{3}{2}\big|
        \big\}
        \,
        \operatorname{dist}(u,\mathfrak{C}),
\]
where $\operatorname{dist}(u,\mathfrak{C}) \eqdef \inf_{c\in \mathfrak{C}} |u-c|$.  Since $\mathfrak{C}$ is closed, the function $\xi$ is continuous.  
Moreover, $\xi$ is compactly supported and is not definable in any o-minimal structure; if it were definable then so would be
$
        \mathfrak{C}
    =
        [-2,-1]\cap \xi^{-1}(\{0\})
$, which is impossible.
\hfill\\
\noindent
For each $k\in \mathbb{N}$, consider the kink class $\mathcal{A}_{1,k}$ of~\citep[page 2]{zhang2024beyond}.  The non-definable $\operatorname{ReLU}^{k+1}$ variant
\[
        \sigma_{1,k,\operatorname{bad}}(u)\eqdef \operatorname{ReLU}(u)^{k+1}+\xi(u)
\]
belongs to $\mathcal{A}_{1,k}$, since near $u=0$ it agrees with $\operatorname{ReLU}(u)^{k+1}$ and hence has the required $k$-th order kink.
Next we consider the smooth ReLU-type class $\mathcal{A}_{2}$, defined through bounded $S$-shaped factors,  cf.~\citep[page 3]{zhang2024deep},
this class contains the non-definable $\operatorname{softplus}$ variant
\[
        \sigma_{2,\operatorname{bad}}(u)
    \eqdef
        \log(1+e^u)
        +\xi(u)
\]
which 
belongs to $\mathcal{A}_{2}$.  
Finally, the class $\mathcal{A}_3$ in~\citep[page 3]{zhang2024deep}, which contains the classical sigmoidal activation functions in~\cite{hornik1989multilayer}, but it also contains the non-definable $\operatorname{sigmoid}$ variant
\[
        \sigma_{3,\operatorname{bad}}(u)
    \eqdef
        \frac{e^u}{1+e^u}+\xi(u).
\]
Thus, the point is \textit{not} that the broad classes of~\cite{zhang2024deep} are inadequate, but rather that definability serves a different purpose: it rules out just enough pathological activation functions to obtain a simple characterization of universal approximation as being equivalent to being non-affine, while retaining real-world non-linearities; esp.\ those noted in~\cite{zhang2024beyond}.
Although statistical guarantees are outside our scope, $\mathcal{A}$NNs with definable activation functions were shown to enjoy finite sample-complexity bounds for classification and regression in~\cite[Theorems~3.1 and~3.2]{Greggy2026}, thus further supporting the suitability of the o-minimal framework for deep learning.

\subsection{Main Quantitative Result: Circuit Complexity Controls \texorpdfstring{$\mathcal{A}$NN}{Neural Network} Complexity}
\label{s:Approximation__ss:Quantitative}
The main quantitative engine behind the paper is the following circuit-to-neural-network compilation theorem. We use the term \textit{compilation} deliberately. At the local level, the proof replaces each elementary $\mathbb{G}$-gate by a neural \textit{emulator}. At the global level, however, the result gives a systematic translation of an entire algorithmic object written in one computational formalism, namely the source language of $\mathbb{G}$-circuits, into an algorithmic object written in another, namely the target language of $\mathcal{A}$NNs. Thus, the theorem compiles the full circuit, including its wiring, depth, width, and size information, into a neural network with explicit quantitative complexity bounds.  

\begin{theorem}[Circuit-to-Neural Network Compilation]
\label{thrm:concrete_surgery}
Let $\mathcal{A}$ satisfy Assumption~\ref{ass:THEASSUMPTIONS}, and fix $0<\varepsilon\le 1$.
Suppose that
$
f:\mathbb{R}^d\to \mathbb{R}^D
$
is $\varepsilon$-computed on $[-1,1]^d$ by a $\mathbb{G}$-circuit (with $\mathbb{G}$ being one of the classes defined in Table~\ref{tab:gate_languages})
$
\mathcal{C}=(V,E,\mathcal{G})
$
with lifting dimension $d_0$, depth $\Delta$, width $\Upsilon$, and gate count $N$, where
$
\bar{\Upsilon}\eqdef \max\{D,\Upsilon+|E|\}
$. 
Then there exists an $\mathcal{A}$NN
$
\hat{f}_{\mathcal{C}}:\mathbb{R}^d\to \mathbb{R}^D
$
which computes $f$ on $[-1,1]^d$ up to error $2\varepsilon$, i.e.,
satisfying
\begin{equation}
\label{eq:epsilon_computation}
    \sup_{x\in [-1,1]^d}
    \,
        \|
            f(x)
            -
            \hat{f}_{\mathcal{C}}(x)
        \|_{\infty}
    \le
        2\varepsilon
.
\end{equation}
More precisely, depending on the gate class $\mathbb{G}$, one may choose $\hat{f}_{\mathcal{C}}$ as follows. (Whenever $r^\star$ appears below, let $s\eqdef \max\{k,r^\star\}$.)
\begin{enumerate}
\item[(i)] If
$
    \mathbb{G}=\mathbb{G}_{alg}^{k,c}
$,
with $k\ge 2$, then $\hat{f}_{\mathcal{C}}$ has
\begin{enumerate}
    \item \textbf{Depth:}
    $
        \mathcal{O}(\Delta\log k)
    $,
    \item \textbf{Width:}
    $
        \mathcal{O}(d+\bar{\Upsilon})
    $,
        \item \textbf{Non-Zero Parameters (Size):}
    $
        \mathcal{O}(d_0+\bar{\Upsilon}\Delta k)
    $.
\end{enumerate}
\item[(ii)] If
$
    \mathbb{G}=\mathbb{G}_{t\text{-}alg}^{k,c}
$,
with $k\ge 2$, then $\hat{f}_{\mathcal{C}}$ has
\begin{enumerate}
    \item \textbf{Depth:}
    $
        \mathcal{O}\Big(
            \Delta\log k
            \Big(
                k^\Delta
                +
                \log\Big(\frac{2\Delta}{\varepsilon}\Big)
            \Big)
            \Big(
                \Delta\log k
                +
                \log^{\circ 2}\Big(\frac{2\Delta}{\varepsilon}\Big)
            \Big)
        \Big)$%
    \footnote{For any set $X$ and map $f:X\to X$, and any $2\le k\in\mathbb{N}$ we write $f^{\circ k}\eqdef f\circ\cdots\circ f$ for the $k$-fold self-composition of $f$.},
    \item \textbf{Width:}
    $
        \mathcal{O}(d+\bar{\Upsilon})
    $,
    \item \textbf{No.\ Parameters (Size):}
    $
        \mathcal{O}\Big(
            d_0+
            \bar{\Upsilon}\Delta k
            \Big(
                k^\Delta
                +
                \log\Big(\frac{2\Delta}{\varepsilon}\Big)
            \Big)
            \Big(
                \Delta\log k
                +
                \log^{\circ 2}\Big(\frac{2\Delta}{\varepsilon}\Big)
            \Big)
        \Big)
    $.
\end{enumerate}
\item[(iii)] If
$
    \mathbb{G}=\mathbb{G}_{rt\text{-}alg}^{k,c,r^\star}
$
with $k\ge 2$ and $r^\star\ge 1$, then $\hat{f}_{\mathcal{C}}$ has
\begin{enumerate}
    \item \textbf{Depth:}
    $
        \mathcal{O}\Big(
            \Delta\log k
            +
            \Delta
            r^\star
            s^\Delta
            \Big(
                s^\Delta
                +
                \log\Big(\frac{2\Delta}{\varepsilon}\Big)
            \Big)
            \Big(
                r^\star
                +
                \Delta
                +
                \log^{\circ 2}\Big(\frac{2\Delta}{\varepsilon}\Big)
            \Big)
        \Big),
    $
    \item \textbf{Width:}
    $
        \mathcal{O}(d+\bar{\Upsilon})
    $,
    \item \textbf{No.\ Parameters (Size):}
    $
        \mathcal{O}\Big(
            d_0
            +
            \bar{\Upsilon}
            \Big[
                \Delta k
                +
                \Delta
                r^\star
                s^\Delta
                \Big(
                    s^\Delta
                    +
                    \log\Big(\frac{2\Delta}{\varepsilon}\Big)
                \Big)
                \Big(
                    r^\star
                    +
                    \Delta
                    +
                    \log^{\circ 2}\Big(\frac{2\Delta}{\varepsilon}\Big)
                \Big)
            \Big]
        \Big)
    $.
\end{enumerate}
\item[(iv)] If $\mathbb{G}=\mathbb{G}_{Rat}^{k,c,r^\star}$, with $k\ge 2$ and $r^\star\ge 1$, and if the circuit $\mathcal{C}$ satisfies the forbidden $m$-compositions condition%
\footnote{See Definition~\ref{defn:no_forbiddencomposition}.}
, then $\hat{f}_{\mathcal{C}}$ has
    \begin{enumerate}
        \item \textbf{Depth:}
        $
            \mathcal{O}\Big(
                \Delta\log k
                +
                \Delta\,s^{2\Delta}
                \log\Big(\frac{2\Delta}{\varepsilon}\Big)
                \Big(
                    \Delta+\log^{\circ 2}\Big(\frac{2\Delta}{\varepsilon}\Big)
                \Big)
            \Big)
        $,
        \item \textbf{Width:}
        $
            \mathcal{O}(d+\bar{\Upsilon})
        $,
        \item \textbf{No.\ Parameters (Size):}
        $
            \mathcal{O}\Big(
                d_0
                +
                \bar{\Upsilon}
                \Big[
                    \Delta k
                    +
                    \Delta\,s^{2\Delta}
                    \log\Big(\frac{2\Delta}{\varepsilon}\Big)
                    \Big(
                        \Delta+\log^{\circ 2}\Big(\frac{2\Delta}{\varepsilon}\Big)
                    \Big)
                \Big]
            \Big)
        $.
    \end{enumerate}
\end{enumerate}
All implicit constants are independent of $\varepsilon$, $\Delta$, $\Upsilon$, $N$, $d$, $D$, and the particular circuit $\mathcal{C}$, except through the displayed quantities.\footnote{They may depend on the fixed gate-language parameters $k,c,r^\star$, and, in the rational-tropical case, on the forbidden-composition parameter $m$.}
\end{theorem}
\begin{proof}
See Section~\ref{s:ProofMain}.
\end{proof}

Note that the division-free cases yield particularly clean bounds, whereas reciprocal gates require the additional non-singularity condition appearing in part~\textup{(iv)} above.

\begin{remark}[Scope of the Quantitative Bounds]
The bounds in Theorem~\ref{thrm:concrete_surgery} are worst-case real-valued representation bounds.  They count depth, width, and non-zero real parameters of the compiled $\mathcal{A}$NN; they do not account for bit complexity, parameter precision, conditioning, memory traffic, or training time.  The depth-dependent factors in the algebro-tropical, radical, and rational cases should therefore be read as part of the generic worst-case compilation cost.
\end{remark}

Before turning to the ``white-box'' analysis, cf.\ Figure~\ref{fig:WhiteBox}, where we directly estimate the complexity of neural networks computing functions whose explicit circuits are known, we first provide several ``grey-box'' lookup tables.  These tables record the specialized consequences proved in Appendix~\ref{s:Proof__ss:thrm:complexity_class_inclusions}, and allow the reader to read off neural-network complexity bounds from two pieces of information: the growth rates of the underlying circuit and the language of elementary gates required to compute the target object.

\subsubsection{{\color{gray}Grey-Box} Lookup Tables: Neural-Network Complexity from Circuit Complexity}
\label{s:Approximation__ss:Quantitative__ss:LookupTables}

Theorem~\ref{thrm:concrete_surgery} yields a general transfer principle: once a target function is computed by a $\mathbb{G}$-circuit, it can be emulated by an $\mathcal{A}$NN with explicit depth, width, and non-zero-parameter overheads.  Here we record the specialized consequences for the complexity regimes most relevant to numerical analysis and constructive approximation theory; the corresponding derivations are given in Appendix~\ref{s:Proof__ss:thrm:complexity_class_inclusions}. E.g.\ finite-dimensional minimax-optimal approximation schemes and many standard numerical solvers are naturally described by $\mathbb{G}_{t\text{-}alg}^{k,c}$, $\mathbb{G}_{rt\text{-}alg}^{k,c,r^\star}$, or $\mathbb{G}_{Rat}^{k,c,r^\star}$ circuits whose gate count and width scale polylogarithmically or polynomially in $\kappa\eqdef \varepsilon^{-1}$.

\begin{table}[H]
\centering
\begin{adjustbox}{max width=\textwidth}
\small
\renewcommand{\arraystretch}{1.25}
\begin{tabular}{@{}llll@{}}
\toprule
\textbf{Assumptions on $(N,\Upsilon)$}
&
$\mathbb{G}_{alg}^{k,c},\,\mathbb{G}_{t\text{-}alg}^{k,c}$ \textbf{non-zero parameters}
&
$\mathbb{G}_{rt\text{-}alg}^{k,c,r^\star}$ \textbf{non-zero parameters}
&
$\mathbb{G}_{Rat}^{k,c,r^\star}$ \textbf{non-zero parameters}
\\
\midrule

$N\in \mathcal{O}((\log \kappa)^p),\ \Upsilon\in \mathcal{O}((\log \kappa)^q)$
&
\sameclass{$\mathcal{O}\big((\log \kappa)^{\max\{p,q\}+r}\big)$}
&
\sameclass{$\mathcal{O}\big((\log \kappa)^{\max\{p,q\}+r+1}(1+\log\log \kappa)\big)$}
&
\neutral{$\mathcal{O}\big((\log \kappa)^{\max\{p,q\}+2r+1}s^{\mathcal{O}((\log \kappa)^r)}\big)$}
\\[0.9em]

$N\in \mathcal{O}(\kappa^p),\ \Upsilon\in \mathcal{O}((\log \kappa)^q)$
&
\samepower{$\mathcal{O}\big(\kappa^p(\log \kappa)^r\big)$}
&
\samepower{$\mathcal{O}\big(\kappa^p(\log \kappa)^{r+1}(1+\log\log \kappa)\big)$}
&
\neutral{$\mathcal{O}\big(\kappa^p(\log \kappa)^{2r+1}s^{\mathcal{O}((\log \kappa)^r)}\big)$}
\\[0.9em]

$N\in \mathcal{O}((\log \kappa)^p),\ \Upsilon\in \mathcal{O}(\kappa^q)$
&
\samepower{$\mathcal{O}\big(\kappa^q(\log \kappa)^r\big)$}
&
\samepower{$\mathcal{O}\big(\kappa^q(\log \kappa)^{r+1}(1+\log\log \kappa)\big)$}
&
\neutral{$\mathcal{O}\big(\kappa^q(\log \kappa)^{2r+1}s^{\mathcal{O}((\log \kappa)^r)}\big)$}
\\[0.9em]

$N\in \mathcal{O}(\kappa^p),\ \Upsilon\in \mathcal{O}(\kappa^q)$
&
\samepower{$\mathcal{O}\big(\kappa^{\max\{p,q\}}(\log \kappa)^r\big)$}
&
\samepower{$\mathcal{O}\big(\kappa^{\max\{p,q\}}(\log \kappa)^{r+1}(1+\log\log \kappa)\big)$}
&
\neutral{$\mathcal{O}\big(\kappa^{\max\{p,q\}}(\log \kappa)^{2r+1}s^{\mathcal{O}((\log \kappa)^r)}\big)$}
\\

\bottomrule
\end{tabular}
\end{adjustbox}
\caption{\textbf{Scaling Laws for Poly\textit{logarithmic}-Depth Circuits} (cf.\ Theorem~\ref{thrm:concrete_surgery} and Appendix~\ref{s:Proof__ss:thrm:complexity_class_inclusions}):
For the inverse error scale 
$
\kappa\eqdef \tfrac{1}{\varepsilon}
$, consider a $\mathbb{G}$-computable function $
f:[-1,1]^d\to \mathbb{R}
$ computable by $\mathbb{G}$-circuits of polylogarithmic depth
$
\Delta\in \mathcal{O}((\log \kappa)^r)
$ using at most $N$ gates and width $\Upsilon$.
The table summarizes the induced $\mathcal{A}$NN non-zero-parameter count under polylogarithmic or polynomial scaling assumptions on $N$ and $\Upsilon$, where $p,q\ge 0$, $r>0$, and $s\eqdef \max\{k,r^\star\}$.
{\color{gray!60}\rule{\linewidth}{0.3pt}}
\textit{Dark green}: indicates preservation of the same dominant power up to mild logarithmic losses;
\hfill\\
\textit{Light green}: indicates preservation of the same qualitative complexity class (polylogarithmic or polynomial);
\hfill\\
\textit{Neutral/yellow}: indicates that, in the rational case, the precise class depends on the additional factor
$
s^{\mathcal{O}((\log \kappa)^r)}
$ and on the forbidden-composition condition from Definition~\ref{defn:no_forbiddencomposition}.
\hfill\\
{\color{gray!60}\rule{\linewidth}{0.3pt}}
}
\label{tab:log_depth}
\end{table}

Table~\ref{tab:log_depth} shows that, under the table assumptions and the specialized bounds proved in Appendix~\ref{s:Proof__ss:thrm:complexity_class_inclusions}, polylogarithmic-depth circuits often yield $\mathcal{A}$NN emulators in the same qualitative complexity class.  In particular, outside the rational case, polylogarithmic gate-count regimes remain polylogarithmic in the induced non-zero-parameter count, and polynomial gate-count regimes remain polynomial, up to the logarithmic losses induced by gate emulation.  Thus, in many regimes arising from constructive approximation and stable numerical algorithms, neural emulation preserves the dominant complexity scale.

\begin{table}[htb!]
\centering
\begin{adjustbox}{max width=\textwidth}
\small
\renewcommand{\arraystretch}{1.25}
\begin{tabular}{@{}llll@{}}
\toprule
\textbf{Assumptions on $(N,\Upsilon)$}
&
$\mathbb{G}_{alg}^{k,c},\,\mathbb{G}_{t\text{-}alg}^{k,c}$ \textbf{non-zero parameters}
&
$\mathbb{G}_{rt\text{-}alg}^{k,c,r^\star}$ \textbf{non-zero parameters}
&
$\mathbb{G}_{Rat}^{k,c,r^\star}$ \textbf{non-zero parameters}
\\
\midrule

$N\in \mathcal{O}((\log \kappa)^p),\ \Upsilon\in \mathcal{O}((\log \kappa)^q)$
&
\badjump{$\mathcal{O}\big(\kappa^r(\log \kappa)^{\max\{p,q\}}\big)$}
&
\badjump{$\mathcal{O}\big(\kappa^r(\log \kappa)^{\max\{p,q\}+1}(1+\log\log \kappa)\big)$}
&
\badjump{$\mathcal{O}\big(\kappa^{2r}(\log \kappa)^{\max\{p,q\}+1}s^{\mathcal{O}(\kappa^r)}\big)$}
\\[0.9em]

$N\in \mathcal{O}(\kappa^p),\ \Upsilon\in \mathcal{O}((\log \kappa)^q)$
&
\sameclass{$\mathcal{O}\big(\kappa^{p+r}\big)$}
&
\sameclass{$\mathcal{O}\big(\kappa^{p+r}\log \kappa\,(1+\log\log \kappa)\big)$}
&
\badjump{$\mathcal{O}\big(\kappa^{p+2r}\log \kappa\, s^{\mathcal{O}(\kappa^r)}\big)$}
\\[0.9em]

$N\in \mathcal{O}((\log \kappa)^p),\ \Upsilon\in \mathcal{O}(\kappa^q)$
&
\sameclass{$\mathcal{O}\big(\kappa^{q+r}\big)$}
&
\sameclass{$\mathcal{O}\big(\kappa^{q+r}\log \kappa\,(1+\log\log \kappa)\big)$}
&
\badjump{$\mathcal{O}\big(\kappa^{q+2r}\log \kappa\, s^{\mathcal{O}(\kappa^r)}\big)$}
\\[0.9em]

$N\in \mathcal{O}(\kappa^p),\ \Upsilon\in \mathcal{O}(\kappa^q)$
&
\sameclass{$\mathcal{O}\big(\kappa^{\max\{p,q\}+r}\big)$}
&
\sameclass{$\mathcal{O}\big(\kappa^{\max\{p,q\}+r}\log \kappa\,(1+\log\log \kappa)\big)$}
&
\badjump{$\mathcal{O}\big(\kappa^{\max\{p,q\}+2r}\log \kappa\, s^{\mathcal{O}(\kappa^r)}\big)$}
\\

\bottomrule
\end{tabular}
\end{adjustbox}
\caption{\textbf{Scaling Laws for \textit{Polynomial}-Depth Circuits} (cf.\ Theorem~\ref{thrm:concrete_surgery} and Appendix~\ref{s:Proof__ss:thrm:complexity_class_inclusions}):
Analogous setting as Table~\ref{tab:log_depth}, but where the $\mathbb{G}$-circuits computing $f$ have polynomial depth $\Delta\in\mathcal{O}(\kappa^r)$ with $r>0$.
\hfill\\
{\color{gray!60}\rule{\linewidth}{0.3pt}}
Shading is as in Table~\ref{tab:log_depth}, but now we use \badjump{dark orange} to emphasize changes in class, from polylogarithmic scaling to polynomial, or from polynomial to exponential.}
\label{tab:poly_depth}
\end{table}

By contrast, Table~\ref{tab:poly_depth} illustrates the cost of deep sequential computation.  When the original circuit has polynomial depth, the emulator may move to a larger complexity class: polylogarithmic gate-count regimes may become polynomial in the induced non-zero-parameter count, and rational circuits may incur exponential-type factors through the term $s^{\mathcal{O}(\kappa^r)}$.  This is not an artifact of the presentation; it reflects the role of circuit depth as a genuine measure of sequential computational hardness.  Throughout both tables, circuit complexity is measured by gate count and width, while $\mathcal{A}$NN complexity is measured by depth, width, and number of non-zero parameters.

Other scaling laws, including non-polynomial regimes, follow from the same derivation and the bounds in Theorem~\ref{thrm:concrete_surgery}.  We do not tabulate them here, since in genuinely exponential regimes the resulting upper bounds are typically less informative.

\section{Implication I: Unified Approximation Theory for \texorpdfstring{$\mathcal{A}$NNs}{Feedforward Architectures}}
\label{s:Implication__Minimax}

In this section, we show that the quantitative neural compilation theorem recovers, and simultaneously extends, standard approximation guarantees for deep networks.  The point is simple: classical constructive approximation schemes are algorithms.  Once such a scheme is written as a circuit, Theorem~\ref{thrm:concrete_surgery} converts it into an $\mathcal{A}$NN with comparable depth, width, and number of non-zero parameters.
We organize the consequences according to the usual regularity scale studied in neural network theory and constructive approximation.  At low regularity, we consider uniformly continuous functions on Ahlfors-regular supports, recovering universal-approximation-type guarantees~\cite{yarotsky2018optimal,shen2021deep,hong2024bridging}.  At moderate regularity, we consider Besov classes $B_{p,q}^s(\mathcal{X})$ on $(\epsilon,\delta)$-domains, with $1\le p<\infty$, $0<q<\infty$, and $0<s<\infty$, recovering the rates of nonlinear spline-wavelet approximation~\cite{devore1988interpolation,devore1993besov}.  At high regularity, we consider functions extending holomorphically to a Bernstein polyellipse, where sparse orthogonal-polynomial methods yield dimension-explicit logarithmic-error complexity~\cite{adcock2022sparse}.  These regimes correspond, up to polylogarithmic factors, to the constructive-approximation rows of Table~\ref{tab:reference_algorithms}; after neural compilation, they yield the $\mathcal{A}$NN complexity bounds summarized in Table~\ref{tab_thrm:approx}.

\begin{table}[htp!]
\centering
\begin{adjustbox}{width=\textwidth, height=\textheight, keepaspectratio, center}
\small
\begin{tabular}{@{\,}
  >{\raggedright\arraybackslash}l
  >{\centering\arraybackslash}l
  >{\centering\arraybackslash}l
  >{\centering\arraybackslash}l
  >{\raggedright\arraybackslash}l
@{}}
\toprule
\textbf{Function Class} & \textbf{Depth} & \textbf{Width} & \textbf{Non-Zero Parameters} & \textbf{Reference} \\
\midrule

Uniformly continuous functions
  &
  finite
  &
  finite
  &
  finite
  &
  Theorem~\ref{thrm:UAT__continouusLowreg}
\\

Besov space $B_{p,q}^s(\mathcal{X})$
  &
  $\mathcal{O}\big(\log^2(\varepsilon^{-1})\big)$
  &
  $\mathcal{O}\big(\varepsilon^{-d/(sp)}\big)$
  &
  $\mathcal{O}\big(\varepsilon^{-d/(sp)}\log^2(\varepsilon^{-1})\big)$
  &
  Theorem~\ref{thrm:Besov_cube_Lp_ANN}
\\

Holomorphic class $\mathcal{H}_{d,\gamma}$
  &
  $\mathcal{O}\big(d\log(\gamma^{-1})+d\log\log(\varepsilon^{-1})\big)$
  &
  $\mathcal{O}\big((\gamma^{-1}\log(\varepsilon^{-1}))^{2d}\big)$
  &
  $\mathcal{O}\big((\gamma^{-1}\log(\varepsilon^{-1}))^{2d}\big)$
  &
  Theorem~\ref{thrm:HolomorphicApproximation}
\\

\bottomrule
\end{tabular}
\end{adjustbox}
\caption{Approximation guarantees obtained by compiling the constructive-approximation schemes from Table~\ref{tab:reference_algorithms} into $\mathcal{A}$NNs.  The table records the resulting depth, width, and number of non-zero parameters required to achieve the error convention stated in the corresponding theorem.  In the $L^p$ rows, $\varepsilon$ denotes the $p$-th power of the $L^p$ error unless explicitly stated otherwise.}
\label{tab_thrm:approx}
\end{table}

Throughout this section, unless explicitly stated otherwise, Assumption~\ref{ass:THEASSUMPTIONS} is in force.

\subsection{Low Regularity: Universal Approximation of Continuous Functions}
For any $0<d\le n$ and any set $A\subseteq \mathbb{R}^n$, we say that a Borel measure $\mu$ supported on $A$ is $d$-Ahlfors regular if there exists a constant $C\geq 1$ such that, for every $x\in A$ and every $0<r\leq \operatorname{diam}(A)$, we have
\begin{equation}
\label{eq:Ahlfors}
    C^{-1}r^d
\le
    \mu\big(A\cap B_2(x,r)\big)
\le
    Cr^d
.
\end{equation}
Examples include Lebesgue measure, Riemannian volume measures when $A$ is an embedded Riemannian submanifold of $\mathbb{R}^n$, and also many self-similar measures on fractal sets; cf.~\cite{heinonen2015sobolev}.  We also recall that if a compact set $A$ has box-covering dimension $d$, then for every $\eta>0$ it can be covered by at most $\mathcal{O}(\rho^{-(d+\eta)})$ Euclidean balls of radius $\rho$ as $\rho\downarrow 0$; cf.~\citep[Chapter 9.1]{robinson2010dimensions}.  
We also provide a sup-norm convergence guarantee, in the classical Stone-Weierstrass-Nachbin style of~\cite{hornik1989multilayer,kidger2020universal,kratsios2022universal,cuchiero2026global,bilokopytov2026universal}.
\begin{theorem}[Universal Approximation on Ahlfors-Regular Supports]
\label{thrm:UAT__continouusLowreg}
Under Assumption~\ref{ass:THEASSUMPTIONS}, 
let $\mathcal{X}\subseteq [0,1]^n$ be compact of box-covering dimension $0<d\le n$, and let $\mu$ be a $d$-Ahlfors regular Borel measure on $\mathcal{X}$.
Let $f:\mathcal{X}\to \mathbb{R}$ be uniformly continuous with strictly increasing concave modulus of continuity $\omega$.
Then, for every $\varepsilon>0$ it holds that:
\begin{enumerate}
    \item[(i)] For every $1\le p<\infty$, there exists an $\mathcal{A}$NN
        $
        \Phi_{f:\varepsilon}:\mathbb{R}^n\to\mathbb{R}
        $
        satisfying
        \[
            \int_{\mathcal{X}}
                |f(x)-\Phi_{f:\varepsilon}(x)|^p
            \,
            d\mu(x)
            \le
            \varepsilon
        .
        \]
    \item[(ii)] There exists an $\mathcal{A}$NN
        $
        \Phi_{f:\varepsilon}:\mathbb{R}^n\to\mathbb{R}
        $
        satisfying 
        \[
                \sup_{x\in \mathcal{X}}\,
                |f(x)-\Phi_{f:\varepsilon}(x)|
            \le
                \varepsilon
        .
        \]
\end{enumerate}
When $1\le p<\infty$, there is a constant $c_{\mathcal{X},\mu,p}>0$%
\footnote{Depending only on $p$, $\mu(\mathcal{X})$, and the geometric constants of $\mathcal{X}$.}~%
such that
$\Phi_{f:\varepsilon}$ has depth
$
\mathcal{O} \left(
    \log^2 \left(
        1/(\varepsilon\,\omega^{-1}(c_{\mathcal{X},\mu,p}\varepsilon^{1/p}))
    \right)
\right)
$,
width
$
\mathcal{O} \left(
    (\omega^{-1}(c_{\mathcal{X},\mu,p}\varepsilon^{1/p}))^{-2d}
\right)
$,
and
$
\mathcal{O} \left(
    (\omega^{-1}(c_{\mathcal{X},\mu,p}\varepsilon^{1/p}))^{-2d}
\right)
$
non-zero parameters.
\end{theorem}
\begin{proof}
See Section~\ref{s:Proofs__approx_theorems}.
\end{proof}

\begin{remark}[Transparency vs.\ Optimality]
Since, below the H\"{o}lder scale, minimax rates for general $\omega$-uniformly continuous functions are highly sensitive to the chosen modulus and are often less informative than the modulus dependence itself, we opt here for a simple and fully transparent proof of Theorem~\ref{thrm:UAT__continouusLowreg}, rather than for a rate-optimized construction in this very general class.  This argument makes explicit the dependence on every relevant parameter, including the modulus of continuity $\omega$, the ambient dimension $n$, and the Ahlfors/box dimension $d$ of the support of $\mu$, that is, the dimension of the region ``where the data lies''.  In the next section, we turn to Besov classes, which contain H\"{o}lder and smooth classes as subclasses, and recover minimax-optimal rates in that regularity regime; see Theorem~\ref{thrm:Besov_cube_Lp_ANN}.  In that setting, the resulting rates are directly comparable to the classical optimal rates.
\end{remark}

\begin{remark}[Definable Supports]
If $\mathcal{X}$ is compact and definable in an o-minimal structure, then its box-counting dimension agrees with its o-minimal dimension \cite{hieronymi_miller_dimension}.  If the target function $f$ is also definable, then the minimal modulus of continuity is definable as well, and so specific properties of the ambient o-minimal structure can sharpen these results. For example, an o-minimal structure is \emph{polynomially bounded} \cite{vanDenDries1998TameTopology} if and only if, for any definable function $g:\reals\rightarrow \reals$, there is some $N\geq 0$ such that $|g(t)|\leq t^N$ for all sufficiently large $t>0$. In this case, the minimal modulus of continuity for definable $f$ on a compact $\mathcal{X}$ is necessarily of H\"older type. The real field expanded by all restricted analytic functions gives an example of a polynomially bounded o-minimal structure. 
\end{remark}

\subsection{Moderate Regularity: Minimax-Optimal Rates for Besov Functions}
\label{s:Approximation__ss:Classical___sss:Besov}
We consider the approximation theory of H\"{o}lder and smooth functions under their natural constructive approximation-theoretic umbrella of Besov spaces, since Besov regularity characterizes approximation by best $N$-term approximants at prescribed rates; cf.~\citep[Theorem 9.2]{DeVoreLorentz__CABook___Vol1_1993}.  In this sense, best $N$-term approximation rates are encoded by smoothness, equivalently, by the Besov space parameters.  We briefly review these notions before presenting our approximation result.

We emphasize the need to work on general domains, since Besov-regular functions arise naturally in PDE theory, Besov spaces subsume Sobolev spaces, and many PDE applications are posed on non-cubic domains.  We restrict here to full-dimensional domains; Besov theory on fractal domains requires additional machinery and lies beyond the scope of this computation-and-approximation-theoretic paper.

\paragraph{Besov Spaces.}
Let $0<p<\infty$ and $0<\alpha$.  The modulus of smoothness of order $\alpha$ of a function $f\in L^p(\mathcal{X})$, defined with respect to the restricted Lebesgue measure on $\mathcal{X}\subseteq \mathbb{R}^d$, is given for each $t>0$ by
\[
	\omega_{\alpha,\mathcal{X}}(f,t)_p
	\eqdef
	\sup_{h\in \mathbb{R}^d\,:\,0<\|h\|\le t}\,
	\|\Delta_{h}^{\lceil \alpha\rceil}(f,\cdot)\|_{L^p(\mathcal{X},h)},
\]
where $\Delta_{h}^{\lceil \alpha\rceil}$ is the $\lceil \alpha\rceil$-th order finite-difference operator with step-size $h\in \mathbb{R}^d\setminus\{0\}$, and
$\|g\|^p_{L^p(\mathcal{X},h)}\eqdef \int |g(u)|^p I_{\mathcal{X},h}(u)\,du$,
where $I_{\mathcal{X},h}$ is the indicator function of the set of all $u\in \mathbb{R}^d$ for which the finite difference $\Delta_{h}^{\lceil \alpha\rceil}(f,u)$ is well-defined inside $\mathcal{X}$.
For any $0<q<\infty$, one of the many equivalent formulations of the Besov space $B_{p,q}^{\alpha}(\mathcal{X})$ is the collection of all $f\in L^p(\mathcal{X})$ for which the following quasi-norm is finite:
\begin{equation}
	\label{eq:Besov_Definition}
	\|f\|_{B_{p,q}^{\alpha}(\mathcal{X})}
	\eqdef
	\biggl(
	\int_0^{\infty} \frac{\omega_{\alpha,\mathcal{X}}(f,t)_p^q}{t^{q\alpha+1}} \, dt
	\biggr)^{1/q}
	+
	\|f\|_{L^p(\mathcal{X})}
	.
\end{equation}
We will focus on the following class of domains.
An open set $\mathcal{X}$ is called an \textit{$(\varepsilon_{\mathcal{X}},\delta_{\mathcal{X}})$-domain} in $\mathbb{R}^d$ if there exist $\varepsilon_{\mathcal{X}},\delta_{\mathcal{X}}>0$ and $C_0\ge 1$ such that, for any $x,y\in \mathcal{X}$ satisfying
\[
	\|x - y\| \leq \delta_{\mathcal{X}},
\]
there exists a rectifiable curve $\gamma:[0,1]\to \mathbb{R}^d$ of length at most $C_0 \|x - y\|$, with $\gamma(0)=x$ and $\gamma(1)=y$, such that for every $t\in [0,1]$,
\[
	\inf_{u\in \partial \mathcal{X}}\, \|\gamma(t)-u\| \geq
	\varepsilon_{\mathcal{X}} \min\{\|\gamma(t) - x\|, \|\gamma(t) - y\|\}.
\]
Examples of $(\varepsilon_{\mathcal{X}},\delta_{\mathcal{X}})$-(locally-uniform) domains include bounded domains with Lipschitz boundaries; cf.~\cite{jones1981quasiconformal}.

\begin{theorem}[{$L^p$-Approximation of Besov Functions}]
\label{thrm:Besov_cube_Lp_ANN}
Let $d\in\mathbb{N}_+$, and let $\mathcal{X}\subseteq \mathbb{R}^d$ be an $(\varepsilon_{\mathcal{X}},\delta_{\mathcal{X}})$-domain.
Let $1\le p<\infty$, let $0<q<\infty$, let $s>0$, and let $r\in\mathbb{N}_+$ satisfy
$
r\ge 2
$
and
$
s<\min\{r,r-1+\frac{1}{p}\}
$.
For every
$
f\in B_{p,q}^{s}(\mathcal{X})
$
and every $0<\varepsilon\le 1$, there is an $\mathcal{A}$NN
$
\Phi_{\varepsilon,f}:\mathbb{R}^d\to\mathbb{R}
$
satisfying
\[
        \int_{\mathcal{X}}
            \big|f(x)-\Phi_{\varepsilon,f}(x)\big|^p\,dx
    \le
        \varepsilon
.
\]
Moreover, $\Phi_{\varepsilon,f}$ has depth
$
\mathcal{O}\big(\log^2(\varepsilon^{-1})\big)
$,
width
$
\mathcal{O}\big(\varepsilon^{-d/(sp)}\big)
$,
and
$
\mathcal{O}\big(\varepsilon^{-d/(sp)}\log^2(\varepsilon^{-1})\big)
$
non-zero parameters.
\end{theorem}
\begin{proof}
See Section~\ref{s:Proofs__ss:thrm:Besov_cube_Lp_ANN}.
\end{proof}

\subsection{High Regularity: Holomorphic Functions}

\label{s:Approximation__ss:Classical___sss:Hol_poly_approx}
In this section, we consider the approximation of holomorphic functions of $d$ variables by $\mathcal{A}$NNs.  We begin by recalling the class of holomorphic functions with ``hidden anisotropy'' used below.  This class was studied in~\cite{adcock2022deep,franco2025practical} for ReLU networks.  Functions of this type arise naturally in the study of parametric ODEs and PDEs; see, e.g.,~\cite[Chapter 4]{adcock2022sparse} and~\cite{cohen2015approximation}.  For simplicity, we focus on functions with inputs in the hypercube $[-1,1]^d$ and outputs in $\mathbb{R}$.  Extensions to functions of infinitely many variables with values in a Hilbert or Banach space, or to operator learning between Hilbert or Banach spaces, are possible but lie beyond the finite-dimensional scalar-valued setting considered here~\cite{adcock2022deep,adcock2025near,adcock2024optimal}.  

Define the $d$-dimensional Bernstein polyellipse with parameter $\varrho=(\varrho_i)_{i=1}^d\ge \mathbf{1}_d$ coordinate-wise as the Cartesian product
\[
\mathcal{B}_{\varrho}  \eqdef  \mathcal{E}_{\varrho_1} \times \cdots \times \mathcal{E}_{\varrho_d} \subseteq \mathbb{C}^d,
\]
where $\mathcal{E}_{r}\eqdef \{(z+z^{-1})/2:\ z\in \mathbb{C},\ 1\le |z|\le r\}$, for $r\ge 1$, is the one-dimensional Bernstein ellipse.  Note that $[-1,1]^d\subseteq \mathcal{B}_{\varrho}$ for every $\varrho\ge \mathbf{1}_d$ and $\mathcal{B}_{\mathbf{1}_d}=[-1,1]^d$.  We will consider functions $f:[-1,1]^d\to \mathbb{R}$ that admit a holomorphic extension to a Bernstein polyellipse $\mathcal{B}_{\varrho}$.

\begin{definition}[Holomorphic Extension]
\label{def:holomorphic_extension}
We say that $f:[-1,1]^d\to \mathbb{R}$ admits a holomorphic extension to a Bernstein polyellipse $\mathcal{B}_{\varrho}$ if there exists an open set $\mathcal{O}\subseteq \mathbb{C}^d$ with $\mathcal{B}_{\varrho}\subseteq \mathcal{O}$ and a holomorphic function $\widetilde{f}:\mathcal{O}\to \mathbb{C}$ such that $f=\widetilde{f}|_{[-1,1]^d}$.  Moreover, we define
\[
    \|f\|_{L^{\infty}(\mathcal{B}_{\varrho})}
    \eqdef
    \inf \left\{\|\widetilde{f}\|_{L^{\infty}(\mathcal{O})}:\ \mathcal{O} \supseteq \mathcal{B}_{\varrho} \text{ open, } \widetilde{f}:\mathcal{O}\to \mathbb{C} \text{ holomorphic, } f=\widetilde{f}|_{[-1,1]^d}\right\}.
\]
\end{definition}

If a function $f$ admits a holomorphic extension to $\mathcal{B}_{\varrho}$, then the parameter $\varrho$ measures its anisotropy; see, e.g.,~\cite{adcock2022sparse}.  Intuitively, the larger the component $\varrho_i$, the more regular $f$ is with respect to the $i$-th variable.  We consider a class of such functions with \emph{hidden anisotropy}, using a slightly simplified version of the setting in~\cite{adcock2022deep,franco2025practical}.

\begin{definition}[Hidden Anisotropy] 
\label{def:hidden_aniso}
For $d\in \mathbb{N}_+$ and $0<\gamma\le 1$, the \emph{hidden anisotropy} class $\mathcal{H}_{d,\gamma}$ consists of all functions $f:[-1,1]^d\to \mathbb{R}$ that admit a holomorphic extension to $\mathcal{B}_{\varrho}$ for some $\varrho>\mathbf{1}_d$ such that
\begin{equation}
\label{eq:hidden_aniso_condition}
\frac{1}{d+1}\left(\frac{d! \prod_{i=1}^d \log(\varrho_i)}{2}\right)^{1/d} \geq \gamma \quad\text{and}\quad \|f\|_{L^{\infty}(\mathcal{B}_{\varrho})} \leq 1.
\end{equation}
\end{definition}
The constraint \eqref{eq:hidden_aniso_condition} is chosen so that the best $N$-term approximation error of any function in $\mathcal{H}_{d,\gamma}$ with respect to Legendre polynomials decays exponentially; see Lemma~\ref{lem:exp_rates_Legendre}. 
Our main theorem shows that functions in $\mathcal{H}_{d,\gamma}$ can be accurately approximated by $\mathcal{A}$NNs with moderate complexity.

\begin{theorem}[$\mathcal{A}$NN Approximation of Holomorphic Functions]
\label{thrm:HolomorphicApproximation}
Let $d\in \mathbb{N}_+$ and $0<\gamma\le 1$.  Under Assumption~\ref{ass:THEASSUMPTIONS}, for every $f\in \mathcal{H}_{d,\gamma}$ and every sufficiently small $\varepsilon>0$, there exists an $\mathcal{A}$NN $\Phi_{\varepsilon,f}:\mathbb{R}^d\to \mathbb{R}$ such that 
\[
\|f-\Phi_{\varepsilon,f}\|_{L^{\infty}([-1,1]^d)} \leq \varepsilon.
\]
Moreover, $\Phi_{\varepsilon,f}$ may be chosen with depth
$
\mathcal{O}\left(d\log(\gamma^{-1})+d\log\log(\varepsilon^{-1})\right)
$,
width
$
\mathcal{O}\left((\gamma^{-1}\log(\varepsilon^{-1}))^{2d}\right)
$,
and
$
\mathcal{O}\left((\gamma^{-1}\log(\varepsilon^{-1}))^{2d}\right)
$
non-zero parameters.
\end{theorem}

The proof of Theorem~\ref{thrm:HolomorphicApproximation} is given in Appendix~\ref{app:proof_holo}.  It is based on emulating sparse Legendre polynomial approximation, in the spirit of~\cite[Theorem 7.4]{adcock2025near}, following ideas from~\cite{daws2019analysis,de2021approximation,opschoor2022exponential}.

\section{Implication II: \texorpdfstring{$\mathcal{A}$NNs}{Feedforward Architectures} Emulate TCS and Numerical Analysis Algorithms}
\label{s:Implication__BeyondMinimaxALGORITHIMS}

In this section, we illustrate the power of our main result, and its auxiliary tools, by showing that several well-studied classical algorithms admit efficient $\mathcal{A}$NN emulators.  These include high-dimensional ODE solvers, maximum-eigenvalue routines, root-finding algorithms, and shortest-path computations.  We emphasize that these are real-valued representation results: the networks constructed below emulate explicit algorithmic computations, and the bounds concern depth, width, and non-zero parameters.  Such structure-aware constructions can be substantially sharper than bounds obtained by treating the same input-output maps only through classical uniform-continuity or Besov-type regularity.

We emphasize that the main result extends well beyond the illustrations in this section: once an explicit computation is written as a circuit over one of our gate languages, Theorem~\ref{thrm:concrete_surgery} gives a systematic route to $\mathcal{A}$NN emulation.  We begin with a classical combinatorial example whose structure is invisible to a purely regularity-based approximation analysis.

\subsection{A Classical TCS Example: Shortest Paths on a Graph}
\label{s:AllShortestPaths}

We first show how our main quantitative result, Theorem~\ref{thrm:concrete_surgery}, directly implies the ability of neural networks to solve problems in classical TCS, which are not related to classical approximation theory in an obvious manner.  
As an example, we consider shortest-path computation on the complete weighted graph $E_k^W\eqdef ([k]_+,E_k,W)$, where $E_k\eqdef \{\{i,j\}:\,i,j\in [k]_+,\ i\neq j\}$ and $W\eqdef (W_e)_{e\in E_k}\in (0,\infty)^{|E_k|}$ is a vector of positive edge weights.  
The all-pairs shortest-path map $\operatorname{all-dist}_k$ sends edge weights $W\in (0,\infty)^{|E_k|}$ to the $k\times k$ distance matrix $D^W$, whose $(i,j)$-th entry is the $W$-weighted shortest-path distance from $i$ to $j$ in $E_k^W$, i.e.,
\begin{equation}
\label{eq:allshortestpaths}
D^W_{i,j}\eqdef \inf_{\gamma_{i\to j}}\, 
\sum_{e\in \gamma_{i\to j}}\, W_e,
\end{equation}
where the minimum in~\eqref{eq:allshortestpaths} is taken over all finite paths $\gamma_{i\to j}=(e_1,\dots,e_I)$ from $i$ to $j$, with $I\in\mathbb{N}_+$ and $e_\ell\in E_k$ for every $\ell\in[I]_+$.

It was shown in~\cite{warshall1962theorem,roy1959transitivite,jukna2023tropical} that $\operatorname{all-dist}_k:W\mapsto D^W$ can be computed by pure dynamic programming.  Equivalently, it can be implemented by a tropical circuit using only the elementary operations $(1,+,\min)$, with depth $\mathcal{O}(k)$ and gate count $\mathcal{O}(k^3)$; this complexity cannot be improved in that language~\cite{kerr1970effect}.  See, e.g.,~\cite[Example 1.8 and Corollary 2.3]{jukna2023tropical} for a self-contained modern exposition.

For the formal neural-emulation statement below, we isolate the source-target distance $D^W_{1,2}$.  By the pure min-plus specialization of the proof of Theorem~\ref{thrm:concrete_surgery}, applied with binary gates and $c=1$, the Bellman-Ford-Moore recursion in Algorithm~\ref{algo:BFM_12} can be encoded into a small $\mathcal{A}$NN.

\begin{algorithm}[H]
\caption{Bellman-Ford-Moore computation of $D^W_{1,2}$ on $E_k^W$}
\label{algo:BFM_12}
\KwIn{Edge weights $W_{\{p,q\}}\in (0,\infty)$ for all $\{p,q\}\in E_k$, with $2\le k\in\mathbb{N}_+$}
\KwOut{The shortest-path distance $D^W_{1,2}$}

\For{$j=2,\dots,k$}{
    $d_j \eqdef W_{\{1,j\}}$\;
}

\For{$r=2,\dots,k-1$}{
    \For{$j=2,\dots,k$}{
        $\widetilde d_j \eqdef
        \min\Big\{
            d_j,\,
            \min_{q\in [k]_+\setminus\{1,j\}}
            \big(d_q+W_{\{q,j\}}\big)
        \Big\}$\;
    }
    \For{$j=2,\dots,k$}{
        $d_j \eqdef \widetilde d_j$\;
    }
}

\Return{$d_2$}\;
\end{algorithm}

\begin{corollary}[Source-Target Shortest-Path Computation]
\label{cor:AllpairsShortestpaths}
Fix $2\le k\in \mathbb{N}_+$ and $0<w_-<1$.  If $\mathcal{A}$ satisfies Assumption~\ref{ass:THEASSUMPTIONS}, then, for every $0<\varepsilon\le 1$, there exists a $\mathcal{A}$NN $\Phi_{\operatorname{dist:1,2}}:\mathbb{R}^{|E_k|}\to \mathbb{R}$ satisfying
\[
    \max_{W\in [w_-,1]^{E_k}}\,
        \big|
            D^W_{1,2}
        -
            \Phi_{\operatorname{dist:1,2}}(W)
        \big|
    <
        \varepsilon
.
\]
Moreover, $\Phi_{\operatorname{dist:1,2}}$ has depth
$
    \mathcal{O}\big(
        k\log k
        \,
        \log(\tfrac{k\log k}{\varepsilon})
        \,
        \log^{\circ 2}(\tfrac{k\log k}{\varepsilon})
    \big)
$
and
$
    \mathcal{O}\big(
        k^4\log k
        \,
        \log(\tfrac{k\log k}{\varepsilon})
        \,
        \log^{\circ 2}(\tfrac{k\log k}{\varepsilon})
    \big)
$
non-zero parameters.
\end{corollary}
\begin{proof}
See Appendix~\ref{a:Proofs__cor:AllpairsShortestpaths}.
\end{proof}

\paragraph{Exponential \textit{Improvement} over regularity-based bounds for combinatorial TCS problems.}
For reference, since the tropical circuit computing $\operatorname{dist}_{1,2}$ is comprised only of the operations $(1,+,\min\{\cdot,\cdot\})$, the map
$
\operatorname{dist}_{1,2}:(0,\infty)^{|E_k|}\to [0,\infty)
$
is piecewise-linear and Lipschitz.  More precisely, it is $1$-Lipschitz with respect to the $\ell^1$ metric on the edge-weight vector, and $(k-1)$-Lipschitz with respect to the $\ell^\infty$ metric.  Indeed,
$
\operatorname{dist}_{1,2}(W)
=
\min_{\gamma:1\to 2}\sum_{e\in\gamma} W_e
$,
where the minimum may be taken over simple paths, each using at most $k-1$ edges.

If one ignores this tropical dynamic-programming structure and treats $\operatorname{dist}_{1,2}$ only as a generic Lipschitz function of $|E_k|$ variables, then smoothness-only ReLU approximation theory gives only a curse-of-dimensionality scale; for instance, the very-deep modulus-of-continuity theory of~\cite{yarotsky2018optimal} yields the generic scale
$
\widetilde{\mathcal{O}}(\varepsilon^{-|E_k|/2})
=
\widetilde{\mathcal{O}}(\varepsilon^{-\Theta(k^2)})
$
for Lipschitz functions, up to constants depending on the Lipschitz constant and the dimension.  
This generic smoothness-based scale is much larger than the complexity guaranteed by our circuit-to-network compilation in Corollary~\ref{cor:AllpairsShortestpaths}, which, for fixed $k$, gives an $\mathcal{A}$NN with
$
\widetilde{\mathcal{O}}(\log(1/\varepsilon))
$
non-zero parameters.  Thus, in this example, the regularity-only viewpoint gives a valid but highly non-adapted upper bound, whereas the tropical circuit description exposes the true algorithmic structure and is reminiscent of the recent breakthroughs in AI-assisted proofs for combinatorial problems, cf.~\cite{openai2026unitdistance}.

\subsection{Maximal Eigenvalue Solver}
\label{s:EigenValueEstimation}
Next, we consider a classical problem in numerical analysis: eigenvalue estimation.  For concreteness, we restrict attention to positive-definite matrices satisfying a uniform ellipticity condition.  Namely, for any $0<\delta\le \Delta$ and any $d\in \mathbb{N}_+$, let $\mathrm{PD}_{\delta,\Delta}^d(\mathbb{R})$ denote the set of $d\times d$ positive-definite matrices $A$ whose smallest eigenvalue $\lambda_{\min}(A)$ is bounded below by $\delta$ and whose largest eigenvalue $\lambda_{\max}(A)$ is bounded above by $\Delta$.  Such matrices arise naturally in elliptic PDEs; cf.~\citep[Chapter 14]{krylov2018sobolev}, and in fully nonlinear BSDEs; see, e.g.,~\cite{furuya2025one}.  We then restrict attention further to the subset of matrices $A \in \mathrm{PD}_{\delta,\Delta}^d(\mathbb{R})$ with a spectral gap of at least $\gamma>0$, i.e., such that
\begin{equation}
\label{eq:spectral_gap__condition}
    \gamma  
\le 
    \lambda_{\max}(A)-\lambda_{\max-1}(A),
\end{equation}
where $\lambda_{\max-1}(A)$ denotes the second-largest eigenvalue of $A$ (necessarily unique in this case).  The collection of all such matrices is denoted by $\mathrm{PD}_{\delta,\gamma,\Delta}^d(\mathbb{R})$.  We note that matrices with spectral gaps arise naturally in the study of mixing times for Markov processes~\cite{Aldous1983,AldousDiaconis1987}, with applications ranging from learning theory with non-i.i.d.\ data~\cite{vogels2023beyond,limmer2024higher} to differential privacy~\cite{zamanlooy2024egamma}.  
We emulate the classical power iteration algorithm; cf.\ Algorithm~\ref{alg:power_iteration} and~\cite{mises1929praktische}.  Its convergence is exponentially fast provided the initialization has non-zero overlap with the top eigenspace; cf.~\citep[Theorem 7.6]{arbenz2012lecture}.

\begin{algorithm}[H]
\caption{Power Iteration}
\label{alg:power_iteration}
\KwIn{$A\in \mathrm{PD}_{\delta,\gamma,\Delta}^d(\mathbb{R})$, Max iterations: $T\in\mathbb{N}_+$}
\KwOut{$x_T$, $\lambda_T$}

$x_0 \leftarrow \tfrac{1}{\sqrt{d}}\cdot \mathbf{1}_d$\;

\For{$k=0,\dots,T-1$}{
    $x_{k+1} \leftarrow \tfrac{1}{\|A x_k\|_2}\cdot A x_k$\;
    $\lambda_{k+1} \leftarrow x_{k+1}^{\top} A x_{k+1}$\;
}

\Return{$\lambda_T$}\;
\end{algorithm}

In what follows, for each $A\in \mathrm{PD}_{\delta,\gamma,\Delta}^d(\mathbb{R})$, we write $v_{\max}(A)$ for a unit eigenvector of $A$ associated with its largest eigenvalue $\lambda_{\max}(A)$.  We then define
$
    \rho(A)^2
    \eqdef
    \biggl\langle v_{\max}(A),\frac{\mathbf{1}_d}{\sqrt d}\biggr\rangle^2
    =
    \biggl(
        \frac{1}{\sqrt d}
        \sum_{i=1}^d v_{\max}(A)_i
    \biggr)^2,
$
which is the squared cosine of the angle between $v_{\max}(A)$ and $\mathbf{1}_d$.  
For $0<\rho_0\le 1$, set
$
    \mathrm{PD}_{\delta,\gamma,\Delta}^{d,\rho_0}(\mathbb{R})
    \eqdef
    \{A\in \mathrm{PD}_{\delta,\gamma,\Delta}^d(\mathbb{R}):\rho(A)^2\ge \rho_0\}.
$

\begin{proposition}[ANN Emulation of Power Iteration with Spectral-Gap and Overlap]
\label{prop:ANN_emulation_power_iteration}
Fix $d\in \mathbb{N}_+$, $0<\delta\le \Delta$, $0<\gamma\le \Delta-\delta$, and $0<\rho_0\le 1$.  
For every $0<\varepsilon<1$ there exists an $\mathcal{A}$NN
$
\widehat{\Lambda}_{\varepsilon}:\mathbb{R}^{d\times d}\to \mathbb{R}
$
such that
\[
     \sup_{A\in \mathrm{PD}_{\delta,\gamma,\Delta}^{d,\rho_0}(\mathbb{R})}
    \,
            \big|
                    \widehat{\Lambda}_{\varepsilon}(A)
                -
                    \lambda_{\max}(A)
            \big|
\le 
    \varepsilon
.
\]
Moreover, $\widehat{\Lambda}_{\varepsilon}$ has depth at most
$
\mathcal{O}\big(
\log(1/\varepsilon)^2
\bigl(
1+\log^{\circ 2}(1/\varepsilon)
)
\big)$, 
width at most
$
\mathcal{O}(1),
$
and at most
$
\mathcal{O}\big(
\log(1/\varepsilon)^2
(
1+\log^{\circ 2}(1/\varepsilon)\bigr)
\big)
$
non-zero parameters.\footnote{All implicit constants may depend on $d,\delta,\gamma,\Delta$, and $\rho_0$.}
\end{proposition}

\begin{remark}
One could perhaps relax the spectral gap assumption in Proposition~\ref{prop:ANN_emulation_power_iteration} using a recent result of~\cite{wang2026beyond}.  However, we relegate that to a more specialized analysis.
\end{remark}

\subsection{Root Finding}
\label{s:NewApplications__ss:RootFinders}

As an example of neural networks implementing root-finding procedures, we consider the computation of radicals. Indeed, a radical amounts to solving
\[
x^{\ell} = y
\]
for a given input $x\geq 0$ and a fixed integer $\ell\in \mathbb{N}_+$. Although this construction appears inside the proof of Proposition~\ref{prop:Step1_GateEmulation}, we isolate it here as an illustrative example of how an $\mathcal{A}$NN can emulate a classical numerical routine.
The proof of the next result shows that $\mathcal{A}$NNs can approximately implement the Newton iterations in Algorithm~\ref{algo:radialization}, using only simpler algebraic gates as building blocks. Moreover, Algorithm~\ref{algo:radialization} calls Algorithm~\ref{algo:rootfinder} as a subroutine, so the resulting ``expert/agent'' network reuses the previously constructed inversion network as a sub-network.

\begin{algorithm}[H]
\caption{Approximate Newton iteration for the $\ell$-th radical}
\label{algo:radialization}
\KwIn{$x\in[m,M]$, integers $\ell\ge 2$, $K\in\mathbb{N}_+$, parameters $0<m<2\le M$, $0<\varepsilon<1$}
\KwOut{$s_K(x)$ approximating $x^{1/\ell}$}

$m_{\mathrm{inv}} \eqdef m^{(\ell-1)/\ell}$\;
$M_{\mathrm{inv}} \eqdef \Big(2M+1+\frac{M}{m^{(\ell-1)/\ell}}\Big)^{\ell-1}$\;
Choose $K_{\mathrm{inv}}\in\mathbb{N}$ so that
$\frac{q^{2^{K_{\mathrm{inv}}-1}}}{m_{\mathrm{inv}}}\le \frac{\varepsilon}{4M}$\;
$s \eqdef M+1$\;

\For{$j=0,\dots,K-1$}{
    $u \eqdef \textsc{RootFinder}(s^{\ell-1};m_{\mathrm{inv}},M_{\mathrm{inv}},K_{\mathrm{inv}})+\frac{\varepsilon}{4M}$ \tcp*[r]{\scriptsize{\textsc{RootFinder} is Algorithm~\ref{algo:rootfinder}}}
    $s \eqdef \frac{1}{\ell}\Big((\ell-1)s+xu\Big)$\;
}

\Return{$s$}\;
\end{algorithm}

\begin{proposition}[$\mathcal{A}NN$ Emulation of the Radical $(\cdot)^{1/\ell}$]
\label{prop:radicals_ANNs}
Under Assumption~\ref{ass:THEASSUMPTIONS},
fix $\ell\in\mathbb{N}_+$ with $\ell\ge 2$.
For every approximation error $0<\varepsilon<1$ and annulus hyperparameters $0<m<2\le M$, there exists an $\mathcal{A}$NN
$
    \mathrm{Rad}_{m,M:\ell,\varepsilon}:\mathbb{R}\to \mathbb{R}
$
satisfying
\[
    \sup_{z\in[m,M]}\,
        \big|
            \mathrm{Rad}_{m,M:\ell,\varepsilon}(z)
            -
            z^{1/\ell}
        \big|
    \le
        \varepsilon
.
\]
Moreover, $\mathrm{Rad}_{m,M:\ell,\varepsilon}$ may be chosen with:
\begin{enumerate}
    \item[(i)] \textbf{Depth:} at most
    $
        2
        +
        K\,\big(
            12K_{\mathrm{inv}}-3
        \big)
    ,
    $
    \item[(ii)] \textbf{Width:} at most
    $
        40
        +
        \left\lceil\frac{\ell-1}{2}\right\rceil
        \big(12\max\{m,n\}+2\big)
    ,
    $
    \item[(iii)] \textbf{Non-Zero Parameters:} at most
    $$
        K\Big(
            2(K_{\mathrm{inv}}-1)\big(8\|x^\star\|_0+41\big)
            +
            30K_{\mathrm{inv}}
            +
            (\ell-2)\big(4(2\|x^\star\|_0+5)+11\big)
            +
            3\ell-1
            +
            \big(8\|x^\star\|_0+41\big)
            +
            24
        \Big)
        +
        12.
    $$
\end{enumerate}
where
$
    m_{\mathrm{inv}}
    \eqdef
        \tfrac{1}{2}m^{(\ell-1)/\ell}
$, $
    M_{\mathrm{inv}}
    \eqdef
        \Big(
            2M+1+\frac{M}{m^{(\ell-1)/\ell}}
        \Big)^{\ell-1}
$, $0<
    q_{\mathrm{inv}}
    \eqdef
        \frac{M_{\mathrm{inv}}^2-m_{\mathrm{inv}}^2}{M_{\mathrm{inv}}^2+m_{\mathrm{inv}}^2}
<1$, $
    \delta_{\mathrm{inv}}
    \eqdef
        \frac{\varepsilon}{8M}
$,
$
        K_{\mathrm{inv}}
    \eqdef
        \Bigg\lceil
            1
            +
            \log_2 \Big(
                \frac{
                    \log \big(
                        2/(m_{\mathrm{inv}}\delta_{\mathrm{inv}})
                    \big)
                }{
                    \log(1/q_{\mathrm{inv}})
                }
            \Big)
        \Bigg\rceil
$, and $
    K
\eqdef
    \Bigg\lceil
        \frac{
            \log_2 \Big(
                \frac{
                    4\,(M+1-m^{1/\ell})
                }{
                    \varepsilon
                }
            \Big)
        }{
            \log_2 \big(
                \tfrac{\ell}{\ell-1}
            \big)
        }
    \Bigg\rceil
.
$
\end{proposition}
\begin{proof}
See Section~\ref{s:Radicals}.
\end{proof}

\subsection{High-Dimensional Structured ODE Solvers Emulation}
\label{s:NewApplications__ss:ODESolvers}

As we know from the theory of (non-linear) widths~\cite{cohen2022optimal,petrova2023lipschitz,siegel2024sharpB} and entropy numbers~\cite{Kolmogorov1956Asymptotic,KolmogorovTikhomirov1959Entropy,kuhn2006entropy,MayerUllrich_2021_EntropyNumbersSmallMixedSmoothness}, the minimax-optimal worst-case complexity for approximating arbitrary $C^{s,1}([0,1]^{d+1},\mathbb{R}^d)$ functions, measured in the uniform $\ell^2$-norm, scales componentwise as in the scalar case, up to constants depending on $d$; this can be discouraging in high dimensions if we cannot guarantee that the target function has any real smoothness (e.g., if only $s=1$).
As an illustration of how our main result allows for such finer structure to be easily incorporated, we consider the following \textit{flow map} 
\begin{equation}
\label{eq:flow_map}
\begin{aligned}
\operatorname{Flow}_{\alpha,\beta}:\mathbb{R}^{d}\times[0,1]& \rightarrow \mathbb{R}^d,
\\
(x,t) & \mapsto X_t^x,
\end{aligned}
\end{equation}
where for any $x\in \mathbb{R}^d$, we use $(X_t^x)_{t\in [0,1]}$ to denote the unique absolutely continuous solution to the following ODE
\begin{equation}
\label{eq:ODE}
\begin{aligned}
    \partial_t \, X_t & = (\alpha(t)A)\,X_t + (\beta(t)B),
\\
    X_0 & = x,
\end{aligned}
\end{equation}
where $\alpha,\beta:\mathbb{R}\to \mathbb{R}$ are Lipschitz, $A\in \mathbb{R}^{d\times d}$, and $B\in \mathbb{R}^d$.

We emphasize that the main point here \textit{is} the stylized structure, which cannot be detected by the ``unstructured'' minimax approximation theorems for classical Besov/$C^{s,1}$-type classes.  Applied componentwise to vector-valued functions and measured in the uniform $\ell^2$-norm, such results only yield dimension-dependent rates governed by the ambient input dimension; which only yield rates of order $\mathcal{O}(\varepsilon^{-d/s})$; cf.~\cite{devore1998nonlinear}.


\begin{proposition}[Flow-Map Emulation for Scalar-Modulated Linear Dynamics]
\label{prop:flowemulating}
Let $A\in \mathbb{R}^{d\times d}$, $B\in \mathbb{R}^d$, and let $\alpha,\beta:[0,1]\to \mathbb{R}$ be Lipschitz.
Consider the vector-valued flow map $\operatorname{Flow}_{\alpha,\beta}$,
cf.~\eqref{eq:flow_map}, associated to the ODE~\eqref{eq:ODE}.  Then
$
    \operatorname{Flow}_{\alpha,\beta}|_{[0,1]^d\times[0,1]}
    \in C^{1,1}([0,1]^{d+1},\mathbb{R}^d)
$.
Under Assumption~\ref{ass:THEASSUMPTIONS}, for every $\varepsilon>0$, there exists an $\mathcal{A}$NN
$
    \widehat{\operatorname{Flow}}_{\alpha,\beta,\varepsilon}:\mathbb{R}^{d+1}\to \mathbb{R}^d
$
such that
\[
        \sup_{\substack{x\in [0,1]^d
        \\
        t\in [0,1]}}
        \,
        \big\|
            \widehat{\operatorname{Flow}}_{\alpha,\beta,\varepsilon}(x,t)
            -
            \underbrace{X_t^x}_{\operatorname{Flow}_{\alpha,\beta}(x,t)}
        \big\|_2
    \le
        \varepsilon
.
\]
Moreover, $\widehat{\operatorname{Flow}}_{\alpha,\beta,\varepsilon}$ may be chosen with:
\begin{enumerate}
    \item[(i)] \textbf{Depth:}
    $
        \mathcal{O}(\log(1/\varepsilon))
    ,
    $
    \item[(ii)] \textbf{Width:}
    $
        \mathcal{O}(\varepsilon^{-3})
    ,
    $
    \item[(iii)] \textbf{Non-Zero Parameters:}
    $
        \mathcal{O}(\varepsilon^{-5})
    .
    $
\end{enumerate}
\end{proposition}
\begin{proof}
See Section~\ref{s:Proof_NewNumericalAlgos__ss:prop:flowemulating}.
\end{proof}

\section{Explanation of New Proof Technique: Abstract Surgery}
\label{s:complexity_classes__ss:whyitworks}
We now briefly explain the new technical engine used to derive Theorem~\ref{thrm:concrete_surgery}, whose lengthy case-by-case proof is given in the appendix~\ref{s:ProofMain__prop:Step1_GateEmulation} in full detail.

\subsection{Step 1: \texorpdfstring{$\mathcal{A}$}{A}NN Dictionaries of ``Expert'' Gate Emulators}
\label{s:complexity_classes__ss:whyitworks___sss:Step1}

Our first main step is to explicitly construct a lookup table of $\mathcal{A}$NNs, each capable of uniformly approximating any gate in our classes to arbitrary precision on arbitrarily large cubes -- or cubical annuli in the case of the inversion map $\tfrac{1}{\cdot}$. The approximating $\mathcal{A}$NNs, together with their complexity in terms of depth, width, and size (number of parameters), as well as the precise results outlining their explicit constructions, are catalogued in Table~\ref{tab_thrm:gate_approximation}.

\begin{table}[htp!]
\begin{adjustbox}{width=\textwidth, height=\textheight, keepaspectratio,center}
\small
\begin{tabular}{@{\,}%
  >{\raggedright\arraybackslash}l  
  >{\centering\arraybackslash}l    
  >{\centering\arraybackslash}l    
  >{\centering\arraybackslash}l    
  >{\raggedright\arraybackslash}l  
@{}}
\toprule
\textbf{Gate} & \textbf{Depth} & \textbf{Width} & \textbf{No.\ Param.} & \textbf{Reference} \\

\midrule
\multicolumn{5}{c}{\textbf{Arithmetic Gates}} \\[3pt]
Unary Constant
  & $\mathcal{O}(1)$
  & $\mathcal{O}(1)$
  & $\mathcal{O}(1)$
  & Corollary~\ref{cor:constant_gate}
\\

Identity
  & $\mathcal{O}(1)$
  & $\mathcal{O}(1)$
  & $\mathcal{O}(1)$
  & Proposition~\ref{prop:identity_emulation__multi}
\\

Binary Addition
  & $\mathcal{O}(1)$
  & $\mathcal{O}(1)$
  & $\mathcal{O}(1)$
  & Corollary~\ref{cor:addition_gate}
\\

Binary Multiplication
  & $\mathcal{O}(1)$
  & $\mathcal{O}(1)$
  & $\mathcal{O}(1)$
  & Proposition~\ref{prop:polarization_mult_gadget}
\\

\midrule

\multicolumn{5}{c}{\textbf{Algebraic Gates}} \\[3pt]

Power ($\cdot^k$)
  & $\mathcal{O}(\log(k))$
  & $\mathcal{O}(k)$
  & $\mathcal{O}(k)$
  & Corollary~\ref{cor:kth_power_gadget__SK__polar__precise}
\\
Exponential ($\exp$)
  & $\mathcal{O}(\log\log(1/\varepsilon))$
  & $\mathcal{O}(\log^2(1/\varepsilon))$
  & $\mathcal{O}(\log^2(1/\varepsilon))$
  & Corollary~\ref{cor:ExpEmulation}
\\

Multiplicative Inversion ($\tfrac{1}{\cdot}$)
  & $\mathcal{O}(\log(1/\varepsilon))$
  & $\mathcal{O}(1)$
  & $\mathcal{O}(\log(1/\varepsilon))$
  & Proposition~\ref{prop:inversion_ANNs}
\\

Radicals ($\sqrt[\ell]{\cdot}$; $\ell\ge 2$)
  & $\mathcal{O}\big(
    \ell \, \log(1/\varepsilon)(\ell+\log^{\circ 2}(1/\varepsilon)
    )
    \big)$
  & $\mathcal{O}(1)$
  & $\mathcal{O}\big(
    \ell \, \log(1/\varepsilon)(\ell+\log^{\circ 2}(1/\varepsilon)
    )
    \big)$
  & Proposition~\ref{prop:radicals_ANNs}
\\

\midrule
\multicolumn{5}{c}{\textbf{Tropical Gates}} \\[3pt]

Absolute Value
  & $\mathcal{O}\big(
    \log(1/\varepsilon)\log^{\circ 2}(1/\varepsilon)
    \big)$
  & $\mathcal{O}(1)$
  & $\mathcal{O}\big(
    \log(1/\varepsilon)\log^{\circ 2}(1/\varepsilon)
    \big)$
  & Proposition~\ref{prop:abs_ANNs}
\\

Binary Maximum
  & $\mathcal{O}\big(
    \log(1/\varepsilon)\log^{\circ 2}(1/\varepsilon)
    \big)$
  & $\mathcal{O}(1)$
  & $\mathcal{O}\big(
    \log(1/\varepsilon)\log^{\circ 2}(1/\varepsilon)
    \big)$
  & Proposition~\ref{prop:max_ANNs}
\\


\bottomrule
\end{tabular}
\end{adjustbox}
\caption{Reference for approximate uniform emulation of the core gates required to compute algebro-tropical functions; by arbitrary   $\mathcal{A}$NNs satisfying only the mild Assumption~\ref{ass:THEASSUMPTIONS}.}
\label{tab_thrm:gate_approximation}
\end{table}

\begin{proposition}[Gate Emulation]
\label{prop:Step1_GateEmulation}
Under Assumption~\ref{ass:THEASSUMPTIONS}, fix $k,c\in \mathbb{N}_+$ and $r^{\star}\ge 0$.
For any gate $g \in \mathbb{G}_{\operatorname{Rat}}^{k,c,r^{\star}}$, any $\varepsilon>0$, $M>1$, and $m>0$  
there exists an $\mathcal{A}$NN $\Phi_g$ satisfying (with the same source as $g$)%
\footnote{If $g:\mathbb{R}^k\to \mathbb{R}$ and $\operatorname{dom}(g)\subseteq  \mathbb{R}^k$ then the source of $g$ is $\mathbb{R}^k$ even if its domain may be strictly smaller.}~%
\begin{equation*}
    \sup_{x\in [-M,M]^k\setminus [-m,m]^k}
    \,
    |g(x)-\Phi_g(x)|
<
    \varepsilon
.
\end{equation*}
Moreover, if $g\neq 1/\cdot$ then $m=0$ is valid.
Furthermore, the depth, width, and size of $\Phi_g$ are summarized in Table~\ref{tab_thrm:gate_approximation}.
\end{proposition}
\begin{proof}
See Section~\ref{s:ProofMain__prop:Step1_GateEmulation}.
\end{proof}

\subsection{Step 2: Surgery}
\label{s:complexity_classes__ss:whyitworks___sss:Surgery}
We now present the main technical vehicle of our approach, building  on~\cite{kratsios2025quantifying,li2026certifiable}, which we term \textit{abstract  surgery}. This mechanism combines the individual approximation guarantees for each gate  in any considered class, e.g., those enumerated in  Table~\ref{tab_thrm:gate_approximation}, into a single $\mathcal{A}$NN, and is presented  in full generality both to expose its core structure and to facilitate future applications.  

The procedure consists of three steps (cf.\ Figures~\ref{fig:G_Circuit__Canonicalized}  and~\ref{fig:G_Circuit__Alignment}): (1) construct a dictionary of approximate  $\mathcal{A}$NN gate emulators, one for each gate in $\mathbb{G}$; (2) rewrite the target  circuit into an equivalent multipartite form aligned with the graph structure of an  $\mathcal{A}$NN; and (3) iteratively replace each gate by its $\mathcal{A}$NN emulator,  padding layers appropriately to ensure alignment. The key difficulty in~(3) is to do so  in a manner that prevents the successive composition of operations of progressively lower  regularity from causing a complexity explosion. The abstract surgery lemma below executes  this procedure, given a dictionary of elementary gate emulators.  

In short, abstract surgery reduces the $\mathcal{A}$NN approximation problem to: (1)  bounding the complexity of approximately emulating each gate by an $\mathcal{A}$NN, and  (2) exhibiting a $\mathbb{G}$-circuit computing the target function. To quantify the  hardness of this problem in terms of the depth, width, and gate count of the circuit, we  require two additional pieces of metric information, one regularity-theoretic and one  bornological. The first is captured by $\mathbb{G}$-propagators, which measure how much a  given gate can stretch a standard cube, modulo points of ill-definition.

\begin{definition}[$\mathbb{G}$-Propagator]
\label{defn:propagator}
Let $\mathbb{G}$ be a family of gates.  
A $\mathbb{G}$-propagator is a map $\Phi_{\mathbb{G}}:[0,\infty)\to [0,\infty]$ such that: 
for every $\mathbb{G}\ni g:\mathbb{R}^{d_g}\supseteq \operatorname{dom}(g)\to \mathbb{R}$, each $M>0$, 
\[
        g([-M,M]^{d_g}\cap \operatorname{dom}(g))
    \subseteq 
        [-\Phi_{\mathbb{G}}(M),\Phi_{\mathbb{G}}(M)]
.
\]
\end{definition}
The second missing piece of information, quantified by $\mathbb{G}$-regulators, describes the worst-case modulus of continuity that any gate in $\mathbb{G}$ can exhibit when restricted to any given cube; omitting points where that gate is ill-defined.
\begin{definition}[$\mathbb{G}$-Regulators]
\label{defn:regulators}
Let $\mathbb{G}$ be a family of gates.  A $\mathbb{G}$-regulator is a map $\mathcal{R}_{\mathbb{G}}:[0,\infty)\to C([0,\infty),[0,\infty))$ 
such that $\mathcal{R}_{\mathbb{G}}(M)$ is a concave modulus of continuity and: 
for every $\mathbb{G}\ni g:\mathbb{R}^{d_g}\supseteq \operatorname{dom}(g)\to \mathbb{R}$, each $M>0$, and every
$x,\tilde{x}\in [-M,M]^{d_g}\cap \operatorname{dom}(g)$,
\[
            |g(x)-g(\tilde{x})|
    \le 
        \mathcal{R}_{\mathbb{G}}(M)(\|x-\tilde{x}\|_{\infty})
.
\]
\end{definition}

In addition, we need one last piece of information, which our analysis to date has estimated for the class of gates considered herein.  Namely, we need the hardness of approximating any individual gate by an $\mathcal{A}$NN.  In this paper, we consider general classes of activation dictionaries and, accordingly, provide worst-case bounds; however, one can easily imagine obtaining specialized bounds by directly applying abstract surgery with specialized classes of activation dictionaries $\mathcal{A}$.
\begin{definition}[$(\mathbb{G},\mathcal{A})$-Hardness]
\label{defn:GAHardness}
Let $\mathbb{G}$ be a family of gates and $\mathcal{A}$ be an activation dictionary.  The $(\mathbb{G},\mathcal{A})$-hardness is a map $\mathcal{H}_{\mathbb{G},\mathcal{A}}:[0,\infty)^2\to \mathbb{N}_+^3$ such that: 
for every $\mathbb{G}\ni g:\mathbb{R}^{d_g}\supseteq \operatorname{dom}(g)\to \mathbb{R}$, each side-length $M>0$, and every uniform approximation error $\varepsilon>0$, writing 
$\mathcal{H}_{\mathbb{G},\mathcal{A}}(M,\varepsilon)=
(\Delta_{M,\varepsilon},\Upsilon_{M,\varepsilon},N_{M,\varepsilon})$,
there exists an $\mathcal{A}$NN $\hat{g}_{\varepsilon,M}:\mathbb{R}^{d_g}\to \mathbb{R}$ 
of depth $\Delta_{M,\varepsilon}$, width $\Upsilon_{M,\varepsilon}$, and with no more than $N_{M,\varepsilon}$ non-zero parameters
satisfying
\[
        \sup_{x \in \operatorname{dom}(g) \cap [-M,M]^{d_g}}
        \,
            |g(x)-\hat{g}_{\varepsilon,M}(x)|
    <
            \varepsilon
.
\]
\end{definition}
Note that, the complexity bounds in Table~\ref{tab_thrm:gate_approximation}, which make up the bulk of our analysis, are precisely estimates on the $(\mathbb{G},\mathcal{A})$-Hardness.

We are now ready to deduce our main result, Theorem~\ref{thrm:concrete_surgery}. Its proof relies on the $\mathbb{G}$-regulator and $\mathbb{G}$-propagator bounds summarized in Table~\ref{tab:G_propagators}, together with the size of the $\mathcal{A}$NNs required to approximate the gates defining each class, as summarized in Table~\ref{tab_thrm:gate_approximation}. All results are stated for inputs in the cube $[0,1]^d$, with $M\ge 1$ and, when relevant, $m>0$. For rational-tropical gates, we restrict attention to functions computed by ``non-singular'' gates, meaning that no \textbf{forbidden $m$-composition} is performed. 
\begin{definition}[Forbidden $m$-Compositions]
\label{defn:no_forbiddencomposition}
Fix $0\le m<1$ and consider a $\mathbb{G}_{Rat}^{k,c,r^\star}$ circuit $\mathcal{C}$ of depth $\Delta$; where $k,\Delta\in \mathbb{N}_+$ and $r,c\ge 0$.  
Then $\mathcal{C}$ satisfies the \textit{forbidden $m$-composition-free} condition if, for every layer $j\in \{1,\dots,\Delta\}$ and every inversion input coordinate $t\in J_j$,
\begin{equation}
\label{eq:forbidden_composition}
    m
<
    \inf_{x\in [-M,M]^d}
    \,
        \big|
            \big(g_{j-1}\circ \cdots \circ g_0(x)\big)_t
        \big|
.
\end{equation}
\end{definition}

\begin{proposition}[Abstract Surgery]
\label{prop:abstract_surgery}
Let $\mathbb{G}$ be a family of gates and suppose $\mathcal{A}$ satisfies Assumption~\ref{ass:THEASSUMPTIONS}.  For every $\mathbb{G}$-circuit $\mathcal{C}=(V,E,g_{\cdot})$ of depth $\Delta$, width $\Upsilon$, using $N$ gates computing a function $f_{\mathcal{C}}:\mathbb{R}^d\to \mathbb{R}^D$, 
and every sequence of approximation errors $\delta_{\cdot}\eqdef (\delta_l)_{l=0}^{\Delta}$ satisfying $\delta_0\eqdef 0$ and $\delta_l>0$ for every $l\in [\Delta]_+$,
there exists an $\mathcal{A}$NN $\hat{f}_{\mathcal{C},\delta_{\cdot}}:\mathbb{R}^d\to \mathbb{R}^D$ satisfying
\begin{equation}
\label{eq:prop:abstract_surgery}
    \sup_{x\in [-1,1]^d}
    \,
        \|
            f_{\mathcal{C}}(x)
        -
            \hat{f}_{\mathcal{C},\delta_{\cdot}}(x)
        \|_{\infty}
    \le 
        \mathfrak{E}_{\Delta},
\end{equation}
where $\delta_0\eqdef 0$, $
\mathfrak{E}_0\eqdef 0$, $
\mathfrak{S}_0\eqdef 1$, and for each $j\in [\Delta]_+$ we define recursively
\begin{equation}
\label{eq:recurive_errors}
\begin{aligned}
    \mathfrak{S}_j
& \eqdef
    \Phi_{\mathbb{G}}\big(
        \mathfrak{S}_{j-1}
        +
        \mathfrak{E}_{j-1}
    \big)
    +
    \delta_j,
\\
    \mathfrak{E}_j
& \eqdef
    \delta_j
    +
    \mathcal{R}_{\mathbb{G}}\big(
        \mathfrak{S}_{j-1}
        +
        \mathfrak{E}_{j-1}
    \big)
    \big(
        \mathfrak{E}_{j-1}
    \big).
\end{aligned}
\end{equation}
Moreover, if for each $l\in [\Delta]_+$ we write
\[
        \mathcal{H}_{\mathbb{G},\mathcal{A}}
        \Big(
            \mathfrak{S}_{l-1}+\mathfrak{E}_{l-1},
            \tfrac{\delta_l}{2}
        \Big)
    =
        \big(
            L_l,
            W_l,
            P_l
        \big)
\,\,\mbox{ and }\,\,
        \bar{\Upsilon}
    \eqdef 
        \max\{D,\Upsilon+|E|\},
\]
then the depth, width, and number of parameters defining $\hat{f}_{\mathcal{C},\delta_{\cdot}}$ are
\begin{enumerate}
    \item[(i)] \textbf{Depth:} 
    $
        \mathcal{O}\Big(
            \sum_{l=1}^{\Delta}\,L_l
        \Big)
    $,
    \item[(ii)] \textbf{Width:} $\mathcal{O}\Big(
            d
            +
            \bar{\Upsilon}\max_{l\in [\Delta]_+}\,W_l
        \Big)$,
    \item[(iii)] \textbf{No. Parameters:} $\mathcal{O}\Big(
            d_0
            +
            \bar{\Upsilon}
            \sum_{l=1}^{\Delta}\,
            \big(
                P_l
                +
                L_l
            \big)
        \Big)
    $,
\end{enumerate}
where $d_0$ is the input-dimension appearing in the multi-partite representation produced by Lemma~\ref{lem:Canonicalizationlemma}.
\end{proposition}

Upon combining Steps $1$ and $2$ with the estimates of the $\mathbb{G}$-propagators and $\mathbb{G}$-regulators of each class, summarized in Table~\ref{tab:G_propagators}, we deduce our main result (Theorem~\ref{thrm:concrete_surgery}).

\begin{table}[htp!]
\centering
\begin{adjustbox}{max width=\textwidth}
\small
\begin{tabular}{@{}lll@{}}
\toprule
\textbf{Circuit Class}
& \textbf{$\mathbb{G}$-Propagator $\Phi_{\mathbb{G}}(M)$}
& \textbf{$\mathbb{G}$-Regulator $\mathcal{R}_{\mathbb{G}}(M)(\delta)$}
\\
\midrule

Algebraic $\mathbb{G}_{alg}^{k,c}$
&
$\max\{c,\,kM,\,M^k\}$
&
$\max\{1,\,k,\,kM^{k-1}\}\,\delta$
\\

\midrule

Algebro-Tropical $\mathbb{G}_{t\text{-}alg}^{k,c}$
&
$\max\{c,\,kM,\,M^k\}$
&
$\max\{1,\,k,\,kM^{k-1}\}\,\delta$
\\

\midrule

Rad-Algebro-Tropical $\mathbb{G}_{rt\text{-}alg}^{k,c,r^\star}$
&
$\max\{c,\,kM,\,M^{\max\{k,r^\star\}}\}$
&
$\max\{1,\,k,\,kM^{k-1},\,r^\star M^{r^\star-1}\}\,(\delta+\delta^{1/r^\star})$
\\

\midrule

\begin{tabular}[c]{@{}l@{}}
Rational-Tropical$^{\dagger}$
$\mathbb{G}_{Rat}^{k,c,r^\star}$
\end{tabular}
&
$\max\{c,\,kM,\,M^{\max\{k,r^\star\}},\,m^{-1}\}$
&
$\max\{m^{-2},\,1,\,k,\,kM^{k-1},\,r^\star M^{r^\star-1}\}\,(\delta+\delta^{1/r^\star})$
\\

\bottomrule
\end{tabular}
\end{adjustbox}
\caption{The $\mathbb{G}$-propagators and $\mathbb{G}$-regulators for listed circuit classes. 
Throughout we assume $M\ge 1$ and $m>0$, and $\delta\ge 0$ denotes the modulus argument. 
Further, $\dagger$ indicates that rational-tropical circuits are only considered under the \textbf{forbidden $m$-composition condition}~\eqref{eq:forbidden_composition}.}
\label{tab:G_propagators}
\end{table}

\subsection{Key Technical Detail: Where o-Minimality Matters}
\label{s:why_definable}
There are two critical places where definability in an o-minimal structure plays a crucial technical role. The first, and most important, is that any non-piecewise linear function which is definable in an o-minimal structure must have at least one ``good point'' at which the function is twice continuously differentiable and has non-degenerate Hessian. Surprisingly, as foreshadowed by the more restrictive technical condition in~\cite{kidger2020universal}, this is already enough for us to build networks that emulate multiplication and, once coupled with the affine layers of an $\mathcal{A}$NN, virtually every other basic gate arising from it.  

Under o-minimality and continuity, the failure of piecewise linearity of a function is equivalent to the existence of a ``good point''; by which we mean a point at which the function is twice continuously differentiable and at which one of its second-order partial derivatives does not vanish. This is reminiscent of the technical condition in~\citep[Theorem 3.2]{kidger2020universal} and the $\mathcal{A}_3$ class in~\citep[page 3]{zhang2024deep}.  
Interestingly, in the o-minimal setting there are only two possibilities: piecewise linearity and the existence of a ``good point''.  
\begin{lemma}[Characterization of Piecewise Linearity in o-Minimal Structures]
\label{lem:piecewiseLinear_ominimalsetting}
Let $d,D\in \mathbb{N}_+$, and let $f:\mathbb{R}^d\to \mathbb{R}^D$ be continuous and definable in an o-minimal structure.  
Then, the following are equivalent: 
\begin{enumerate}
    \item[(i)] \textbf{Non-Piecewise Linearity:} 
    $f$ is not piecewise linear,
    \item[(ii)] There exists a point $x\in \mathbb{R}^d$ and a radius $r>0$ such that for every $x'\in B(x,r)$ and every $r'>0$ with $B(x',r')\subseteq B(x,r)$, $f|_{B(x',r')}$ is not affine,
    \item[(iii)] \textbf{Good Point:} 
    There exists a ``good point'' $x^{\star}\in \mathbb{R}^d$ at which $f$ is twice continuously differentiable and at which one of its second-order partial derivatives does not vanish. 
\end{enumerate}
\end{lemma}
\begin{proof}
See Section~\ref{s:ProofMain}.
\end{proof}

\section{Conclusion}
\label{s:Conclusion}
This paper develops a complexity-theoretic view of neural-network approximation.  Its main point is that neural-network expressivity is not governed only by the regularity of the target function, but also by the complexity of the best elementary algorithm computing it.  Theorem~\ref{thrm:Universal_Computation} gives the qualitative statement: within the definable feedforward setting, universal approximation is equivalent to the presence of a non-affine nonlinearity.  Thus universality is generic, stable under composition, and not an artifact of a particular shallow architecture.  Theorem~\ref{thrm:concrete_surgery} gives the quantitative statement: every $\mathbb{G}$-circuit computing a target map can be compiled into an $\mathcal{A}$NN with explicit depth, width, and size bounds inherited from the circuit.  Consequently, standard approximation schemes, numerical algorithms, and TCS-style dynamic programs all become neural-network constructions once written as elementary circuits.  This recovers minimax-optimal approximation guarantees for classical regularity classes, and also shows that $\mathcal{A}$NNs can compute algorithms from numerical analysis and combinatorial optimization at comparable worst-case complexity.

\subsection{Future Work}
\label{s:Conclusion_ss:FutureWork}
On the theoretical side, a natural new question is to determine whether $\mathcal{A}$NNs are only efficient emulators of classical real-valued circuits, or whether their richer dictionaries of trainable nonlinearities, shared modules, attention, pooling, and normalization can yield provable complexity separations.  This would require lower bounds for restricted circuit languages, for which may be possible by using the recent theory of Lipschitz widths~\cite{petrova2023lipschitz} or their manifold width counterpart~\cite{cohen2022optimal}.

On the empirical side, it would be interesting to use our results to construct a ``\textit{neural network compiler}'' building on  Theorem~\ref{thrm:concrete_surgery}.  Such a meta-algorithm would accept pseudo-code as input, akin to the ones peppered throughout our manuscript and emulated by the neural networks we constructed in the case-by-case analysis of Proposition~\ref{prop:Step1_GateEmulation}, and would return a concrete neural network which certifiably computes that algorithm.  We expect that the framework developed herein would allow this, since we explicitly constructed most of those neural networks in the proofs.  

\addtocontents{toc}{\protect\setcounter{tocdepth}{-1}}
\appendix

\section{Proof of \texorpdfstring{Theorem~\ref{thrm:concrete_surgery}}{The Main Result}}
\label{s:ProofMain}

We begin with the proof of the technical tools on which our first few steps build.
\begin{proof}[{Proof of Lemma~\ref{lem:piecewiseLinear_ominimalsetting}}]
Suppose that $f$ is piecewise linear. Let $\mathbb{R}^d= \cup_{p=1}^P\, \Pi^{(p)}$ be a finite partition such that $f$ is affine on each $\Pi^{(p)}$. Let $B(x,r)$ be an arbitrary open ball. Then $\Pi^{(p)}\cap B(x,r)$ has non-empty interior for at least one $p$, and so we may always find a sub-ball $B(x',r')\subseteq B(x,r)$ such that the restriction of $f$ to $B(x', r')$ is affine. Thus (ii) fails; equivalently, by contraposition, (ii) implies (i).

Conversely, suppose that $f$ is not piecewise linear. By o-minimal cell decomposition, there is a finite partition $\mathbb{R}^d= U_1\cup\ldots\cup U_k$ such that $f_i=f\mid_{U_i}$ is $C^2$, and such that cells of dimension $d$ are open. Without loss of generality, there is at least one cell $U_i$ of dimension $d$ and indices $j\in [D]_+$ and $i_1,i_2\in [d]_+$ such that $\partial_{i_1}\partial_{i_2}(f_i)_j$ is not identically $0$; if not, then $f$ would be affine on each full-dimensional cell, and, by continuity and induction on the dimension, piecewise linear on $\mathbb{R}^d$. Consider the set $O=\{x\in U_i: \partial_{i_1}\partial_{i_2}(f_i)_j(x)\neq 0\}$. Since $(f_i)_j$ is $C^2$, $\partial_{i_1}\partial_{i_2}(f_i)_j:U_i\rightarrow \mathbb{R}$ is continuous, and so $O$ is non-empty and open. Therefore, there is some ball $B(x,r)\subseteq O$ such that, on $B(x,r)$, $\partial_{i_1}\partial_{i_2}(f_i)_j$ is nowhere $0$. In particular, $f_i$ (and hence $f$) is not affine on any sub-ball of $B(x,r)$; thus (i) implies (ii).

It remains to show that (ii) and (iii) are equivalent.
Since every open ball is definable, \citep[Theorem 6.7]{coste1999introduction} implies that there exists a finite number of definable $C^2$-submanifolds $\mathcal{M}^1,\dots,\mathcal{M}^p\subseteq B(x,r)$ on which $f|_{\mathcal{M}^k}$ is twice continuously differentiable. Now, since there are finitely many such definable manifolds covering $B(x,r)$, there must exist some $p'\in [p]$ such that $\mathcal{M}^{p'}$ has full dimension ($\operatorname{dim}(\mathcal{M}^{p'})=d$). Consequently, there must exist some open ball $B(x',r')\subseteq \mathcal{M}^{p'}$; whence, $f|_{B(x',r')}$ is $C^2$.  
If for every $j\in[D]_+$ all second partials $\partial_{i_1}\partial_{i_2} f_j$ vanished identically on $B(x',r')$,
then each component $f_j$ would be affine on $B(x',r')$, hence $f$ would be affine on $B(x',r')$,
contradicting (ii). Therefore, there exist $j\in [D]_+$ and $i_1,i_2\in[d]_+$ such that $\partial_{i_1}\partial_{i_2} f_j$ is not identically zero on $B(x',r')$, and thus there exists $x_0\in B(x',r')$ with $\partial_{i_1}\partial_{i_2} f_j(x_0)\neq 0$, proving (iii).
Conversely, if (iii) holds, then by continuity of the relevant second-order partial derivative there exists a ball $B(x^\star,r)$ on which it is nowhere zero. Hence $f$ cannot be affine on any sub-ball of $B(x^\star,r)$.
\end{proof}

\subsection{Proof of Step 1: \texorpdfstring{Proposition~\ref{prop:Step1_GateEmulation}}{Gate Emulation}}
\label{s:ProofMain__prop:Step1_GateEmulation}

We now prove Proposition~\ref{prop:Step1_GateEmulation}. The proof proceeds via a set of case-by-case arguments, wherein each gate is encoded into a general $\mathcal{A}$NN whose  activation dictionary $\mathcal{A}$ contains at least one non-piecewise affine function.  We begin with the simulation of elementary building blocks and simpler gates  (e.g., $+$ and $\times$), which are in turn called as sub-networks, simulating  sub-algorithms, within more complex procedures encoded by networks computing the more sophisticated gates (e.g., $\sqrt{\cdot}$ or $\tfrac{1}{\cdot}$) and sub-networks (e.g., $k$-ary maximization, tensorized splines, or sparse polynomials).

\subsection{Basic Algebraic Gates}
\label{s:BasicAlgOps}

For any $x\in \mathbb{R}^n$ and each $r>0$ we denote the ball at $x$ of radius $r>0$ by $B(x,r) \eqdef  \{z\in \mathbb{R}^n:\, \|x-z\|_2<r\}$.

\subsubsection{Identity Blocks}
\label{s:BasicAlgOps__ss:identity}
We now build the identity map, which, although it seems simple, will play a key role when concatenating neural networks (so-called parallelization; cf.~\cite{petersen2024mathematical}).  Our strategy is to encode the following Algorithm~\ref{alg:identity_emulation__multi}.

\begin{algorithm}[H]
\caption{$m$-Dimensional Approximate Identity Gate}
\label{alg:identity_emulation__multi}
\KwIn{
    input $x=(x_1,\dots,x_m)\in\mathbb{R}^m$; (small) $h>0$; $f\in\mathcal{A}$; $x^\star\in\mathbb{R}$ with $f'(x^\star)\neq 0$
}
\KwOut{
    $\widehat{\operatorname{Id}}(x) \approx x$
}
$
    z \eqdef h x + x^\star \mathbf{1}_m
$
\tcp*[r]{{\scriptsize 1st affine layer}}
$
    a \eqdef f(z)
$
\tcp*[r]{{\scriptsize nonlinear activation}}
$
\widehat{\operatorname{Id}}(x)
\eqdef
\frac{1}{h f'(x^\star)}
    \big(
        a-f(x^\star)\mathbf{1}_m
    \big)
$
\tcp*[r]{{\scriptsize 2nd affine layer (finite difference)}}
\Return{$\widehat{\operatorname{Id}}(x)$}
\end{algorithm}

\begin{proposition}[Approximate $m$-Dimensional Identity Gate]
\label{prop:identity_emulation__multi}
Assume Assumption~\ref{ass:THEASSUMPTIONS}.
Then, for every $m\in\mathbb{N}_+$, every $\varepsilon>0$, and every $M>0$, there exists an $\mathcal{A}$NN
$
    \operatorname{Id}_{\varepsilon,M,m}:\mathbb{R}^m\to\mathbb{R}^m
$
such that
\[
    \sup_{\|x\|_\infty\le M}
        \big\|
            \operatorname{Id}_{\varepsilon,M,m}(x)-x
        \big\|_\infty
    \le \varepsilon
.
\]
Moreover, $\operatorname{Id}_{\varepsilon,M,m}$ may be chosen with depth at most $1$, width at most $m$, and with at most $4m$ non-zero parameters.
\end{proposition}

We now prove the identity emulation result.

\begin{lemma}[Approximate Identity Gate]
\label{lem:identity_emulation__multi}
Under Assumption~\ref{ass:THEASSUMPTIONS},  there exist $x^\star\in \mathbb{R}$ and $f\in \mathcal{A}$ with $f'(x^\star)\neq 0$ such that, for every $\varepsilon>0$ and $M>0$, there exists an $\mathcal{A}$NN
$
    \operatorname{Id}_{\varepsilon,M}:\mathbb{R}\to\mathbb{R}
$
satisfying
\[
    \sup_{|x|\le M}
        \big|
            \operatorname{Id}_{\varepsilon,M}(x)-x
        \big|
    \le \varepsilon
.
\]
Moreover, $\operatorname{Id}_{\varepsilon,M}$ may be chosen with depth $1$, width $1$, and $4$ non-zero parameters.
\end{lemma}
\begin{proof}[{Proof of Lemma~\ref{lem:identity_emulation__multi}}]
By Assumption~\ref{ass:THEASSUMPTIONS} (iii), $\mathcal{A}$ contains a function $f$ which is not piecewise linear; thus, Lemma~\ref{lem:piecewiseLinear_ominimalsetting} (i) and (iii) imply that there exists some $x^{\star}$ such that $f^{\prime}(x^{\star})\neq 0$.
Now, fix $\varepsilon>0$ and $M>0$.  
For $h>0$, define
$
       \operatorname{Id}_{h}(x)
    \eqdef
        \frac{
            f(x^\star+h x)-f(x^\star)
        }{
            h\,f'(x^\star)
        }
$.
This is an $\mathcal{A}$NN with depth at most $1$, width at most $1$, and at most $4$ non-zero parameters.
Since $f$ is differentiable at $x^\star$, we have
$
    \lim\limits_{t\downarrow 0}
    \frac{
        f(x^\star+t)-f(x^\star)
    }{
        t
    }
    =
    f'(x^\star)
    {\neq} 0
$. 
Hence, for $h>0$ sufficiently small,
$
    \sup_{|x|\le M}
        \Big|
            \frac{
                f(x^\star+h x)-f(x^\star)
            }{
                h\,f'(x^\star)
            }
            -x
        \Big|
    \le 
        \varepsilon
,
$
which yields the claim. 
\end{proof}

\begin{proof}[{Proof of Proposition~\ref{prop:identity_emulation__multi}}]
Let $\operatorname{Id}_{\varepsilon,M}$ be as in Lemma~\ref{lem:identity_emulation__multi}, and define
\[
        \operatorname{Id}_{\varepsilon,M,m}(x_1,\dots,x_m)
    \eqdef
        \big(
            \operatorname{Id}_{\varepsilon,M}(x_1),
            \dots,
            \operatorname{Id}_{\varepsilon,M}(x_m)
        \big)
.
\]
Then $\operatorname{Id}_{\varepsilon,M,m}$ is an $\mathcal{A}$NN with depth at most $1$, width at most $m$, and at most $4m$ non-zero parameters.  
\end{proof}

\subsubsection{Squaring}
Our first step is to approximately compute the squaring operation $x\mapsto x^2$ onto a small expert agent; by emulating Algorithm~\ref{algo:square_from_fd}.  Our construction is much simpler than the Telgarsky construction; cf.~\cite{telgarsky2015representation} for $\operatorname{ReLU}$-MLPs, and more efficient by a log factor.

\begin{algorithm}[H]
\caption{Squaring Expert}
\label{algo:square_from_fd}
\KwIn{Non-linearity $f:\mathbb{R}^m\to\mathbb{R}$ with $f\in \mathcal{A}$, point $b = x^{\star} \in \mathbb{R}^m$, indices $i,j\in[m]_+$, $h>0$ ($h\approx 0$) such that $c\eqdef \partial_{x_j}\partial_{x_i} f(x^{\star})\neq 0$.}
\KwOut{$\widehat s(x)\approx x^2$}
$u
\eqdef
\begin{bmatrix}
    b+hx(e_i+e_j)\\
    b+hxe_i\\
    b+hx e_j \\
    b 
\end{bmatrix}$
\tcp*[r]{{\scriptsize affine layer (four parallel branches)}}

$v
\eqdef
\begin{bmatrix}
    f(u_1)\\
    f(u_2)\\
    f(u_3) \\
    f(u_4)
\end{bmatrix}$
\tcp*[r]{{\scriptsize activation block $(f,f,f,f)^\top\in\mathcal{A}$}}

$\widehat s(x)
\eqdef \widehat{\operatorname{Id}}\left(
\displaystyle\frac{v_1 - v_2 - v_3 + v_4}{c\,h^2}\right)$
\tcp*[r]{{\scriptsize affine layer (finite difference) and identity gate from Algorithm~\ref{alg:identity_emulation__multi}}}

\Return{$\widehat s(x)$}\;
\end{algorithm}

We begin by illustrating the basic building block underlying most algebraic gates, namely the univariate quadratic function, using a single function with non-vanishing Hessian; that is, a function with at least one non-vanishing second-order partial derivative, possibly mixed. We will not need vector-valued functions in what follows, since one may always project onto a suitable coordinate.  

\begin{proposition}[Approximate Squaring Gate]
\label{prop:SomwhereNonFlatDefinable_Yields_Squarable}
Under Assumption~\ref{ass:THEASSUMPTIONS}, there exist $m\in \mathbb{N}_+$, 
a map $f: \mathbb{R}^m \to \mathbb{R}$,
a point $x^\star \in \mathbb{R}^m$, and 
indices
$i,j \in [m]_+$ such that $
\partial_{x_j}\partial_{x_i}
f(x^\star)\neq 0$. Then, for any $\varepsilon>0$ and $M>0$ there exists a network $g=g_{\varepsilon,M}$ such that
$$
\sup_{|x| \leq M} |g(x) - x^2| \leq \varepsilon.
$$
Moreover, $g$ has depth $2$ and width $4 
m
$, with at most $4 \|x^\star\|_0+8$ non-zero parameters.
Furthermore, the weights of $g$ depend on $M$ and on $\varepsilon$.
\end{proposition}
\begin{proof}[{Proof of Proposition~\ref{prop:SomwhereNonFlatDefinable_Yields_Squarable}}]
Again, under Assumption~\ref{ass:THEASSUMPTIONS} (iii), Lemma~\ref{lem:piecewiseLinear_ominimalsetting} (i) and (iii) imply the existence of $x^{\star}$.  
We will use a finite difference approach, following \cite{kidger2020universal,kratsios2022universal}. Let $b = x^\star$ such that
$
c\eqdef
\partial_{x_j}\partial_{x_i}\, f
(x^\star)
$
is non-zero.  
For any $h>0$, define the four-point finite-difference operator with directional step-size $h>0$ for all $x \in \mathbb{R}$ by
$$
D_h(x)
\eqdef
\frac{
f
(b+hxe_i+hxe_j)
-
f
(b+hxe_i)
-
f
(b+hxe_j)
+
f
(b)
}{h^2}
.
$$
By standard estimates for mixed finite differences, for any $M \geq 0$ and for any $\delta >0$, there exists $h(\delta,M) >0$ such that
$
    \sup_{|x| \leq M}
        \,
        \big|
        D_h(x) - c\,x^2
        \big|
\leq 
    \delta
$. 
Hence,
$
    \sup_{|x|\leq M}
    \,
    \big|
        x^2 - \frac{1}{c}D_h(x)
    \big|
\le
    \frac{\delta}{|c|}
.
$
Consequently, the network is given by $g = \mathcal{L}_2 \circ \mathcal{L}_1$ where
$
\mathcal{L}_1(x) = f_1(A^{(1)} x+ b^{(1)})
$
and
$
\mathcal{L}_2(y) = f_2(A^{(2)} y),
$
where 
\begin{align*}
    (A^{(1)}, b^{(1)})
    & \eqdef 
    \left(
    \begin{bmatrix}
        h(e_i+e_j)\\
        h e_i\\
        h e_j\\
        \mathbf{0}_m
    \end{bmatrix}
    ,
    \begin{bmatrix}
        b\\
        b\\
        b\\
        b
    \end{bmatrix}
    \right),\\
    f_1\left(
        \begin{bmatrix}x_1\\ x_2\\ x_3\\ x_4\end{bmatrix}
    \right)
    & \eqdef 
    \begin{bmatrix}
    f(x_1)\\
    f(x_2)\\
    f(x_3)\\
    f(x_4)
    \end{bmatrix},\\
    A^{(2)} & 
    \eqdef 
    \frac{1}{c\,h^2}
    \begin{bmatrix}
        1
    ,\,
        -1
    ,\,
        -1
    ,\,
        1
    \end{bmatrix},
\\
        f_2 &
     \eqdef 
        \mathrm{id}.
\end{align*}
Note that $f_1=(f,f,f,f)^{\top}\in \mathcal{A}$ thanks to Assumption~\ref{ass:THEASSUMPTIONS} (ii).
Applying Proposition~\ref{prop:identity_emulation__multi} we deduce the existence of an $\mathcal{A}$NN $\check{f}_2:\mathbb{R}\to \mathbb{R}$ of depth $1$, width $1$, and size $4$ satisfying
\begin{equation}
\label{eq:identify_me}
        \max_{|x|\le M^2 + \tfrac{\delta}{|c|}}
        \, 
        \big|
            f_2(x)-\check{f}_2(x)
        \big| 
<
        \tfrac{\varepsilon}{2}
.
\end{equation}
Replacing $f_2$ by $\check{f}_2$ and
retroactively setting $\varepsilon\eqdef \tfrac{\delta}{2\,|c|}$ yields the conclusion.\footnote{Note that the choice and norm of $A^{(1)}$ and $A^{(2)}$ depend on $h$ which, in turn, depends on $M$ and $\delta$.}%
\end{proof}

Throughout the remainder of this part, unless explicitly stated otherwise,
Assumption~\ref{ass:THEASSUMPTIONS} is in force. We fix once and for all the
objects $m_\star,f_\star,x^\star,i_\star,j_\star$ supplied by
Proposition~\ref{prop:SomwhereNonFlatDefinable_Yields_Squarable}.
All subsequent occurrences of these symbols refer to these fixed witnesses.

\paragraph{Basic Algebraic Approximate-Gates.}
We now construct our approximate emulation of each of the elementary gates in Table~\ref{tab_thrm:gate_approximation}.
Using the polarization trick, cf.~\cite{yarotsky2017error}, we may now approximately realize the binary multiplication (algebraic) gate as an $\mathcal{A}$NN.
We do so by encoding Algorithm~\ref{algo:polarization_mult} onto a general $\mathcal{A}$NN.

\begin{algorithm}[H]
\caption{Binary Multiplication Expert}
\label{algo:polarization_mult}
\KwIn{small $h,\delta>0$ ($h,\delta\approx 0$)}
\KwOut{$\widehat{m}(u,v)\approx uv$}

$a\eqdef
\begin{bmatrix}
    u+v\\
    u\\
    v
\end{bmatrix}$
\tcp*[r]{{\scriptsize affine lift into three parallel branches}}

$q\eqdef
\begin{bmatrix}
    \widehat{s}(a_1;b=x^\star,i,j,h)\\
    \widehat{s}(a_2;b=x^\star,i,j,h)\\
    \widehat{s}(a_3;b=x^\star,i,j,h)
\end{bmatrix}$
\tcp*[r]{{\scriptsize each branch calls the squaring expert $\widehat{s}$ from Algorithm~\ref{algo:square_from_fd}}}

$\widehat m(u,v)\eqdef \frac{1}{2}
\begin{bmatrix}
    1 & -1 & -1
\end{bmatrix}
q$
\tcp*[r]{{\scriptsize affine polarization readout}}
\Return{$\widehat m(u,v)$}
\end{algorithm}

\begin{proposition}[Approximate Binary Multiplication Gate]
\label{prop:polarization_mult_gadget}
Under Assumption~\ref{ass:THEASSUMPTIONS},
for every $M>0$ and $\varepsilon>0$ there exists an $\mathcal{A}$NN 
$
    \mathrm{Mult}_{\varepsilon,M}:\mathbb{R}^2\to\mathbb{R}
$
satisfying
\[
    \sup_{|u|\le M,\ |v|\le M}\, \big|\mathrm{Mult}_{\varepsilon,M}(u,v)-uv\big|\le \varepsilon.
\]
Moreover, $\mathrm{Mult}_{\varepsilon,M}$ has depth $3$, width $12\max\{m,n\}+2$, and with at most
$4(2\|x^\star\|_0+5)+11$ non-zero parameters.  Furthermore, the weights depend on $M$ and on $\varepsilon$.
\end{proposition}
\begin{proof}[Proof of Proposition~\ref{prop:polarization_mult_gadget}]
Fix $M>0$ and $\varepsilon>0$, and set $\delta\eqdef 2\varepsilon/3$.
Let $S\eqdef f_{\delta,2M}:\mathbb{R}\to\mathbb{R}$ be the depth-$2$ $\mathcal{A}$NN from
Proposition~\ref{prop:SomwhereNonFlatDefinable_Yields_Squarable}, so that
\[
    \sup_{|x|\le 2M}\, \big|S(x)-x^2\big|\le \delta.
\]
Define $
    \mathrm{Mult}_{\varepsilon,M}(u,v)
    \eqdef
    \tfrac12\Big(S(u+v)-S(u)-S(v)\Big)
$.
For $|u|,|v|\le M$ we have $|u+v|\le 2M$, hence
\[
\begin{aligned}
    \big|\mathrm{Mult}_{\varepsilon,M}(u,v)-uv\big|
    &=
    \tfrac12\big|\big(S(u+v)-(u+v)^2\big)-\big(S(u)-u^2\big)-\big(S(v)-v^2\big)\big|  \\
    &\le \tfrac12(\delta+\delta+\delta)
    = \tfrac32\,\delta
    = \varepsilon.
\end{aligned}
\]
The stated depth/width/sparsity bounds follow by wiring three copies of $S$ in parallel (fed with $u+v$, $u$, $v$ via affine maps)
and post-composing with one affine output map.
\end{proof}

Having approximately implemented the squaring function, we may cheaply obtain an exponentiation gadget which we can use as a sub-network of our network wherever needed.



\begin{figure}
    \centering
    \includegraphics[width=0.5\linewidth]{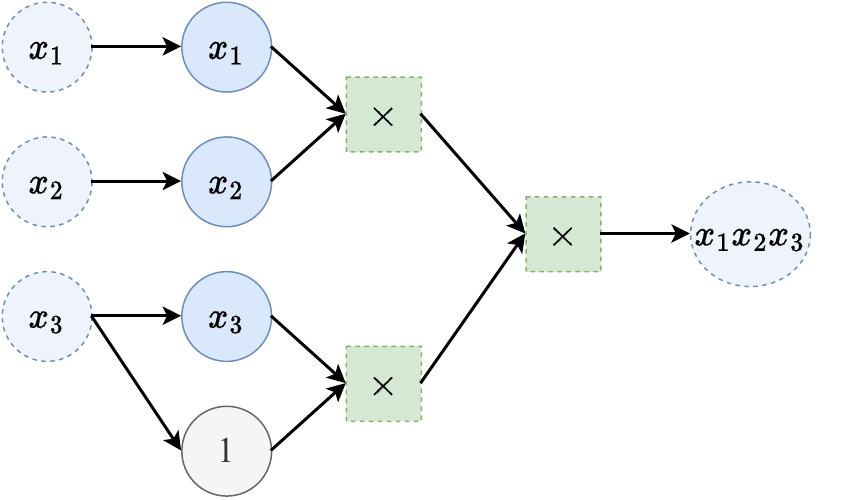}
    \caption{We obtain the $k$-fold product (here $k=3$) by nesting the binary product in Proposition~\ref{prop:polarization_mult_gadget} and padding by $1$s until the inputs to the nested binary products are a power of $2$ (here $2^2$).}
    \label{fig:prof}
\end{figure}

\begin{proposition}[Approximate $k$-Fold Multiplication Gate]
\label{prop:kth_multiplication__SK__polar__precise}
Fix $k\in \mathbb{N}$ with $k\ge 2$.
For every $\varepsilon>0$ and $M>0$ there exists an $\mathcal{A}$NN $ 
    \mathrm{Mult}^{(k)}_{\varepsilon,M}:\mathbb{R}^k\to\mathbb{R}
$
satisfying
\[
    \sup_{|x|\le M}\, \big|
        \mathrm{Mult}^{(k)}_{\varepsilon,M}
        (x)-
        \prod_{i=1}^k\, x_i
    \big|\le \varepsilon
.
\]
Moreover, $\mathrm{Mult}^{(k)}_{\varepsilon,M}$ has depth at most $3\lfloor \log_2(k)\rfloor  + 5$, width at most $\lceil k/2\rceil\big(12 m +2\big)$, and at most
$2 (k-1) (4 \|x^\star\|_0 + 17) + 5$
non-zero parameters. 
\end{proposition}
\begin{proof}[Proof of Proposition~\ref{prop:kth_multiplication__SK__polar__precise}]
Fix $\varepsilon>0$ and $M>0$.
Set $t_i\eqdef x_i/M$, so that $|t_i|\le 1$ whenever $|x|\le M$.
Let $
    \eta\eqdef \min\Big\{\varepsilon\big/(2M^k(k-1))\, ,\, (\log 2)\big/(k-1)\Big\}
$ and $
    \delta\eqdef 2\eta/3
$.
Let $S\eqdef f_{\delta,4}$ be the depth-$2$ squaring gadget from
Proposition~\ref{prop:SomwhereNonFlatDefinable_Yields_Squarable}, so that
$
    \sup_{|z|\le 4}|S(z)-z^2|\le \delta
.
$
Define
$
    \mathrm{Mult}_\eta(u,v)\eqdef \tfrac12\big(S(u+v)-S(u)-S(v)\big)
.
$
Then
\[
    \sup_{|u|\le 2,\ |v|\le 2}\,|\mathrm{Mult}_\eta(u,v)-uv|\le \eta
.
\]
Let $L\eqdef \lceil \log_2(k)\rceil$ and $\tilde{k}\eqdef 2^L-k$.
Define the padding affine map
\[
    \mathrm{Pad}(t)\eqdef (t_1,\dots,t_k,\mathbf{1}_{\tilde{k}})\in\mathbb{R}^{2^L}
.
\]
Compute $\prod_{i=1}^k t_i$ by composing $\mathrm{Mult}_\eta$ along a balanced binary tree of depth $L$
(obtained from the complete depth-$L$ binary tree on $2^L$ leaves by pruning the multiplications by $1$ coming from $\mathrm{Pad}$),
and denote the output by $\widehat{P}(t)$.

We claim that, for all $|t_i|\le 1$,
$
    |\widehat{P}(t)|
    \le
    (1+\eta)^{k-1}
    \le
    2
$ and 
$
    \big|\widehat{P}(t)-\prod_{i=1}^k t_i\big|
    \le
    (1+\eta)^{k-1}-1
    \le
    2(k-1)\eta
$.
Indeed, writing $E_r$ for the worst-case absolute error after multiplying $r$ inputs, one has the recurrence 
\begin{equation}
\label{eq:recurrence}
\begin{aligned}
    E_{a+b} & \le \eta+E_a+E_b+E_aE_b, \\
    E_1     & =0,
\end{aligned}
\end{equation}
since each internal node multiplies two (possibly approximate) partial products and contributes an additional error at most $\eta$.
\begin{equation}
\label{eq:recursion_implies_control}
    E_k
\le 
    (1+\eta)^{k-1}-1
\end{equation}
and hence $|\widehat{P}|\le 1+E_k\le (1+\eta)^{k-1}$.
Finally, since $\eta\le (\log 2)/(k-1)$, we have $(1+\eta)^{k-1}\le e^{(k-1)\eta}\le 2$, and
\[
    (1+\eta)^{k-1}-1 \le e^{(k-1)\eta}-1 \le (k-1)\eta\,e^{(k-1)\eta}\le 2(k-1)\eta
.
\]
Therefore, 
$
    \sup_{|t_i|\le 1}\,|\widehat{P}(t)-\prod_{i=1}^k t_i|
    \le
    2(k-1)\eta
    \le
    \varepsilon/M^k
$.  
Define the multiplication gate by
$
    \mathrm{Mult}^{(k)}_{\varepsilon,M}(x)\eqdef M^k\,\widehat{P}(x/M)
.
$
Consequently,
\[
    \sup_{|x|\le M}\,
    \Big|
        \mathrm{Mult}^{(k)}_{\varepsilon,M}(x)-\prod_{i=1}^k x_i
    \Big|
    \le
    M^k\cdot \varepsilon/M^k
    =
    \varepsilon
.
\]
It remains to tally the complexity.
The tree has $\lceil\log_2(k)\rceil$ levels of multiplications, so the depth is at most
$3\lceil\log_2(k)\rceil$ for the multiplication blocks, plus $1$ affine layer for the input rescaling and $1$ affine layer for the output rescaling.
The maximal parallelism occurs at the first level, which contains $\lceil k/2\rceil$ multiplications in parallel, hence the stated width bound.
Finally, there are $k-1$ multiplication blocks (after pruning the padded $1$ leaves), each contributing at most $4(2\|x^\star\|_0+5)+11$ non-zero parameters, plus $k$ for the affine input rescaling,
plus at most $2k+2$ for wiring and the final affine output scaling.
\end{proof}

\begin{corollary}[Approximate $k$-th Power Gate]
\label{cor:kth_power_gadget__SK__polar__precise}
Fix $k\in\mathbb{N}_+$.
For every $\varepsilon>0$ and $M>0$ there exists an $\mathcal{A}$NN
$
    \mathrm{Pow}^{(k)}_{\varepsilon,M}:\mathbb{R}\to\mathbb{R}
$
satisfying
\[
    \sup_{|x|\le M}\,
    \Big|
        \mathrm{Pow}^{(k)}_{\varepsilon,M}(x)
        -
        x^k
    \Big|
    \le
    \varepsilon
.
\]
Moreover, $\mathrm{Pow}^{(k)}_{\varepsilon,M}$ has depth at most
$
    3\lceil \log_2(k)\rceil+2,
$
width at most
$
    \left\lceil\frac{k}{2}\right\rceil\big(12 m +2\big),
$
and at most
$
    (k-1)\big(4(2\|x^\star\|_0+5)+11\big)+3k+2
$
non-zero parameters.
\end{corollary}
\begin{proof}[{Proof of Corollary~\ref{cor:kth_power_gadget__SK__polar__precise}}]
If $k=1$, then by the affine depth-zero convention following Definition~\ref{defn:ANN}, we may take $\mathrm{Pow}^{(1)}_{\varepsilon,M}(x)\eqdef x$. This map has depth $0$, width $1$, and one non-zero parameter, so the approximation error is $0$ and the stated depth, width, and non-zero parameter bounds are immediate. Hence assume below that $k\ge 2$.

Define the affine diagonal map
$
        \mathrm{diag}_k(x)
    \eqdef 
        (x,\dots,x)\in\mathbb{R}^k
$ and consider the composition
$
    \mathrm{Pow}^{(k)}_{\varepsilon,M}(x)
    \eqdef
    \mathrm{Mult}^{(k)}_{\varepsilon,M}\big(\mathrm{diag}_k(x)\big)
$.  Since $
    \prod_{i=1}^k (\mathrm{diag}_k(x))_i
    =
    x^k,
$
Proposition~\ref{prop:kth_multiplication__SK__polar__precise} yields
$
    \sup_{|x|\le M}\,
    \Big|
        \mathrm{Pow}^{(k)}_{\varepsilon,M}(x)
        -
        x^k
    \Big|
    \le
    \varepsilon
$.
Finally, the affine duplication map may be absorbed into the input affine layer in the proof of Proposition~\ref{prop:kth_multiplication__SK__polar__precise}, and therefore the same depth, width, and non-zero parameter bounds hold.
\end{proof}

\subsubsection{Exponentiation}
\label{s:Exponentiation}

As a consequence, we may approximate the exponential map by emulating the Picard iteration for the one-dimensional ODE $u'(x)=u(x)$, whose Picard operator is given by $u(\cdot)\mapsto 1+\int_0^x u(t)\,dt$. We remark that the iterates of this map, evaluated first at the constant function $u(t)=1$ for all $t$, are given exactly, and conveniently, by the truncations of the Taylor series expansion of $e^x$. 
\hfill\\
\indent
We mention this because the Picard iteration fundamentally uses depth, while, in contrast, an Euler-type approximation such as $e^x\approx (1+x/k)^k$ does not. Moreover, we now have enough elementary gates emulated to efficiently encode the convenient truncated Taylor-series representation of these iterates.
\begin{corollary}[Exponential Emulator]
\label{cor:ExpEmulation}
Assume that the activation dictionary $\mathcal{A}$ consists only of definable functions, $\mathrm{id}\in \mathcal{A}$, and contains the ``parallelized activation function'' $(f^{\top},f^{\top},f^{\top})^{\top}$ for some somewhere non-flat $f:\mathbb{R}^m\to\mathbb{R}^n$, where $m,n\in \mathbb{N}_+$.
Under Assumption~\ref{ass:THEASSUMPTIONS}, for every $\varepsilon,M>0$, define
$
    N_{M,\varepsilon}
    \eqdef
    \left\lceil
        \max\left\{
            2eM,\,
            \log_2\left(\frac{2e^M}{\varepsilon}\right)
        \right\}
    \right\rceil
$
there exists an $\mathcal{A}$NN
$
    \mathrm{Exp}_{\varepsilon,M}:\mathbb{R}\to\mathbb{R}
$
satisfying
\[
    \sup_{|x|\le M}
    \,
    \Big|
        \mathrm{Exp}_{\varepsilon,M}(x)
        -
        e^x
    \Big|
    \le
    \varepsilon
.
\]
Moreover, $\mathrm{Exp}_{\varepsilon,M}$ has depth at most
$
    3\lceil \log_2(N_{M,\varepsilon})\rceil+3,
$
\\
width at most
$
    \left\lceil\frac{N_{M,\varepsilon}}{2}\right\rceil
    \left\lceil\frac{N_{M,\varepsilon}+1}{2}\right\rceil
    \big(12\max\{m,n\}+2\big),
$
and at most
$
    \sum_{j=1}^{N_{M,\varepsilon}}
    [
        (j-1)\big(4(2\|x^\star\|_0+5)+11\big)+3j+2
    ]
    +N_{M,\varepsilon}+1
$
non-zero parameters.
\end{corollary}
\begin{proof}[{Proof of Corollary~\ref{cor:ExpEmulation}}]
Write $N\eqdef N_{M,\varepsilon}$.
For each $j\in [N]_+$, apply Corollary~\ref{cor:kth_power_gadget__SK__polar__precise} with accuracy
$
    \varepsilon_j
    \eqdef
    \frac{\varepsilon\, j!}{2N}
$
to obtain an $\mathcal{A}$NN
$
    \mathrm{Pow}^{(j)}_{\varepsilon_j,M}:\mathbb{R}\to\mathbb{R}
$
such that
$
    \sup_{|x|\le M}\,
    \Big|
        \mathrm{Pow}^{(j)}_{\varepsilon_j,M}(x)
        -
        x^j
    \Big|
    \le
    \varepsilon_j
.
$
Now define
$
    \mathrm{Exp}_{\varepsilon,M}(x)
    \eqdef
    1+
    \sum_{j=1}^N
    \frac{1}{j!}\,
    \mathrm{Pow}^{(j)}_{\varepsilon_j,M}(x)
.
$
Then, for every $|x|\le M$,
\[
\begin{aligned}
    \Big|
        \mathrm{Exp}_{\varepsilon,M}(x)
        -
        e^x
    \Big|
    &\le
    \Big|
        \mathrm{Exp}_{\varepsilon,M}(x)
        -
        \sum_{j=0}^N \frac{x^j}{j!}
    \Big|
    +
    \Big|
        \sum_{j=0}^N \frac{x^j}{j!}
        -
        e^x
    \Big|
\\
    &\le
    \sum_{j=1}^N
    \frac{1}{j!}\,
    \Big|
        \mathrm{Pow}^{(j)}_{\varepsilon_j,M}(x)-x^j
    \Big|
    +
    e^M\frac{M^{N+1}}{(N+1)!}
\\
    &\le
    \sum_{j=1}^N
    \frac{\varepsilon_j}{j!}
    +
    e^M\frac{M^{N+1}}{(N+1)!}
\\
    &\le
    \frac{\varepsilon}{2}
    +
    e^M\frac{M^{N+1}}{(N+1)!}
.
\end{aligned}
\]
By Stirling's lower bound,
$
    (N+1)!
    \ge
    \left(\frac{N+1}{e}\right)^{N+1}
,
$
and since $N\ge 2eM$, it follows that
$
    \frac{eM}{N+1}
    \le
    \frac{1}{2}
.
$
Hence
$
    e^M\frac{M^{N+1}}{(N+1)!}
    \le
    e^M\left(\frac{eM}{N+1}\right)^{N+1}
    \le
    e^M\,2^{-(N+1)}
    \le
    \frac{\varepsilon}{2}
,
$
where the last inequality follows from $N\ge \log_2\left(\frac{2e^M}{\varepsilon}\right)$.
Therefore
$
    \sup_{|x|\le M}
    \,
    \Big|
        \mathrm{Exp}_{\varepsilon,M}(x)-e^x
    \Big|
    \le
    \varepsilon
.
$
\hfill\\
\indent
For the complexity bounds, parallelize the $N$ branches
$
    \mathrm{Pow}^{(j)}_{\varepsilon_j,M}
$
for $j\in [N]_+$, pad the shallower ones with identity layers up to depth
$
    3\lceil\log_2(N)\rceil+2,
$
and then apply one final affine output layer with weights $(1/j!)_{j=1}^N$ and bias $1$.
Thus the depth is at most
$
    3\lceil \log_2(N)\rceil+3.
$
The width is bounded by the sum of the branch widths:
$
    \sum_{j=1}^N
    \left\lceil\frac{j}{2}\right\rceil
    \big(12\max\{m,n\}+2\big)
    =
    \left\lceil\frac{N}{2}\right\rceil
    \left\lceil\frac{N+1}{2}\right\rceil
    \big(12\max\{m,n\}+2\big).
$
The number of non-zero parameters is bounded by the sum of the branch counts, plus $N+1$ non-zero parameters in the final affine output layer; substituting $N=N_{M,\varepsilon}$ yields the conclusion.
\end{proof}

\paragraph{Sparse Polynomials.}
We thus directly obtain an approximate emulation of one of the most powerful classical objects in approximation theory and algebraic circuit complexity: sparse polynomials.
\begin{proposition}[Sparse Polynomial Gates via Power and Tree-Multiplication]
\label{prop:sparse_poly_gate__Pow_and_TreeMult}
Fix $S,\Delta\in\mathbb{N}_+$ and let $p:\mathbb{R}^n\to\mathbb{R}$ be of the form
\begin{equation}
\label{eq:polyrep}
    p(x)
    \eqdef
    \sum_{s=1}^S \beta_s \prod_{i=1}^{I_s} x_i^{\alpha^{(s)}_i},
\end{equation}
where $1\le I_s\le n$, $
    \big|\alpha^{(s)}\big|
    \eqdef
    \sum_{i=1}^{I_s}\alpha^{(s)}_i
    \le
    \Delta
$ and $\beta_s\in \mathbb{R}$ for every $s\in[S]$.
For every $\varepsilon>0$ and $M>0$ there exists an $\mathcal{A}$NN 
$
    \Phi_{p,\varepsilon,M}:\mathbb{R}^n\to\mathbb{R}
$
satisfying
\[
    \sup_{|x|\le M}\,|\Phi_{p,\varepsilon,M}(x)-p(x)|
    \le
    \varepsilon
.
\]
Moreover, $\Phi_{p,\varepsilon,M}$ may be chosen with:
\begin{enumerate}
    \item[(i)] \textbf{Depth:} 
    $
        \mathcal{O}(\log(\Delta))
    ,
    $
    \item[(ii)] \textbf{Width:}
    $
        \mathcal{O}\Big(\sum_{s=1}^S\big(I_s+|\alpha^{(s)}|\big)\Big)
    ,
    $
    \item[(iii)] \textbf{Non-Zero Parameters:}
    $
        \mathcal{O}\Big(\sum_{s=1}^S\big(I_s+|\alpha^{(s)}|\big)\Big)
    .
    $
\end{enumerate}
\end{proposition}

\begin{proof}[Proof of Proposition~\ref{prop:sparse_poly_gate__Pow_and_TreeMult}]
Fix $\varepsilon>0$ and $M>0$ and set $t\eqdef x/M$, so that $|t|\le 1$ whenever $|x|\le M$.
For each $s\in[S]$, write
\[
\begin{aligned}
        p(x)
    & =
        \sum_{s=1}^S
            \gamma_s
            \prod_{i\in J_s} t_i^{\alpha^{(s)}_i},
\\
    \gamma_s & \eqdef \beta_s\,M^{\sum_{i\in J_s}\alpha^{(s)}_i},
\end{aligned}
\]
where $J_s\eqdef \{1,\dots,I_s\}$ and $k_s\eqdef |J_s|=I_s$.
Set $
    C_p \eqdef \sum_{s=1}^S |\gamma_s|
$.  If $C_p=0$ then $p\equiv 0$ and we set $\Phi_{p,\varepsilon,M}\eqdef 0$.
Assume $C_p>0$ and set $\varepsilon_s\eqdef \varepsilon/C_p$ for $s\in[S]$.

For each $s\in[S]$ and each $i\in J_s$, apply Corollary~\ref{cor:kth_power_gadget__SK__polar__precise} (with input bound $M=1$)
to obtain an $\mathcal{A}$NN
\[
    \mathrm{Pow}_{\alpha^{(s)}_i,\eta_s,1}:\mathbb{R}\to\mathbb{R}
\]
approximating the map $u\mapsto u^{\alpha^{(s)}_i}$ on $[-1,1]$, with $\eta_s>0$ to be chosen below.
Define
\[
    u^{(s)}(t)\eqdef \big(\mathrm{Pow}_{\alpha^{(s)}_i,\eta_s,1}(t_i)\big)_{i\in J_s}\in\mathbb{R}^{k_s}.
\]
Next, apply Proposition~\ref{prop:kth_multiplication__SK__polar__precise}
(again with input bound $M=1$) to obtain
\[
    \mathrm{Mult}^{(k_s)}_{\eta_s,1}:\mathbb{R}^{k_s}\to\mathbb{R}
\]
approximating the product on $[-1,1]^{k_s}$.
Set 
$
        \widehat{m}_s(t)
    \eqdef
       \mathrm{Mult}^{(k_s)}_{\eta_s,1}\big(u^{(s)}(t)\big)
,
$
with the convention $\widehat{m}_s(t)\eqdef 1$ if $k_s=0$.
Since the above construction of $\widehat{m}_s$ is a finite composition of continuous maps on a compact domain, there exists
$\eta_s=\eta_s(\varepsilon_s,k_s,\alpha^{(s)})>0$ such that
\[
    \sup_{|t|\le 1}\,
    \Big|
        \widehat{m}_s(t)-\prod_{i\in J_s} t_i^{\alpha^{(s)}_i}
    \Big|
    \le
    \varepsilon_s
.
\]
Finally define $
    \Phi_{p,\varepsilon,M}(x)
    \eqdef
    \sum_{s=1}^S \gamma_s\,\widehat{m}_s(x/M)
$.  
Then, $
\sup_{|x|\le M}|\Phi_{p,\varepsilon,M}(x)-p(x)|
\le
\sum_{s=1}^S |\gamma_s|\,
\sup_{|t|\le 1}\Big|\widehat{m}_s(t)-\prod_{i\in J_s} t_i^{\alpha^{(s)}_i}\Big|
\le
\sum_{s=1}^S |\gamma_s|\,\varepsilon_s
=
\varepsilon
$.

The complexity bounds follow by parallelizing all $S$ monomial branches:
each branch $s$ consists of (i) a parallel block of $k_s$ power gadgets (depth $3L_s$), followed by
(ii) one $k_s$-fold multiplication tree (depth $\le 3\lceil\log_2(\max\{k_s,1\})\rceil+2$),
together with one shared affine rescaling $x\mapsto x/M$ and one final affine output sum.
By Corollary~\ref{cor:kth_power_gadget__SK__polar__precise}, the parallel power block in branch $s$ contributes width and non-zero parameters of order $\mathcal{O}(|\alpha^{(s)}|)$, while Proposition~\ref{prop:kth_multiplication__SK__polar__precise} contributes $\mathcal{O}(k_s)=\mathcal{O}(I_s)$ for the final multiplication tree. Summing over $s\in[S]$ gives the stated width and non-zero parameter bounds.
\end{proof}

Proposition~\ref{prop:sparse_poly_gate__Pow_and_TreeMult} automatically gives us the three remaining elementary algebraic gates which we consider herein.
\begin{cor}[$\mathcal{A}NN$ Emulation of Constant Gates]
\label{cor:constant_gate}
Fix $c\in\mathbb{R}$.
For every $M>0$ there exists an $\mathcal{A}$NN
$
    \Phi_{c,M}:\mathbb{R}\to\mathbb{R}
$
satisfying
\[
    \sup_{|x|\le M}\,|\Phi_{c,M}(x)-c|
    =
    0
.
\]
Moreover, $\Phi_{c,M}$ may be chosen with:
\begin{enumerate}
    \item[(i)] \textbf{Depth:} $1$,
    \item[(ii)] \textbf{Width:} $1$,
    \item[(iii)] \textbf{Non-Zero Parameters:} $1$.
\end{enumerate}
\end{cor}
\begin{proof}
Take $\Phi_{c,M}(x)\eqdef c$ as an affine layer.
\end{proof}

\begin{cor}[$\mathcal{A}NN$ Emulation of Binary Addition Gates]
\label{cor:addition_gate}
For every $M>0$ there exists an $\mathcal{A}$NN
$
    \Phi_{+,M}:\mathbb{R}^2\to\mathbb{R}
$
satisfying
\[
    \sup_{|x|\le M,\,|y|\le M}\,|\Phi_{+,M}(x,y)-(x+y)|
    =
    0
.
\]
Moreover, $\Phi_{+,M}$ may be chosen with depth $1$, width $2$, and size $2$.
\end{cor}
\begin{proof}
Take $\Phi_{+,M}(x,y)\eqdef x+y$ as an affine layer.
\end{proof}

\begin{cor}[$\mathcal{A}NN$ emulation of binary subtraction gates]
	\label{cor:(subtraction_gate)}
	For every $M>0$ there exists an $\mathcal{A}$NN
	$
		\Phi_{-,M}:\mathbb{R}^2\to\mathbb{R}
	$
	satisfying
	\[
		\sup_{|x|\le M, |y|\le M}\lvert \Phi_{-,M}(x,y)-(x-y)\rvert
		=
		0
		.
	\]
	Moreover, $\Phi_{-,M}$ may be chosen with depth $1$, width $2$, and $2$ non-zero parameters.
\end{cor}
\begin{proof}
	Take $\Phi_{-,M}(x,y)\eqdef x-y$ as an affine layer.
\end{proof}

\subsubsection{Inversion}
We will now prove that the $\mathcal{A}$NNs we consider in this paper have the capacity to emulate division.  We will be encoding the following Newton iteration onto our network layers, in Algorithm~\ref{algo:rootfinder}, into an expert sub-network/agent.

\begin{algorithm}[H]
\caption{Newton Inverse Iteration on $K_{m,M}$}
\label{algo:rootfinder}
\KwIn{Fix $K\in\mathbb{N}_+$, parameters $0<m<M$, and $x\in [-M,-m]\cup [m,M]$.}
\KwOut{$y_K(x)$ approximating $1/x$}

$c \eqdef \frac{2}{m^2+M^2}$\;
$y \eqdef c\,x$\;

\For{$j=1,\dots,K-1$}{
    $y \eqdef y\,\big(2-x\,y\big)$\;
}

\Return{$y$}\;
\end{algorithm}
\begin{proposition}[{$\mathcal{A}NN$ Computation of Inversion Algorithm~\ref{algo:rootfinder}}]
\label{prop:inversion_ANNs}
For every approximation error $0<\varepsilon<1$ and annulus hyperparameters $0<m<2\le M$, there exists an $\mathcal{A}$NN
$
    \mathrm{Inv}_{m,M:\varepsilon}:\mathbb{R}\to \mathbb{R}
$
satisfying
\[
    \sup_{m\le |z|\le M}\,
        \big|
            \mathrm{Inv}_{m,M:\varepsilon}(z)
            -
            \tfrac{1}{z}
        \big|
    \le
        \varepsilon
.
\]
Moreover, $\mathrm{Inv}_{m,M:\varepsilon}$ may be chosen with:
\begin{enumerate}
    \item[(i)] \textbf{Depth:} at most
    $
        2+18(K-1)
    $
    ,
    \item[(ii)] \textbf{Width:} at most
    $
        12\max\{m,n\}+8
    $,

    \item[(iii)] \textbf{Non-Zero Parameters:} at most
    $
        2(K-1)\big(8\|x^\star\|_0+39\big)
        +
        30K
    ,
    $
\end{enumerate}
where $K\eqdef \bigg\lceil
                1
                +
                \log_2 \Big(
                    \frac{\log \big(2/(m\varepsilon)\big)}{\log(1/q)}
                \Big)
            \bigg\rceil$ and $0<q\eqdef \frac{M^2-m^2}{M^2+m^2} < 1$.
\end{proposition}

\paragraph{{Proof of Proposition~\ref{prop:inversion_ANNs}}.} We start by proving an auxiliary result.

\begin{lem}[Iterative (Local) Uniform Approximation of the Multiplicative Inverse]
\label{lem:inversion_algorithm}
Fix $0<m<M$, and consider the compact annulus
\begin{equation}
\label{eq:annulus}
        K_{m,M}
    \eqdef 
        [-M,-m]\cup[m,M]
.
\end{equation}
Define the constants $0<c\eqdef \frac{2}{m^2+M^2}$ (depending only on $m$ and on $M$) and $
0<q\eqdef \frac{M^2-m^2}{M^2+m^2}<1.
$
For each $k\in\mathbb{N}$, define the functions $y_k:\mathbb{R}\to\mathbb{R}$ recursively by
\begin{equation}
\label{eq:recursion}
    y_1(x)\eqdef c\,x,
    \qquad
    y_{k+1}(x)\eqdef y_k(x)\,\big(2-x\,y_k(x)\big),
    \qquad x\in K_{m,M}.
\end{equation}
For every $k\in \mathbb{N}$ we have 
\begin{equation}
\label{eq:uniform_estimation__inverstion}
    \sup_{x\in K_{m,M}}\,
        \Big|
            y_k(x)-\frac{1}{x}
        \Big|
\le
        \frac{q^{2^{k-1}}}{m}
.
\end{equation}
For any $\varepsilon\in(0,1)$, if either
$
    m\varepsilon\ge 1
    \text{ and }
    K\ge 1,
$
or
$
    m\varepsilon<1
    \text{ and }
    K \ge
    1+
    \Bigg\lceil
    \log_2 \Big(
    \frac{\log(1/(m\varepsilon))}{\log(1/q)}
    \Big)
    \Bigg\rceil,
$
then, $\sup_{x\in K_{m,M}}\,\big|y_K(x)-\frac{1}{x}\big|\le \varepsilon$.
\end{lem}

\begin{proof}[{Proof of Lemma~\ref{lem:inversion_algorithm}}]
For each $k\in\mathbb{N}$ and $x\in K_{m,M}$ define the multiplicative residual
\[
    e_k(x)\eqdef 1-x\,y_k(x)
.
\]
Thus, for every $k\in\mathbb{N}$ and $x\in K_{m,M}$, by \eqref{eq:recursion} we have
\[
    e_{k+1}(x)
    =
    1-x\,y_{k+1}(x)
    =
    1-x\,y_k(x)\,\big(2-x\,y_k(x)\big)
    =
    \big(1-x\,y_k(x)\big)^2
    =
    e_k(x)^2
.
\]
Next, we bound the initial residual.  Since $y_1(x)=c\,x$, we have
\[
    e_1(x)=1-cx^2
.
\]
As $x\in K_{m,M}$ implies $x^2\in[m^2,M^2]$, and since $c=2/(m^2+M^2)$, it follows that
\[
    \sup_{x\in K_{m,M}}|e_1(x)|
    =
    \max\Big\{
        \Big|1-\frac{2m^2}{m^2+M^2}\Big|,
        \Big|1-\frac{2M^2}{m^2+M^2}\Big|
    \Big\}
    =
    \frac{M^2-m^2}{M^2+m^2}
    =
    q
.
\]
Combining this with $e_{k+1}(x)=e_k(x)^2$ yields, for every $k\in\mathbb{N}$,
\[
    \sup_{x\in K_{m,M}}|e_k(x)|
\le
    \biggl(
        \sup_{x\in K_{m,M}}|e_1(x)|
    \biggr)^{2^{k-1}}
    =
    q^{2^{k-1}}
.
\]
Finally, for any $x\in K_{m,M}$ we may write
$
    y_k(x)-\frac{1}{x}
=
    \frac{x\,y_k(x)-1}{x}
=
    -\frac{e_k(x)}{x}
$.
Now, since $|x|\ge m$ on $K_{m,M}$, then 
$
    \Big|y_k(x)-\frac{1}{x}\Big|
\le
    \frac{|e_k(x)|}{m}
$.  Taking suprema yields
\[
    \sup_{x\in K_{m,M}}
        \Big|
            y_k(x)-\frac{1}{x}
        \Big|
    \le
        \frac{q^{2^{k-1}}}{m}
.
\]
Thus the estimate in~\eqref{eq:uniform_estimation__inverstion} holds.  
For the final claim, 
consider $\varepsilon\in (0,1)$.  If $m\varepsilon\ge 1$, then any $K\ge 1$ yields
$\sup_{x\in K_{m,M}}\,\big|y_K(x)-\frac{1}{x}\big|\le \varepsilon$, since $q^{2^{K-1}}\le 1$ and thus $q^{2^{K-1}}/m\le 1/m\le \varepsilon$; otherwise 
choose $K\in\mathbb{N}$ such that $q^{2^{K-1}}/m\le\varepsilon$, which is equivalent to
\begin{equation}
\label{eq:tweaked_bound}
    2^{K-1}
\ge 
    \frac{\log(1/(m\varepsilon))}{\log(1/q)}
,
\end{equation}
and hence \eqref{eq:tweaked_bound} implies the bound on $K$, since we are considering the case where $m\varepsilon<1$.
\end{proof}

We will now deduce that $\mathcal{A}$NNs can efficiently approximate the inverse function on a compact annulus not containing the origin by encoding the previous Newton iteration onto its layers. 

\begin{proof}[{Proof of Proposition~\ref{prop:inversion_ANNs}}]
Let $0<\delta\le 1$ be a parameter to be chosen retroactively.
Fix $0<m<2\le M$ and $0<\varepsilon<1$, and set
$0<
    c\eqdef \frac{2}{m^2+M^2}
$ and $
    0<q\eqdef \frac{M^2-m^2}{M^2+m^2}<1$.  
Let $(y_k)_{k\ge 1}$ denote iterates from Lemma~\ref{lem:inversion_algorithm}, i.e.,
\[
    y_1(z)\eqdef c\,z,
    \qquad
    y_{k+1}(z)\eqdef y_k(z)\,\big(2-z\,y_k(z)\big),
    \qquad m\le |z|\le M
.
\]
Let $K$ be defined by $q^{2^{K-1}}/m\le \varepsilon/2$, so that the uniform estimate
$
        \sup_{m\le |z|\le M}\,
            \big|
                y_K(z)-\tfrac{1}{z}
            \big|
    \le
        \frac{\varepsilon}{2}
$ holds.  
Define $
    B\eqdef \frac{3}{m}
$ and $
    \bar{M}\eqdef 
         3+MB
    =
        3+\frac{3M}{m}
$.  
By Lemma~\ref{lem:inversion_algorithm}, for every $k\ge 1$ and every $m\le |z|\le M$, we have
\[
    |y_k(z)|
    \le
            \big|\tfrac{1}{z}\big|
        +
            \Big|y_k(z)-\tfrac{1}{z}\Big|
    \le
        \frac{1}{m}+\frac{1}{m}
    =
        \frac{2}{m}
    \le 
        B
.
\]
Let
$
    \mathrm{Mult}^{(2)}_{\delta,\bar{M}}:\mathbb{R}^2\to\mathbb{R}
$
be the $2$-fold multiplication gate from Proposition~\ref{prop:kth_multiplication__SK__polar__precise} (with $k=2$), so that we may use $\varepsilon=\delta$. 
Define the $\mathcal{A}$NN step-map
$
    \mathrm{Step}_\delta:\mathbb{R}^2\to\mathbb{R}^2
$
by
\[
    \mathrm{Step}_\delta(z,y)
    \eqdef
    \Big(
        z,
        \mathrm{Mult}^{(2)}_{\delta,\bar{M}}\big(
            y,
            2-\mathrm{Mult}^{(2)}_{\delta,\bar{M}}(z,y)
        \big)
    \Big)
.
\]
Define also the affine initialization map $\mathrm{Init}:\mathbb{R}\to\mathbb{R}^2$ by
$
    \mathrm{Init}(z)\eqdef (z,cz)
$,
and set
\[
        \mathrm{Inv}_{m,M:\varepsilon}(z)
    \eqdef
        \pi_2
        \circ 
            \mathrm{Step}_\delta^{\circ(K-1)}
        \circ 
            \mathrm{Init}(z)
,
\]
where $\pi_2(z,y)\eqdef y$ and $f^{\circ k}$ denotes the $k$-fold composition of a function $f$ with itself (provided that its domain and codomain are the same).

\medskip
\noindent\textit{Error bound.}
For $m\le |z|\le M$, let $(\widetilde{y}_k(z))_{k\ge 1}$ be defined by
$
    \widetilde{y}_1(z)\eqdef y_1(z)=cz
$
and
$
    (z,\widetilde{y}_{k+1}(z))
    \eqdef
    \mathrm{Step}_\delta(z,\widetilde{y}_k(z))
$.
Then $\mathrm{Inv}_{m,M:\varepsilon}(z)=\widetilde{y}_K(z)$.

Consider the exact Newton update as $F_z(y)\eqdef y(2-zy)$.  By construction of $\mathrm{Step}_\delta$ and the definition of $\mathrm{Mult}^{(2)}_{\delta,\bar{M}}$, we have that
\[
    \widetilde{y}_{k+1}(z)
    =
    F_z\big(\widetilde{y}_k(z)\big)
    +
    r_k(z)
\]
with $|r_k(z)|
    \le
    (B+1)\,\delta$ 
whenever $|\widetilde{y}_k(z)|\le B$ and $m\le |z|\le M$
\footnote{Here we used $|z|\le M$, $|y|\le B$, the bound $|\mathrm{Mult}^{(2)}_{\delta,\bar{M}}(z,y)-zy|\le \delta$, and the choice $\bar{M}=3+MB$ (with $\delta\le 1$) to ensure every argument to $\mathrm{Mult}^{(2)}_{\delta,\bar{M}}$ remains within $[-\bar{M},\bar{M}]$.}.
Moreover, for every $|z|\le M$ and every $|y|\le B$,
$
    |F'_z(y)|
        =   
    |2-2zy|
        \le
    2+2MB
       \eqdef
    L
$.  
Setting $d_k(z)\eqdef|\widetilde{y}_k(z)-y_k(z)|$, we obtain $d_1(z)=0$ and for each $k\in \mathbb{N}_+$ we have that 
$
       d_{k+1}(z)
    \le
        L\,d_k(z) + (B+1)\delta
.
$

We claim that, for every $k\in\{1,\dots,K\}$ and every $m\le |z|\le M$, one has
$|\widetilde{y}_k(z)|\le B$.
We prove this by induction on $k$.
For $k=1$, since $\widetilde{y}_1(z)=y_1(z)$, the bound follows from the estimate on $|y_k(z)|$ above.
Now assume $|\widetilde{y}_k(z)|\le B$.
Then the recurrence for $d_{k+1}(z)$ is valid, and hence
$
    d_{k+1}(z)
    \le
    (B+1)\delta\sum_{j=0}^{k-1}L^j
    \le
    (B+1)\delta\,\frac{L^{K-1}}{L-1}
$.
Upon choosing $\delta$ so that $
    \delta
    \le
    \frac{1}{m}\frac{L-1}{(B+1)L^{K-1}}
$,
we obtain $d_{k+1}(z)\le 1/m$.
Since $|y_{k+1}(z)|\le 2/m$, it follows that
$
       |\widetilde{y}_{k+1}(z)|
    \le
        |y_{k+1}(z)|+d_{k+1}(z)
    \le
        \frac{2}{m}+\frac{1}{m}
    =
        B
$; completing the induction.
\noindent
Consequently, for every $k\in \mathbb{N}_+$, we have that
\[
    d_K(z)
    \le
    (B+1)\delta\sum_{j=0}^{K-2}L^j
    \le
    (B+1)\delta\,\frac{L^{K-1}}{L-1}
.
\]
Upon retroactively setting $\delta>0$ so as to satisfy
$
        \delta
    \le
    \min\Big\{
        1,\,
        \frac{1}{m}
        \frac{L-1}{(B+1)L^{K-1}}
        ,\,
        \frac{\varepsilon}{2}\frac{L-1}{(B+1)L^{K-1}}
    \Big\}
,
$ we find that
\begin{equation}
\label{eq:retroactive_bound}
    \sup_{m\le |z|\le M} \,
        d_K(z)
\le 
    \tfrac{\varepsilon}{2}
.
\end{equation}
Combining~\eqref{eq:retroactive_bound} with the choice of $K$ yields, namely, 
$
K=\bigg\lceil
                1
                +
                \log_2\Big(
                    \frac{\log\big(2/(m\varepsilon)\big)}{\log(1/q)}
                \Big)
            \bigg\rceil
$ we find that
for all $m\le |z|\le M$, we have the desired bound since
\[
        \Big|\mathrm{Inv}_{m,M:\varepsilon}(z)-\tfrac{1}{z}\Big|
    \le
        \Big|\widetilde{y}_K(z)-y_K(z)\Big|
        +
        \Big|y_K(z)-\tfrac{1}{z}\Big|
    \le
        \frac{\varepsilon}{2}+\frac{\varepsilon}{2}
    =
        \varepsilon
.
\]

It only remains to tally the parametric complexity of our $\mathcal{A}$NN approximate implementation.  
Each application of $\mathrm{Step}_\delta$ uses two instances of $\mathrm{Mult}^{(2)}_{\delta,\bar{M}}$ plus affine wiring.
By Proposition~\ref{prop:kth_multiplication__SK__polar__precise} with $k=2$, each $\mathrm{Mult}^{(2)}_{\delta,\bar{M}}$ may be chosen with depth $\le 8$, width $\le 12\max\{m,n\}+2$, and non-zero parameters $\le 8\|x^\star\|_0+39$.
Therefore each $\mathrm{Step}_\delta$ may be chosen with depth at most $18$, width at most $12\max\{m,n\}+8$, and at most $2(8\|x^\star\|_0+39)+30$ non-zero parameters.
Stacking $(K-1)$ steps and accounting for affine initialization/projection yields the stated bounds.
\end{proof}

\begin{cor}[$\mathcal{A}NN$ Emulation of Binary Division]
	\label{cor:division_gate}
	For every $0<\varepsilon<1$ and $0<m<2\le M$ there exists an $\mathcal{A}$NN
	$
		\mathrm{Div}_{m,M:\varepsilon}:\mathbb{R}^2\to\mathbb{R}
	$
	satisfying
	\[
		\sup_{|x|\le M, m\le |y|\le M}
		\left\lvert
		\mathrm{Div}_{m,M:\varepsilon}(x,y)-\frac{x}{y}
		\right\rvert
		\le
		\varepsilon
		.
	\]
	Moreover, $\mathrm{Div}_{m,M:\varepsilon}$ may be chosen with the same asymptotic
	depth, width, and non-zero parameter complexity as in Proposition~\ref{prop:inversion_ANNs}.
\end{cor}

\begin{proof}
	Set
	\[
		\delta\eqdef \min\left\{\frac{\varepsilon}{2M},\frac{1}{m}\right\}
		\text{ and }
		\bar{M}\eqdef \max\left\{M,\frac{2}{m}\right\}.
	\]
	Let
	$
		\mathrm{Inv}_{m,M:\delta}
	$
	be as in Proposition~\ref{prop:inversion_ANNs}, and let
	$
		\mathrm{Mult}^{(2)}_{\varepsilon/2,\bar{M}}
	$
	be the $2$-fold multiplication gate from Proposition~\ref{prop:kth_multiplication__SK__polar__precise}.
	Define
	\[
		\mathrm{Div}_{m,M:\varepsilon}(x,y)
		\eqdef
		\mathrm{Mult}^{(2)}_{\varepsilon/2,\bar{M}}
		(
		x,
		\mathrm{Inv}_{m,M:\delta}(y)
		)
		.
	\]
	For $m\le |y|\le M$, Proposition~\ref{prop:inversion_ANNs} gives
	\[
		\lvert
		\mathrm{Inv}_{m,M:\delta}(y)-\frac{1}{y}
		\rvert
		\le
		\delta
		.
	\]
	Hence
	$
		\lvert
		\mathrm{Inv}_{m,M:\delta}(y)
		\rvert
		\le
		\frac{1}{m}+\delta
		\le
		\frac{2}{m}
		\le
		\bar{M},
	$; thus, the multiplication gate is applied on its validity domain.
	Therefore, for $|x|\le M$ and $m\le |y|\le M$, and so
	\[
		\begin{aligned}
			\lvert
			\mathrm{Div}_{m,M:\varepsilon}(x,y)-\frac{x}{y}
			\rvert
			 & \le
			\lvert
			\mathrm{Mult}^{(2)}_{\varepsilon/2,\bar{M}}
			(x,\mathrm{Inv}_{m,M:\delta}(y))
			-
			x\,\mathrm{Inv}_{m,M:\delta}(y)
			\rvert
			+
			|x|
			\lvert
			\mathrm{Inv}_{m,M:\delta}(y)-\frac{1}{y}
			\rvert
			\\
			 & \le
			\frac{\varepsilon}{2}
			+
			M\delta
			\le
			\varepsilon
			.
		\end{aligned}
	\]
	The complexity claim follows by composing one inversion block with one binary multiplication block.
\end{proof}

\subsubsection{Radicals}
\label{s:Radicals}
We are now prepared to implement the radical operation. This will allow us to unlock several approximate algorithms, especially the ability to approximately emulate the tropical operations. We show that $\mathcal{A}$NNs can approximately compute the Newton iterations for radicals given in Algorithm~\ref{algo:radialization}. This procedure calls Algorithm~\ref{algo:rootfinder} as a subroutine; equivalently, the resulting ``expert/agent'' network incorporates the previously constructed inversion network as a sub-network.

\paragraph{Proof of Proposition~\ref{prop:radicals_ANNs}.} We start by proving an auxiliary lemma.

\begin{lem}[Iterative Uniform Approximation of Radicals]
\label{lem:radialization}
Fix $\ell\in\mathbb{N}_+$ with $\ell\ge 2$.
Fix an approximation error $0<\varepsilon<1$ and annulus hyperparameters $0<m<2\le M$.
Define
$
    m_{\mathrm{inv}}
    \eqdef
        m^{(\ell-1)/\ell}
$ and $
    M_{\mathrm{inv}}
    \eqdef
        \Big(
            2M+1+\frac{M}{m^{(\ell-1)/\ell}}
        \Big)^{\ell-1}
$
, and set
$
    0<q\eqdef \sqrt{\frac{M_{\mathrm{inv}}^2-m_{\mathrm{inv}}^2}{M_{\mathrm{inv}}^2+m_{\mathrm{inv}}^2}}<1
$
.  
Let $K_{\mathrm{inv}}\in\mathbb{N}$ satisfy
$
K_{\mathrm{inv}}
\ge  
    \log_2 \Big(
        \frac{\log(4M/(m_{\mathrm{inv}}\varepsilon))}{\log(1/q)}
    \Big)
$ and let $y_{K_{\mathrm{inv}}}:\mathbb{R}\to\mathbb{R}
$ be as in~\eqref{eq:recursion}, i.e., 
satisfying
\begin{equation}
\label{eq:inv_accuracy_eps_over_2}
    \sup_{z\in 
    [m_{\mathrm{inv}},M_{\mathrm{inv}}]
    }\,
        \Big|
            y_{K_{\mathrm{inv}}}(z)-\frac{1}{z}
        \Big|
    \le
        \tfrac{\varepsilon}{4M}
.
\end{equation}
Let $
    \mathrm{Inv}_{\varepsilon}(z)
    \eqdef
        y_{K_{\mathrm{inv}}}(z)+\frac{\varepsilon}{4M}
$ and 
define $s_0(x)\eqdef M+1$ and for every $k\in \mathbb{N}_+$ define the surrogate-Newton iterates recursively $s_k(x)$, for each $x\in [m,M]$ by
\begin{equation}
\label{eq:sqrt_inexact_newton_iterates}
    s_{k+1}(x)
    \eqdef
        \frac{1}{\ell}\Big(
            (\ell-1)\,s_k(x)
            +
            x\,\mathrm{Inv}_{\varepsilon}\big(s_k(x)^{\ell-1}\big)
        \Big)
.
\end{equation}
Then, the following holds.
\begin{enumerate}
    \item[(i)]
    For every $k\in\mathbb{N}_0$ and $x\in[m,M]$,
    \begin{equation}
    \label{eq:sqrt_iterates_in_domain}
        m^{1/\ell}
        \le
            s_k(x)
        \le
            M_{\mathrm{inv}}^{1/(\ell-1)}
    .
    \end{equation}

    \item[(ii)] %
    If $k\in\mathbb{N}_0$ satisfies $
        k
        \ge
            \frac{
                \log_2 \big(
                    \tfrac{2(M+1-m^{1/\ell})}{\varepsilon}
                \big)
            }{
                \log_2 \big(
                    \tfrac{\ell}{\ell-1}
                \big)
            }
    $
    then
    $
        \sup_{x\in[m,M]}\,
            \big|s_k(x)-x^{1/\ell}\big|
        \le
            \varepsilon
    $.
\end{enumerate}
\end{lem}

\begin{proof}[{Proof of Lemma~\ref{lem:radialization}}]
For each $x\in[m,M]$, set $r\eqdef x^{1/\ell}$.
For each $k\in\mathbb{N}$, abbreviate the residual approximation error by
$
    w_k(x)\eqdef s_k(x)-r
$.
We begin by recording the one-sided reciprocal error on the inversion domain.
By \eqref{eq:inv_accuracy_eps_over_2} and the definition of $\mathrm{Inv}_{\varepsilon}$ we have, for every $z\in[m_{\mathrm{inv}},M_{\mathrm{inv}}]$,
\begin{equation}
\label{eq:Inv_eps_one_sided__proof}
    0
    \le
        \mathrm{Inv}_{\varepsilon}(z)-\frac{1}{z}
    \le
        \frac{\varepsilon}{2M}
.
\end{equation}
We argue by induction on $k$.
Since $s_0(x)=M+1\ge r\ge m^{1/\ell}$, the lower bound holds for $k=0$.
Assume now that $m^{1/\ell}\le s_k(x)\le M_{\mathrm{inv}}^{1/(\ell-1)}$ for some $k\in\mathbb{N}_0$.
Then \eqref{eq:Inv_eps_one_sided__proof} gives $\mathrm{Inv}_{\varepsilon}(s_k(x)^{\ell-1})\ge 1/s_k(x)^{\ell-1}$, and hence
\[
\begin{aligned}
        s_{k+1}(x)
    & =
        \frac{1}{\ell}\Big(
            (\ell-1)\,s_k(x)
            +
            x\,\mathrm{Inv}_{\varepsilon}\big(s_k(x)^{\ell-1}\big)
        \Big)
\\
    & \ge
        \frac{1}{\ell}\Big(
            (\ell-1)\,s_k(x)
            +
            \frac{x}{s_k(x)^{\ell-1}}
        \Big)
\\
    &\ge
        x^{1/\ell}
\\
    & \ge
        m^{1/\ell}
,
\end{aligned}
\]
where the penultimate inequality follows from the weighted AM-GM inequality; e.g., as implied by the 
Brascamp-Lieb inequality%
\footnote{To see this, take $f_j(x)\eqdef e^{-\pi\,\alpha_j x^2}$ with $\alpha_j>0$ and weights $\lambda_j\ge 0$ satisfying $\sum_{j=1}^m\lambda_j=1$.  Then
$
\prod_{j=1}^m f_j(x)^{\lambda_j}
=
\exp \Big(
-\pi\big(\sum_{j=1}^m\lambda_j\alpha_j\big)x^2
\Big)
$,
so
$
\int_{\mathbb{R}}\prod_{j=1}^m f_j(x)^{\lambda_j}\,dx
=
\big(\sum_{j=1}^m\lambda_j\alpha_j\big)^{-1/2}
$,
while $\int_{\mathbb{R}} f_j(x)\,dx=\alpha_j^{-1/2}$.  Thus the Brascamp-Lieb case with $B_j=\mathrm{id}_{\mathbb{R}}$ gives
$
(\sum_{j=1}^m\lambda_j\alpha_j)^{-1/2}
\le
\prod_{j=1}^m(\alpha_j^{-1/2})^{\lambda_j}
$,
equivalently
$
\sum_{j=1}^m\lambda_j\alpha_j
\ge
\prod_{j=1}^m \alpha_j^{\lambda_j}
$.
}~%
cf.~\citep[Equation (1)]{bennett2008brascamp}.
For the upper bound, \eqref{eq:Inv_eps_one_sided__proof} yields
$\mathrm{Inv}_{\varepsilon}(s_k(x)^{\ell-1})\le 1/s_k(x)^{\ell-1}+\varepsilon/(2M)$,
and thus, using $x\le M$ and $s_k(x)\ge m^{1/\ell}$,
\[
\begin{aligned}
    s_{k+1}(x)
    & \le
        \frac{1}{\ell}\Big(
            (\ell-1)\,s_k(x)
            +
            \frac{x}{s_k(x)^{\ell-1}}
        \Big)
        +
        \frac{1}{\ell}x\frac{\varepsilon}{2M}
\\
    & \le
        \frac{\ell-1}{\ell}s_k(x)
        +
        \frac{1}{\ell}\frac{M}{m^{(\ell-1)/\ell}}
        +
        \frac{1}{2\ell}
\\
    & \le
        M_{\mathrm{inv}}^{1/(\ell-1)}
,
\end{aligned}
\]
where the last inequality follows from $s_k(x)\le M_{\mathrm{inv}}^{1/(\ell-1)}$, $\varepsilon<1$, and the definition of $M_{\mathrm{inv}}$; showing (i).
Next, using \eqref{eq:sqrt_inexact_newton_iterates} and adding and subtracting $x/s_k(x)^{\ell-1}$, we obtain
\[
    w_{k+1}(x)
    =
        \Big(
            \mathsf{N}_x\big(s_k(x)\big)-r
        \Big)
        +
        \frac{1}{\ell}x\Big(
            \mathrm{Inv}_{\varepsilon}\big(s_k(x)^{\ell-1}\big)-\frac{1}{s_k(x)^{\ell-1}}
        \Big)
,
\]
where
$
    \mathsf{N}_x(s)
\eqdef
    \frac{1}{\ell}\Big(
        (\ell-1)\,s+\frac{x}{s^{\ell-1}}
    \Big)
$
for $s>0$.
Since $\mathsf{N}_x(r)=r$ and since 
\hfill\\
$
    \mathsf{N}_x'(s)
    =
        \frac{\ell-1}{\ell}\Big(
            1-\frac{x}{s^\ell}
        \Big)
$
for all $s\ge r$, then we have that
\[
    0
    \le
        \mathsf{N}_x\big(s_k(x)\big)-r
    \le
        \frac{\ell-1}{\ell}\,w_k(x)
.
\]
Moreover, by \eqref{eq:Inv_eps_one_sided__proof} we have
$
0\le \mathrm{Inv}_{\varepsilon}(s_k(x)^{\ell-1})-\frac{1}{s_k(x)^{\ell-1}}\le \varepsilon/(2M)
$,
and hence
\[
        w_{k+1}(x)
    \le
        \frac{\ell-1}{\ell}\,w_k(x)
        +
        \frac{1}{\ell}x\frac{\varepsilon}{2M}
    \le
        \frac{\ell-1}{\ell}\,w_k(x)
        +
        \frac{\varepsilon}{2\ell}
.
\]
Thus
\begin{equation}
\label{eq:w_linear_recursion__radialization}
    w_{k+1}(x)
    \le
        \frac{\ell-1}{\ell}w_k(x)
        +
        \frac{\varepsilon}{2\ell}
.
\end{equation}
Iterating \eqref{eq:w_linear_recursion__radialization} yields
$
    w_k(x)
    \le
        \big(\tfrac{\ell-1}{\ell}\big)^k\,w_0(x)
        +
        \big(1-\big(\tfrac{\ell-1}{\ell}\big)^k\big)\frac{\varepsilon}{2}
    \le
        \big(\tfrac{\ell-1}{\ell}\big)^k\,(M+1-m^{1/\ell})
        +
        \frac{\varepsilon}{2}
$
; where we used the fact that $w_0(x)=M+1-x^{1/\ell}\le M+1-m^{1/\ell}$.
Taking suprema over $x\in[m,M]$ yields
\begin{equation}
\label{eq:le_bound}
    \sup_{x\in[m,M]}\,
        \big|s_k(x)-x^{1/\ell}\big|
    \le
        \big(\tfrac{\ell-1}{\ell}\big)^k\,(M+1-m^{1/\ell})
        +
        \tfrac{\varepsilon}{2}
.
\end{equation}
In particular, if $k\in\mathbb{N}$ satisfies
$
\big(\tfrac{\ell-1}{\ell}\big)^k\,(M+1-m^{1/\ell})\le \varepsilon/2
$,
then
$
\sup_{x\in[m,M]}\,\big|s_k(x)-x^{1/\ell}\big|\le \varepsilon
$.
Note that it is sufficient that
\[
    k
    \ge
        \frac{
            \log_2 \big(
                \tfrac{2(M+1-m^{1/\ell})}{\varepsilon}
            \big)
        }{
            \log_2 \big(
                \tfrac{\ell}{\ell-1}
            \big)
        }
\]
for this to hold, which shows (ii).
\end{proof}

We are now ready to show that $\mathcal{A}$NNs can implement the square-root operation on a compact sub-interval of the positive real line.  
Our proof is to emulate the algorithm in Lemma~\ref{lem:radialization}.
\begin{proof}[Proof of Proposition~\ref{prop:radicals_ANNs}]
Fix $\ell\in\mathbb{N}_+$ with $\ell\ge 2$.
Fix $0<m<2\le M$ and $0<\varepsilon<1$.
Set
$
a_{\mathrm{inv}}\eqdef m^{(\ell-1)/\ell}
$, $
m_{\mathrm{inv}}\eqdef \tfrac{1}{2}a_{\mathrm{inv}}
$, 
and
$
M_{\mathrm{inv}}\eqdef \Big(
            2M+1+\frac{M}{a_{\mathrm{inv}}}
        \Big)^{\ell-1}
$.

\noindent
Let
$
    \operatorname{Inv}_{m_{\mathrm{inv}},M_{\mathrm{inv}}:\delta_{\mathrm{inv}}}
$
be as in Proposition~\ref{prop:inversion_ANNs}, with
$
\delta_{\mathrm{inv}}\eqdef \varepsilon/(8M)
$.
Define the shifted inverse gate
$
       \widehat{\operatorname{Inv}}_{\varepsilon}(z)
    \eqdef
        \operatorname{Inv}_{m_{\mathrm{inv}},M_{\mathrm{inv}}:\delta_{\mathrm{inv}}}(z)
    +
        \delta_{\mathrm{inv}}
$.
Then, for every $z\in[m_{\mathrm{inv}},M_{\mathrm{inv}}]$
\begin{equation}
\label{eq:Inv_one_sided__Rad_prop}
    0
    \le
        \widehat{\operatorname{Inv}}_{\varepsilon}(z)-\frac{1}{z}
    \le
        2\delta_{\mathrm{inv}}
    =
        \frac{\varepsilon}{4M}
.
\end{equation}
Next, set $\delta_{\mathrm{mult}}\eqdef \varepsilon/8$ and let
$
    B_{\mathrm{inv}}
    \eqdef
        \frac{2}{m_{\mathrm{inv}}}+1
$.
Then, let
$
    \mathrm{Mult}^{(2)}_{\delta_{\mathrm{mult}},\,M+B_{\mathrm{inv}}}
$
be the $2$-fold multiplication gate from Proposition~\ref{prop:kth_multiplication__SK__polar__precise}.
Define the shifted product gate
$
       \widehat{\mathrm{Mult}}_{\varepsilon}(u,v)
    \eqdef
        \mathrm{Mult}^{(2)}_{\delta_{\mathrm{mult}},\,M+B_{\mathrm{inv}}}(u,v)
        +
        \delta_{\mathrm{mult}}
$.
Then, for all $|u|\le M$ and $|v|\le B_{\mathrm{inv}}$ we have
\begin{equation}
\label{eq:Mult_one_sided__Rad_prop}
    0
    \le
        \widehat{\mathrm{Mult}}_{\varepsilon}(u,v)-uv
    \le
        2\delta_{\mathrm{mult}}
    =
        \frac{\varepsilon}{4}
.
\end{equation}
Next, set $\delta_{\mathrm{pow}}\eqdef \varepsilon m_{\mathrm{inv}}^2/(8M)$ and let
$
    B_{\mathrm{pow}}
    \eqdef
        M_{\mathrm{inv}}^{1/(\ell-1)}
$.
Then,
$
    \mathrm{Pow}^{(\ell-1)}_{\delta_{\mathrm{pow}},\,B_{\mathrm{pow}}}
$
be the $(\ell-1)$-fold power gate from Corollary~\ref{cor:kth_power_gadget__SK__polar__precise}.
Define the shifted power gate
$
       \widehat{\mathrm{Pow}}_{\varepsilon}(s)
    \eqdef
        \mathrm{Pow}^{(\ell-1)}_{\delta_{\mathrm{pow}},\,B_{\mathrm{pow}}}(s)
        -
        \delta_{\mathrm{pow}}
$.
Then, for all $0\le s\le B_{\mathrm{pow}}$ we have
\begin{equation}
\label{eq:Pow_one_sided__Rad_prop}
    0
    \le
        s^{\ell-1}-\widehat{\mathrm{Pow}}_{\varepsilon}(s)
    \le
        2\delta_{\mathrm{pow}}
    =
        \frac{\varepsilon m_{\mathrm{inv}}^2}{4M}
.
\end{equation}
We are now prepared to emulate the Newton iterations from Lemma~\ref{lem:radialization}.  To this end, define $\mathrm{Init}:\mathbb{R}\to\mathbb{R}^2$ by
$
    \mathrm{Init}(z)
    \eqdef
    (z,M+1)
,
$ and define $\mathrm{Step}:\mathbb{R}^2\to\mathbb{R}^2$ by
$
        \mathrm{Step}(z,s)
    \eqdef
        \Bigg(
            z,
            \frac{1}{\ell}\Big(
                (\ell-1)\,s
                +
                \widehat{\mathrm{Mult}}_{\varepsilon}\big(
                    z,
                    \widehat{\operatorname{Inv}}_{\varepsilon}\big(
                        \widehat{\mathrm{Pow}}_{\varepsilon}(s)
                    \big)
                \big)
            \Big)
        \Bigg)
$.  
Finally set
$
    \mathrm{Rad}_{m,M:\ell,\varepsilon}(z)
    \eqdef
        \pi_2\circ \mathrm{Step}^{\circ K}\circ \mathrm{Init}(z)
$, where $
    \pi_2(z,s)\eqdef s
$.  
We claim that, for every $k\in\{0,\dots,K\}$ and every $z\in[m,M]$ we have
$
        (z,s_k(z))
    \eqdef
        \mathrm{Step}^{\circ k}\circ \mathrm{Init}(z)
$, so that $\mathrm{Rad}_{m,M:\ell,\varepsilon}(z)=s_K(z)$.

We claim that, for every $k\in [K]$ the following holds: for every $z\in[m,M]$ we have
\begin{equation}
\label{eq:s_k_in_domain__Rad_prop}
        m^{1/\ell}
    \le
        s_k(z)
    \le
        B_{\mathrm{pow}}
.
\end{equation}
We proceed by induction.  Clearly, $s_0(z)=M+1\in[m^{1/\ell},B_{\mathrm{pow}}]$.  
Suppose now that~\eqref{eq:s_k_in_domain__Rad_prop} holds for some $k\in [K-1]$.
Then~\eqref{eq:Pow_one_sided__Rad_prop} yields
$\widehat{\mathrm{Pow}}_{\varepsilon}(s_k(z))\le s_k(z)^{\ell-1}$. Moreover, since
$s_k(z)\ge m^{1/\ell}$, we have $s_k(z)^{\ell-1}\ge a_{\mathrm{inv}}=2m_{\mathrm{inv}}$, while
$2\delta_{\mathrm{pow}}=\varepsilon m_{\mathrm{inv}}^2/(4M)\le m_{\mathrm{inv}}$. 
Hence
$\widehat{\mathrm{Pow}}_{\varepsilon}(s_k(z))\ge m_{\mathrm{inv}}$. Also,
$\widehat{\mathrm{Pow}}_{\varepsilon}(s_k(z))\le s_k(z)^{\ell-1}\le B_{\mathrm{pow}}^{\ell-1}=M_{\mathrm{inv}}$, and therefore
\[
        \big|
            \widehat{\operatorname{Inv}}_{\varepsilon}
            \,
            \big(
                \widehat{\mathrm{Pow}}_{\varepsilon}(s_k(z))
            \big)
        \big|
    \le
        \frac{1}{m_{\mathrm{inv}}}
        +
        \frac{\varepsilon}{4M}
    \le
        B_{\mathrm{inv}}
.
\]
Consequently, together~\eqref{eq:Inv_one_sided__Rad_prop} and \eqref{eq:Mult_one_sided__Rad_prop} yield
\begin{equation}
\label{eq:induction_bound__I}
        \widehat{\mathrm{Mult}}_{\varepsilon}\big(
            z,
            \widehat{\operatorname{Inv}}_{\varepsilon}\big(
                \widehat{\mathrm{Pow}}_{\varepsilon}(s_k(z))
            \big)
        \big)
    \ge
        z\cdot \widehat{\operatorname{Inv}}_{\varepsilon}\big(
                \widehat{\mathrm{Pow}}_{\varepsilon}(s_k(z))
            \big)
    \ge
        \frac{z}{s_k(z)^{\ell-1}}
.
\end{equation}
Consequently,~\eqref{eq:induction_bound__I} implies that
\[
    s_{k+1}(z)
    \ge
        \frac{1}{\ell}\Big(
            (\ell-1)\,s_k(z)+\frac{z}{s_k(z)^{\ell-1}}
        \Big)
    \ge
        z^{1/\ell}
    \ge
        m^{1/\ell}
\]
(where we have used the weighted AM--GM inequality in the penultimate inequality),
which establishes the lower-bound in our induction.  
For the upper bound, first note that~\eqref{eq:Pow_one_sided__Rad_prop}, together with
$\widehat{\mathrm{Pow}}_{\varepsilon}(s_k(z))\ge m_{\mathrm{inv}}$ and
$s_k(z)^{\ell-1}\ge a_{\mathrm{inv}}$, gives
\[
    0
    \le
        \frac{1}{\widehat{\mathrm{Pow}}_{\varepsilon}(s_k(z))}
        -
        \frac{1}{s_k(z)^{\ell-1}}
    =
        \frac{
            s_k(z)^{\ell-1}
            -
            \widehat{\mathrm{Pow}}_{\varepsilon}(s_k(z))
        }{
            \widehat{\mathrm{Pow}}_{\varepsilon}(s_k(z))\,s_k(z)^{\ell-1}
        }
    \le
        \frac{2\delta_{\mathrm{pow}}}{m_{\mathrm{inv}}a_{\mathrm{inv}}}
    \le
        \frac{\varepsilon}{4M}
.
\]
Therefore, by~\eqref{eq:Inv_one_sided__Rad_prop},
\[
    \widehat{\operatorname{Inv}}_{\varepsilon}\big(
                \widehat{\mathrm{Pow}}_{\varepsilon}(s_k(z))
            \big)
    \le
        \frac{1}{s_k(z)^{\ell-1}}
        +
        \frac{\varepsilon}{2M}
.
\]
Using $z\le M$ and~\eqref{eq:Mult_one_sided__Rad_prop}, we obtain
\begin{equation}
\label{eq:estimatosku__industoku}
        \widehat{\mathrm{Mult}}_{\varepsilon}\big(
            z,
            \widehat{\operatorname{Inv}}_{\varepsilon}\big(
                \widehat{\mathrm{Pow}}_{\varepsilon}(s_k(z))
            \big)
        \big)
    \le
        \frac{z}{s_k(z)^{\ell-1}}
        +
        \frac{3\varepsilon}{4}
.
\end{equation}
Thus,
\begin{equation}
\label{eq:el_detailos}
        s_{k+1}(z)
    \le
        \frac{1}{\ell}\Big(
            (\ell-1)\,s_k(z)+\frac{z}{s_k(z)^{\ell-1}}
        \Big)
        +
        \frac{3\varepsilon}{4\ell}
    \le
        \frac{\ell-1}{\ell}s_k(z)
        +
        \frac{1}{\ell}\frac{M}{a_{\mathrm{inv}}}
        +
        \frac{3}{4\ell}
    \le
        B_{\mathrm{pow}},
\end{equation}
where the last inequality follows from $s_k(z)\le B_{\mathrm{pow}}$, $\varepsilon<1$, and
$B_{\mathrm{pow}}=2M+1+M/a_{\mathrm{inv}}$.

Having established~\eqref{eq:s_k_in_domain__Rad_prop}, and 
in particular, $|\widehat{\operatorname{Inv}}_{\varepsilon}(\widehat{\mathrm{Pow}}_{\varepsilon}(s_k(z)))|\le B_{\mathrm{inv}}$ for all $k\in [K]$, we are now prepared to enter into our recursive argument.    
Fix $z\in[m,M]$ and set $w_k(z)\eqdef s_k(z)-z^{1/\ell}\ge 0$.
Define
$
    \mathsf{N}_z(s)
\eqdef
    \frac{1}{\ell}\Big(
        (\ell-1)\,s+\frac{z}{s^{\ell-1}}
    \Big)
$
for $s>0$.
Since $\mathsf{N}_z(z^{1/\ell})=z^{1/\ell}$ and
$
    \mathsf{N}_z'(s)
    =
        \frac{\ell-1}{\ell}\Big(
            1-\frac{z}{s^\ell}
        \Big)
    \in
        \big[0,\tfrac{\ell-1}{\ell}\big]
$
for all $s\ge z^{1/\ell}$, we have
\[
    0
    \le
        \mathsf{N}_z\big(s_k(z)\big)-z^{1/\ell}
    \le
        \frac{\ell-1}{\ell}\,w_k(z)
.
\]
Moreover, by adding and subtracting
$z\,\widehat{\operatorname{Inv}}_{\varepsilon}(\widehat{\mathrm{Pow}}_{\varepsilon}(s_k(z)))$,
$z/\widehat{\mathrm{Pow}}_{\varepsilon}(s_k(z))$, and $z/s_k(z)^{\ell-1}$, we compute
\[
    w_{k+1}(z)
    =
        \Big(
            \mathsf{N}_z\big(s_k(z)\big)-z^{1/\ell}
        \Big)
        +
        \frac{1}{\ell}\Big(
            \widehat{\mathrm{Mult}}_{\varepsilon}\big(
                z,
                \widehat{\operatorname{Inv}}_{\varepsilon}\big(
                    \widehat{\mathrm{Pow}}_{\varepsilon}(s_k(z))
                \big)
            \big)
            -
            z\,\widehat{\operatorname{Inv}}_{\varepsilon}\big(
                    \widehat{\mathrm{Pow}}_{\varepsilon}(s_k(z))
                \big)
        \Big)
\]
\[
        +
        \frac{1}{\ell}z\Big(
            \widehat{\operatorname{Inv}}_{\varepsilon}\big(
                    \widehat{\mathrm{Pow}}_{\varepsilon}(s_k(z))
                \big)-\frac{1}{\widehat{\mathrm{Pow}}_{\varepsilon}(s_k(z))}
        \Big)
        +
        \frac{1}{\ell}z\Big(
            \frac{1}{\widehat{\mathrm{Pow}}_{\varepsilon}(s_k(z))}-\frac{1}{s_k(z)^{\ell-1}}
        \Big)
.
\]
By \eqref{eq:Mult_one_sided__Rad_prop} we have
$
0\le \widehat{\mathrm{Mult}}_{\varepsilon}(\cdot,\cdot)-(\cdot)(\cdot)\le \varepsilon/4
$
on the relevant domain, and by \eqref{eq:Inv_one_sided__Rad_prop} we have
$
0\le \widehat{\operatorname{Inv}}_{\varepsilon}(\cdot)-1/(\cdot)\le \varepsilon/(4M)
$
on $[m_{\mathrm{inv}},M_{\mathrm{inv}}]$.
Finally, by \eqref{eq:Pow_one_sided__Rad_prop}, $\widehat{\mathrm{Pow}}_{\varepsilon}(s_k(z))\ge m_{\mathrm{inv}}$, and $s_k(z)^{\ell-1}\ge a_{\mathrm{inv}}$, we have
\[
    0
    \le
        \frac{1}{\widehat{\mathrm{Pow}}_{\varepsilon}(s_k(z))}-\frac{1}{s_k(z)^{\ell-1}}
    =
        \frac{
            s_k(z)^{\ell-1}-\widehat{\mathrm{Pow}}_{\varepsilon}(s_k(z))
        }{
            \widehat{\mathrm{Pow}}_{\varepsilon}(s_k(z))\,s_k(z)^{\ell-1}
        }
    \le
        \frac{\varepsilon}{4M}
.
\]
Thus
\[
    w_{k+1}(z)
    \le
        \frac{\ell-1}{\ell}\,w_k(z)
        +
        \frac{1}{\ell}\frac{\varepsilon}{4}
        +
        \frac{1}{\ell}\frac{\varepsilon}{4}
        +
        \frac{1}{\ell}\frac{\varepsilon}{4}
    =
        \frac{\ell-1}{\ell}\,w_k(z)
        +
        \frac{3\varepsilon}{4\ell}
.
\]
Hence
\begin{equation}
\label{eq:w_linear_recursion__Rad_prop}
    w_{k+1}(z)
    \le
        \frac{\ell-1}{\ell}w_k(z)
        +
        \frac{3\varepsilon}{4\ell}
.
\end{equation}
Iterating \eqref{eq:w_linear_recursion__Rad_prop} yields
\begin{equation}
\label{eq:boundmeupbb}
        w_K(z)
    \le
        \Big(\tfrac{\ell-1}{\ell}\Big)^K w_0(z)
        +
        \bigg(1-\Big(\tfrac{\ell-1}{\ell}\Big)^K\bigg)\frac{3\varepsilon}{4}
    \le
        \Big(\tfrac{\ell-1}{\ell}\Big)^K\,(M+1-m^{1/\ell})
        +
        \frac{3\varepsilon}{4}
.
\end{equation}
Recalling our definition of $K=\bigg\lceil
        \frac{
            \log_2 \Big(
                \frac{
                    4\,(M+1-m^{1/\ell})
                }{
                    \varepsilon
                }
            \Big)
        }{
            \log_2 \big(
                \tfrac{\ell}{\ell-1}
            \big)
        }
\bigg\rceil$, we have $\big(\tfrac{\ell-1}{\ell}\big)^K(M+1-m^{1/\ell})\le \varepsilon/4$.  Consequently,~\eqref{eq:boundmeupbb} implies that the desired uniform approximation bound
\[
    \sup_{z\in[m,M]}\,
        \big|s_K(z)-z^{1/\ell}\big|
    =
    \sup_{z\in[m,M]} w_K(z)
    \le
        \varepsilon
.
\]
This proves the approximation claim.  Denote $\mathrm{Rad}_{m,M:\ell,\varepsilon}\eqdef s_K$.

It now only remains to tally the parametric complexity of our $\mathcal{A}$NN $\mathrm{Rad}_{m,M:\ell,\varepsilon}$.
Each application of $\mathrm{Step}$ uses one copy of the inversion network from Proposition~\ref{prop:inversion_ANNs}, one copy of the $(\ell-1)$-fold power gate from Corollary~\ref{cor:kth_power_gadget__SK__polar__precise}, and one copy of the $2$-fold multiplication gate from Proposition~\ref{prop:kth_multiplication__SK__polar__precise}, plus affine wiring and the scalar shifts by $\delta_{\mathrm{inv}}$, $\delta_{\mathrm{mult}}$, and $\delta_{\mathrm{pow}}$.
The bounds from Proposition~\ref{prop:inversion_ANNs}, Corollary~\ref{cor:kth_power_gadget__SK__polar__precise} with $k=\ell-1$, and Proposition~\ref{prop:kth_multiplication__SK__polar__precise} with $k=2$, together with the existing affine wiring overhead, yield the stated width and non-zero-parameter bounds after stacking $K$ steps.
\end{proof}

\subsubsection{Tropical Operations}

We begin by emulating the binary maximum, as the $k$-ary maximum may be derived from it by the usual tree-based iteration ($NC^0$-type tropical circuit).  
Likewise, the binary minimum is obtained by conjugating the binary maximum with a negation, which can be computed by an $\mathcal{A}$NN.  In order to obtain the maximum we must first emulate the absolute value operation.  This too can be obtained via the elementary identity: for every $x\in \mathbb{R}$
\begin{equation}
\label{eq:elementaryID_abs_sq_rad}
    |x| = \sqrt{x^2}
.
\end{equation}

Next, we exhibit an $\mathcal{A}$NN which approximately computes the left-hand side of~\eqref{eq:elementaryID_abs_sq_rad} by calling the radical and squaring ``expert/agent'' networks as sub-routines; this is performed in the following Algorithm~\ref{algo:absolute_value}.

\begin{algorithm}[H]
\caption{$\mathcal{A}$NN Absolute Value via Squaring and Radicals}
\label{algo:absolute_value}
\KwIn{$x\in[-M,M]$, parameters $M\ge 2$, $0<\varepsilon<1$}
\KwOut{$\widehat a \approx |x|$}

$u \eqdef \textsc{Square}(x)$
\tcp*[r]{{\scriptsize \textsc{Square} is Proposition~\ref{prop:SomwhereNonFlatDefinable_Yields_Squarable}}}

$\widehat a \eqdef \textsc{Radical}(u; \ell=2)$
\tcp*[r]{{\scriptsize \textsc{Radical} is Algorithm~\ref{algo:radialization}}}

\Return{$\widehat a$}\;
\end{algorithm}

Note that, we have previously shown that all of the operations in~\eqref{eq:elementaryID_abs_sq_rad} are approximately (and efficiently) emulatable by $\mathcal{A}$NNs.  Thus, so must the absolute value be; and consequently so must the binary maximum (respectively, minimum) be.

\begin{proposition}[$\mathcal{A}NN$ Emulation of the Absolute Value $|\cdot|$]
\label{prop:abs_ANNs}
For every $0<\varepsilon<1$ and $M\ge 2$ there exists an $\mathcal{A}$NN
$
    \mathrm{Abs}_{M:\varepsilon}:\mathbb{R}\to\mathbb{R}
$
satisfying
\[
    \sup_{|x|\le M}\,
        \big|
            \mathrm{Abs}_{M:\varepsilon}(x)-|x|
        \big|
    \le
        \varepsilon
.
\]
Moreover, $\mathrm{Abs}_{M:\varepsilon}$ may be chosen to have
\begin{enumerate}
    \item[(i)] \textbf{Depth:} at most
    $
        7
        +
        K_{\mathrm{abs}}\,\big(
            12K_{\mathrm{abs,inv}}-3
        \big)
    $,
    \item[(ii)] \textbf{Width:} at most
    $
        40
    $,
    \item[(iii)] \textbf{Non-zero parameters:} at most
    \[
\begin{aligned}
        K_{\mathrm{abs}}\Big(
            2(K_{\mathrm{abs,inv}}-1)\big(8\|x^\star\|_0+41\big)
            +
            30K_{\mathrm{abs,inv}}
            +
            \big(8\|x^\star\|_0+41\big)
            +
            24
        \Big)
    +
        12
        +
        \big(8\|x^\star\|_0+41\big)
        +
        12
    ,
\end{aligned}
    \]
\end{enumerate}
where
$
    \delta_{\mathrm{abs}}
    \eqdef
        \big(\varepsilon/4\big)^2
$, $
    m_{\mathrm{abs}}
    \eqdef
        \delta_{\mathrm{abs}}/2
$, $
    M_{\mathrm{abs}}
    \eqdef
        M^2+2\delta_{\mathrm{abs}}
$, $
    m_{\mathrm{abs,inv}}
    \eqdef
        \frac12\sqrt{m_{\mathrm{abs}}}
$, $
    M_{\mathrm{abs,inv}}
    \eqdef
        2M_{\mathrm{abs}}+1+\frac{M_{\mathrm{abs}}}{m_{\mathrm{abs,inv}}}
$, $0<
    q_{\mathrm{abs,inv}}
    \eqdef
        \frac{M_{\mathrm{abs,inv}}^2-m_{\mathrm{abs,inv}}^2}{M_{\mathrm{abs,inv}}^2+m_{\mathrm{abs,inv}}^2}
<1$, $
    \delta_{\mathrm{abs,inv}}
    \eqdef
        \frac{\varepsilon}{32M_{\mathrm{abs}}}
$, $K_{\mathrm{abs,inv}}
    \eqdef
        \Bigg\lceil
            1
            +
            \log_2 \Big(
                \frac{
                    \log \big(
                        2/(m_{\mathrm{abs,inv}}\delta_{\mathrm{abs,inv}})
                    \big)
                }{
                    \log(1/q_{\mathrm{abs,inv}})
                }
            \Big)
        \Bigg\rceil
$, and $
    K_{\mathrm{abs}}
\eqdef
    \Bigg\lceil
        \log_2 \Big(
            \frac{
                4\,(M_{\mathrm{abs}}+1-\sqrt{m_{\mathrm{abs}}})
            }{
                \varepsilon/4
            }
        \Big)
    \Bigg\rceil
$.
\end{proposition}

\begin{proof}[{Proof of Proposition~\ref{prop:abs_ANNs}}]
Fix $0<\varepsilon<1$ and $M\ge 2$, and set
\[
    \delta\eqdef \Big(\frac{\varepsilon}{4}\Big)^2
    \qquad \mbox{ and } \qquad
    m_{\mathrm{abs}}\eqdef \frac{\delta}{2},
    \qquad
    M_{\mathrm{abs}}\eqdef M^2+2\delta
.
\]
Let
$
    \mathrm{Mult}^{(2)}_{\delta/4,\,M}:\mathbb{R}^2\to\mathbb{R}
$
be as in Proposition~\ref{prop:kth_multiplication__SK__polar__precise} (with $k=2$), and define the approximate square gate
$
       \mathrm{Sq}_{M:\varepsilon}(x)
    \eqdef
        \mathrm{Mult}^{(2)}_{\delta/4,\,M}(x,x)
$.
Then
$
    \sup_{|x|\le M}
    \big|
        \mathrm{Sq}_{M:\varepsilon}(x)-x^2
    \big|
    \le
        \frac{\delta}{4}
$.  
Next, let
$
    \mathrm{Rad}_{m_{\mathrm{abs}},M_{\mathrm{abs}}:\varepsilon/4}
$
be as in Proposition~\ref{prop:radicals_ANNs}.
Define our main $\mathcal{A}$NN of interest herein
\[
        \mathrm{Abs}_{M:\varepsilon}(x)
    \eqdef
        \mathrm{Rad}_{m_{\mathrm{abs}},M_{\mathrm{abs}}:\varepsilon/4}\big(
            \mathrm{Sq}_{M:\varepsilon}(x)+\delta
        \big)
\]
and where we suppress the notational dependence of $\ell=2$ in $\mathrm{Rad}_{m_{\mathrm{abs}},M_{\mathrm{abs}}:2,\varepsilon/4}\eqdef \mathrm{Rad}_{m_{\mathrm{abs}},M_{\mathrm{abs}}:\varepsilon/4}$.
For $|x|\le M$ we have
$
    \mathrm{Sq}_{M:\varepsilon}(x)+\delta
\in
    [\tfrac{3}{4}\delta,\ M^2+\tfrac{5}{4}\delta]
\subset
    [m_{\mathrm{abs}},M_{\mathrm{abs}}]
$,
so the radical gate is applied on its validity domain.
We now decompose the error:
\begin{align}
    \big|
        \mathrm{Abs}_{M:\varepsilon}(x)-|x|
    \big|
\le & 
\label{eq:termuno}
\Big|
    \mathrm{Rad}_{m_{\mathrm{abs}},M_{\mathrm{abs}}:\varepsilon/4}\big(
        \mathrm{Sq}_{M:\varepsilon}(x)+\delta
    \big)
    -
    \sqrt{\mathrm{Sq}_{M:\varepsilon}(x)+\delta}
\Big|
\\
\label{eq:termdos}
& +
    \Big|
        \sqrt{\mathrm{Sq}_{M:\varepsilon}(x)+\delta}
        -
        \sqrt{x^2+\delta}
    \Big|
\\
\label{eq:termtres}
&
    +
    \Big|
        \sqrt{x^2+\delta}
        -
        |x|
    \Big|
.
\end{align}
Term~\eqref{eq:termuno} is at most $\varepsilon/4$ by Proposition~\ref{prop:radicals_ANNs}.
Term~\eqref{eq:termtres} is controlled via the identity
$
\sqrt{x^2+\delta}-|x|
=
\delta/(\sqrt{x^2+\delta}+|x|)
$, since this identity implies that
$
0\le \sqrt{x^2+\delta}-|x|\le \sqrt{\delta}=\varepsilon/4
$.
It remains to control the middle term~\eqref{eq:termdos}.  Note that, since $t\mapsto \sqrt{t}$ is $1/\sqrt{2\delta}$-Lipschitz on $[m_{\mathrm{abs}},\infty)$ and
$
|\mathrm{Sq}_{M:\varepsilon}(x)-x^2|\le \delta/4
$,
we obtain
\begin{equation}
\label{eq:leestimo}
    \Big|
        \sqrt{\mathrm{Sq}_{M:\varepsilon}(x)+\delta}
        -
        \sqrt{x^2+\delta}
    \Big|
\le
    \frac{1}{\sqrt{2\delta}}\cdot \frac{\delta}{4}
=
    \frac{\sqrt{\delta}}{4\sqrt{2}}
<
    \frac{\varepsilon}{4}
.
\end{equation}
Thus our three bounds on terms \eqref{eq:termuno}, \eqref{eq:termdos}, and \eqref{eq:termtres} imply the desired uniform estimate, since
\[
        \sup_{|x|\le M}|\,
        \mathrm{Abs}_{M:\varepsilon}(x)-|x||
    \le 
        \varepsilon
.
\]

It remains to tally the complexity of $\mathrm{Abs}_{M:\varepsilon}$.  To this end, by composing one $2$-fold multiplication gate with one radical gate, plus affine wiring; thus the depth/width/non-zero parameter counts are those of Proposition~\ref{prop:radicals_ANNs} up to additive constants, and hence have the stated asymptotic dependence on $\varepsilon$.
\end{proof}

Finally, in order to approximately compute the binary maximum function, it is enough for us to consider the following well-known identity which reduces the maximum function to gates which we have previously shown are computable by $\mathcal{A}$NNs; the identity is valid for every $x,y\in \mathbb{R}$
\begin{equation}
\label{eq:max_abs_identity}
    \max\{x,y\}
=
    \tfrac{1}{2}\,
    \big(
        x
        +
        y
        +
        |x-y|
    \big)
.
\end{equation}
As before, we approximately compute the right-hand side of~\eqref{eq:max_abs_identity} by calling the ``expert/agent'' sub-networks performing addition, rescaling, and absolute values as sub-routines; thus emulating the following algorithm~\ref{algo:binary_maximum}.

\begin{algorithm}[H]
\caption{$\mathcal{A}$NN Binary Maximum via the Absolute Value Identity}
\label{algo:binary_maximum}
\KwIn{$x,y\in[-M,M]$, parameter $M\ge 2$, approximation error $0<\varepsilon<1$}
\KwOut{$\widehat m \approx \max\{x,y\}$}

$u \eqdef x-y$\;

$v \eqdef \textsc{AbsoluteValue}(u)$ \tcp*[r]{{\scriptsize \textsc{AbsoluteValue} is Algorithm~\ref{algo:absolute_value}}}

$\widehat m \eqdef \tfrac{1}{2}(x+y+v)$\;

\Return{$\widehat m$}\;
\end{algorithm}

\begin{proposition}[$\mathcal{A}NN$ Emulation of the Binary Maximum $\max\{\cdot,\cdot\}$]
\label{prop:max_ANNs}
For every $0<\varepsilon<1$ and $M\ge 2$ there exists an $\mathcal{A}$NN
$
    \mathrm{Max}_{M:\varepsilon}:\mathbb{R}^2\to\mathbb{R}
$
satisfying
\[
    \sup_{|x|\le M,\,|y|\le M}\,
        \big|
            \mathrm{Max}_{M:\varepsilon}(x,y)-\max\{x,y\}
        \big|
    \le
        \varepsilon
.
\]
Moreover, $\mathrm{Max}_{M:\varepsilon}$ may be chosen to have
\begin{enumerate}
    \item[(i)] \textbf{Depth:} at most
    $
        9
        +
        K_{\mathrm{max}}\,\big(
            12K_{\mathrm{max,inv}}-3
        \big)
    $,
    \item[(ii)] \textbf{Width:} at most
    $
        41
    $,
    \item[(iii)] \textbf{Non-zero parameters:} at most
    \[
\begin{aligned}
        K_{\mathrm{max}}\Big(
            2(K_{\mathrm{max,inv}}-1)\big(8\|x^\star\|_0+41\big)
            +
            30K_{\mathrm{max,inv}}
            +
            \big(8\|x^\star\|_0+41\big)
            +
            24
        \Big)
        +
        12
        +
        \big(8\|x^\star\|_0+41\big)
        +
        12
        +
        13,
\end{aligned}
    \]
\end{enumerate}
where
$
    \delta_{\mathrm{max}}
    \eqdef
        \big(\varepsilon/4\big)^2
$, $
    m_{\mathrm{max}}
    \eqdef
        \delta_{\mathrm{max}}/2
$, $
    M_{\mathrm{max}}
    \eqdef
        (2M)^2+2\delta_{\mathrm{max}}
$, $
    m_{\mathrm{max,inv}}
    \eqdef
        \frac12\sqrt{m_{\mathrm{max}}}
$, $
    M_{\mathrm{max,inv}}
    \eqdef
        2M_{\mathrm{max}}+1+\frac{M_{\mathrm{max}}}{m_{\mathrm{max,inv}}}
$, $0<
    q_{\mathrm{max,inv}}
    \eqdef
        \frac{M_{\mathrm{max,inv}}^2-m_{\mathrm{max,inv}}^2}{M_{\mathrm{max,inv}}^2+m_{\mathrm{max,inv}}^2}
<1$, $
    \delta_{\mathrm{max,inv}}
    \eqdef
        \frac{\varepsilon}{32M_{\mathrm{max}}}
$,
$
    K_{\mathrm{max,inv}}
    \eqdef
        \Bigg\lceil
            1
            +
            \log_2 \Big(
                \frac{
                    \log \big(
                        2/(m_{\mathrm{max,inv}}\delta_{\mathrm{max,inv}})
                    \big)
                }{
                    \log(1/q_{\mathrm{max,inv}})
                }
            \Big)
        \Bigg\rceil
$, and $
    K_{\mathrm{max}}
\eqdef
    \Bigg\lceil
        \log_2 \Big(
            \frac{
                4\,(M_{\mathrm{max}}+1-\sqrt{m_{\mathrm{max}}})
            }{
                \varepsilon/4
            }
        \Big)
    \Bigg\rceil
$.
\end{proposition}
\begin{proof}[{Proof of Proposition~\ref{prop:max_ANNs}}]
Fix $0<\varepsilon<1$ and $M\ge 2$.
Let $\mathrm{Abs}_{2M:\varepsilon}:\mathbb{R}\to\mathbb{R}$ be as in Proposition~\ref{prop:abs_ANNs}.
Define
\[
    \mathrm{Max}_{M:\varepsilon}(x,y)
    \eqdef
    \frac{1}{2}\Big(
        x+y+\mathrm{Abs}_{2M:\varepsilon}(x-y)
    \Big)
.
\]
Then, for $|x|\le M$ and $|y|\le M$, we have $|x-y|\le 2M$, and thus
\[
    \big|
        \mathrm{Max}_{M:\varepsilon}(x,y)-\max\{x,y\}
    \big|
    =
    \frac{1}{2}\,
    \big|
        \mathrm{Abs}_{2M:\varepsilon}(x-y)-|x-y|
    \big|
    \le
    \frac{\varepsilon}{2}
    \le
    \varepsilon
,
\]
using the identity~\eqref{eq:max_abs_identity}.
The complexity bounds follow from Proposition~\ref{prop:abs_ANNs}.
\end{proof}

\begin{cor}[$\mathcal{A}NN$ Emulation of the Binary Minimum $\min\{\cdot,\cdot\}$]
	\label{cor:min_ANNs}
	For every $0<\varepsilon<1$ and $M\ge 2$ there exists an $\mathcal{A}$NN
	$
		\mathrm{Min}_{M:\varepsilon}:\mathbb{R}^2\to\mathbb{R}
	$
	satisfying
	\[
		\sup_{|x|\le M, |y|\le M}
		\lvert
		\mathrm{Min}_{M:\varepsilon}(x,y)-\min\{x,y\}
		\rvert
		\le
		\varepsilon
		.
	\]
	Moreover, $\mathrm{Min}_{M:\varepsilon}$ may be chosen with the same asymptotic
	depth, width, and non-zero parameter complexity as in Proposition~\ref{prop:max_ANNs}.
\end{cor}
\begin{proof}
	Define
	\[
		\mathrm{Min}_{M:\varepsilon}(x,y)
		\eqdef
		-\mathrm{Max}_{M:\varepsilon}(-x,-y).
	\]
	Since $\min\{x,y\}=-\max\{-x,-y\}$, the claim follows immediately from Proposition~\ref{prop:max_ANNs}.
\end{proof}

\begin{remark}[A Curious Contrast with the $\operatorname{ReLU}$-MLP arguments]
\label{rem:curiouscontrast}
The reader familiar with the usual proofs of universal approximation specialized to the highly particular structure of the $\operatorname{ReLU}$-MLP model, cf.~\cite{yarotsky2017error,zhang2024deep,petersen2024mathematical}, which is used less and less in real-world AI applications, will notice a curious feature of our arguments to date.  Namely, in the general non-flat o-minimal setting $\mathcal{A}$NNs easily emulate squares but have trouble emulating the absolute value function; this is exactly the opposite of what happens in the specialized $\operatorname{ReLU}$-MLP model where one has to make use of Telgarsky's sawtooth construction, cf.~\cite{telgarsky2015representation}, to emulate the quadratic monomial.
\end{remark}

\subsubsection{Tensorized $B$-Splines and Bump Functions of Moderate Regularity}
\label{splines}
There are at least two polynomial splines which play core roles in several central results in approximation theory and numerical analysis.  These are moderate-regularity bump functions and cardinal $B$-splines.  We now obtain $\mathcal{A}$NN computations of these rudimentary objects, by stringing together our previous basic building blocks; in other words we are encoding basic $\mathbb{G}_{\operatorname{alg}}^k$ circuits computing them; for suitable but small $k$.

\paragraph{Bump Functions of Moderate Regularity.}
Let $0<\eta<1\le M$ and $k\in \mathbb{N}_+$.  Consider the associated ``smooth cutoff function'' $S:\mathbb{R}\to [0,1]$ defined for every $t\in \mathbb{R}$ by
\begin{equation}
\label{eq:smoothcutoff}
S(t)\eqdef
\begin{cases}
0, & t\le 0,
\\[1mm]
2t^2, & 0\le t\le \frac12,
\\[1mm]
1-2(1-t)^2, & \frac12\le t\le 1,
\\[1mm]
1, & t\ge 1.
\end{cases}
\end{equation}
Then define the $C^1$ bump function $\phi:\mathbb{R}\to [0,1]$ for every $x\in \mathbb{R}$ by
\begin{equation}
\label{eq:def_phi_cutoff}
\phi(x)\eqdef
S\Bigl(\frac{x+1+\eta}{\eta}\Bigr)
-
S\Bigl(\frac{x-1}{\eta}\Bigr)
.
\end{equation}
Consider its $k$-fold tensorization $\phi^{\otimes k:\eta}:\mathbb{R}^k\to [0,1]$ defined for any $x\in\mathbb{R}^k$ by
\begin{equation}
\label{eq:tensorized_phik}
\phi^{\otimes k:\eta}(x)\eqdef \prod_{i=1}^k \phi(x_i)
.
\end{equation}
Then $\phi^{\otimes k:\eta}$ satisfies
$
\phi^{\otimes k:\eta}(x)=1
$
if $x\in[-1,1]^k$ and
$
\phi^{\otimes k:\eta}(x)=0
$
for all $x\in\mathbb{R}^k\setminus(-1-\eta,1+\eta)^k$.
We now see that the tensorized $C^1$ bump function $\phi^{\otimes k:\eta}$, which is a first example of a tensorized spline, can be approximately computed by an $\mathcal{A}$NN.

The properties will be key for the approximation of $\varphi^{\otimes n}$ will be instrumental in approximating a continuous function on $\mathbb{R}^n\to \mathbb{R}$, at the minimax rate, while scaling with the (intrinsic/box) metric dimension of the domain we are approximating it on.

\begin{proposition}[{$C^1$ Bump Function~\eqref{eq:tensorized_phik}}]
\label{prop:C1_bump}
For every $k\in \mathbb{N}_+$, $M,\varepsilon>0$, and $0<\eta<1$ there is an $\mathcal{A}$NN
$
\Phi^{\phi^{\otimes k:\eta}}_{\varepsilon,M,\eta}:\mathbb{R}^k\to\mathbb{R}
$
satisfying
\[
    \sup_{|x|\le M}
    \,
    \big|
            \Phi^{\phi^{\otimes k:\eta}}_{\varepsilon,M,\eta}(x)
        -
            \phi^{\otimes k:\eta}(x)
    \big|
\le
    \varepsilon
.
\]
Moreover, the network $\Phi^{\phi^{\otimes k:\eta}}_{\varepsilon,M,\eta}$ may be chosen to have depth
$
\mathcal{O}\big(\log^2((M+1+\eta)k/(\eta\varepsilon))\big)+\mathcal{O}(\log k)
$,
width
$
\mathcal{O}(k)
$,
and
$
\mathcal{O}\big(k\log^2((M+1+\eta)k/(\eta\varepsilon))\big)
$
non-zero parameters.
\end{proposition}
Next, we consider a fundamental spline family for constructive approximation.
\paragraph{$B$-Splines.}
We now implement tensorized $B$-splines, which were introduced by~\cite{Schoenberg1946b}, and can be used as the basic building blocks defining Besov spaces on cubes; cf.~\cite{devore1993besov,DAHMEN1981299}.  
Using the representation of~\citep[Equation (4.28)]{mhaskar1992approximation}, these (cardinal) $B$-spline of degree $\delta\in \mathbb{N}_+$ is the map $s_{\delta}:\mathbb{R}\to \mathbb{R}$ defined by
\[
    s_{\delta}(x)
\eqdef
    \frac{1}{\delta!}
    \sum_{j=0}^{\delta+1}
    (-1)^j
    \binom{\delta+1}{j}
    (x-j)_+^{\delta}
.
\]
Its integer translates, at any $m\in \mathbb{Z}$, are denoted by
$
    s_{\delta:m}(x)\eqdef s_{\delta}(x-m)
$.
Now, for each $n\in \mathbb{N}_+$ and every pair of multi-indices $\Delta\in \mathbb{N}_+^n$ and $m\in \mathbb{Z}^n$, define the tensorized cardinal $B$-spline
$
    \mathfrak{S}_{\Delta:m}:\mathbb{R}^n\to\mathbb{R}
$
by
\begin{equation}
\label{eq:def_tensor_bspline_target}
    \mathfrak{S}_{\Delta:m}(x)
\eqdef
    \prod_{i=1}^n s_{\Delta_i:m_i}(x_i)
,
\qquad x=(x_1,\dots,x_n)\in\mathbb{R}^n
.
\end{equation}
Since~\eqref{eq:def_tensor_bspline_target} is effectively a tree-shaped $\mathbb{G}_{\operatorname{alg}}^k$ representation, for a small $k\in \mathbb{N}_+$ depending only on the coordinates of $m$, we obtain the following $\mathcal{A}$NN emulator.
\begin{proposition}[Tensor Products of Cardinal $B$-Splines]
\label{prop:tensor_product_Bsplines}
Let $n\in \mathbb{N}_+$ and fix multi-indices $\Delta\in \mathbb{N}_+^n$ and $m\in \mathbb{Z}^n$.  Then, for every $\varepsilon,M>0$, there exists an $\mathcal{A}$NN
$
    \Phi^{\mathrm{BSp}}_{\varepsilon,\Delta:m:M}:\mathbb{R}^n\to\mathbb{R}
$
satisfying
\begin{equation}
\label{eq:ANN_spline}
        \sup_{|x|\le M}
        \,
        \big|
            \Phi^{\mathrm{BSp}}_{\varepsilon,\Delta:m:M}(x)
            -
            \mathfrak{S}_{\Delta:m}(x)
        \big|
    \le
        \varepsilon
.
\end{equation}
Moreover, $\Phi^{\mathrm{BSp}}_{\varepsilon,\Delta:m:M}$ has 
depth
$
    \mathcal{O}\big(\log^2((M+1)/\varepsilon)+\log(n)\big)
$,
width
$
    \mathcal{O}(n)
$,
and
$
    \mathcal{O}\big(n\log^2((M+1)/\varepsilon)\big)
$
non-zero parameters.
\end{proposition}

\paragraph{Proofs of Spline Results.}
We now prove the results of this subsection.

\begin{proof}[{Proof of Proposition~\ref{prop:C1_bump}}]
Write
$
\rho(u)\eqdef (u)_+
$.
Since
$\rho(u)=\frac{u+|u|}{2}$, 
Proposition~\ref{prop:abs_ANNs} yields, for every $0<\delta<1$ and every $B\ge 2$, an $\mathcal{A}$NN
$
\rho_{B,\delta}:\mathbb{R}\to\mathbb{R}
$
such that
$
    \sup_{|u|\le B}
    \,
    |\rho_{B,\delta}(u)-\rho(u)|
\le
    \delta
$.
Moreover, $\rho_{B,\delta}$ has depth in
$
    \mathcal{O}\big(\log^2(B/\delta)\big)
$,
width at most $40$, and non-zero parameters in
$
    \mathcal{O}\big(\log^2(B/\delta)\big)
$.
For the remainder of the proof, set
$
	B \eqdef \frac{M+1+\eta}{\eta}.
$
Since \(0<\eta<1\), we have $B\ge 2$.

Now, for $|x|\le M$, each of the six arguments
$
    \frac{x+1+\eta}{\eta}
$, $
    \frac{x+1+\eta}{\eta}-\frac12
$, $
    \frac{x+1+\eta}{\eta}-1
$, $
    \frac{x-1}{\eta}
$, $
    \frac{x-1}{\eta}-\frac12
$, and
$
    \frac{x-1}{\eta}-1
$
has absolute value at most $B$.
Since
\[
S(t)
=
2\rho(t)^2
-
4\rho\Big(t-\frac12\Big)^2
+
2\rho(t-1)^2
,
\]
it follows that
\[
\begin{aligned}
    \phi(x)
=
    &
    2\rho\Big(\frac{x+1+\eta}{\eta}\Big)^2
    -
    4\rho\Big(\frac{x+1+\eta}{\eta}-\frac12\Big)^2
    +
    2\rho\Big(\frac{x+1+\eta}{\eta}-1\Big)^2
    \\
    &
    -
    2\rho\Big(\frac{x-1}{\eta}\Big)^2
    +
    4\rho\Big(\frac{x-1}{\eta}-\frac12\Big)^2
    -
    2\rho\Big(\frac{x-1}{\eta}-1\Big)^2
    .
\end{aligned}
\]
Since $|\rho_{B,\delta}(u)-\rho(u)|\le \delta$ for $|u|\le B$, with $0\le \rho(u)\le B$, and since $0<\delta<1$, the outputs of $\rho_{B,\delta}$ lie in $[-\delta,B+\delta]\subseteq [-(B+1),B+1]$. Using six parallel copies of $\rho_{B,\delta}$, followed by six parallel copies of $\mathrm{Mult}^{(2)}_{\delta,B+1}$ from Proposition~\ref{prop:kth_multiplication__SK__polar__precise} to square the resulting coordinates, and then taking the above affine linear combination, we obtain, for every $0<\delta<1$, an $\mathcal{A}$NN
$
\widetilde{\phi}_{\delta,M,\eta}:\mathbb{R}\to\mathbb{R}
$
satisfying
\begin{equation}
\label{eq:splinetime}
\sup_{|x|\le M}
\,
    |\widetilde{\phi}_{\delta,M,\eta}(x)-\phi(x)|
\le
    \delta
.
\end{equation}
Moreover, $\widetilde{\phi}_{\delta,M,\eta}$ has depth at most
$
    \mathcal{O}\big(\log^2(B/\delta)\big)+8
$,
width at most
$
    240
    +
    6\big(
        12\max\{m_{\star},n_{\star}\}+2
    \big)
$,
and at most
$
    \mathcal{O}\big(\log^2(B/\delta)\big)
    +
    6\big(
        8\|x^\star\|_0+39
    \big)
$
non-zero parameters.

Set
$
    \delta\eqdef \min\Big\{
        \frac{1}{k},
        \frac{\varepsilon}{4ek},
        \frac{\varepsilon}{2}
    \Big\}
.
$
Now define
$
    \widetilde{\Phi}_{\delta,M,\eta}(x_1,\dots,x_k)
\eqdef
    \mathrm{Mult}^{(k)}_{\delta,\,2}
    \Big(
    \widetilde{\phi}_{\delta,M,\eta}(x_1),\dots,\widetilde{\phi}_{\delta,M,\eta}(x_k)
    \Big)
$, 
where $\mathrm{Mult}^{(k)}_{\delta,\,2}$ is given by Proposition~\ref{prop:kth_multiplication__SK__polar__precise}.  Since $0\le \phi\le 1$, \eqref{eq:splinetime} implies that
$
    -\delta
    \le
    \widetilde{\phi}_{\delta,M,\eta}(x)
    \le
    1+\delta
$
for all $|x|\le M$; in particular, since $\delta\le 1$, we have
$
    |\widetilde{\phi}_{\delta,M,\eta}(x)|
    \le
    2
$
for all $|x|\le M$, so the multiplication network is evaluated on its domain of accuracy.  Thus,~\eqref{eq:splinetime} implies that: for $|x|\le M$,
\[
\begin{aligned}
    \big|
    \widetilde{\Phi}_{\delta,M,\eta}(x)-\phi^{\otimes k:\eta}(x)
    \big|
& \le
    \big|
    \mathrm{Mult}^{(k)}_{\delta,\,2}
    \big(
    \widetilde{\phi}_{\delta,M,\eta}(x_1),\dots,\widetilde{\phi}_{\delta,M,\eta}(x_k)
    \big)
\\
& \qquad
-
\prod_{i=1}^k \widetilde{\phi}_{\delta,M,\eta}(x_i)
\big|
+
\big|
\prod_{i=1}^k \widetilde{\phi}_{\delta,M,\eta}(x_i)
-
\prod_{i=1}^k \phi(x_i)
\big|
.
\end{aligned}
\]
The first term is at most $\delta$.  For the second term, by telescoping and using that
$
    |\widetilde{\phi}_{\delta,M,\eta}(x_i)|
    \le
    1+\delta
$
and
$
    |\phi(x_i)|
    \le
    1
$,
we obtain
$
    \big|
        \prod_{i=1}^k \widetilde{\phi}_{\delta,M,\eta}(x_i)
        -
        \prod_{i=1}^k \phi(x_i)
    \big|
    \le
    k\delta(1+\delta)^{k-1}
    \le
    ek\delta
    \le
    \varepsilon/4
$,
since $\delta\le 1/k$.  Also, $\delta\le \varepsilon/2$, so the first term is at most $\varepsilon/2$.  Therefore,
$
\sup_{|x|\le M}
\big|
\widetilde{\Phi}_{\delta,M,\eta}(x)-\phi^{\otimes k:\eta}(x)
\big|
\le
\varepsilon
$.
Finally, the claimed bounds follow by composing the $k$ coordinate-wise bump emulators with the final $k$-fold multiplication gate: the depth adds, while the width and number of non-zero parameters are obtained by summing the coordinate-wise contributions and adding the final multiplication overhead.  Since
$
    \delta
    =
    \mathcal{O}(\varepsilon/k)
$
and $B=\frac{M+1+\eta}{\eta}$, we have $\log(B/\delta)=\mathcal{O}\big(\log((M+1+\eta)k/(\eta\varepsilon))\big)$. Therefore the depth is $\mathcal{O}\big(\log^2((M+1+\eta)k/(\eta\varepsilon))\big)+\mathcal{O}(\log k)$, the width is $\mathcal{O}(k)$, and the number of non-zero parameters is $\mathcal{O}\big(k\log^2((M+1+\eta)k/(\eta\varepsilon))\big)$.
This yields the stated estimates after renaming
$
    \widetilde{\Phi}_{\delta,M,\eta}
$
as
$
    \Phi^{\phi^{\otimes k:\eta}}_{\varepsilon,M,\eta}
$.
\end{proof}

\begin{proof}[{Proof of Proposition~\ref{prop:tensor_product_Bsplines}}]
Fix $\varepsilon,M>0$.  Set
$
    \Delta_{\max}\eqdef \|\Delta\|_{\infty}
$
and
$
    B\eqdef M+\|m\|_{\infty}+\Delta_{\max}+1
$.
For each $i\in\{1,\dots,n\}$ and every $t\in\mathbb{R}$, the translated cardinal spline admits the representation
\begin{equation}
\label{eq:translated_cardinal_bspline}
    s_{\Delta_i:m_i}(t)
=
    \frac{1}{\Delta_i!}
    \sum_{j=0}^{\Delta_i+1}
    (-1)^j
    \binom{\Delta_i+1}{j}
    (t-m_i-j)_+^{\Delta_i}
.
\end{equation}
Let
$
    \rho(u)\eqdef (u)_+
$.
Since
$
    \rho(u)=\frac{u+|u|}{2}
$,
Proposition~\ref{prop:abs_ANNs} implies that, for every $0<\delta<1$, there exists an $\mathcal{A}$NN
$
    \rho_{B:\delta}:\mathbb{R}\to\mathbb{R}
$
such that
$
    \sup_{|u|\le B}|\rho_{B:\delta}(u)-\rho(u)|\le \delta
$.
Moreover, $\rho_{B:\delta}$ has depth in
$
    \mathcal{O}\big(\log^2(B/\delta)\big)
$,
width at most $40$, and non-zero parameters in
$
    \mathcal{O}\big(\log^2(B/\delta)\big)
$.
Now fix $i\in\{1,\dots,n\}$.  Define the polynomial
$
    p_i:\mathbb{R}^{\Delta_i+2}\to\mathbb{R}
$
by
\begin{equation}
\label{eq:def_poly_for_bspline}
    p_i(z_0,\dots,z_{\Delta_i+1})
\eqdef
    \frac{1}{\Delta_i!}
    \sum_{j=0}^{\Delta_i+1}
    (-1)^j
    \binom{\Delta_i+1}{j}
    z_j^{\Delta_i}
.
\end{equation}
Then $p_i$ is a sparse polynomial with $S=\Delta_i+2$ monomials and local degree at most $\Delta_i$.  Hence, by Proposition~\ref{prop:sparse_poly_gate__Pow_and_TreeMult}, for every $0<\delta<1$ there exists an $\mathcal{A}$NN
$
    P_{i,\delta,B}:\mathbb{R}^{\Delta_i+2}\to\mathbb{R}
$
such that
$
    \sup_{|z|\le B+1}|P_{i,\delta,B}(z)-p_i(z)|\le \delta
$,
with depth $\mathcal{O}(\log(\Delta_i))$, width $\mathcal{O}(\Delta_i)$, and $\mathcal{O}(\Delta_i\log(\Delta_i))$ non-zero parameters.
Define
$
    U_i:\mathbb{R}\to\mathbb{R}^{\Delta_i+2}
$
by
$
    U_i(t)\eqdef (t-m_i-j)_{j=0}^{\Delta_i+1}
$,
and let
$
    \rho_{i,B:\delta}^{\|(\Delta_i+2)}
$
denote the parallel application of $\rho_{B:\delta}$ to each coordinate of $\mathbb{R}^{\Delta_i+2}$.
Now define
\begin{equation}
\label{eq:def_single_bspline_emulator}
    \Psi_{i,\delta,B}
\eqdef
    P_{i,\delta,B}
    \circ
    \rho_{i,B:\delta}^{\|(\Delta_i+2)}
    \circ
    U_i
.
\end{equation}
By construction, $\Psi_{i,\delta,B}:\mathbb{R}\to\mathbb{R}$ is an $\mathcal{A}$NN.  Moreover, for every $|t|\le M$, one has $|t-m_i-j|\le B$, and therefore
$
    \rho_{B:\delta}(t-m_i-j)
$
approximates
$
    (t-m_i-j)_+
$
uniformly on $[-M,M]$.  Combining this with~\eqref{eq:translated_cardinal_bspline} and~\eqref{eq:def_poly_for_bspline}, we obtain
$
    \lim\limits_{\delta\downarrow 0}\,
    \sup_{|t|\le M}|\Psi_{i,\delta,B}(t)-s_{\Delta_i:m_i}(t)|=0
$
.
Moreover, $\Psi_{i,\delta,B}$ has depth at most
$
    \mathcal{O}\big(\log^2(B/\delta)\big)+\mathcal{O}\big(\log(\Delta_i)\big)
$,
width at most
$
    40(\Delta_i+2)+\mathcal{O}(\Delta_i)
$,
and at most
$
    \mathcal{O}\big((\Delta_i+2)\log^2(B/\delta)\big)+\mathcal{O}\big(\Delta_i\log(\Delta_i)\big)
$
non-zero parameters.
Since each cardinal $B$-spline is bounded by $1$, by choosing $\delta>0$ sufficiently small we may ensure that
$
    |\Psi_{i,\delta,B}(t)|\le 2
$
for every $|t|\le M$ and every $i\in\{1,\dots,n\}$.
We now define the global emulator by
\begin{equation}
\label{eq:def_global_bspline_emulator}
    \Phi^{\mathrm{BSp}}_{\delta,\Delta:m:M}(x)
\eqdef
    \mathrm{Mult}^{(n)}_{\delta,2}
    \big(
        \Psi_{1,\delta,B}(x_1),\dots,\Psi_{n,\delta,B}(x_n)
    \big)
,
\qquad x\in\mathbb{R}^n
.
\end{equation}
Then Proposition~\ref{prop:kth_multiplication__SK__polar__precise} yields
$
    \sup_{|y|\le 2}
    \big|
        \mathrm{Mult}^{(n)}_{\delta,2}(y)
        -
        \prod_{i=1}^n y_i
    \big|
    \le
    \delta
$,
with depth at most
$
    3(\lfloor \log_2(n)\rfloor+1)+2
$,
width at most
$
    \lceil n/2\rceil\big(12\max\{m_{\star},n_{\star}\}+2\big)
$,
and at most
$
    (n-1)\big(4(2\|x^\star\|_0+5)+11\big)+3n+2
$
non-zero parameters.
Therefore, for every $|x|\le M$,
\[
\begin{aligned}
    \big|
        \Phi^{\mathrm{BSp}}_{\delta,\Delta:m:M}(x)
        -
        \mathfrak{S}_{\Delta:m}(x)
    \big|
&\le
    \big|
        \mathrm{Mult}^{(n)}_{\delta,2}
        \big(
            \Psi_{1,\delta,B}(x_1),\dots,\Psi_{n,\delta,B}(x_n)
        \big)
        -
        \prod_{i=1}^n \Psi_{i,\delta,B}(x_i)
    \big|
\\
&
\qquad
+
    \big|
        \prod_{i=1}^n \Psi_{i,\delta,B}(x_i)
        -
        \prod_{i=1}^n s_{\Delta_i:m_i}(x_i)
    \big|
    .
\end{aligned}
\]
The first term is at most $\delta$.  For the second term, since multiplication is Lipschitz on $[-2,2]^n$, there exists a constant $C_n>0$, depending only on $n$, such that
$
    \big|
        \prod_{i=1}^n a_i-\prod_{i=1}^n b_i
    \big|
    \le
    C_n \max_{1\le i\le n}|a_i-b_i|
$
for all $a,b\in[-2,2]^n$.  Hence,
$
    \lim\limits_{\delta \downarrow 0}\,
    \sup_{|x|\le M}
    \,
    \big|
        \prod_{i=1}^n \Psi_{i,\delta,B}(x_i)
        -
        \prod_{i=1}^n s_{\Delta_i:m_i}(x_i)
    \big|
=
    0
$
.
Choosing $\delta>0$ sufficiently small therefore yields
$
    \sup_{|x|\le M}
    \big|
        \Phi^{\mathrm{BSp}}_{\delta,\Delta:m:M}(x)
        -
        \mathfrak{S}_{\Delta:m}(x)
    \big|
    \le
    \varepsilon
$.
Renaming $\Phi^{\mathrm{BSp}}_{\delta,\Delta:m:M}$ as $\Phi^{\mathrm{BSp}}_{\varepsilon,\Delta:m:M}$ completes our estimates.
Note that the depth of $\Phi^{\mathrm{BSp}}_{\varepsilon,\Delta:m:M}$ is the sum of the depth of the coordinate-wise spline emulators and the final multiplication network, while its width and number of non-zero parameters are obtained by summing the coordinate-wise contributions and adding the final multiplication overhead. This yields the following complexity estimates
\begin{enumerate}
    \item[(i)] Depth:
    $
        \mathcal{O}\Big(
            \log^2\Big(
                \frac{M+\|m\|_{\infty}+\Delta_{\max}+1}{\varepsilon}
            \Big)
        \Big)
        +
        \mathcal{O}\big(\log(\Delta_{\max})\big)
        +
        3(\lfloor \log_2(n)\rfloor+1)+2
    ,
    $
    \item[(ii)] Width:
    $
        40\big(\|\Delta\|_1+2n\big)
        +
        \mathcal{O}\big(\|\Delta\|_1\big)
        +
        \lceil n/2\rceil\big(12\max\{m_{\star},n_{\star}\}+2\big)
    ,
    $
    \item[(iii)] Non-zero parameters:
    \[
        \mathcal{O}\Big(
            \big(\|\Delta\|_1+2n\big)
            \log^2\Big(
                \frac{M+\|m\|_{\infty}+\Delta_{\max}+1}{\varepsilon}
            \Big)
        \Big)
        +
        \mathcal{O}\big(\|\Delta\|_1\log(\Delta_{\max})\big)
        +
        (n-1)\big(4(2\|x^\star\|_0+5)+11\big)+3n+2
    ,
    \]
\end{enumerate}
where
$
    \Delta_{\max}\eqdef \|\Delta\|_{\infty}
$,
$
    \|\Delta\|_1\eqdef \sum_{i=1}^n \Delta_i
$,
and $m_{\star},n_{\star}$ are as in Proposition~\ref{prop:kth_multiplication__SK__polar__precise}.
\end{proof}

With the lengthy argument proving Proposition~\ref{prop:Step1_GateEmulation}, and its surrounding gadgets, coming to a close we now move on to the second step of the proof of our main ``surgery''.

\subsection{Proof of Step 2: \texorpdfstring{Proposition~\ref{prop:abstract_surgery}}{Surgery}}
\label{s:ProofMain__prop:prop:abstract_surgery}
The remainder of this section is devoted to proving our abstract surgery theorem; Proposition~\ref{prop:abstract_surgery}.
\subsubsection{Step 1: Circuit Canonicalization: Matching the Computational Graphs of a $\mathcal{A}$NN}
\label{s:Canonicalization}
A general $\mathbb{G}$-circuit need not have a multi-partite computational graph structure (resolved in Lemma~\ref{lem:Canonicalizationlemma}); moreover,  the second issue is that the nodes on the same part may require different complexity to emulate by a neural network (resolved in Lemma~\ref{lem:subnetwork_alignment}).  Once both these issues are shown to be resolvable in a ``tame manner'' then all the approximate gate emulation computations we have made, summarized in Table~\ref{tab_thrm:gate_approximation} will directly imply our main \textit{universal computation theorem}: Theorem~\ref{thrm:Universal_Computation}.
\begin{figure}[H]
    \centering
    \begin{subfigure}[t]{0.48\linewidth}
        \centering
        \includegraphics[width=\linewidth]{Bookkeeping/Graphics/GCircuit.pdf}
        \caption{The original $\mathbb{G}$-circuit.}
        \label{fig:G_Circuit__copypasta}
    \end{subfigure}
    \hfill
    \begin{subfigure}[t]{0.48\linewidth}
        \centering
        \includegraphics[width=\linewidth]{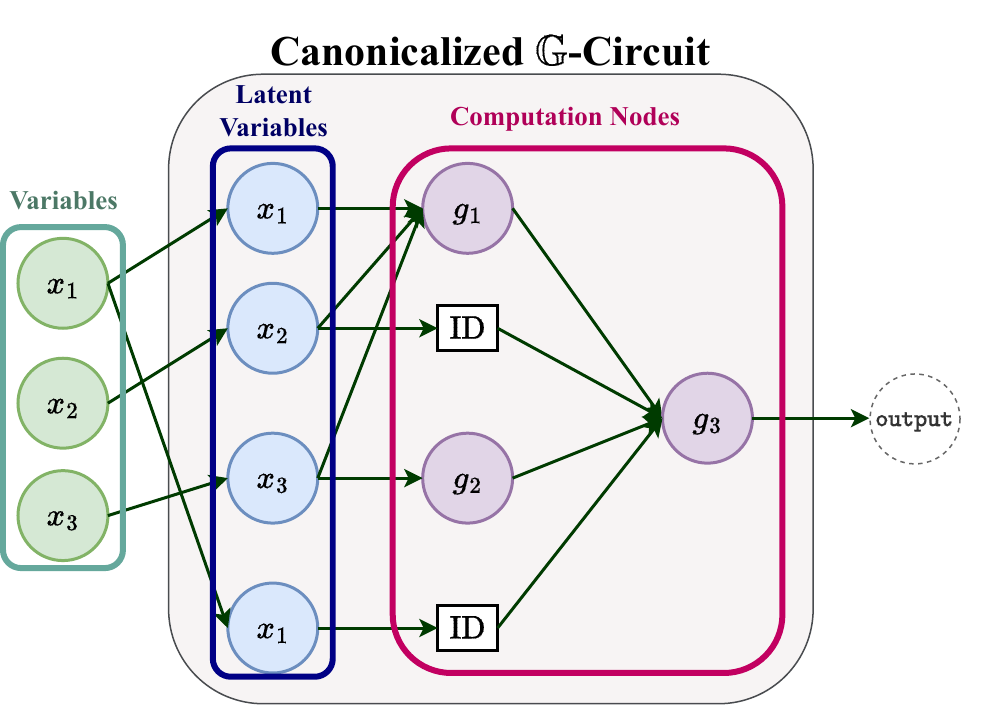}
        \caption{The canonicalized $\mathbb{G}^{\star}$-circuit.}
        \label{fig:G_Circuit__Canonicalized}
    \end{subfigure}
    \caption{\textbf{Canonicalization of the $\mathbb{G}$-Circuit in Figure~\ref{fig:G_Circuit__copypasta}:}
    The circuit computing the \textit{same} function $f$ is endowed with a multi-partite structure, thereby mimicking the ``uniform'' structure of an $\mathcal{A}$NN.  This is achieved by first aligning computation nodes according to their distance from the input variables, and then subdividing edges which ``hop across'' levels of the circuit by inserting new nodes labelled with identity gates.}
\end{figure}

Without loss of generality, we only consider the case of a one-dimensional output, as one may directly obtain the general multi-dimensional output case by parallelization; which is no issue by Assumption~\ref{ass:THEASSUMPTIONS} (ii).

\begin{lemma}[Canonicalization Lemma]
\label{lem:Canonicalizationlemma}
Let $\mathcal{C}=(V,E,g_{\cdot})$ be a $\mathbb{G}$-circuit whose underlying DAG $G=(V,E)$ 
with $V_{\operatorname{out}}=1$ of 
has depth $L\in\mathbb{N}_+$.
Let $\operatorname{Id}:\mathbb{R}\to\mathbb{R}$ be the unary identity gate $\operatorname{Id}(t)\eqdef t$, and define the extended gate set $\mathbb{G}^{\star}\eqdef \mathbb{G}\cup\{\operatorname{Id}\}$.
There exists a $L$-multi-partite DAG $(V^{\star},E^{\star})$ and a $\mathbb{G}^{\star}$-circuit
$\mathcal{C}^{\star}\eqdef (V^{\star},E^{\star},g^{\star}_{\cdot})$ such that:
\begin{enumerate}
\item[(i)] \textbf{Coherent Computation with $\mathcal{C}$:} 
For every $x\in\mathbb{R}^d$ we have $f_{\mathcal{C}^{\star}}(x)=\Rep[\mathcal{C}](x)$.

\item[(ii)] \textbf{Canonical Layering:} 
$(V^{\star},E^{\star})$ admits a layering
$
V^{\star}=\bigsqcup_{l=0}^{L} V_l^{\star}
$
with $V_0^{\star}=V_{\operatorname{in}}$, $V_L^{\star}=\{v_{\operatorname{out}}\}$, and every edge $(u,v)\in E^{\star}$ satisfies
$u\in V_l^{\star}$ and $v\in V_{l+1}^{\star}$ for some $l\in\{0,1,\dots,L-1\}$.

\item[(iii)] \textbf{Extension by Identity Gates:}
$V\subseteq V^{\star}$ and $g^{\star}_v=g_v$ for every $v\in V_{\operatorname{comp}}$,
while $g^{\star}_v=\operatorname{Id}$ for each $v\in V^{\star}\setminus V$.

\item[(iv)] \textbf{Complexity Bounds:}
\begin{enumerate}
    \item[(iv - a)] \textbf{Edge Count:}
    The number of edges satisfies
    \[
            |E^{\star}|
        =
            |E|+\sum_{(u,v)\in E}\bigl(\ell(v)-\ell(u)-1\bigr)
        \le
            L\,|E|,
    \]
    where $\ell(\cdot)$ is the level function defined in the proof. 
    In particular, $|E^{\star}|=\mathcal{O}(L\,|E|)$.
    \item[(iv - b)] \textbf{Node Count:}
    For each $l\in\{0,\dots,L\}$, writing $V_l\eqdef\{v\in V:\,\ell(v)=l\}$, we have 
$
        |V_l^{\star}|
    \le
        |V_l|+|E|
$
where $\Upsilon$ denotes the width of $(V,E)$.
\hfill\\
    In particular, 
\[
        |V^{\star}|
    \le 
        |V| + L|E|
.
\]
\end{enumerate}
\end{enumerate}
\end{lemma}
\begin{proof}[{Proof of Lemma~\ref{lem:Canonicalizationlemma}}]
We first create ``levels'' of nodes.
Since $G$ is a DAG with roots $V_{\operatorname{in}}$ and a unique terminal node $v_{\operatorname{out}}$,
we define a level function $\ell:V\to\{0,1,\dots,L\}$ by
\[
    \ell(v)
    \eqdef
    \max\bigl\{\text{length of a directed path from some }u\in V_{\operatorname{in}}\text{ to }v\bigr\}
.
\]
Then $\ell(u)=0$ for each $u\in V_{\operatorname{in}}$, and $\ell(v_{\operatorname{out}})=L$ since the depth of $G$ is $L$.
Moreover, if $(u,v)\in E$, then every directed path from an input node to $u$ extends to a directed path to $v$, and therefore
\[
    \ell(v)\ge \ell(u)+1
.
\]
Set
$
    V_l\eqdef \{v\in V:\,\ell(v)=l\}
$
for each $l\in\{0,\dots,L\}$, so that
$
    V=\bigsqcup_{l=0}^L V_l
$.

Next, we sub-divide every edge which skips a level.
Initialize $V^{\star}\gets V$, $E^{\star}\gets E$, and $g^{\star}_v\gets g_v$ for every $v\in V_{\operatorname{comp}}$.
Now fix an edge $(u,v)\in E$ and write $d\eqdef \ell(v)-\ell(u)\ge 1$.
If $d=1$, do nothing.
If $d\ge 2$, remove $(u,v)$ from $E^{\star}$ and introduce new nodes
\[
w^{(u,v)}_1,\dots,w^{(u,v)}_{d-1},
\]
together with edges
\[
(u,w^{(u,v)}_1),\ (w^{(u,v)}_1,w^{(u,v)}_2),\ \dots,\ (w^{(u,v)}_{d-2},w^{(u,v)}_{d-1}),\ (w^{(u,v)}_{d-1},v),
\]
and define their labels by $g^{\star}_{w^{(u,v)}_i}\eqdef \operatorname{Id}$ for each $i\in[d-1]_+$.
Finally, declare $w^{(u,v)}_i\in V^{\star}_{\ell(u)+i}$ for each $i$, and keep all original nodes $z\in V_l$ in the same layer $V^{\star}_l$. We repeat this operation for every edge $(u,v)\in E$ with $d\ge 2$.

By construction, every new edge goes from level $l$ to level $l+1$, hence $(V^{\star},E^{\star})$ is $L$-multi-partite with layers
$V^{\star}=\bigsqcup_{l=0}^L V^{\star}_l$,
where $V^{\star}_0=V_{\operatorname{in}}$ and $V^{\star}_L=\{v_{\operatorname{out}}\}$.
Moreover, the longest directed path still has length $L$ (we only subdivided edges), so the depth remains $L$.

Since each subdivided edge $(u,v)$ has been replaced by a directed path whose internal nodes apply $\operatorname{Id}$,
then the value transmitted from $u$ to $v$ is unchanged. Since all original gates are preserved, the output at $v_{\operatorname{out}}$ is unchanged.
Hence $f_{\mathcal{C}^{\star}}=\Rep[\mathcal{C}]$.

It remains only to tally the number of added nodes and edges.
For an original edge $(u,v)\in E$ with $d=\ell(v)-\ell(u)$ we replaced one edge by exactly $d$ edges, i.e., we added $d-1$ edges.
Therefore
\[
        |E^{\star}|
    =
        |E|+\sum_{(u,v)\in E}(d-1)
    =
        |E|+\sum_{(u,v)\in E}\bigl(\ell(v)-\ell(u)-1\bigr)
    \le 
        |E|+(L-1)|E|
    \le 
        L|E|
.
\]
Thus $|E^{\star}|=\mathcal{O}(L|E|)$.
Finally, each long edge contributes at most one new node to any intermediate layer,
so for each fixed $l$ we have $|V_l^{\star}|\le |V_l|+|E|$. If $|V_l|\le \Upsilon$ then $|V_l^{\star}|\le \Upsilon+|E|$.
\end{proof}

\subsubsection{Step 2: Layer Alignment via Identity Padding}

We already know that parallelization is possible in our language, namely, if $f,g\in \mathcal{A}$ then so is $[f^{\top},g^{\top}]^{\top}$.  However, it is less clear that several subnetworks \textit{of different depths} can be parallelized.  In the continuous piecewise linear MLP case (with finitely many pieces), e.g., $\operatorname{ReLU}$, this is well-known to be \textit{exactly} possible since such networks can exactly compute the identity map; cf.~\citep[Proposition 5]{cheridito2022efficient} or~\citep[Lemma 2.17]{gribonval2022approximation} for a specialized ReLU MLP case.

\begin{figure}[H]
    \centering
    \begin{subfigure}[t]{0.48\linewidth}
        \centering
        \includegraphics[width=\linewidth]{Bookkeeping/Graphics/GCircuit_Canonicalized.pdf}
        \caption{The canonicalized $\mathbb{G}^{\star}$-circuit.}
        \label{fig:G_Circuit__Canonicalized__copypasta}
    \end{subfigure}
    \hfill
    \begin{subfigure}[t]{0.48\linewidth}
        \centering
        \includegraphics[width=\linewidth]{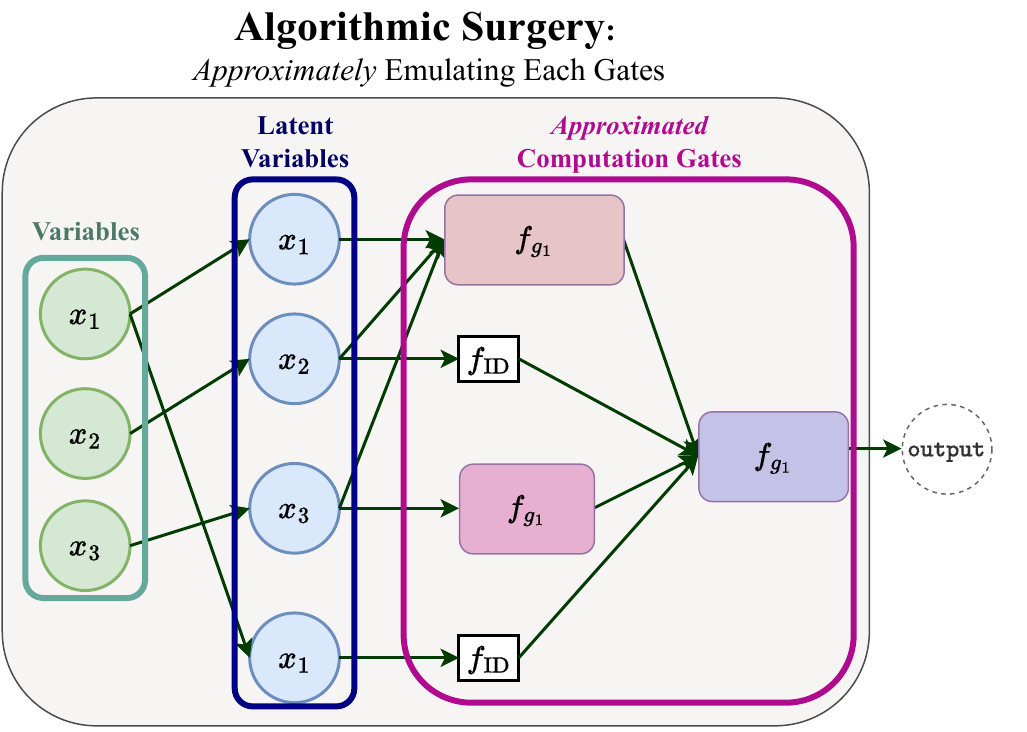}
        \caption{Surgery: Gates swapped for $\mathcal{A}$NN approximate emulators.}
        \label{fig:G_Circuit_Surgerized}
    \end{subfigure}
    \caption{\textbf{Algorithmic Surgery of the $\mathbb{G}$-Circuit in Figure~\ref{fig:Summary}:}
    Upon canonicalizing the $\mathbb{G}$-circuit computing $f$, layer by layer, we replace each node and gate by a small ``expert'' $\mathcal{A}$NN sub-network approximately performing the same computation, according to the results in Table~\ref{tab_thrm:gate_approximation}. The result is a circuit whose sub-circuits are each exactly $\mathcal{A}$NNs, and whose totality approximately computes the same function $f$ as the original $\mathbb{G}$-circuit in Figure~\ref{fig:G_Circuit__copypasta}, with approximately the same complexity (depending only approximately logarithmically on $1/\varepsilon$), and with roughly the same internal graph structure as the DAG encoding the order of operations of the original $\mathbb{G}$-circuit.
    The idea of ``algorithmic surgery'' builds on the technique of~\cite{kratsios2025quantifying,li2026certifiable}.
    }
    \label{fig:G_Circuit__Surgery}
\end{figure}

\begin{lemma}[Sub-Network Alignment]
\label{lem:subnetwork_alignment}
Let $K\in \mathbb{N}_+$, and let $d_{in:k},d_{out:k}\in \mathbb{N}_+$ for each $k\in [K]$.  
For each $k\in [K]$, let
$
    f^{(k)}:\mathbb{R}^{d_{in:k}}\to\mathbb{R}^{d_{out:k}}
$
be an $\mathcal{A}$NN with depth at most $L_k$, width at most $W_k$, and with at most $N_k$ non-zero parameters.  
Then, for every $\varepsilon,M>0$, there exists an $\mathcal{A}$NN
\[
    \bigoplus_{\varepsilon:k=1}^K\,f^{(k)}
    :
    \mathbb{R}^{\sum_{k=1}^K d_{in:k}}
    \to
    \mathbb{R}^{\sum_{k=1}^K d_{out:k}}
\]
satisfying the ``approximate parallelization'' guarantee\footnote{
Where for every
$
    x=(x^{(1)},\dots,x^{(K)}) \in \mathbb{R}^{\sum_{k=1}^K\ d_{in:k}}
$ we write 
$
    x^{(k)}\in \mathbb{R}^{d_{in:k}}
$.
}
\begin{equation}
\label{eq:parallelization}
        \sup_{\|x\|_\infty\le M}
    \,
        \left\|
                \bigoplus_{\varepsilon:k=1}^K\,f^{(k)} (x)
            -
                \big(
                    f^{(1)}(x^{(1)}),
                    \dots,
                    f^{(K)}(x^{(K)})
                \big)
        \right\|_\infty
    \le
        \varepsilon
.
\end{equation}
Moreover, $\oplus_{\varepsilon:k=1}^K f^{(k)}\,$ has depth at most
$
    L^\star
    \eqdef
    \max_{k\in [K]} L_k
$, width at most $
    \sum_{k=1}^K \,\max\{W_k,d_{out:k}\}
$, 
and at most
$
    \sum_{k=1}^K
    \big(
        N_k
        +
        4d_{out:k}(L^\star-L_k)
    \big)
$ non-zero parameters.
\end{lemma}

\begin{proof}[{Proof of Lemma~\ref{lem:subnetwork_alignment}}]
Fix $\varepsilon>0$ and $M>0$.  For each $k\in [K]$, set $
    B_k
    \eqdef
    \sup_{\|x\|_\infty\le M}
        \|f^{(k)}(x)\|_\infty
$.
Since $f^{(k)}$ is continuous, $B_k<\infty$.
For each $k\in [K]$, if $L_k=L^\star$, set
$
    \widetilde{f}^{(k)}\eqdef f^{(k)}
$.
Otherwise, let
$
    \operatorname{Id}^{(k)}_r:\mathbb{R}^{d_{out:k}}\to\mathbb{R}^{d_{out:k}}
$
for $r=1,\dots,L^\star-L_k$
be given by Proposition~\ref{prop:identity_emulation__multi} with accuracy
$
    \varepsilon/L^\star
$
and radius
$
    B_k+\varepsilon
$.
Define
\[
    \widetilde{f}^{(k)}
    \eqdef
    \operatorname{Id}^{(k)}_{L^\star-L_k}
    \circ
    \cdots
    \circ
    \operatorname{Id}^{(k)}_1
    \circ
    f^{(k)}
.
\]
Then $\widetilde{f}^{(k)}$ is an $\mathcal{A}$NN with depth at most $L^\star$, width at most $\max\{W_k,d_{out:k}\}$, and with at most $
    N_k+4d_{out:k}(L^\star-L_k)
$ non-zero parameters.
For every $\|x\|_\infty\le M$, setting
$
    y_0\eqdef f^{(k)}(x)
$
and
$
    y_r\eqdef \operatorname{Id}^{(k)}_r(y_{r-1})
$
for $r=1,\dots,L^\star-L_k$, one has inductively
$
       \|y_r\|_\infty
    \le
        B_k+r\frac{\varepsilon}{L^\star}
    \le
        B_k+\varepsilon
$; whence, 
$
    \|y_r-y_{r-1}\|_\infty
    \le
    \frac{\varepsilon}{L^\star}
    \qquad
    \forall r=1,\dots,L^\star-L_k
$.  
Hence, $
        \|\widetilde{f}^{(k)}(x)-f^{(k)}(x)\|_\infty
    \le
        (L^\star-L_k)\frac{\varepsilon}{L^\star}
    \le
        \varepsilon
$; set $
        \bigoplus_{\varepsilon:k=1}^K f^{(k)}\,(x^{(1)},\dots,x^{(K)})
    \eqdef
        \big(
            \widetilde{f}^{(1)}(x^{(1)}),
            \dots,
            \widetilde{f}^{(K)}(x^{(K)})
        \big)
$.
This is realized by taking block-diagonal affine maps layer-by-layer.  Therefore $\bigoplus_{\varepsilon:k=1}^K f^{(k)}\,$ is an $\mathcal{A}$NN with depth at most $L^\star$, width at most
$
    \sum_{k=1}^K \max\{W_k,d_{out:k}\}
$,
and with at most
$
    \sum_{k=1}^K
    \,
    (
        N_k
        +
        4d_{out:k}(L^\star-L_k)
    )
$
non-zero parameters.  The estimate~\eqref{eq:parallelization} follows immediately from the construction.
\end{proof}

\subsubsection{Step 3: Surgery Performed}

\begin{figure}[H]
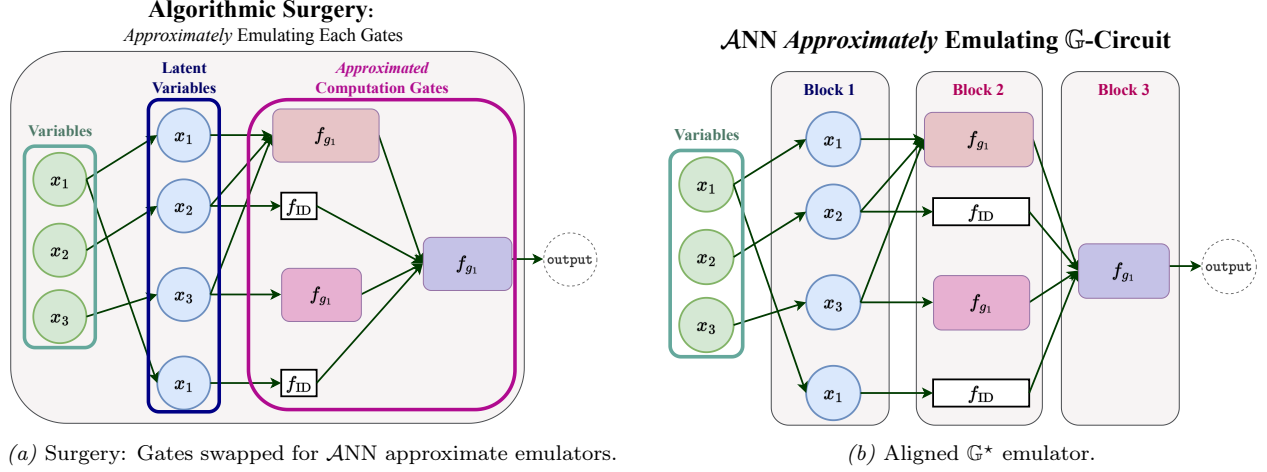

    \centering
    \begin{subfigure}[t]{0.48\linewidth}
        \centering
        \includegraphics[width=\linewidth]{Bookkeeping/Graphics/GCircuit_Surgery.pdf}
        \caption{Surgery: Gates swapped for $\mathcal{A}$NN approximate emulators.}
        \label{fig:G_Circuit_Surgerized__copypasta}
    \end{subfigure}
    \hfill
    \begin{subfigure}[t]{0.48\linewidth}
        \centering
        \includegraphics[width=\linewidth]{Bookkeeping/Graphics/GCircuit_ANNEMulator.pdf}
        \caption{Aligned $\mathbb{G}^{\star}$ emulator.}
        \label{fig:G_Circuit_aligned}
    \end{subfigure}
    \caption{\textbf{Alignment of $\mathcal{A}$NN sub-networks emulating the $\mathbb{G}$-Circuit in Figure~\ref{fig:G_Circuit__Surgery}:}
    Finally, to obtain a genuine $\mathcal{A}$NN, we need that our object has the canonical multi-partite structure.  For this, we simply pad each sub-network in each layer with identity networks to guarantee that each layer has the same size; this is no trouble since approximation of the identity map of any dimension by $\mathcal{A}$NN sub-networks have exactly $1$ hidden layer.}
    \label{fig:G_Circuit__Alignment}
\end{figure}

\begin{proof}[{Proof of Proposition~\ref{prop:abstract_surgery}}]
For every $l\in [\Delta]_+$ 
recall the approximation error $\delta_l>0$ 
consider the side-lengths given by our $\mathbb{G}$-propagator:  $M_0\eqdef 1$, $\varepsilon_0\eqdef 0$, and
\begin{equation}
\label{eq:propagation__usage}
    M_l\eqdef \Phi_{\mathbb{G}}(M_{l-1}+\varepsilon_{l-1}) + \delta_l
\end{equation}
and using our $\mathbb{G}$-regulator, we write
\begin{equation}
\label{eq:regulator__usage}
        \varepsilon_l
    \eqdef 
        \delta_l
        + 
        \mathcal{R}_{\mathbb{G}}(M_{l-1}+\varepsilon_{l-1})
            \big(
                \varepsilon_{l-1}
            \big)
.
\end{equation}
Having set up our errors, we are ready to construct our approximation and then retroactively optimize parameters.

By Lemma~\ref{lem:Canonicalizationlemma}, there exists a $\mathbb{G}$-circuit $\mathcal{C}^{\star}\eqdef (V^{\star},E^{\star},g^{\star}_{\cdot})$ such that $f_{\mathcal{C}^{\star}}=f_{\mathcal{C}}$ whose underlying graph $(V^{\star},E^{\star})$ is $\Delta-1$ multi-partite, with depth $\Delta^{\star}$, width $\Upsilon^{\star}$, and number of gates $N^{\star}$ satisfying
\begin{equation}
\label{eq:complexity_Estimates__flattenedbuddy}
\Delta^{\star}= \Delta
,\qquad
\Upsilon^{\star}\le \Upsilon+|E|
,\qquad
N^{\star} \le N + \Delta|E|
.
\end{equation}
Since $(V^{\star},E^{\star})$ has a multi-partite structure, we may represent
\begin{equation}
\label{eq:layering_consequence}
    f_{\mathcal{C}^{\star}}(x) 
= 
    g_{\Delta}\circ \dots \circ g_{1}(\Pi x)
\end{equation}
for each $x\in \mathbb{R}^d$, where for each $l\in [\Delta]_+$, $g_l:\mathbb{R}^{d_{l-1}}\to \mathbb{R}^{d_l}$ has components in $\mathbb{G}$, and $\Pi\in \{0,1\}^{d_0\times d}$ is a row-stochastic matrix; moreover $\max_{l\in [\Delta-1]_+}\,d_l\le \Upsilon^{\star}$ and $d_{\Delta}=D$.

Now, for every $l\in [\Delta]_+$ and for each $i\in [d_l]_+$ write $g_{l:i}\eqdef g_l(\cdot)_i$; 
let an $\mathcal{A}$NN be $\hat{g}_{l:i}:\mathbb{R}^{d_{l-1}}\to \mathbb{R}$  satisfying the following ``gate level approximation''
\begin{equation}
\label{eq:gate_specificapproximation_each_layer}
        \sup_{x \in \operatorname{dom}(g_{l:i}) \cap [-M_{l-1}-\varepsilon_{l-1},M_{l-1}+\varepsilon_{l-1}]^{d_{l-1}}}
        \,
            \big|
                    g_{l:i}(x)
                -
                    \hat{g}_{l:i}(x)
            \big|
    <
            \tfrac{\delta_l}{2}
.
\end{equation}
By the $(\mathbb{G},\mathcal{A})$-Hardness we know that $\hat{g}_{l:i}$ must have depth at most $\Delta_{M_{l-1}+\varepsilon_{l-1},\delta_l/2}$, width at most $\Upsilon_{M_{l-1}+\varepsilon_{l-1},\delta_l/2}$, number of parameters at most $N_{M_{l-1}+\varepsilon_{l-1},\delta_l/2}$; where $\mathcal{H}_{\mathbb{G},\mathcal{A}}(M_{l-1}+\varepsilon_{l-1},\delta_l/2)$ is
\[
\mathcal{H}_{\mathbb{G},\mathcal{A}}(M_{l-1}+\varepsilon_{l-1},\delta_l/2)=(\Delta_{M_{l-1}+\varepsilon_{l-1},\delta_l/2},\Upsilon_{M_{l-1}+\varepsilon_{l-1},\delta_l/2},N_{M_{l-1}+\varepsilon_{l-1},\delta_l/2})
.
\]

Then, applying the sub-network alignment lemma, Lemma~\ref{lem:subnetwork_alignment}, we have that there exists an $\mathcal{A}$NN $\hat{g}_l:\mathbb{R}^{d_{l-1}}\to \mathbb{R}^{d_l}$ satisfying
\begin{equation}
\label{eq:allignederrors_each_layer}
    \sup_{x\in [-M_{l-1}-\varepsilon_{l-1},M_{l-1}+\varepsilon_{l-1}]^{d_{l-1}}}
    \,
        \big\|
            \hat{g}_l(x)
            -
            \big(
                    \hat{g}_{l:1}(x)
                ,
                    \dots
                ,
                    \hat{g}_{l:d_l}(x)
            \big)
        \big\|_{\infty}
    \le
        \tfrac{\delta_l}{2}
.
\end{equation}
Combining~\eqref{eq:gate_specificapproximation_each_layer} with~\eqref{eq:allignederrors_each_layer} we obtain
\begin{equation}
\label{eq:layerwise_bounds}
\begin{aligned}
&
    \sup_{x\in [-M_{l-1}-\varepsilon_{l-1},M_{l-1}+\varepsilon_{l-1}]^{d_{l-1}}}
    \,
        \big\|
            g_l(x)
            -
            \hat{g}_l(x)
        \big\|_{\infty}
\\
\le & 
    \sup_{x\in [-M_{l-1}-\varepsilon_{l-1},M_{l-1}+\varepsilon_{l-1}]^{d_{l-1}}}
    \,
        \big\|
            g_l(x)
            -
            \big(
                    \hat{g}_{l:1}(x)
                ,
                    \dots
                ,
                    \hat{g}_{l:d_l}(x)
            \big)
        \big\|_{\infty}
\\
&\qquad +
    \sup_{x\in [-M_{l-1}-\varepsilon_{l-1},M_{l-1}+\varepsilon_{l-1}]^{d_{l-1}}}
    \,
        \big\|
            \hat{g}_l(x)
            -
            \big(
                    \hat{g}_{l:1}(x)
                ,
                    \dots
                ,
                    \hat{g}_{l:d_l}(x)
            \big)
        \big\|_{\infty}
\\
\le &
    \delta_l
.
\end{aligned}
\end{equation}
Define our $\mathcal{A}$NN approximator $\hat{f}_{\mathcal{C},\delta_{\cdot}}\eqdef 
\hat{g}_{\Delta}\circ \dots \circ \hat{g}_1(\Pi \cdot)
$.  Then the ``canonicalized'' representation of $f_{\mathcal{C}}$, namely $f_{\mathcal{C}^{\star}}$ in~\eqref{eq:layering_consequence} implies that it suffices to show, by induction on $l\in [\Delta]_+$, that
\begin{equation}
\label{eq:inductive_error_bound}
    \sup_{x\in [-1,1]^d}
    \,
        \big\|
            g_l\circ \dots \circ g_1(\Pi x)
            -
            \hat{g}_l\circ \dots \circ \hat{g}_1(\Pi x)
        \big\|_{\infty}
    \le
        \varepsilon_l
.
\end{equation}
For $l=1$, this is immediate from~\eqref{eq:layerwise_bounds} since $\Pi x\in [-1,1]^{d_0}=[-M_0-\varepsilon_0,M_0+\varepsilon_0]^{d_0}$.  
Now suppose that~\eqref{eq:inductive_error_bound} holds for $l-1$, fix $x\in [-1,1]^d$, and write
\[
    z
    \eqdef
        g_{l-1}\circ \dots \circ g_1(\Pi x)
\,\,\mbox{ where }\,\,
    \hat{z}
    \eqdef
        \hat{g}_{l-1}\circ \dots \circ \hat{g}_1(\Pi x)
.
\]
Then, by the induction hypothesis,
$
       \|z-\hat{z}\|_{\infty}
    \le
        \varepsilon_{l-1}
$.
Moreover, by the definition of $M_{l-1}$ and the propagator bound, we have
$
       z\in [-M_{l-1},M_{l-1}]^{d_{l-1}}
$,
while by~\eqref{eq:inductive_error_bound},
$
       \hat{z}\in [-M_{l-1}-\varepsilon_{l-1},M_{l-1}+\varepsilon_{l-1}]^{d_{l-1}}
$.
In particular, both $z$ and $\hat{z}$ belong to
$
       [-M_{l-1}-\varepsilon_{l-1},M_{l-1}+\varepsilon_{l-1}]^{d_{l-1}}
$.
Thus,
\begin{align*}
    \big\|
        g_l\circ \dots \circ g_1(\Pi x)
        -
        \hat{g}_l\circ \dots \circ \hat{g}_1(\Pi x)
    \big\|_{\infty}
&=
    \big\|
        g_l(z)
        -
        \hat{g}_l(\hat{z})
    \big\|_{\infty}
\\
&\le
    \big\|
        g_l(z)
        -
        g_l(\hat{z})
    \big\|_{\infty}
    +
    \big\|
        g_l(\hat{z})
        -
        \hat{g}_l(\hat{z})
    \big\|_{\infty}
\\
&\le
    \mathcal{R}_{\mathbb{G}}(M_{l-1}+\varepsilon_{l-1})
    \big(
        \|z-\hat{z}\|_{\infty}
    \big)
    +
    \delta_l
\\
&\le
    \mathcal{R}_{\mathbb{G}}(M_{l-1}+\varepsilon_{l-1})
    \big(
        \varepsilon_{l-1}
    \big)
    +
    \delta_l
\\
&=
    \varepsilon_l
,
\end{align*}
where the last line is exactly~\eqref{eq:regulator__usage}.  
This proves~\eqref{eq:inductive_error_bound} for all $l\in [\Delta]_+$.  
In particular, taking $l=\Delta$ yields
\[
    \sup_{x\in [-1,1]^d}
    \,
        \big\|
            f_{\mathcal{C}}(x)
            -
            \hat{f}_{\mathcal{C},\delta_{\cdot}}(x)
        \big\|_{\infty}
    \le
    \varepsilon_{\Delta}
.
\]
Since the recursions~\eqref{eq:propagation__usage} and~\eqref{eq:regulator__usage} coincide with~\eqref{eq:recurive_errors}, we have $M_j=\mathfrak{S}_j$ and $\varepsilon_j=\mathfrak{E}_j$ for every $j\in [\Delta]_+$.  Thus~\eqref{eq:prop:abstract_surgery} follows.

For complexity estimates: the definition of the $(\mathbb{G},\mathcal{A})$-hardness $\mathcal{G}_{\mathcal{A}}$ function, the complexity estimates in the sub-network alignment lemma (Lemma~\ref{lem:subnetwork_alignment}), and the definitions of~\eqref{eq:recurive_errors} yield
For each $l\in [\Delta]_+$, write
\begin{equation}
\label{eq:complexity_buddies}
\mathcal{H}_{\mathbb{G},\mathcal{A}}
\Big(
    \mathfrak{S}_{l-1}+\mathfrak{E}_{l-1},
    \tfrac{\delta_l}{2}
\Big)
=
\big(
    L_l,
    W_l,
    P_l
\big).
\end{equation}
Then, by the $(\mathbb{G},\mathcal{A})$-Hardness, each scalar approximation $\hat{g}_{l:i}$ may be chosen to have depth at most $L_l$, width at most $W_l$, and at most $P_l$ non-zero parameters.  Hence, by Lemma~\ref{lem:subnetwork_alignment}, the vector-valued network $\hat{g}_l:\mathbb{R}^{d_{l-1}}\to \mathbb{R}^{d_l}$ may be chosen to have depth at most
$
    L_l,
$
width at most
$
d_l\max\{W_l,1\},
$
and at most
$
d_l(P_l+4L_l)
$
non-zero parameters.
Writing
$
\bar{\Upsilon}\eqdef \max\{D,\Upsilon+|E|\}
$,
we have $d_l\le \bar{\Upsilon}$ for every $l\in [\Delta]_+$.  Therefore, since
\[
\hat{f}_{\mathcal{C},\delta_{\cdot}}
\eqdef
\hat{g}_{\Delta}\circ \dots \circ \hat{g}_1(\Pi\cdot),
\]
it follows that $\hat{f}_{\mathcal{C},\delta_{\cdot}}$ has depth
$
\mathcal{O}\Big(
    \sum_{l=1}^{\Delta}\,L_l
\Big)
$, 
width
$
\mathcal{O}\Big(
    d+\bar{\Upsilon}\max_{l\in [\Delta]_+}\,W_l
\Big)
$,
and number of non-zero parameters
$
\mathcal{O}\Big(
    d_0
    +
    \bar{\Upsilon}
    \sum_{l=1}^{\Delta}\,(P_l+L_l)
\Big)
$.
\end{proof}

\subsection{{Proof of \texorpdfstring{Theorem~\ref{thrm:concrete_surgery}}{Main Quantitative Theorem}}}
\label{s:Proof__ss:thrm:concrete_surgery}
We are now set up to prove our main quantitative theorem.
\begin{proof}[{Proof of Theorem~\ref{thrm:concrete_surgery}}]
Write $
A_j
\eqdef
\mathfrak{S}_j+\mathfrak{E}_j$ and $j\in \{0,\dots,\Delta\}$.  
Since the input domain is $[-1,1]^d$, Proposition~\ref{prop:abstract_surgery} gives
$
A_0=1
$.
We argue on a case-by-case basis.  

\noindent
\textit{Case 1 - Algebraic and algebro-tropical classes:}
\hfill\\
\noindent
\textbf{Case 1(a): Shared error allocation.}
\hfill\\
First, consider the case where
$
    \mathbb{G}\in \{\mathbb{G}_{alg}^{k,c},\mathbb{G}_{t\text{-}alg}^{k,c}\}
$.
By Proposition~\ref{prop:Step1_GateEmulation}; cf.~Table~\ref{tab:G_propagators}, the class's $\mathbb{G}$-propagator is
$
    \Phi_{\mathbb{G}}(M)
    =
    \max\{c,kM,M^k\},
$
and its $\mathbb{G}$-regulator is
$
    \mathcal{R}_{\mathbb{G}}(M)(\delta)
    =
    \max\{1,k,kM^{k-1}\}\delta.
$
Since $M\ge 1$ and $k\ge 2$, we have
$
    \Phi_{\mathbb{G}}(M)
    \le
    \max\{c,k,1\}M^k
$
and
$
    \mathcal{R}_{\mathbb{G}}(M)(\delta)
    \le
    kM^{k-1}\delta.
$
Hence, for each $j\in [\Delta]_+$, we have that
\[
\begin{aligned}
    A_j
    &=
    \mathfrak{S}_j+\mathfrak{E}_j
\\
    &\le
    \Phi_{\mathbb{G}}(A_{j-1})
    +
    \delta_j
    +
    \delta_j
    +
    \mathcal{R}_{\mathbb{G}}(A_{j-1})(\mathfrak{E}_{j-1})
\\
    &\le
    \max\{c,k,1\}A_{j-1}^k
    +
    2\delta_j
    +
    kA_{j-1}^{k-1}\mathfrak{E}_{j-1}
\\
    &\le
    \big(
        \max\{c,k,1\}+k
    \big)
    A_{j-1}^k
    +
    2\delta_j
\\
    &=
    C_{alg}A_{j-1}^k+2\delta_j.
\end{aligned}
\]
Let
$
    D_{alg}\eqdef C_{alg}+2
$.
For each $l\in[\Delta]_+$, set $
    w_l^{alg}
\eqdef
    k^{\Delta-l}
    D_{alg}^{\sum_{q=l+1}^{\Delta}(k^{q-1}-1)}
$, 
with the convention that $w_\Delta^{alg}=1$, and choose
$
    \delta_l^\star
\eqdef
    \frac{\varepsilon}{\Delta w_l^{alg}}
$.
Since $w_l^{alg}\ge 1$ and $0<\varepsilon\le 1$, we have
$
    0<\delta_l^\star\le \varepsilon/\Delta\le 1
$.
Therefore, for each $j\in[\Delta]_+$ we have
$
    A_j
\le
    D_{alg}A_{j-1}^k$; and by induction we find that
$
    A_j
\le
    D_{alg}^{\sum_{t=0}^{j-1}k^t}
$ for each $j\in[\Delta]_+$.
Appealing again to Proposition~\ref{prop:abstract_surgery}, we recursively have
\[
    \mathfrak{E}_j
    \le
    \delta_j^\star
    +
    kA_{j-1}^{k-1}\mathfrak{E}_{j-1}
    \le
    \delta_j^\star
    +
    \lambda_j^{alg}\mathfrak{E}_{j-1},
\]
where
$
    \lambda_j^{alg}
\eqdef
    kD_{alg}^{k^{j-1}-1}
$.
Iterating this recursion gives
$
    \mathfrak{E}_\Delta
\le
    \sum_{l=1}^{\Delta}
        \delta_l^\star
        \prod_{q=l+1}^{\Delta}
            \lambda_q^{alg}
$.  
Since
$
    \prod_{q=l+1}^{\Delta}\lambda_q^{alg}
=
    k^{\Delta-l}
    D_{alg}^{\sum_{q=l+1}^{\Delta}(k^{q-1}-1)}
=
    w_l^{alg}
$; then, we deduce that
$
    \mathfrak{E}_\Delta
\le
    \sum_{l=1}^{\Delta}
        \delta_l^\star w_l^{alg}
=
    \sum_{l=1}^{\Delta}
        \frac{\varepsilon}{\Delta}
=
    \varepsilon
$.

It remains to translate this error allocation into depth and size bounds.  If
$
    \mathbb{G}=\mathbb{G}_{alg}^{k,c},
$
then every gate in the algebraic language is emulated with depth $\mathcal{O}(\log k)$, width $\mathcal{O}(1)$, and $\mathcal{O}(k)$ non-zero parameters.  Proposition~\ref{prop:abstract_surgery} therefore gives depth
$
    \mathcal{O}(\Delta\log k)
$
and number of non-zero parameters
$
    \mathcal{O}(d_0+\bar{\Upsilon}\Delta k).
$
\hfill\\
\noindent
\textbf{Case 1(b): Algebro-tropical complexity extraction.}
\hfill\\
If instead
$
    \mathbb{G}=\mathbb{G}_{t\text{-}alg}^{k,c} ,
$
then the tropical gates contribute, at layer $l$, an additional accuracy-dependent factor of order
$
    \log(2/\delta_l^\star)\log^{\circ 2}(2/\delta_l^\star).
$
Set
$
    L_l^{trop}(\varepsilon)
    \eqdef
    \log\Big(\frac{2}{\delta_l^\star}\Big).
$
Since
$
    \delta_l^\star
    =
    \varepsilon/(\Delta w_l^{alg}),
$
we have
\[
    L_l^{trop}(\varepsilon)
    =
    \log\Big(\frac{2\Delta}{\varepsilon}\Big)
    +
    \log w_l^{alg}.
\]
Moreover,
$
    \log w_l^{alg}
=
    (\Delta-l)\log k
    +
    \Big(
        \sum_{q=l+1}^{\Delta}
            (k^{q-1}-1)
    \Big)
    \log D_{alg}
\le
    C_0 k^\Delta
$,
for some constant $C_0>0$ depending only on $k,c$.  Hence, 
$
    L_l^{trop}(\varepsilon)
\le
    C_1
    \Big(
        k^\Delta
        +
        \log\Big(\frac{2\Delta}{\varepsilon}\Big)
    \Big)
$, 
and therefore
\[
    \log L_l^{trop}(\varepsilon)
\le
    C_2
    \Big(
        \Delta\log k
        +
        \log^{\circ 2}\Big(\frac{2\Delta}{\varepsilon}\Big)
    \Big).
\]
Thus the total tropical contribution to the depth is bounded by
\[
    \sum_{l=1}^{\Delta}
        \log k\,
        L_l^{trop}(\varepsilon)
        \log L_l^{trop}(\varepsilon)
    \le
    C_3
    \Delta\log k
    \Big(
        k^\Delta
        +
        \log\Big(\frac{2\Delta}{\varepsilon}\Big)
    \Big)
    \,
    \Big(
        \Delta\log k
        +
        \log^{\circ 2}\Big(\frac{2\Delta}{\varepsilon}\Big)
    \Big),
\]
while the total tropical contribution to the number of non-zero parameters is bounded by
\[
    \sum_{l=1}^{\Delta}
        k\,
        L_l^{trop}(\varepsilon)
        \log L_l^{trop}(\varepsilon)
    \le
    C_4
    \Delta k
    \Big(
        k^\Delta
        +
        \log\Big(\frac{2\Delta}{\varepsilon}\Big)
    \Big)
    \,
    \Big(
        \Delta\log k
        +
        \log^{\circ 2}\Big(\frac{2\Delta}{\varepsilon}\Big)
    \Big)
.
\]
Proposition~\ref{prop:abstract_surgery} gives the corresponding algebro-tropical depth and non-zero-parameter estimates.
Finally, since $f_{\mathcal{C}}$ $\varepsilon$-computes $f$ on $[-1,1]^d$ and since
$
    \sup_{x\in[-1,1]^d}
    \big\|
        f_{\mathcal{C}}(x)-\hat{f}_{\mathcal{C}}(x)
    \big\|_{\infty}
\le
    \varepsilon
$, 
thus, we deduce that
\[
    \sup_{x\in[-1,1]^d}
    \big\|
        f(x)-\hat{f}_{\mathcal{C}}(x)
    \big\|_{\infty}
    \le
    \sup_{x\in[-1,1]^d}
    \big\|
        f(x)-f_{\mathcal{C}}(x)
    \big\|_{\infty}
    +
    \sup_{x\in[-1,1]^d}
    \big\|
        f_{\mathcal{C}}(x)-\hat{f}_{\mathcal{C}}(x)
    \big\|_{\infty}
    \le
    2\varepsilon
\]
resolving Case 1.
\hfill\\
\noindent
\textit{Case 2 - Rad-algebro-tropical class:}
\hfill\\
For every $0<t<e^{-1}$, define the map $\Gamma_{r^\star}:(0,\infty)\to \mathbb{R}$ by
\begin{equation}
\label{eq:uglyfunction}
    \Gamma_{r^\star}(t)
    \eqdef
    r^\star
    \,
    \log(1/t)
    \,
    \big(
        r^\star+\log^{\circ 2}(1/t)
    \big)
\end{equation}
(for larger values of $t$, the corresponding gate-emulation costs are absorbed into the implicit constants below).
Assume next that
$
    \mathbb{G}=\mathbb{G}_{rt\text{-}alg}^{k,c,r^\star}
$.
Set
$
    s\eqdef \max\{k,r^\star\}
$.
Applying Proposition~\ref{prop:Step1_GateEmulation}, together with Table~\ref{tab:G_propagators}, the class's $\mathbb{G}$-propagator is
$
    \Phi_{\mathbb{G}}(M)
    =
    \max\{c,kM,M^s\}
$
and its $\mathbb{G}$-regulator is
$
    \mathcal{R}_{\mathbb{G}}(M)(\delta)
    =
    \max\{1,k,kM^{k-1},r^\star M^{r^\star-1}\}
    (\delta+\delta^{1/r^\star}).
$
Since $M\ge 1$, there exist constants $C_\Phi^{rad},C_R^{rad}>0$, depending only on $k,c,r^\star$, such that
$
    \Phi_{\mathbb{G}}(M)
    \le
    C_\Phi^{rad}M^s
$
and
$
    \mathcal{R}_{\mathbb{G}}(M)(\delta)
    \le
    C_R^{rad}M^{s-1}(\delta+\delta^{1/r^\star})
$.
Write
$
    A_j\eqdef \mathfrak{S}_j+\mathfrak{E}_j
$.
Then, using the recursion in Proposition~\ref{prop:abstract_surgery},
\[
\begin{aligned}
    A_j
    &=
    \mathfrak{S}_j+\mathfrak{E}_j
\\
    &\le
    C_\Phi^{rad}A_{j-1}^s
    +
    2\delta_j^\star
    +
    C_R^{rad}A_{j-1}^{s-1}
    \Big(
        \mathfrak{E}_{j-1}
        +
        \mathfrak{E}_{j-1}^{1/r^\star}
    \Big).
\end{aligned}
\]
Since $\mathfrak{E}_{j-1}\le A_{j-1}$ and $A_{j-1}\ge 1$, we have
$
    \mathfrak{E}_{j-1}^{1/r^\star}\le A_{j-1}
$.
Thus
\[
    A_j
    \le
    \Big(
        C_\Phi^{rad}+2C_R^{rad}
    \Big)
    A_{j-1}^s
    +
    2\delta_j^\star.
\]
Since $w_l^{rad}(\varepsilon)\ge 1$ and $0<\varepsilon\le 1$, the choice of $\delta_l^\star$ below gives
$
    0<\delta_l^\star\le \varepsilon/\Delta\le 1
$.
Therefore, after increasing the constant if necessary,
$
    A_j
\le
    D_{rad}A_{j-1}^s
$, 
where $D_{rad}>0$ depends only on $k,c,r^\star$.  Hence, by induction,
$
    A_j
\le
    M_j^{rad}
\eqdef
    D_{rad}^{\sum_{t=0}^{j-1}s^t}
$ for each $j\in[\Delta]_+$.
Next, we estimate the propagated error.  From the same recursion,
\[
\begin{aligned}
    \mathfrak{E}_j
    &\le
    \delta_j^\star
    +
    C_R^{rad}A_{j-1}^{s-1}
    \Big(
        \mathfrak{E}_{j-1}
        +
        \mathfrak{E}_{j-1}^{1/r^\star}
    \Big)
\\
    &=
    \delta_j^\star
    +
    C_R^{rad}A_{j-1}^{s-1}
    \mathfrak{E}_{j-1}^{1/r^\star}
    \Big(
        \mathfrak{E}_{j-1}^{1-1/r^\star}+1
    \Big)
\\
    &\le
    \delta_j^\star
    +
    2C_R^{rad}A_{j-1}^{s-1}
    A_{j-1}^{1-1/r^\star}
    \mathfrak{E}_{j-1}^{1/r^\star}
\\
    &\le
    \delta_j^\star
    +
    a_j^{rad}
    \mathfrak{E}_{j-1}^{1/r^\star},
\end{aligned}
\]
where
$
    a_j^{rad}
\eqdef
    2C_R^{rad}
    (M_{j-1}^{rad})^{s-1}
    (M_{j-1}^{rad})^{1-1/r^\star}
$.  Now, for each $j\in[\Delta]_+$, define
$
    \alpha_j
\eqdef
    (r^\star)^{-(\Delta-j)}
$ and $Z_j
\eqdef
    \mathfrak{E}_j^{\alpha_j}
$. 
Since $0<\alpha_j\le 1$, the map $x\mapsto x^{\alpha_j}$ is subadditive on $[0,\infty)$.  Consequently,
\[
\begin{aligned}
    Z_j
    =
    \mathfrak{E}_j^{\alpha_j}
    \le
    \Big(
        \delta_j^\star
        +
        a_j^{rad}
        \mathfrak{E}_{j-1}^{1/r^\star}
    \Big)^{\alpha_j}
    \le
    (\delta_j^\star)^{\alpha_j}
    +
    (a_j^{rad})^{\alpha_j}
    \mathfrak{E}_{j-1}^{\alpha_j/r^\star}
    =
    (\delta_j^\star)^{\alpha_j}
    +
    (a_j^{rad})^{\alpha_j}
    Z_{j-1}.
\end{aligned}
\]
Since $Z_0=0$, iteration gives
$
    Z_\Delta
\le
    \sum_{l=1}^{\Delta}
        (\delta_l^\star)^{\alpha_l}
        \prod_{q=l+1}^{\Delta}
            (a_q^{rad})^{\alpha_q}
$.  For each $l\in[\Delta]_+$, write $
    w_l^{rad}
\eqdef
    \prod_{q=l+1}^{\Delta}
        (a_q^{rad})^{\alpha_q}
$; where we use the convention that $w_\Delta^{rad}=1$.  Set $
    \delta_l^\star
\eqdef
    \Big(
        \frac{\varepsilon}{\Delta w_l^{rad}}
    \Big)^{1/\alpha_l}
$.  
Then
\[
    Z_\Delta
    \le
    \sum_{l=1}^{\Delta}
        (\delta_l^\star)^{\alpha_l}
        w_l^{rad}
    =
    \sum_{l=1}^{\Delta}
        \frac{\varepsilon}{\Delta}
    =
    \varepsilon.
\]
Since $\alpha_\Delta=1$, we have $Z_\Delta=\mathfrak{E}_\Delta$, and therefore
$
    \mathfrak{E}_\Delta\le\varepsilon
$

It remains to bound the depth and number of non-zero parameters.  This is the point at which one must not use $\delta_l^\star\le \varepsilon/\Delta$ to upper-bound $\Gamma_{r^\star}(\delta_l^\star/2)$, since $\Gamma_{r^\star}(t)$ increases as $t$ converges down to $0$ from the right.  We instead estimate $\delta_l^\star$ directly.
Set
$
    L_l^{rad}(\varepsilon)
\eqdef
    \log\Big(\frac{2}{\delta_l^\star}\Big)
$.
Using the definition of $\delta_l^\star$, we obtain
\[
\begin{aligned}
    L_l^{rad}(\varepsilon)
    &=
    \log 2
    +
    \frac{1}{\alpha_l}
    \Big(
        \log\Big(\frac{\Delta}{\varepsilon}\Big)
        +
        \log w_l^{rad}
    \Big)
\\
    &=
    \log 2
    +
    (r^\star)^{\Delta-l}
    \Big(
        \log\Big(\frac{\Delta}{\varepsilon}\Big)
        +
        \log w_l^{rad}
    \Big).
\end{aligned}
\]
We now bound $\log w_l^{rad}$.  Since
$
    \log M_{j-1}^{rad}
=
    \Big(
        \sum_{t=0}^{j-2}s^t
    \Big)
    \log D_{rad}
\le
    C_0s^{j-1}
$, then, there exists a constant $C_1>0$, depending only on $k,c,r^\star$, such that: for each $j\in[\Delta]_+$ we have $
    \log a_j^{rad}
\le
    C_1s^{j-1}
$.  Consequently, 
\[
\begin{aligned}
    \log(w_l^{rad})
    &=
    \sum_{q=l+1}^{\Delta}
        \alpha_q
        \log a_q^{rad}
\\
    &\le
    \sum_{q=l+1}^{\Delta}
        \log a_q^{rad}
\\
    &\le
    C_1
    \sum_{q=l+1}^{\Delta}
        s^{q-1}
\\
    &\le
    C_2s^\Delta,
\end{aligned}
\]
for some constant $C_2>0$ depending only on $k,c,r^\star$.  Since $r^\star\le s$, we get
\[
\begin{aligned}
    L_l^{rad}(\varepsilon)
    &\le
    \log 2
    +
    s^\Delta
    \Big(
        \log\Big(\frac{\Delta}{\varepsilon}\Big)
        +
        C_2s^\Delta
    \Big)
\\
    &\le
    C_3s^\Delta
    \Big(
        s^\Delta
        +
        \log\Big(\frac{2\Delta}{\varepsilon}\Big)
    \Big),
\end{aligned}
\]
where $C_3>0$ depends only on $k,c,r^\star$.  Hence
$
    \log L_l^{rad}(\varepsilon)
\le
    C_4
    \Big(
        \Delta
        +
        \log^{\circ 2}\Big(\frac{2\Delta}{\varepsilon}\Big)
    \Big)
$,
again with $C_4>0$ depending only on $k,c,r^\star$.  Therefore,
\[
\begin{aligned}
    \Gamma_{r^\star}\Big(\frac{\delta_l^\star}{2}\Big)
    &=
    r^\star
    L_l^{rad}(\varepsilon)
    \Big(
        r^\star
        +
        \log L_l^{rad}(\varepsilon)
    \Big)
\\
    &\le
    C_5
    r^\star
    s^\Delta
    \Big(
        s^\Delta
        +
        \log\Big(\frac{2\Delta}{\varepsilon}\Big)
    \Big)
    \Big(
        r^\star
        +
        \Delta
        +
        \log^{\circ 2}\Big(\frac{2\Delta}{\varepsilon}\Big)
    \Big),
\end{aligned}
\]
where $C_5>0$ depends only on $k,c,r^\star$.  Therefore, we have
\[
\begin{aligned}
    \sum_{l=1}^{\Delta}
        \Gamma_{r^\star}\Big(\frac{\delta_l^\star}{2}\Big)
    &\le
    C_5
    \Delta
    r^\star
    s^\Delta
    \Big(
        s^\Delta
        +
        \log\Big(\frac{2\Delta}{\varepsilon}\Big)
    \Big)
    \Big(
        r^\star
        +
        \Delta
        +
        \log^{\circ 2}\Big(\frac{2\Delta}{\varepsilon}\Big)
    \Big).
\end{aligned}
\]
Since Proposition~\ref{prop:abstract_surgery} gives depth
$
    \mathcal{O}\Big(
        \Delta\log k
        +
        \sum_{l=1}^{\Delta}
            \Gamma_{r^\star}\Big(\frac{\delta_l^\star}{2}\Big)
    \Big),
$
and number of non-zero parameters
$
    \mathcal{O}\Big(
        d_0
        +
        \bar{\Upsilon}
        \Big[
            \Delta k
            +
            \sum_{l=1}^{\Delta}
                \Gamma_{r^\star}\Big(\frac{\delta_l^\star}{2}\Big)
        \Big]
    \Big),
$ then, the resulting $\mathcal{A}$NN has depth
\[
    \mathcal{O}\Big(
        \Delta\log k
        +
        \Delta
        r^\star
        s^\Delta
        \Big(
            s^\Delta
            +
            \log\Big(\frac{2\Delta}{\varepsilon}\Big)
        \Big)
        \Big(
            r^\star
            +
            \Delta
            +
            \log^{\circ 2}\Big(\frac{2\Delta}{\varepsilon}\Big)
        \Big)
    \Big)
\]
and the number of its non-zero parameters is on the order of
\[
    \mathcal{O}\Big(
        d_0
        +
        \bar{\Upsilon}
        \Big[
            \Delta k
            +
            \Delta
            r^\star
            s^\Delta
            \Big(
                s^\Delta
                +
                \log\Big(\frac{2\Delta}{\varepsilon}\Big)
            \Big)
            \Big(
                r^\star
                +
                \Delta
                +
                \log^{\circ 2}\Big(\frac{2\Delta}{\varepsilon}\Big)
            \Big)
        \Big]
    \Big).
\]
Finally, since $f_{\mathcal{C}}$ $\varepsilon$-computes $f$ on $[-1,1]^d$ and since
$
    \sup_{x\in [-1,1]^d}
    \,
    \big\|
        f_{\mathcal{C}}(x)-\hat{f}_{\mathcal{C}}(x)
    \big\|_{\infty}
\le
    \varepsilon
$; whence, since $\hat{f}_{\mathcal{C}}$ computes $f$ to $\varepsilon$ accuracy on $[-1,1]^d$ then we deduce that
\[
    \sup_{x\in[-1,1]^d}
    \big\|
        f(x)-\hat{f}_{\mathcal{C}}(x)
    \big\|_{\infty}
    \le
    \sup_{x\in[-1,1]^d}
    \big\|
        f(x)-f_{\mathcal{C}}(x)
    \big\|_{\infty}
    +
    \sup_{x\in[-1,1]^d}
    \big\|
        f_{\mathcal{C}}(x)-\hat{f}_{\mathcal{C}}(x)
    \big\|_{\infty}
    \le
    2\varepsilon.
\]

\noindent
\textit{Case 3 - Rational-tropical class:}
\hfill\\
Assume next that
$
\mathbb{G}=\mathbb{G}_{Rat}^{k,c,r^\star}
$
and that $\mathcal{C}$ satisfies the forbidden $m$-compositions condition from Definition~\ref{defn:no_forbiddencomposition} on $[-1,1]^d$ for some $0<m<1$.  
Again by Table~\ref{tab:G_propagators}, we find that the class's $\mathbb{G}$-propagator is
$
    \Phi_{\mathbb{G}}(M)
    =
    \max\{c,kM,M^s,m^{-1}\}
$
and its $\mathbb{G}$-regulator is
$
    \mathcal{R}_{\mathbb{G}}(M)(\delta)
    =
    \max\{m^{-2},1,k,kM^{k-1},r^\star M^{r^\star-1}\}
    (\delta+\delta^{1/r^\star}).
$
Since $M\ge 1$ and $s=\max\{k,r^\star\}$, we simplify $\Phi_{\mathbb{G}}$ to
$
    \Phi_{\mathbb{G}}(M)
    \le
    C_\Phi^{rat}M^s,
$
and $\mathcal{R}_{\mathbb{G}}$ as
$
    \mathcal{R}_{\mathbb{G}}(M)(\delta)
    \le
    C_R^{rat}M^{s-1}(\delta+\delta^{1/r^\star}),
$
where
$
    C_\Phi^{rat}\eqdef \max\{c,k,1,m^{-1}\}
$
and
$
    C_R^{rat}\eqdef \max\{m^{-2},1,k,r^\star\}.
$
Hence
\[
\begin{aligned}
    A_j
&=
    \mathfrak{S}_j+\mathfrak{E}_j
\\
&\le
    C_\Phi^{rat}A_{j-1}^s
    +
    2\delta_j^\star
    +
    C_R^{rat}A_{j-1}^{s-1}
    \big(
        \mathfrak{E}_{j-1}
        +
        \mathfrak{E}_{j-1}^{1/r^\star}
    \big).
\end{aligned}
\]
Since
$
\mathfrak{E}_{j-1}\le A_{j-1}
$
and
$
A_{j-1}\ge 1
$,
we have
$
\mathfrak{E}_{j-1}^{1/r^\star}\le A_{j-1}
$,
and therefore
\[
    A_j
\le
    \big(
        C_\Phi^{rat}+2C_R^{rat}
    \big)
    A_{j-1}^s
    +
    2\delta_j^\star
=
    C_{rat}A_{j-1}^s+2\delta_j^\star,
\]
where
$
C_{rat}\eqdef C_\Phi^{rat}+2C_R^{rat}
$.
Now
$
w_\Delta^{rat}(\varepsilon)=1
$,
hence
$
\delta_\Delta^\star=\varepsilon/\Delta\le 1
$,
and, since $w_l^{rat}(\varepsilon)\ge 1$ and $0<\varepsilon\le 1$,
for each $l$ we have $0<\delta_l^\star\le \frac{\varepsilon}{\Delta}$.  
Thus
$
    A_j
\le
    D_{rat}(\varepsilon)A_{j-1}^s
$,
where
$
D_{rat}(\varepsilon)\eqdef C_{rat}+\frac{2\varepsilon}{\Delta}
$.
By induction, we have
\[
    A_j
\le
    M_j^{rat}
    \eqdef
    D_{rat}(\varepsilon)^{\sum_{t=0}^{j-1}s^t}.
\]

Next, we may bound each $\mathfrak{E}_j$ via
\[
\begin{aligned}
    \mathfrak{E}_j
&\le
    \delta_j^\star
    +
    C_R^{rat}A_{j-1}^{s-1}
    \big(
        \mathfrak{E}_{j-1}
        +
        \mathfrak{E}_{j-1}^{1/r^\star}
    \big)
\\
&=
    \delta_j^\star
    +
    C_R^{rat}A_{j-1}^{s-1}
    \mathfrak{E}_{j-1}^{1/r^\star}
    \big(
        \mathfrak{E}_{j-1}^{1-1/r^\star}+1
    \big)
\\
&\le
    \delta_j^\star
    +
    2C_R^{rat}A_{j-1}^{s-1}A_{j-1}^{1-1/r^\star}
    \mathfrak{E}_{j-1}^{1/r^\star}
\\
&\le
    \delta_j^\star
    +
    a_j^{rat}(\varepsilon)
    \mathfrak{E}_{j-1}^{1/r^\star},
\end{aligned}
\]
where
$
a_j^{rat}(\varepsilon)\eqdef 2C_R^{rat}\big(M_{j-1}^{rat}\big)^{s-\frac{1}{r^\star}}.
$
Upon defining
$
    \alpha_j
    \eqdef
    (r^\star)^{-(\Delta-j)}
$
and
$
    Z_j
    \eqdef
    \mathfrak{E}_j^{\alpha_j},
$
the same argument as in Case 2 yields
\[
    Z_\Delta
\le
    \sum_{l=1}^{\Delta}
        (\delta_l^\star)^{(r^\star)^{-(\Delta-l)}}
        w_l^{rat}(\varepsilon)
=
    \sum_{l=1}^{\Delta}
        \frac{\varepsilon}{\Delta}
=
    \varepsilon.
\]
Since $Z_\Delta=\mathfrak{E}_\Delta$, this yields the last claim with $\varepsilon$.

It remains to bound the depth and size.  By the exact rational estimate, the network's depth is
$
    \mathcal{O}\Big(
        \Delta\log k
        +
        \sum_{l=1}^{\Delta}
            r^\star
            L_l^{rat}(\varepsilon)
            \big(
                r^\star+\log L_l^{rat}(\varepsilon)
            \big)
    \Big)
$
while its number of non-zero parameters is
$
    \mathcal{O}\Big(
        d_0
        +
        \bar{\Upsilon}
        \Big[
            \Delta k
            +
            \sum_{l=1}^{\Delta}
                r^\star
                L_l^{rat}(\varepsilon)
                \big(
                    r^\star+\log L_l^{rat}(\varepsilon)
                \big)
        \Big]
    \Big),
$
where
\[
    L_l^{rat}(\varepsilon)
    \eqdef
    \log\Big(\frac{2}{\delta_l^\star}\Big)
    =
    \log 2
    +
    (r^\star)^{\Delta-l}
    \Big(
        \log(\Delta/\varepsilon)
        +
        \log w_l^{rat}(\varepsilon)
    \Big).
\]
Since $0<\varepsilon\le 1$, we have
$
D_{rat}(\varepsilon)\le C_{rat}+2
$.
Hence, for some constant $C_0>0$ depending only on $k,c,r^\star,m$, we have
\[
    \log a_j^{rat}(\varepsilon)
    \le
    C_0 s^{j-1}
    \qquad
    \mbox{for every }j\in [\Delta]_+.
\]
Therefore, for each $l\in [\Delta]_+$,
\[
\begin{aligned}
    \log w_l^{rat}(\varepsilon)
&=
    \sum_{q=l+1}^{\Delta}
        (r^\star)^{-(\Delta-q)}
        \log a_q^{rat}(\varepsilon)
\\
&\le
    \sum_{q=l+1}^{\Delta}
        \log a_q^{rat}(\varepsilon)
\\
&\le
    C_0
    \sum_{q=l+1}^{\Delta}
        s^{q-1}
\\
&\le
    C_1 s^\Delta
\end{aligned}
\]
for some constant $C_1>0$ depending only on $k,c,r^\star,m$.  Consequently,
\[
\begin{aligned}
    L_l^{rat}(\varepsilon)
&\le
    \log 2
    +
    s^\Delta
    \Big(
        \log(\Delta/\varepsilon)
        +
        C_1 s^\Delta
    \Big)
\\
&\le
    C_2 s^{2\Delta}\log\Big(\frac{2\Delta}{\varepsilon}\Big)
\end{aligned}
\]
for some constant $C_2>0$ depending only on $k,c,r^\star,m$.  In particular,
\[
    \log L_l^{rat}(\varepsilon)
    \le
    C_3
    \Big(
        \Delta+\log^{\circ 2}\Big(\frac{2\Delta}{\varepsilon}\Big)
    \Big)
\]
for some constant $C_3>0$ depending only on $k,c,r^\star,m$.  Thus,
\[
    r^\star
    L_l^{rat}(\varepsilon)
    \big(
        r^\star+\log L_l^{rat}(\varepsilon)
    \big)
    \le
    C_4
    s^{2\Delta}
    \log\Big(\frac{2\Delta}{\varepsilon}\Big)
    \Big(
        \Delta+\log^{\circ 2}\Big(\frac{2\Delta}{\varepsilon}\Big)
    \Big)
\]
for some constant $C_4>0$ depending only on $k,c,r^\star,m$.  Summing over $l\in [\Delta]_+$ yields the announced depth and size bounds.

\hfill\\
\noindent
Finally, since $f_{\mathcal{C}}$ $\varepsilon$-computes $f$ on $[-1,1]^d$ and since we proved that
$
\sup_{x\in [-1,1]^d}
\|
    f_{\mathcal{C}}(x)-\hat{f}_{\mathcal{C}}(x)
\|_{\infty}
\le
\varepsilon,
$
the triangle inequality yields
$
\sup_{x\in [-1,1]^d}
\|
    f(x)-\hat{f}_{\mathcal{C}}(x)
\|_{\infty}
\le
2\varepsilon.
$
\end{proof}

\subsection{Proofs of rates in Tables~\ref{tab:log_depth} and~\ref{tab:poly_depth}}
\label{s:Proof__ss:thrm:complexity_class_inclusions}
We are now ready to prove the inclusions of the complexity classes reported in the beginning of our manuscript.

\begin{proof}[{Proof of Table~\ref{tab:log_depth}}]
Fix $0<\varepsilon<1$, and set
$
\kappa\eqdef \tfrac{1}{\varepsilon}
$.
Let $\mathcal{C}_{\varepsilon}$ be a $\mathbb{G}$-circuit which $\varepsilon$-computes $f$ on $[-1,1]^d$, with depth $\Delta_\varepsilon$, width $\Upsilon_\varepsilon$, and using $N_\varepsilon$ gates.  Write
$
\bar{\Upsilon}_\varepsilon\eqdef \max\{1,\Upsilon_\varepsilon+|E_\varepsilon|\}
$.
Since the circuit is $k$-ary, we have
$
|E_\varepsilon|\in \mathcal{O}(N_\varepsilon)
$,
and therefore:
\[
\bar{\Upsilon}_\varepsilon \in
\begin{cases}
\mathcal{O}\big((\log \kappa)^{\max\{p,q\}}\big),
& \mbox{if } N_\varepsilon\in \mathcal{O}((\log \kappa)^p),\ \Upsilon_\varepsilon\in \mathcal{O}((\log \kappa)^q),
\\[0.4em]
\mathcal{O}(\kappa^p),
& \mbox{if } N_\varepsilon\in \mathcal{O}(\kappa^p),\ \Upsilon_\varepsilon\in \mathcal{O}((\log \kappa)^q),
\\[0.4em]
\mathcal{O}(\kappa^q),
& \mbox{if } N_\varepsilon\in \mathcal{O}((\log \kappa)^p),\ \Upsilon_\varepsilon\in \mathcal{O}(\kappa^q),
\\[0.4em]
\mathcal{O}\big(\kappa^{\max\{p,q\}}\big),
& \mbox{if } N_\varepsilon\in \mathcal{O}(\kappa^p),\ \Upsilon_\varepsilon\in \mathcal{O}(\kappa^q).
\end{cases}
\]

Assume now that
$
\Delta_\varepsilon\in \mathcal{O}((\log \kappa)^r)
$
with $r>0$.
Then
\[
\log\Big(\frac{4\Delta_\varepsilon}{\varepsilon}\Big)\in \mathcal{O}(\log \kappa),
\qquad
\log^{\circ 2}\Big(\frac{4\Delta_\varepsilon}{\varepsilon}\Big)\in \mathcal{O}(\log\log \kappa).
\]

Applying Theorem~\ref{thrm:concrete_surgery}, we obtain the following master size bounds.

\smallskip

\noindent
\textbf{Case 1:}
if
$
\mathbb{G}=\mathbb{G}_{alg}^{k,c}
$
or
$
\mathbb{G}=\mathbb{G}_{t\text{-}alg}^{k,c}
$,
then
\[
\mathrm{Size}(f_\varepsilon)
\in
\mathcal{O}\big(
d_0+\bar{\Upsilon}_\varepsilon\Delta_\varepsilon k
\big)
\subseteq
\mathcal{O}\big(
\bar{\Upsilon}_\varepsilon(\log \kappa)^r
\big).
\]

\noindent
\textbf{Case 2:}
if
$
\mathbb{G}=\mathbb{G}_{rt\text{-}alg}^{k,c,r^\star}
$,
then
\[
\mathrm{Size}(f_\varepsilon)
\in
\mathcal{O}\Big(
d_0
+
\bar{\Upsilon}_\varepsilon
\Big[
\Delta_\varepsilon k
+
\Delta_\varepsilon
\log\Big(\frac{4\Delta_\varepsilon}{\varepsilon}\Big)
\Big(
1+\log^{\circ 2}\Big(\frac{4\Delta_\varepsilon}{\varepsilon}\Big)
\Big)
\Big]
\Big),
\]
and therefore
\[
\mathrm{Size}(f_\varepsilon)
\subseteq
\mathcal{O}\big(
\bar{\Upsilon}_\varepsilon
(\log \kappa)^{r+1}
(1+\log\log \kappa)
\big).
\]

\noindent
\textbf{Case 3:}
if
$
\mathbb{G}=\mathbb{G}_{Rat}^{k,c,r^\star}
$,
then
\[
\mathrm{Size}(f_\varepsilon)
\in
\mathcal{O}\Big(
d_0
+
\bar{\Upsilon}_\varepsilon
\Big[
\Delta_\varepsilon k
+
\Delta_\varepsilon
s^{2\Delta_\varepsilon}
\log\Big(\frac{4\Delta_\varepsilon}{\varepsilon}\Big)
\Big(
\Delta_\varepsilon+\log^{\circ 2}\Big(\frac{4\Delta_\varepsilon}{\varepsilon}\Big)
\Big)
\Big]
\Big).
\]
Since $r>0$, we have
$
(\log \kappa)^r\gg \log\log \kappa
$
as $\kappa\uparrow\infty$, whence
\[
\mathrm{Size}(f_\varepsilon)
\subseteq
\mathcal{O}\big(
\bar{\Upsilon}_\varepsilon
(\log \kappa)^{2r+1}
s^{\mathcal{O}((\log \kappa)^r)}
\big).
\]

Substituting the above four possible scalings of $\bar{\Upsilon}_\varepsilon$ into these three master bounds yields exactly the entries in Table~\ref{tab:log_depth}.
\end{proof}

\begin{proof}[{Proof of Table~\ref{tab:poly_depth}}]
Fix $0<\varepsilon<1$, and set
$
\kappa\eqdef \tfrac{1}{\varepsilon}
$.
Let $\mathcal{C}_{\varepsilon}$ be a $\mathbb{G}$-circuit which $\varepsilon$-computes $f$ on $[-1,1]^d$, with depth $\Delta_\varepsilon$, width $\Upsilon_\varepsilon$, and using $N_\varepsilon$ gates.  Write
$
\bar{\Upsilon}_\varepsilon\eqdef \max\{1,\Upsilon_\varepsilon+|E_\varepsilon|\}
$.
As in the proof of Table~\ref{tab:log_depth}, since the circuit is $k$-ary we have
$
|E_\varepsilon|\in \mathcal{O}(N_\varepsilon)
$,
and therefore:
\[
\bar{\Upsilon}_\varepsilon \in
\begin{cases}
\mathcal{O}\big((\log \kappa)^{\max\{p,q\}}\big),
& \mbox{if } N_\varepsilon\in \mathcal{O}((\log \kappa)^p),\ \Upsilon_\varepsilon\in \mathcal{O}((\log \kappa)^q),
\\[0.4em]
\mathcal{O}(\kappa^p),
& \mbox{if } N_\varepsilon\in \mathcal{O}(\kappa^p),\ \Upsilon_\varepsilon\in \mathcal{O}((\log \kappa)^q),
\\[0.4em]
\mathcal{O}(\kappa^q),
& \mbox{if } N_\varepsilon\in \mathcal{O}((\log \kappa)^p),\ \Upsilon_\varepsilon\in \mathcal{O}(\kappa^q),
\\[0.4em]
\mathcal{O}\big(\kappa^{\max\{p,q\}}\big),
& \mbox{if } N_\varepsilon\in \mathcal{O}(\kappa^p),\ \Upsilon_\varepsilon\in \mathcal{O}(\kappa^q).
\end{cases}
\]

Assume now that
$
\Delta_\varepsilon\in \mathcal{O}(\kappa^r)
$
with $r>0$.
Then
\[
\log\Big(\frac{4\Delta_\varepsilon}{\varepsilon}\Big)\in \mathcal{O}(\log \kappa),
\qquad
\log^{\circ 2}\Big(\frac{4\Delta_\varepsilon}{\varepsilon}\Big)\in \mathcal{O}(\log\log \kappa).
\]

Applying Theorem~\ref{thrm:concrete_surgery}, we obtain the following master size bounds.

\smallskip

\noindent
\textbf{Case 1:}
if
$
\mathbb{G}=\mathbb{G}_{alg}^{k,c}
$
or
$
\mathbb{G}=\mathbb{G}_{t\text{-}alg}^{k,c}
$,
then
\[
\mathrm{Size}(f_\varepsilon)
\in
\mathcal{O}\big(
d_0+\bar{\Upsilon}_\varepsilon\Delta_\varepsilon k
\big)
\subseteq
\mathcal{O}\big(
\bar{\Upsilon}_\varepsilon\kappa^r
\big).
\]

\noindent
\textbf{Case 2:}
if
$
\mathbb{G}=\mathbb{G}_{rt\text{-}alg}^{k,c,r^\star}
$,
then
\[
\mathrm{Size}(f_\varepsilon)
\in
\mathcal{O}\Big(
d_0
+
\bar{\Upsilon}_\varepsilon
\Big[
\Delta_\varepsilon k
+
\Delta_\varepsilon
\log\Big(\frac{4\Delta_\varepsilon}{\varepsilon}\Big)
\Big(
1+\log^{\circ 2}\Big(\frac{4\Delta_\varepsilon}{\varepsilon}\Big)
\Big)
\Big]
\Big),
\]
and therefore
\[
\mathrm{Size}(f_\varepsilon)
\subseteq
\mathcal{O}\big(
\bar{\Upsilon}_\varepsilon
\kappa^r
\log \kappa
(1+\log\log \kappa)
\big).
\]

\noindent
\textbf{Case 3:}
if
$
\mathbb{G}=\mathbb{G}_{Rat}^{k,c,r^\star}
$,
then
\[
\mathrm{Size}(f_\varepsilon)
\in
\mathcal{O}\Big(
d_0
+
\bar{\Upsilon}_\varepsilon
\Big[
\Delta_\varepsilon k
+
\Delta_\varepsilon
s^{2\Delta_\varepsilon}
\log\Big(\frac{4\Delta_\varepsilon}{\varepsilon}\Big)
\Big(
\Delta_\varepsilon+\log^{\circ 2}\Big(\frac{4\Delta_\varepsilon}{\varepsilon}\Big)
\Big)
\Big]
\Big).
\]
Since $r>0$, we have
$
\kappa^r\gg \log\log \kappa
$
as $\kappa\uparrow\infty$, whence
$
\mathrm{Size}(f_\varepsilon)
\subseteq
\mathcal{O}\big(
\bar{\Upsilon}_\varepsilon
\kappa^{2r}
\log \kappa
\, s^{\mathcal{O}(\kappa^r)}
\big).
$
Substituting the above four possible scalings of $\bar{\Upsilon}_\varepsilon$ into these three master bounds yields exactly the entries in Table~\ref{tab:poly_depth}.
\end{proof}

\section{Proofs of Implications I: Minimax-Optimal Approximation}
\label{s:Proofs__approx_theorems}
\subsection{Proof of Optimal Approximation of Uniformly-Continuous Functions}

\begin{proof}[{Proof of Theorem~\ref{thrm:UAT__continouusLowreg}}]
The proof will first handle the finite $p$-case and then proceed to the sup-norm case.
\\
\noindent
\textit{Finite $L^p$-Norm Case:}
\hfill\\
Since $f$ is uniformly continuous and $\mathcal{X}\subseteq [0,1]^n$ is compact, $f$ is bounded on $\mathcal{X}$; whence applying the McShane extension theorem, cf.~\cite[Theorem 2 ]{beer2020mcshane}, we deduce the existence of a uniformly continuous function $f^{\uparrow}:[0,1]^n\to \mathbb{R}$ satisfying $f^{\uparrow}(x)=f(x)$ for all $x\in \mathcal{X}$ and such that $f^{\uparrow}$ has modulus of continuity $\omega$ also.  We thus approximate $f^{\uparrow}$ on $\mathcal{X}$ (plus some possible thickened region around $\mathcal{X}$, unambiguously) and, thus, notational convenience relabel $f^{\uparrow}$ as $f$.
Fix $\varepsilon>0$, and write
$
M_f\eqdef \sup_{x\in [0,1]^n}|f(x)|
<\infty
$.
Since $\mu(\mathcal{X})<\infty$ and $\lim\limits_{t\downarrow 0}\, \omega(t)= 0$, we may choose $0<r<1/4$ such that
\begin{equation}
\label{eq:choose_r}
    \mu(\mathcal{X})\,\omega(2r)^p
    \le
    \frac{\varepsilon}{2^p}
.
\end{equation}
Let $\{x_i\}_{i=1}^N\subseteq \mathcal{X}$ be an $r$-net of minimal cardinality, so that
$
N=N_r(\mathcal{X})
$
and
$
    \mathcal{X}
    \subseteq
    \bigcup_{i=1}^N B_2(x_i,r)
$.  
Fix $0<\eta<r$.
For each $i\in [N]_+$, let $u_i:\mathbb{R}^n\to [0,1]$ be the radial $C^1$ bump defined by
$
    u_i(x)
\eqdef
    \vartheta_{r,\eta} \big(\|x-x_i\|_2^2\big)
$, 
where $\vartheta_{r,\eta}:\mathbb{R}\to [0,1]$ satisfies
$
    \vartheta_{r,\eta}(t)=1
$
for $t\le r^2$,
and $\vartheta_{r,\eta}(t)=0$ whenever $t\ge (r+\eta)^2$.  
Hence, for every $i\in[N]_+$,
$
    u_i(x)=1
$ if $\|x-x_i\|_2\le r$ and 
$u_i(x)=0$ whenever $\|x-x_i\|_2\ge r+\eta$.  
In particular, since $\{B_2(x_i,r)\}_{i=1}^N$ covers $\mathcal{X}$, for every $x\in\mathcal{X}$ there exists at least one index $i$ with $u_i(x)=1$.
Now define
\[
    \theta_1(x)\eqdef u_1(x),
    \qquad
    \theta_i(x)\eqdef u_i(x)\prod_{j=1}^{i-1}\bigl(1-u_j(x)\bigr),
    \quad i=2,\dots,N
.
\]
Then each $\theta_i$ takes values in $[0,1]$, and
\[
    \sum_{i=1}^N \theta_i(x)
    =
    1-\prod_{i=1}^N\bigl(1-u_i(x)\bigr)
    =
    1
    \qquad
    \forall x\in \mathcal{X}
,
\]
because at least one of the factors $1-u_i(x)$ vanishes on $\mathcal{X}$.
Moreover, if $\theta_i(x)\neq 0$, then necessarily $u_i(x)\neq 0$, and therefore
$
\|x-x_i\|_2<r+\eta<2r
$.
Consider the map
\[
    g_r(x)
    \eqdef
    \sum_{i=1}^N f(x_i)\,\theta_i(x)
    =
    P\bigl(u_1(x),\dots,u_N(x)\bigr),
\]
where
$
    P(z_1,\dots,z_N)
    \eqdef
    \sum_{i=1}^N
        f(x_i)\,
        z_i\prod_{j=1}^{i-1}(1-z_j)
$.  
Since $\sum_{i=1}^N\theta_i(x)=1$ on $\mathcal{X}$ and each $\theta_i(x)\ge 0$, we obtain, for every $x\in \mathcal{X}$
\allowdisplaybreaks
\begin{align}
|f(x)-g_r(x)|
& =
    \left|
        \sum_{i=1}^N \theta_i(x)\bigl(f(x)-f(x_i)\bigr)
    \right|
\notag\\
& \le
    \sum_{i=1}^N
        \theta_i(x)\,
        |f(x)-f(x_i)|
\notag\\
& \le
    \sum_{i=1}^N
        \theta_i(x)\,
        \omega(\|x-x_i\|_2)
\notag\\
& \le
    \omega(2r)\sum_{i=1}^N \theta_i(x)
=
    \omega(2r)
.
\label{eq:exact_partition_error}
\end{align}
Consequently,~\eqref{eq:choose_r} implies that
\begin{equation}
\label{eq:exact_Lp_error}
    \int_{\mathcal{X}}
        |f(x)-g_r(x)|^p
    \,
    d\mu(x)
\le
    \mu(\mathcal{X})\,\omega(2r)^p
\le
    \frac{\varepsilon}{2^p}
.
\end{equation}
It remains to approximate $g_r$ by an $\mathcal{A}$NN.
Fix $\tau>0$ so small that
\begin{equation}
\label{eq:choose_tau}
    \mu(\mathcal{X})\,(2\tau)^p
    \le
    \frac{\varepsilon}{2^p}
.
\end{equation}
Since $P$ is continuous on the compact cube $[-1,2]^N$, it is uniformly continuous there.  Hence there exists $\delta\in (0,1)$ such that, whenever
$z,z'\in [-1,2]^N$ satisfy $\|z-z'\|_{\infty}\le \delta$, one has:
for all $z,z'\in [-1,2]^N$ 
\begin{equation}
\label{eq:P_uniform_cont}
    |P(z)-P(z')|
\le
    \tau
.
\end{equation}
For each $i\in[N]_+$, the map
$
x\mapsto \|x-x_i\|_2^2
$
is a quadratic polynomial on $\mathbb{R}^n$.
Therefore, by Proposition~\ref{prop:sparse_poly_gate__Pow_and_TreeMult}, there exists an $\mathcal{A}$NN
$
\Psi_i:\mathbb{R}^n\to\mathbb{R}
$
such that
\[
    \sup_{x\in [0,1]^n}
    \big|
        \Psi_i(x)-\|x-x_i\|_2^2
    \big|
    \le
    \delta
.
\]
By Proposition~\ref{prop:C1_bump} in dimension $k=1$, there exists an $\mathcal{A}$NN
$
B_{r,\eta,\delta}:\mathbb{R}\to\mathbb{R}
$
such that
$
    \sup_{t\in [0,n]}
    \big|
        B_{r,\eta,\delta}(t)-\vartheta_{r,\eta}(t)
    \big|
\le
    \delta
$.  
By composition, for each $i\in[N]_+$ there exists an $\mathcal{A}$NN
$
U_i\eqdef B_{r,\eta,\delta}\circ \Psi_i
$
satisfying
\begin{equation}
\label{eq:Ui_approx_ui}
    \sup_{x\in [0,1]^n}
    |U_i(x)-u_i(x)|
    \le
    \delta
.
\end{equation}
Now apply Proposition~\ref{prop:sparse_poly_gate__Pow_and_TreeMult} to the polynomial
$
P:\mathbb{R}^N\to\mathbb{R}
$
on the cube $[-1,2]^N$.
There exists an $\mathcal{A}$NN
$
Q:\mathbb{R}^N\to\mathbb{R}
$
such that
\begin{equation}
\label{eq:Q_approx_P}
    \sup_{z\in [-1,2]^N}
    |Q(z)-P(z)|
    \le
    \tau
.
\end{equation}
Let
$
    U(x)\eqdef \bigl(U_1(x),\dots,U_N(x)\bigr)
$ and $\Phi_{f:\varepsilon}(x)\eqdef Q\bigl(U(x)\bigr)$.  
Then, by~\eqref{eq:Ui_approx_ui}, for every $x\in [0,1]^n$,
$
    \|U(x)-(u_1(x),\dots,u_N(x))\|_{\infty}
\le
    \delta
$.  
Since each $u_i(x)\in [0,1]$ and $\delta<1$, both vectors lie in $[-1,2]^N$.
Therefore, by~\eqref{eq:P_uniform_cont} and~\eqref{eq:Q_approx_P}, for every $x\in [0,1]^n$ we have
\begin{align*}
    |\Phi_{f:\varepsilon}(x)-g_r(x)|
& \le
    \big|
        Q(U(x))-P(U(x))
    \big|
\\
&\quad +
    \big|
        P(U(x))-P(u_1(x),\dots,u_N(x))
    \big|
\le
    \tau+\tau
=
    2\tau
.
\end{align*}
Hence
\begin{equation}
\label{eq:network_Lp_error}
    \int_{\mathcal{X}}
        |\Phi_{f:\varepsilon}(x)-g_r(x)|^p
    \,
    d\mu(x)
    \le
    \mu(\mathcal{X})\,(2\tau)^p
    \le
    \frac{\varepsilon}{2^p}
,
\end{equation}
by~\eqref{eq:choose_tau}.  By~\eqref{eq:exact_Lp_error},~\eqref{eq:network_Lp_error}, and the elementary inequality
$
|a+b|^p\le 2^{p-1}(|a|^p+|b|^p)
$,
we find that
\begin{align*}
    \int_{\mathcal{X}}
        |f(x)-\Phi_{f:\varepsilon}(x)|^p
    \,
    d\mu(x)
& \le
    2^{p-1}
    \int_{\mathcal{X}}
        |f(x)-g_r(x)|^p
    \,
    d\mu(x)
\\
& \quad
    +
    2^{p-1}
    \int_{\mathcal{X}}
        |g_r(x)-\Phi_{f:\varepsilon}(x)|^p
    \,
    d\mu(x)
\\
& \le
    2^{p-1}
    \left(
        \frac{\varepsilon}{2^p}
        +
        \frac{\varepsilon}{2^p}
    \right)
=
    \varepsilon
.
\end{align*}
To tally the architecture complexity, choose $c_{\mathcal X,\mu,p}>0$ small enough that, with
$
r_{\varepsilon}\eqdef \frac12\,\omega^{-1}(c_{\mathcal X,\mu,p}\varepsilon^{1/p}),
$
one has
$
    \mu(\mathcal X)\,\omega(2r_{\varepsilon})^p
    \le
    \varepsilon/2^p
$.
Write
$
N_{\varepsilon}\eqdef N_{r_{\varepsilon}}(\mathcal{X})
$
and
$
L_{\varepsilon}\eqdef
\log \left(
    1/(\varepsilon\,\omega^{-1}(c_{\mathcal X,\mu,p}\varepsilon^{1/p}))
\right).
$
Since $\mu$ is $d$-Ahlfors regular on $\mathcal{X}$, the covering estimate gives
$
N_{\varepsilon}
=
\mathcal{O} \left(
    (\omega^{-1}(c_{\mathcal X,\mu,p}\varepsilon^{1/p}))^{-d}
\right),
$
where the implicit constant depends on the Ahlfors-regularity constants of $\mathcal X$.

For each $i\in [N_{\varepsilon}]_+$, Proposition~\ref{prop:sparse_poly_gate__Pow_and_TreeMult} gives a network $\Psi_i$ approximating $x\mapsto \|x-x_i\|_2^2$ with depth, width, and number of non-zero parameters all of order $\mathcal{O}(1)$; while Proposition~\ref{prop:C1_bump} gives a scalar bump network $B_{r,\eta,\delta}$ with depth $\mathcal{O}(L_{\varepsilon}^2)$, width $\mathcal{O}(1)$, and $\mathcal{O}(L_{\varepsilon}^2)$ non-zero parameters. Hence each composition
$
U_i\eqdef B_{r,\eta,\delta}\circ \Psi_i
$
has depth $\mathcal{O}(L_{\varepsilon}^2)$, width $\mathcal{O}(1)$, and $\mathcal{O}(L_{\varepsilon}^2)$ non-zero parameters. Parallelizing the $N_{\varepsilon}$ blocks yields the map
$
U=(U_1,\dots,U_{N_{\varepsilon}})
$
with depth $\mathcal{O}(L_{\varepsilon}^2)$, width $\mathcal{O}(N_{\varepsilon})$, and $\mathcal{O}(N_{\varepsilon}L_{\varepsilon}^2)$ non-zero parameters.

The selector polynomial
$
P(z)
=
\sum_{i=1}^{N_{\varepsilon}}
f(x_i)\,z_i\prod_{j=1}^{i-1}(1-z_j)
$
has $S=N_{\varepsilon}$ summands and maximal degree $\Delta=N_{\varepsilon}$. Therefore, Proposition~\ref{prop:sparse_poly_gate__Pow_and_TreeMult} yields a network $Q$ approximating $P$ with depth $\mathcal{O}(\log N_{\varepsilon})$, width $\mathcal{O}(N_{\varepsilon}^2)$, and $\mathcal{O}(N_{\varepsilon}^2)$ non-zero parameters. Since $\log N_{\varepsilon}\lesssim L_{\varepsilon}$, the composition
$
\Phi_{f:\varepsilon}=Q\circ U
$
has depth $\mathcal{O}(L_{\varepsilon}^2)$, width $\mathcal{O}(N_{\varepsilon}^2)$, and
$
\mathcal{O}(N_{\varepsilon}^2+N_{\varepsilon}L_{\varepsilon}^2)
$
non-zero parameters. Substituting the preceding estimate on $N_{\varepsilon}$ gives:
Moreover, there exists a constant $c_{\mathcal{X},\mu,p}>0$%
\footnote{Depending only on $p$,
$\mu(\mathcal{X})$, and the geometric constants of $\mathcal{X}$.}%
, such that
$\Phi_{f:\varepsilon}$ has depth
$
\mathcal{O} \left(
    \log^2 \left(
        1/(\varepsilon\,\omega^{-1}(c_{\mathcal{X},\mu,p}\varepsilon^{1/p}))
    \right)
\right)
$,
width
$
\mathcal{O} \left(
    (\omega^{-1}(c_{\mathcal{X},\mu,p}\varepsilon^{1/p}))^{-2d}
\right)
$,
and its number of non-zero parameters is
\[
\mathcal{O} \left(
    (\omega^{-1}(c_{\mathcal{X},\mu,p}\varepsilon^{1/p}))^{-2d}
    +
    (\omega^{-1}(c_{\mathcal{X},\mu,p}\varepsilon^{1/p}))^{-d}
    \log^2 \left(
        1/(\varepsilon\,\omega^{-1}(c_{\mathcal{X},\mu,p}\varepsilon^{1/p}))
    \right)
\right)
=
\mathcal{O} \left(
    (\omega^{-1}(c_{\mathcal{X},\mu,p}\varepsilon^{1/p}))^{-2d}
\right)
.
\]
This concludes the proof of the finite $L^p$-norm base.

\noindent
\textit{Sup-Norm Case:}
\hfill\\
Since $\mathcal{X}\subseteq \mathbb{R}^n$ is compact, it is closed and bounded by the Heine-Borel theorem; whence there exists some $M>0$ such that $\mathcal{X}\subseteq [-M,M]^n$.  Since $f:\mathcal{X}\to \mathbb{R}$ is continuous and since $\mathcal{X}$ is compact, the Dugundji-Tietze Extension theorem, cf.~\cite[Theorem 5.1]{dugundji1951extension}, implies that there exists a continuous map $f^{\uparrow}:\mathbb{R}^n\to \mathbb{R}$ satisfying: 
\begin{equation}
\label{eq:Extension}
    f(x)=f^{\uparrow}(x)
\end{equation}
for every $x\in \mathcal{X}$.
We focus on the restriction of $f^{\uparrow}$ to $[-M,M]^n$, which we denote by $F\eqdef f^{\uparrow}|_{[-M,M]^n}$; and we remark that by~\eqref{eq:Extension}, the map $F$ satisfies:
\begin{equation}
\label{eq:Extension_dos}
    f(x)=F(x)
\end{equation}
for every $x\in \mathcal{X}$.
Now, the compactness of $[-M,M]^n$ and the continuity of $F$ allow us to invoke
the Stone-Weierstrass theorem guaranteeing that: for every $\varepsilon>0$ there exists a polynomial $p_{\varepsilon}:\mathbb{R}^n\to \mathbb{R}$ satisfying
\begin{equation}
\label{eq:density_stoneWeirestrass}      
        \sup_{x\in [-M,M]^n}\, 
            |F(x)-p_{\varepsilon}(x)|
    \le 
        \frac{\varepsilon}{2}
.
\end{equation}
Applying Proposition~\ref{prop:sparse_poly_gate__Pow_and_TreeMult} to the fixed polynomial $p_{\varepsilon}$, we deduce that there is an $\mathcal{A}$NN $\Phi:\mathbb{R}^n \to \mathbb{R}$ satisfying
\begin{equation}
\label{eq:poly_approximation__bound}      
        \sup_{x\in [-M,M]^n}\, 
            |\Phi(x)-p_{\varepsilon}(x)|
    \le 
        \frac{\varepsilon}{2}
.
\end{equation}
Upon combining~\eqref{eq:Extension_dos} with~\eqref{eq:density_stoneWeirestrass} and~\eqref{eq:poly_approximation__bound}, we find that
\allowdisplaybreaks
\begin{align*}
    \sup_{x\in \mathcal{X}}\,
        |f(x)-\Phi(x)|
& =
    \sup_{x\in \mathcal{X}}\,
        |F(x)-\Phi(x)|
\\
& \le
    \sup_{x\in [-M,M]^n}\,
        |F(x)-\Phi(x)|
\\
& \le
    \sup_{x\in [-M,M]^n}\,
        |F(x)-p_{\varepsilon}(x)|
    +
    \sup_{x\in [-M,M]^n}\,
        |p_{\varepsilon}(x)-\Phi(x)|
\\
& \le
    \frac{\varepsilon}{2}+ \frac{\varepsilon}{2} =\varepsilon.
\end{align*}
This completes the proof of the second case.
\end{proof}

\subsection{Proof of Optimal Approximation of Besov Functions}
\label{s:Proofs__ss:thrm:Besov_cube_Lp_ANN}
\begin{proof}[{Proof of Theorem~\ref{thrm:Besov_cube_Lp_ANN}}]
Since $\mathcal{X}$ is an $(\epsilon,\delta)$-domain, the main results of~\cite{rogers2006degree} 
imply that there exists a bounded linear extension operator $\mathcal{E}:B_{p,q}^{\alpha}(\mathcal{X})\to
		B_{p,q}^{\alpha}(\mathbb{R}^d)$.  Upon composing the extension operator $\mathcal{E}$ with the restriction operator%
\footnote{Defined on smooth test functions supported on $(0,1)^n$; cf.~\cite{triebel2008function}.}~%
$\mathcal{E}:B_{p,q}^{\alpha}(\mathbb{R}^d)\ni f \mapsto f|_{(0,1)^n}\in B_{p,q}^{\alpha}((0,1)^n)$ we obtain an $f^{\uparrow}\in B_{p,q}^{\alpha}((0,1)^n)$ whose restriction to $\mathcal{X}$ coincides with $f$ and such that $\|f^{\uparrow}\|_{B_{p,q}^{\alpha}((0,1)^n)} \le C\, \|f^{\uparrow}\|_{B_{p,q}^{\alpha}((0,1)^n)} $ for some absolute constant $C>0$; i.e., their Besov norms differ only by an absolute constant.  Thus, we may, without loss of generality, approximate $f^{\uparrow}$ on $(0,1)^n$ and then infer an approximation of $f$ on $\mathcal{X}$ by restriction; using the simple fact that $\|\cdot\|_{L^p(\mathcal{X})}\le  \|\cdot\|_{L^p( (0,1)^n)}$.  Thus, for convenience, we may relabel $f^{\uparrow}$ by $f$ and argue on the cube $(0,1)^n$.

Without loss of generality, we may assume that $0<\varepsilon$.
By the standard spline characterization of Besov spaces on $[0,1]^d$, there exist coefficients
$
\{\beta_{k,\nu}(f)\}_{k\in\mathbb{N}_0,\nu\in\Lambda_k}
$
such that
$
f=\sum_{k=0}^{\infty}\sum_{\nu\in\Lambda_k}\beta_{k,\nu}(f)\mathfrak{S}_{k,\nu}^{(r)}
$
in $L^p([0,1]^d)$, and such that, writing
$
A_k(f)\eqdef \Big(\sum_{\nu\in\Lambda_k}|\beta_{k,\nu}(f)|^p 2^{-kd}\Big)^{1/p},
$
one has
$
\sup_{k\in\mathbb{N}_0}2^{ks}A_k(f)<\infty
$.
Hence there exists a constant $C_f>0$ such that
\begin{equation}
\label{eq:Ak_decay_cube}
A_k(f)\le C_f\,2^{-ks}
\qquad
\text{for every }k\in\mathbb{N}_0
.
\end{equation}

For $J\in\mathbb{N}_0$, define the spline truncation
$
S_Jf\eqdef \sum_{k=0}^{J}\sum_{\nu\in\Lambda_k}\beta_{k,\nu}(f)\mathfrak{S}_{k,\nu}^{(r)}
$.
Since the family $\{\mathfrak{S}_{k,\nu}^{(r)}\}_{\nu\in\Lambda_k}$ has uniformly bounded overlap and each atom is bounded by $1$, there exists a constant $C_{d,p,r}>0$ such that, for every finitely supported family $\{c_\nu\}_{\nu\in\Lambda_k}$,
$
    \Big\|
    \sum_{\nu\in\Lambda_k} c_\nu \mathfrak{S}_{k,\nu}^{(r)}
    \Big\|_{L^p([0,1]^d)}
\le
    C_{d,p,r}
    \Big(
    \sum_{\nu\in\Lambda_k}|c_\nu|^p 2^{-kd}
    \Big)^{1/p}
$. 
Therefore,
\begin{align}
\|f-S_Jf\|_{L^p([0,1]^d)}
&\le
\sum_{k=J+1}^{\infty}
\Big\|
\sum_{\nu\in\Lambda_k}\beta_{k,\nu}(f)\mathfrak{S}_{k,\nu}^{(r)}
\Big\|_{L^p([0,1]^d)}
\notag
\\
&\le
C_{d,p,r}
\sum_{k=J+1}^{\infty}A_k(f)
\notag
\\
&\le
C_{d,p,r}C_f
\sum_{k=J+1}^{\infty}2^{-ks}
\notag
\\
&\le
C_f'2^{-Js}
\label{eq:tail_Lp_cube}
\end{align}
for some constant $C_f'>0$ depending only on $d,p,q,s,r$ and on $f$.
Setting $J=J(\varepsilon)\in\mathbb{N}_0$ so that
$
C_f'2^{-Js}\le (\varepsilon/2^p)^{1/p}
$.  
Thus, 
$
J=\mathcal{O}(\log(1/\varepsilon))
$
and, since $|\Lambda_k|\asymp_{d,r}2^{kd}$ then we have
\begin{equation}
\label{eq:mJ_cube_again}
m_J\eqdef \sum_{k=0}^{J}|\Lambda_k|
\asymp_{d,r}
2^{Jd}
=
\mathcal{O}\big(\varepsilon^{-d/(sp)}\big)
.
\end{equation}
Now fix $0<\delta<1$, to be chosen below.  For each $0\le k\le J$ and each $\nu\in\Lambda_k$, since $x\in[0,1]^d$ implies
$
|2^kx_i-\nu_i|\le 2^J+r
$
for every $i\in[d]$, Proposition~\ref{prop:tensor_product_Bsplines} yields an $\mathcal{A}$NN
$
\widehat{\mathfrak{S}}_{k,\nu:\delta}:\mathbb{R}^d\to\mathbb{R}
$
such that
\begin{equation}
\label{eq:atom_error_uniform_cube}
\sup_{x\in[0,1]^d}
\big|
\widehat{\mathfrak{S}}_{k,\nu:\delta}(x)-\mathfrak{S}_{k,\nu}^{(r)}(x)
\big|
\le
\delta
,
\end{equation}
and such that
$
\widehat{\mathfrak{S}}_{k,\nu:\delta}
$
has depth
$
\mathcal{O}\big(\log^2(2^J/\delta)\big)
$,
width
$
\mathcal{O}(1)
$,
and
$
\mathcal{O}\big(\log^2(2^J/\delta)\big)
$
non-zero parameters.
Here we used that $d$ and $r$ are fixed.  This is exactly the spline-atom implementation block from tensorized cardinal $B$-spline proposition. 
Enumerate the active atoms
$
\{\mathfrak{S}_{k,\nu}^{(r)}:\ 0\le k\le J,\ \nu\in\Lambda_k\}
$
as
$
\{\mathfrak{S}^{(1)},\dots,\mathfrak{S}^{(m_J)}\}
$,
and similarly enumerate their approximators as
$
\{\widehat{\mathfrak{S}}^{(1)}_{\delta},\dots,\widehat{\mathfrak{S}}^{(m_J)}_{\delta}\}
$.
Let
$
D_{m_J}:\mathbb{R}^d\to\mathbb{R}^{dm_J}
$
be the affine duplication map
$
D_{m_J}(x)\eqdef (x,\dots,x)
$.
Applying Lemma~\ref{lem:subnetwork_alignment} to the family
$
\{\widehat{\mathfrak{S}}^{(\ell)}_{\delta}\}_{\ell=1}^{m_J}
$
and then composing with $D_{m_J}$, we obtain an $\mathcal{A}$NN
$
\mathbf{S}_{J,\delta}:\mathbb{R}^d\to\mathbb{R}^{m_J}
$
whose coordinates are precisely the sub-networks
$
\widehat{\mathfrak{S}}^{(\ell)}_{\delta}
$.
It has depth
$
\mathcal{O}\big(\log^2(2^J/\delta)\big)
$,
width
$
\mathcal{O}(m_J)
$,
and
$
\mathcal{O}\big(m_J\log^2(2^J/\delta)\big)
$
non-zero parameters.
Now define
\begin{equation}
\label{eq:def_hatSJ_again}
\widehat{S}_{J,\delta}f(x)
\eqdef
\sum_{k=0}^{J}\sum_{\nu\in\Lambda_k}
\beta_{k,\nu}(f)\widehat{\mathfrak{S}}_{k,\nu:\delta}(x)
,
\qquad x\in\mathbb{R}^d
.
\end{equation}
This is obtained by composing $\mathbf{S}_{J,\delta}$ with one final affine output layer, and hence
$
\widehat{S}_{J,\delta}f
$
has the same depth, width
$
\mathcal{O}(m_J)
$,
and
$
\mathcal{O}\big(m_J\log^2(2^J/\delta)\big)
$
non-zero parameters.
It remains to bound
$
\|S_Jf-\widehat{S}_{J,\delta}f\|_{L^p([0,1]^d)}
$.
Since $|[0,1]^d|=1$, \eqref{eq:atom_error_uniform_cube} implies
\begin{align}
\|S_Jf-\widehat{S}_{J,\delta}f\|_{L^p([0,1]^d)}
&\le
\sup_{x\in[0,1]^d}
\big|
S_Jf(x)-\widehat{S}_{J,\delta}f(x)
\big|
\notag
\\
&\le
\delta
\sum_{k=0}^{J}\sum_{\nu\in\Lambda_k}
|\beta_{k,\nu}(f)|
.
\label{eq:uniform_to_Lp_cube}
\end{align}
For each fixed $k$, Hölder's inequality and \eqref{eq:Ak_decay_cube} give
\begin{align}
\sum_{\nu\in\Lambda_k}|\beta_{k,\nu}(f)|
&\le
|\Lambda_k|^{1-\frac1p}
\Big(
\sum_{\nu\in\Lambda_k}|\beta_{k,\nu}(f)|^p
\Big)^{1/p}
\notag
\\
&=
|\Lambda_k|^{1-\frac1p}2^{kd/p}A_k(f)
\notag
\\
&\lesssim_{d,r}
2^{kd(1-\frac1p)}2^{kd/p}2^{-ks}
\notag
\\
&\lesssim_f
2^{k(d-s)}
.
\label{eq:coeff_sum_per_scale}
\end{align}
Therefore, $
\sum_{k=0}^{J}\sum_{\nu\in\Lambda_k}
|\beta_{k,\nu}(f)|
\lesssim_f
\sum_{k=0}^{J}2^{k(d-s)}
\lesssim_f
(J+1)\,2^{J(d-s)_+}
$.
Choose $\delta=\delta(\varepsilon)$ so that $
\delta\,(J+1)\,2^{J(d-s)_+}
\le
(\varepsilon/2^p)^{1/p}/C_f''
$, 
where $C_f''>0$ is the implicit constant. Then
$
    \|S_Jf-\widehat{S}_{J,\delta}f\|_{L^p([0,1]^d)}
\le
    (\varepsilon/2^p)^{1/p}
$.
Combining this with \eqref{eq:tail_Lp_cube} and the triangle inequality yields
$
\|f-\widehat{S}_{J,\delta}f\|_{L^p([0,1]^d)}
\le
\varepsilon^{1/p}
$.
Set
$
\Phi_{\varepsilon,f}\eqdef \widehat{S}_{J,\delta}f
$.
By definition of $m_J$, we have
$
J=\mathcal{O}(\log(1/\varepsilon))
$
and
$
\log(1/\delta)=\mathcal{O}(\log(1/\varepsilon))
$.
Hence
$
\log(2^J/\delta)=\mathcal{O}(\log(1/\varepsilon))
$.
Using \eqref{eq:mJ_cube_again}, we conclude that
$
\Phi_{\varepsilon,f}
$
has depth
$
\mathcal{O}\big(\log^2(1/\varepsilon)\big)
$,
width
$
\mathcal{O}\big(\varepsilon^{-d/(sp)}\big)
$,
and
$
\mathcal{O}\big(\varepsilon^{-d/(sp)}\log^2(1/\varepsilon)\big)
$
non-zero parameters.
\end{proof}

\subsection{Proof of Optimal Approximation of Holomorphic functions}
\label{app:proof_holo}

Before proving Theorem~\ref{thrm:HolomorphicApproximation}, which is the main goal of this section, we introduce some technical ingredients regarding polynomial approximation that are needed for the proof. We start by defining orthogonal Legendre polynomials. In particular, we consider the basis $(\Psi_\nu)_{\nu \in \mathbb{N}_0^d}$ of $L^2([-1,1]^d)$-orthonormal tensorized Legendre polynomials, defined as 
\begin{equation}
\label{eq:def_Legendre}
\Psi_{\nu} = \bigotimes_{i\in \mathrm{supp}(\nu)} \psi_{\nu_i}, \quad \forall \nu =(\nu_i)_{i=1}^d \in \mathbb{N}_0^d,
\end{equation}
where $(\psi_n)_{n=0}^\infty$ is the one-dimensional $L^2$-orthonormal basis of Legendre polynomials over the interval $[-1,1]$; see, e.g., \cite[Section 2.2.2]{adcock2022sparse} for an analytic definition of $\psi_n$. Here the space $L^2([-1,1])$ is equipped with the uniform probability measure (with constant density $1/2$) and, analogously, $L^2([-1,1]^d)$ is equipped with the uniform probability measure (having constant density $1/2^d$). For an introduction to multivariate Legendre polynomials we refer the reader to, e.g., \cite[Chapter 2]{adcock2022sparse}. Before stating the main auxiliary result, Lemma~\ref{lem:exp_rates_Legendre}, we introduce the notion of \emph{lower} (also known as \emph{downward closed}) multi-index set, which will be instrumental in our analysis. Figure~\ref{fig:lower} provides examples of lower and non-lower multi-index sets of $\mathbb{N}_0^2$.
\begin{figure}[t]
\centering
\includegraphics[width=0.4\textwidth]{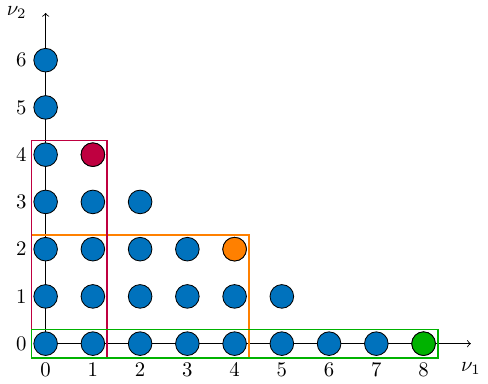}
\includegraphics[width=0.4\textwidth]{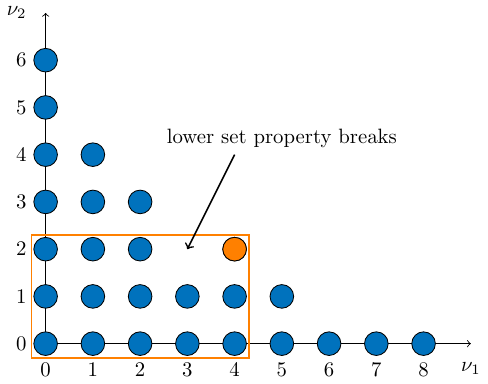}
\caption{\label{fig:lower}Illustration of lower and non-lower sets of $\mathbb{N}_0^2$. \textit{Left:} the blue dots provide an example of lower set $S \subseteq \mathbb{N}_0^2$. Dots and rectangles in red, orange, and green color illustrate the property that $\times_{i=1}^d [0,\nu_i] \cap \mathbb{N}_0^d \subseteq S$ for each $\nu \in S$. \textit{Right:} a multi-index set of $\mathbb{N}_0^2$ not satisfying the lower set property. The lattice points in the orange box corresponding to the multi-index $(4,2)$ in orange are not included in $S$.}
\end{figure}
\begin{definition}[Lower Set]
\label{def:lower_set}
A multi-index set $S \subseteq \mathbb{N}_0^d$ is called a \emph{lower set} if the following property holds for every $\nu, \mu \in \mathbb{N}_0^d$:
$$
\text{if }\nu \in S \text{ and } \mu \leq \nu, \text{ then } \mu \in S.
$$
Here the inequality is intended componentwise, i.e., $\mu \leq \nu$ means that  $\mu_i \leq \nu_i$ for every $i \in [d]$.
\end{definition}
A useful property of lower sets that will be used later is that for every $\nu \in S$, the box $\prod_{i=1}^d [0,\nu_i] \cap \mathbb{N}_0^d$ is contained in $S$ (see Figure~\ref{fig:lower}).
We are now ready to state the main approximation lemma. The version stated here corresponds to \cite[Theorem~3.15]{adcock2022sparse} (with the choice $\epsilon = 1$), based on a result from \cite{opschoor2022exponential}. 

\begin{lemma}[Exponential Rates via Legendre Polynomials]
\label{lem:exp_rates_Legendre}
Let $d\in \mathbb{N}$, $\gamma > 0$ and $f \in \mathcal{H}_{d, \gamma}$ from Definition~\ref{def:hidden_aniso}. Then, for every $s \in \mathbb{N}$ large enough, there exists a lower set $S \subseteq \mathbb{N}_0^d$ of cardinality $|S|\leq s$ and a coefficient vector $(c_\nu)_{\nu \in S}$ such that 
\begin{equation}
\label{eq:exp_decay_rate}
\left\|f - \sum_{\nu \in S} c_{\nu} \Psi_{\nu}\right\|_{L^\infty([-1,1]^d)} \leq \exp( - \gamma s^{1/d}), 
\end{equation}
where $(\Psi_\nu)_{\nu \in \mathbb{N}_0^d}$ are the $L^2([-1,1]^d)$-orthogonal tensorized Legendre polynomials defined in \eqref{eq:def_Legendre}.
\end{lemma}

We are now ready to prove Theorem~\ref{thrm:HolomorphicApproximation}.
\begin{proof}[Proof of Theorem~\ref{thrm:HolomorphicApproximation}]
The idea is to emulate sparse polynomial approximation based on Legendre polynomials with $\mathcal{A}$NNs and invoke Lemma~\ref{lem:exp_rates_Legendre} to obtain a convergence rate. With this aim, the first step is to construct $\mathcal{A}$NN approximations to Legendre orthogonal polynomials. First note that we can write
$$
\Psi_{\nu} = \bigotimes_{i\in \mathrm{supp}(\nu)} \psi_{\nu_i}, \quad \forall \nu =(\nu_i)_{i=1}^d \in \mathbb{N}_0^d,
$$
because $\psi_0 \equiv 1$ (recall that we are equipping $L^2([-1,1])$ with the uniform probability measure, hence $1$ is the right scaling), which means that we can omit the constant terms from the above product. Now, observe that each $\psi_n$ has $n$ distinct real roots $r_1^{(n)}, \ldots, r_n^{(n)} \in (-1,1)$ and that its leading coefficient is $a_n\eqdef \sqrt{2n+1}\,2^{-n}\binom{2n}{n}$ (see, e.g., \cite[Eq.~(2.11)]{adcock2022sparse}). Therefore, the fundamental theorem of algebra yields
$$
\Psi_\nu(x) 
= \prod_{i \in \mathrm{supp}(\nu)} \psi_{\nu_i}(x_i)
= \prod_{i \in \mathrm{supp}(\nu)} \prod_{j=1}^{\nu_i} a_{\nu_i}^{1/\nu_i} \cdot (x_i - r_{j}^{(\nu_i)}), \quad \forall x \in \mathbb{R}^d, \quad \forall \nu \in \mathbb{N}_0^d.
$$
This formula allows us to approximate each polynomial $\Psi_{\nu}(x)$ by composing the affine map 
$$
\mathbb{R}^d \ni x \mapsto (a_{\nu_i}^{1/\nu_i} \cdot (x_i - r_{j}^{(\nu_i)}))_{i \in \mathrm{supp}(\nu),\; j \in [\nu_i]} =: A^{(\nu)}(x)\in \bigtimes_{i \in \mathrm{\supp(\nu)}}\mathbb{R}^{\nu_i} \cong \mathbb{R}^{\|\nu\|_1}
$$ 
with $\mathrm{Mult}_{\delta, M}^{k}(x)$ from Proposition~\ref{prop:kth_multiplication__SK__polar__precise}, i.e., the $\mathcal{A}$NN  that emulates the $k$-fold multiplication gate. Specifically, we define
$$
\Phi_\nu(x)  \eqdef  \mathrm{Mult}_{\delta, M(\nu)}^{\|\nu\|_1}\left(\mathrm{vec}(A^{(\nu)}(x))\right),
$$
with $\delta >0$ a parameter that will be picked later, $\mathrm{vec}(\cdot)$ the vectorization operation that maps a matrix into the vector obtained by stacking its columns atop, and where we let $M(\nu) \eqdef 2\max\{1,a_{\nu_i}^{1/\nu_i}:i\in\mathrm{supp}(\nu)\}$, which is a uniform bound to the absolute entries of $A^{(\nu)}(x)$ for all $x \in [-1,1]^d$. Such a network satisfies 
\begin{equation}
\label{eq:approx_Legendre_poly_with_nets}
\|\Psi_\nu - \Phi_\nu\|_{L^{\infty}([-1,1]^d)} \leq \delta
\end{equation}
by construction. 

Now recall from Proposition~\ref{prop:kth_multiplication__SK__polar__precise} that $\mathrm{Mult}_{\delta, M}^{k}$ has depth at most $\mathcal{O}(\log  k)$, width at most $\mathcal{O}(k \cdot \max\{m,n\})$, and at most
$\mathcal{O}(k \cdot \|x^\star\|_0)$
non-zero parameters. Moreover, the affine map $A^{(\nu)}$ has depth $\mathcal{O}(1)$, width $\mathcal{O}(\max\{d,\|\nu\|_1\})$ and at most $\mathcal{O}(\|\nu\|_1)$ non-zero parameters. This implies that $\Phi_\nu$ has depth at most $\mathcal{O}(\log \|\nu\|_1)$, width at most $\mathcal{O}(\max\{\|\nu\|_1 \cdot \max\{m,n\}, d\})$, and at most
$\mathcal{O}(\|\nu\|_1 \cdot \|x^\star\|_0)$ non-zero parameters.

Now let $s \in \mathbb{N}$ large enough to be determined later. We invoke Lemma~\ref{lem:exp_rates_Legendre} and let $S\subseteq\mathbb{N}_0^d$ be the corresponding lower set of cardinality at most $s$ realizing the exponential error decay rate \eqref{eq:exp_decay_rate}. The next step is to construct an $\mathcal{A}$NN that is able to accurately approximate linear combinations of Legendre polynomials $(\Psi_{\nu})_{\nu \in S}$. This is simply obtained by parallelizing the networks $\Phi_{\nu}$ for $\nu \in S$ and adding a last linear layer. This yields an $\mathcal{A}$NN $\Phi$ of the form
$$
\Phi := \sum_{\nu \in S} c_{\nu} \Phi_{\nu},
$$
where $(c_{\nu})_{\nu \in S}$ are the coefficients from Lemma~\ref{lem:exp_rates_Legendre}. Thanks to \eqref{eq:exp_decay_rate} and  \eqref{eq:approx_Legendre_poly_with_nets}, we see that 
\begin{align*}
\|f-\Phi\|_{L^\infty([-1,1]^d)}
& \leq \left\|f-\sum_{\nu \in S} c_{\nu} \Psi_{\nu}\right\|_{L^\infty([-1,1]^d)} + \left\|\sum_{\nu \in S} c_{\nu} (\Psi_{\nu}-\Phi_{\nu})\right\|_{L^\infty([-1,1]^d)}\\ 
& \leq \exp(-\gamma \cdot s^{1/d}) + \|(c_\nu)_{\nu \in S}\|_1 \cdot \delta.
\end{align*}
Hence, choosing $\delta = \exp(-\gamma \cdot s^{1/d}) / \|(c_\nu)_{\nu \in S}\|_1$ leads to 
\begin{equation}
\label{eq:exp_error_bound_Phi}
\|f-\Phi\|_{L^\infty([-1,1]^d)}
\leq 2 \exp(-\gamma \cdot s^{1/d}).
\end{equation}
The network $\Phi$ has depth at most $\mathcal{O}\left(\log \left(\max_{\nu \in S}\|\nu\|_1\right)\right)$, width at most $\mathcal{O}(|S| \cdot \max\{\max_{\nu \in S}\|\nu\|_1 \cdot \max\{m,n\}, d\})$, and at most
$\mathcal{O}(\sum_{\nu \in S}\|\nu\|_1 \|x^\star\|_0)$ non-zero parameters.

Note that since $S$ is a lower set (recall Definition~\ref{def:lower_set}), for any $\nu \in S$ we have
$$
\|\nu\|_1 
\leq \prod_{i=1}^d (\nu_i + 1)
=|\{\mu \in \mathbb{N}_0^d : \mu \leq \nu \}|
\leq |S|,
$$
where the last inequality holds because $\times_{i=1}^d [0,\nu_i] \cap\mathbb{N}_0^d \subseteq S$ as soon as $\nu \in S$. This fact is immediately implied by the definition of lower set (see also Figure~\ref{fig:lower}). Hence, since $|S|\leq s$ we have $\max_{\nu \in S} \|\nu\|_1 \leq s$. This implies that $\Phi$ has depth at most $\mathcal{O}\left(\log s\right)$, width at most $\mathcal{O}(s \cdot \max\{ s\cdot  \max\{m,n\}, d\})$, and at most
$\mathcal{O}(s^2 \cdot \|x^\star\|_0)$ non-zero parameters.

We conclude by making a suitable choice for $s \in \mathbb{N}$. In view of \eqref{eq:exp_error_bound_Phi}, we choose 
$$
s = s(d, \varepsilon,\gamma) := \left\lceil(\gamma^{-1} \log(2\varepsilon^{-1}))^d\right\rceil
$$
and assume that $\varepsilon > 0$ is small enough such that this choice is admissible (recall that Lemma~\ref{lem:exp_rates_Legendre} holds for $s \in \mathbb{N}$ large enough). This implies $\|f-\Phi\|_{L^\infty([-1,1]^d)}
\leq \varepsilon$, as desired. The final step is to write the width, depth and non-zero parameters bounds for $\Phi$ as a function of the parameter $\varepsilon>0$. Plugging the formula for $s$ into the bounds just obtained, we see that the network $\Phi$ has depth at most $\mathcal{O}\left(d\log(\gamma^{-1})+d\log\log(\varepsilon^{-1})\right)$, width at most $\mathcal{O}( (\gamma^{-1} \log(\varepsilon^{-1}))^{2d}\cdot  \max\{m,n\})$ (where we are implicitly assuming $\varepsilon$ small enough such that $s \cdot \max\{m,n\} \geq d$), and at most
$\mathcal{O}((\gamma^{-1} \log(\varepsilon^{-1}))^{2d} \cdot \|x^\star\|_0)$ non-zero parameters. Letting $\Phi_{f,\varepsilon} = \Phi$ concludes the proof.
\end{proof}

\section{Proof of Implications II: Simulation of Numerical Schemes}
\label{s:Proof_NewNumericalAlgos}

\subsection{ODE Flow Solvers}
\label{s:Proof_NewNumericalAlgos__ss:prop:flowemulating}

Our proof of Proposition~\ref{prop:flowemulating} will take a detour to first show that quadrature methods for ODEs can be efficiently encoded.
\subsubsection{Quadrature Method Emulators for Low-Regularity Integrands}
\label{s:NewApplications__ss:ODESolvers___sss:IntegralSolvers}
We now show that any $\mathcal{A}$NN can efficiently approximate quadrature rules for integral approximation. In many cases, this yields better rates than those guaranteed by direct minimax-optimal approximation of the antiderivative, when only its regularity is known (e.g., Lipschitz or $C^{s,1}$ for some $s\in \mathbb{N}_+$).

\begin{proposition}[Integral Approximation via Trapezoid-Quadrature Emulation]
\label{prop:integral_by_quadrature}
Let $\alpha:[0,1]\to \mathbb{R}$ be $L$-Lipschitz, and define its integral
\[
\begin{aligned}
I_\alpha:[0,1]^2 &\rightarrow \mathbb{R}\\
(x,t)&\mapsto x+\int_0^t \alpha(s)\,ds
.
\end{aligned}
\]
Under Assumption~\ref{ass:THEASSUMPTIONS}, for every $\varepsilon>0$, there is an $\mathcal{A}$NN
$
    \widehat{I}_{\alpha,\varepsilon}:\mathbb{R}^2\to \mathbb{R}
$
satisfying
\[
        \sup_{(x,t)\in [0,1]^2}
        \,
        \big|
                \widehat{I}_{\alpha,\varepsilon}(x,t)
            -
                I_\alpha(x,t)
        \big|
    \le
        \varepsilon
.
\]
Moreover, $\widehat{I}_{\alpha,\varepsilon}$ has depth at most
$
    \mathcal{O}(\log(1/\varepsilon))
,
$ width at most 
$
    \mathcal{O}(\varepsilon^{-3})
,
$ and no more than
$
    \mathcal{O}(\varepsilon^{-5})
$
non-zero parameters.
\end{proposition}
\begin{proof}
See Section~\ref{s:Proof_NewNumericalAlgos}.
\end{proof}

\begin{proof}[{Proof of Proposition~\ref{prop:integral_by_quadrature}}]
Since $\alpha$ is $L$-Lipschitz on $[0,1]$, its modulus of continuity satisfies
$
    \mathcal{X}(r)\le Lr
$
for every $r\in [0,1]$.
Applying Theorem~\ref{thrm:UAT__continouusLowreg} with $n=1$ yields an $\mathcal{A}$NN
$
    \Phi_{\alpha,\Delta}:\mathbb{R}\to \mathbb{R}
$
such that
\[
    \sup_{u\in [0,1]}
    \big|
        \alpha(u)-\Phi_{\alpha,\Delta}(u)
    \big|
    \le
    2\Big(1+\frac14\Big)\frac{L}{\sqrt{\Delta}}
    =
    \frac{5L}{2\sqrt{\Delta}}
.
\]
Moreover, $\Phi_{\alpha,\Delta}$ has depth $\mathcal{O}(\log(\Delta))$, width $\mathcal{O}(\Delta)$, and $\mathcal{O}(\Delta^2)$ non-zero parameters.
Now define the rescaled composite trapezoidal quadrature
\[
    Q_{\alpha,\Delta,N}(t)
    \eqdef
    \frac{t}{2N}
    \Bigg(
        \Phi_{\alpha,\Delta}(0)
        +
        2\sum_{j=1}^{N-1}
        \Phi_{\alpha,\Delta} \Big(\frac{jt}{N}\Big)
        +
        \Phi_{\alpha,\Delta}(t)
    \Bigg)
.
\]
Set $
    \widehat{I}_{\alpha,\Delta,N}(x,t)
    \eqdef
    x+Q_{\alpha,\Delta,N}(t)
$. 
Since affine rescalings, additions, and multiplication are all available in our gate library, this defines an $\mathcal{A}$NN. By parallelizing the $N+1$ evaluations of $\Phi_{\alpha,\Delta}$, its depth is $\mathcal{O}(\log(\Delta))$, width is $\mathcal{O}(N\Delta)$, and the number of non-zero parameters is $\mathcal{O}(N\Delta^2)$.
\hfill\\
\indent
It remains to estimate the error. Fix $(x,t)\in [0,1]^2$, and define
$
    g_t(u)\eqdef \alpha(tu)
$
for $u\in [0,1]$.
Then $g_t$ is $L$-Lipschitz on $[0,1]$, since $t\le 1$.
Also,
\[
    \int_0^t \alpha(s)\,ds
    =
    t\int_0^1 g_t(u)\,du
.
\]
Hence $\operatorname{err}(t)\eqdef \Big|
        \int_0^t \alpha(s)\,ds
        -
        Q_{\alpha,\Delta,N}(t)
    \Big|$ satisfies
\[
\begin{aligned}
    \operatorname{err}(t) 
&
\le
    t\Big|
        \int_0^1 g_t(u)\,du
        -
        \frac{1}{2N}
        \Bigg(
            g_t(0)
            +
            2\sum_{j=1}^{N-1}g_t \Big(\frac{j}{N}\Big)
            +
            g_t(1)
        \Bigg)
    \Big|
\\
&\quad
    +
    \frac{t}{2N}
    \Bigg(
        \big|g_t(0)-\Phi_{\alpha,\Delta}(0)\big|
        +
        2\sum_{j=1}^{N-1}
        \big|g_t(j/N)-\Phi_{\alpha,\Delta}(jt/N)\big|
        +
        \big|g_t(1)-\Phi_{\alpha,\Delta}(t)\big|
    \Bigg)
.
\end{aligned}
\]
The second term is bounded by
\[
    \frac{t}{2N}
    \Big(
        1+2(N-1)+1
    \Big)
    \sup_{u\in [0,1]}
    |\alpha(u)-\Phi_{\alpha,\Delta}(u)|
    \le
    \frac{5L}{2\sqrt{\Delta}}
.
\]
For the first term, the composite trapezoidal rule for an $L$-Lipschitz function on $[0,1]$ incurs error at most $L/(2N)$. Therefore,
\begin{equation}
\label{eq:integral_bound}
        \Big|
            \int_0^t \alpha(s)\,ds
            -
            Q_{\alpha,\Delta,N}(t)
        \Big|
    \le
        \frac{L}{2N}
        +
        \frac{5L}{2\sqrt{\Delta}}
.
\end{equation}
Adding $x$ to both terms does not change the error bound in~\eqref{eq:integral_bound}; whence
\begin{equation}
\label{eq:integral_bound__eldosoi}
    \sup_{(x,t)\in [0,1]^2}
    \big|
        \widehat{I}_{\alpha,\Delta,N}(x,t)-I_\alpha(x,t)
    \big|
    \le
    \frac{5L}{2\sqrt{\Delta}}
    +
    \frac{L}{2N}
.
\end{equation}
Setting $\Delta=N^2$, the right-hand side of~\eqref{eq:integral_bound__eldosoi} is $\tfrac{3L}{N}$. Let $\varepsilon>0$ and retroactively set $N=\big\lceil \tfrac{3L}{\varepsilon}\big\rceil$. Plugging this value of $N$ into our complexity estimates and setting $\widehat{I}_{\alpha,\varepsilon}\eqdef \widehat{I}_{\alpha,N^2,N}$ yields the result.
\end{proof}
We return to the proof of our structured high-dimensional flow approximation result.
\subsubsection{{Return to the Proof of Proposition~\ref{prop:flowemulating}}}
\begin{proof}[{Proof of Proposition~\ref{prop:flowemulating}}]
Define
$
    a:[0,1]\to \mathbb{R}
$
and
$
    E:\mathbb{R}\to \mathbb{R}^{d\times d}
$
by
$
    a(t)\eqdef \int_0^t \alpha(s)\,ds
$ and $
    E(u)\eqdef e^{uA}
$.  
Since $\alpha$ is Lipschitz, $a\in C^{1,1}([0,1])$ and
$
    a'(t)=\alpha(t)
$
for every $t\in [0,1]$.
Next define $
    g(t)\eqdef \beta(t)E(-a(t))B
$ and $
    z(t)\eqdef \int_0^t g(s)\,ds
$.  
Consequently, $E$ is smooth, hence Lipschitz on every compact interval.
Define $
    M_{\alpha}\eqdef \|\alpha\|_{L^\infty([0,1])}
$, 
and consider the compact interval $
    K\eqdef [-M_{\alpha}-1,M_{\alpha}+1]
$.  
Since $|a(t)|\le M_{\alpha}$ for all $t\in [0,1]$, there exist finite constants
$
    M_E\eqdef \sup_{u\in K}|E(u)|_\infty
$ and $
    L_E\eqdef \operatorname{Lip}(E;K)
$.  
Hence $g$ is Lipschitz on $[0,1]$. Indeed, for $s,t\in [0,1]$,
\[
\begin{aligned}
    \|g(t)-g(s)\|_{\ell^\infty}
    &\le
    |\beta(t)-\beta(s)|\,\|E(-a(t))B\|_{\ell^\infty}
    +
    |\beta(s)|\,|E(-a(t))-E(-a(s))|_\infty\,\|B\|_{\ell^\infty}
\\
    &\le
    \Big(
        \operatorname{Lip}(\beta)\,M_E\|B\|_{\ell^\infty}
        +
        \|\beta\|_{L^\infty([0,1])}\,L_E\,\|\alpha\|_{L^\infty([0,1])}\|B\|_{\ell^\infty}
    \Big)
    |t-s|
.
\end{aligned}
\]
Therefore,
\[
\begin{aligned}
    \|g(t)-g(s)\|_{\ell^2}
    &\le
    \sqrt{d}\,
    \|g(t)-g(s)\|_{\ell^\infty}
\\
    &\le
    \sqrt{d}\,
    \Big(
        \operatorname{Lip}(\beta)\,M_E\|B\|_{\ell^\infty}
        +
        \|\beta\|_{L^\infty([0,1])}\,L_E\,\|\alpha\|_{L^\infty([0,1])}\|B\|_{\ell^\infty}
    \Big)
    |t-s|
.
\end{aligned}
\]
In particular, each coordinate $g_i$ is Lipschitz.
We now compute the solution explicitly. Let
$
    Y_t\eqdef E(-a(t))X_t^x
$ for each $t\in [0,1]$.  
Since $a\in C^{1,1}$ and $E$ is smooth, the map $t\mapsto E(-a(t))$ is $C^{1,1}$; since $X^x_\cdot$ is absolutely continuous, so is $Y_\cdot$.
For each $t\in [0,1]$,
\[
\begin{aligned}
    \partial_t Y_t
    &=
    -\alpha(t)AE(-a(t))X_t^x
    +
    E(-a(t))
    \big(
        \alpha(t)AX_t^x+\beta(t)B
    \big)
\\
    &=
    \beta(t)E(-a(t))B
    =
    g(t)
.
\end{aligned}
\]
Integrating using $Y_0=X_0^x=x$ gives $
    Y_t
    =
    x+\int_0^t g(s)\,ds
    =
    x+z(t)
$. 
Therefore
\begin{equation}
\label{eq:closed_form_flow}
    X_t^x
    =
    E(a(t))
    \big(
        x+z(t)
    \big)
.
\end{equation}

We next verify the claimed regularity on the compact approximation domain.
Since $g$ is Lipschitz, $z\in C^{1,1}([0,1],\mathbb{R}^d)$.
Since $a\in C^{1,1}([0,1])$ and $E$ is smooth, the map $t\mapsto E(a(t))$ is $C^{1,1}$.
By~\eqref{eq:closed_form_flow},
\[
    \operatorname{Flow}_{\alpha,\beta}(x,t)
    =
    E(a(t))(x+z(t))
.
\]
Hence $\operatorname{Flow}_{\alpha,\beta}$ is affine in $x$ and $C^{1,1}$ in $t$.  Moreover, on $[0,1]^d\times[0,1]$, all first-order partial derivatives are Lipschitz, with constants depending only on $A,B,\alpha,\beta$, and $d$.  Therefore
$
    \operatorname{Flow}_{\alpha,\beta}|_{[0,1]^d\times[0,1]}
    \in C^{1,1}([0,1]^{d+1},\mathbb{R}^d)
$.

We now construct the emulator of the flow map in~\eqref{eq:flow_map}.  We first emulate $a$.
Apply Proposition~\ref{prop:integral_by_quadrature} to $\alpha$.
For every $\delta>0$, there exists an $\mathcal{A}$NN
$
    \widehat{I}_{\alpha,\delta}:\mathbb{R}^2\to \mathbb{R}
$
such that
\[
    \sup_{(x,t)\in [0,1]^2}
    \big|
        \widehat{I}_{\alpha,\delta}(x,t)
        -
        \big(
            x+\int_0^t \alpha(s)\,ds
        \big)
    \big|
    \le
    \delta
.
\]
Define $
    \widehat{a}_\delta(t)\eqdef \widehat{I}_{\alpha,\delta}(0,t)
$.  
Then
\begin{equation}
\label{eq:a_error}
    \sup_{t\in [0,1]}
    |\widehat{a}_\delta(t)-a(t)|
    \le
    \delta
,
\end{equation}
and $\widehat{a}_\delta$ has depth $\mathcal{O}(\log(1/\delta))$, width $\mathcal{O}(\delta^{-3})$, and $\mathcal{O}(\delta^{-5})$ non-zero parameters.
We next emulate $z$.
For each $i\in [d]_+$, apply Proposition~\ref{prop:integral_by_quadrature} to the scalar Lipschitz function $g_i$.
Thus, for every $\delta>0$, there exists an $\mathcal{A}$NN
$
    \widehat{I}_{g_i,\delta}:\mathbb{R}^2\to \mathbb{R}
$
such that
\[
    \sup_{(x,t)\in [0,1]^2}
    \Big|
        \widehat{I}_{g_i,\delta}(x,t)
        -
        \Big(
            x+\int_0^t g_i(s)\,ds
        \Big)
    \Big|
    \le
    \delta
.
\]
Define $
    \widehat{z}_{i,\delta}(t)\eqdef \widehat{I}_{g_i,\delta}(0,t)
$.
Then
\[
    \sup_{t\in [0,1]}
    \Big|
        \widehat{z}_{i,\delta}(t)
        -
        \int_0^t g_i(s)\,ds
    \Big|
    \le
    \delta
.
\]
Define $
    \widehat{z}_\delta(t)
    \eqdef
    \big(
        \widehat{z}_{1,\delta}(t),\dots,\widehat{z}_{d,\delta}(t)
    \big)
$.
Then
\begin{equation}
\label{eq:z_error}
    \sup_{t\in [0,1]}
    \|\widehat{z}_\delta(t)-z(t)\|_{\ell^\infty}
    \le
    \delta
,
\end{equation}
and hence $
    \sup_{t\in [0,1]}
    \|\widehat{z}_\delta(t)-z(t)\|_{\ell^2}
    \le
    \sqrt{d}\,\delta
$.  
Since $d$ is fixed, $\widehat{z}_\delta$ still has depth $\mathcal{O}(\log(1/\delta))$, width $\mathcal{O}(\delta^{-3})$, and $\mathcal{O}(\delta^{-5})$ non-zero parameters, up to constants depending on $d$.

It remains to emulate $E(u)=e^{uA}$ on $K$.
For each $i,j\in [d]_+$, the scalar map
$
    E_{ij}:K\to \mathbb{R}
$
is Lipschitz.
After affinely rescaling $K$ onto $[0,1]$, Theorem~\ref{thrm:UAT__continouusLowreg} yields, for every $\delta>0$, an $\mathcal{A}$NN
$
    \widehat{E}_{ij,\delta}:\mathbb{R}\to \mathbb{R}
$
such that
$
    \sup_{u\in K}
    \big|
        \widehat{E}_{ij,\delta}(u)-E_{ij}(u)
    \big|
    \le
    \delta
$, 
with depth $\mathcal{O}(\log(1/\delta))$, width $\mathcal{O}(\delta^{-2})$, and $\mathcal{O}(\delta^{-4})$ non-zero parameters.
Define
$
    \widehat{E}_\delta(u)
    \eqdef
    \big(
        \widehat{E}_{ij,\delta}(u)
    \big)_{i,j=1}^d
.
$
Then
\begin{equation}
\label{eq:E_error}
    \sup_{u\in K}
    |\widehat{E}_\delta(u)-E(u)|_\infty
    \le
    \delta
.
\end{equation}
Now, using affine layers together with finitely many additions and binary multiplications, define
$
    \widehat{\operatorname{Flow}}_{\alpha,\beta,\delta}(x,t)
    \eqdef
    \widehat{E}_\delta\big(\widehat{a}_\delta(t)\big)
    \big(
        x+\widehat{z}_\delta(t)
    \big)
$.
Since $|a(t)|\le M_{\alpha}$ and, by~\eqref{eq:a_error}, $|\widehat{a}_\delta(t)-a(t)|\le \delta$, if $\delta\le 1$ then
$
    \widehat{a}_\delta(t)\in K
$
for all $t\in [0,1]$.
Set $
    M_z\eqdef \sup_{t\in [0,1]}\|z(t)\|_{\ell^\infty}
$
and
$
    C_\ast
    \eqdef
    \big(1+L_E\big)\big(2+M_z\big)+M_E
$.  
Assume $\delta\le 1$.
Then, for every $(x,t)\in [0,1]^d\times [0,1]$,
\[
\begin{aligned}
    \Big\|
        \widehat{\operatorname{Flow}}_{\alpha,\beta,\delta}(x,t)
        -
        \operatorname{Flow}_{\alpha,\beta}(x,t)
    \Big\|_{\ell^2}
& \le
    \sqrt{d}\,
    \Big|
        \widehat{E}_\delta(\widehat{a}_\delta(t))
        -
        E(a(t))
    \Big|_\infty
    \,
    \Big\|
        x+\widehat{z}_\delta(t)
    \Big\|_{\ell^\infty}
\\
    &\quad
    +
    \sqrt{d}\,
    |E(a(t))|_\infty
    \,
    \|\widehat{z}_\delta(t)-z(t)\|_{\ell^\infty}
.
\end{aligned}
\]
By~\eqref{eq:E_error},
$
    \Big|
        \widehat{E}_\delta(\widehat{a}_\delta(t))
        -
        E(\widehat{a}_\delta(t))
    \Big|_\infty
    \le
    \delta
$,
and, since $E$ is $L_E$-Lipschitz on $K$, by~\eqref{eq:a_error},
$
    \Big|
        E(\widehat{a}_\delta(t))
        -
        E(a(t))
    \Big|
    \le
    L_E\delta
$.  
Hence $
    \Big|
        \widehat{E}_\delta(\widehat{a}_\delta(t))
        -
        E(a(t))
    \Big|
    \le
    (1+L_E)\delta
.
$  By~\eqref{eq:z_error}, we have
$
        \Big\|
            x+\widehat{z}_\delta(t)
        \Big\|_{\ell^\infty}
    \le
        \|x\|_{\ell^\infty}+\|z(t)\|_{\ell^\infty}+\delta
    \le
        1+M_z+1
    =
      2+M_z
$.
Therefore
\[
\begin{aligned}
    \Big\|
        \widehat{\operatorname{Flow}}_{\alpha,\beta,\delta}(x,t)
        -
        \operatorname{Flow}_{\alpha,\beta}(x,t)
    \Big\|_{\ell^2}
    &\le
    \sqrt{d}\,
    \big(1+L_E\big)\big(2+M_z\big)\delta
    +
    \sqrt{d}\,M_E\delta
\\
    &= \sqrt{d}\,C_\ast \delta
.
\end{aligned}
\]
Upon setting $
    \delta
    \eqdef
    \min\left\{
        1,\frac{\varepsilon}{\sqrt{d}\,C_\ast}
    \right\}
$, we may define
$
    \widehat{\operatorname{Flow}}_{\alpha,\beta,\varepsilon}
    \eqdef
    \widehat{\operatorname{Flow}}_{\alpha,\beta,\delta}
$.  
Then
\[
    \sup_{\substack{x\in [0,1]^d\\ t\in [0,1]}}
    \,
        \Big\|
            \widehat{\operatorname{Flow}}_{\alpha,\beta,\varepsilon}(x,t)
            -
            \operatorname{Flow}_{\alpha,\beta}(x,t)
        \Big\|_{\ell^2}
    \le
        \varepsilon
.
\]
Finally, the depth of the assembled network is $\mathcal{O}(\log(1/\delta))=\mathcal{O}(\log(1/\varepsilon))$.
Its width is dominated by the primitive emulators for $a$ and $z$, hence is $\mathcal{O}(\delta^{-3})=\mathcal{O}(\varepsilon^{-3})$.
Its number of non-zero parameters is likewise dominated by those primitive emulators and is $\mathcal{O}(\delta^{-5})=\mathcal{O}(\varepsilon^{-5})$.
This completes the proof.
\end{proof}

\subsection{Maximal Eigenvalue Solver}
\label{s:Proof__ss:classicalAlgos}

\begin{proof}[{Proof of Proposition~\ref{prop:ANN_emulation_power_iteration}}]
For each $A\in \mathrm{PD}_{\delta,\gamma,\Delta}^{d,\rho_0}(\mathbb{R})$, let
$
\lambda_T(A)
$
denote the Rayleigh-quotient output of the exact power iteration after $T$ steps, and set
$
    x_0
\eqdef
    \tfrac{1}{\sqrt d}\mathbf{1}_d
$.
Set
\[
T_\varepsilon
\eqdef
\left\lceil
\frac{
    \log \big(2(\Delta-\delta)/(\rho_0\varepsilon)\big)
}{
    2\log \big((1-\gamma/\Delta)^{-1}\big)
}
\right\rceil.
\]
Since $\delta,\gamma,\Delta,\rho_0$ are fixed, we have
$
T_\varepsilon=\mathcal{O}(\log(1/\varepsilon))
$.  
For any fixed 
$
A\in \mathrm{PD}_{\delta,\gamma,\Delta}^{d,\rho_0}(\mathbb{R})
$, since
$
\lambda_{\max}(A)-\lambda_{\max-1}(A)\ge \gamma
$,
$
\lambda_{\max}(A)\le \Delta
$,
$
\lambda_{\min}(A)\ge \delta
$, and
$
\rho(A)^2\ge \rho_0
$, \citep[Theorem 7.6]{arbenz2012lecture} yields
\begin{equation}
\label{eq:power_iteration__method_bound}
    \bigl|
        \lambda_{\max}(A)-\lambda_{T_\varepsilon}(A)
    \bigr|
\le
    \bigl(
        \lambda_{\max}(A)-\lambda_{\min}(A)
    \bigr)
    \biggl(
        \frac{\lambda_{\max-1}(A)}{\lambda_{\max}(A)}
    \biggr)^{2T_\varepsilon}
    \frac{1-\rho(A)^2}{\rho(A)^2}
.
\end{equation}
Since
$
\lambda_{\max}(A)-\lambda_{\min}(A)\le \Delta-\delta
$ and
\[
\frac{\lambda_{\max-1}(A)}{\lambda_{\max}(A)}
=
1-
\frac{
    \lambda_{\max}(A)-\lambda_{\max-1}(A)
}{
    \lambda_{\max}(A)
}
\le
1-\frac{\gamma}{\Delta},
\]
then~\eqref{eq:power_iteration__method_bound} simplifies to
\begin{equation}
\label{eq:main_bound__pt__unoski}
\bigl|
    \lambda_{\max}(A)-\lambda_{T_\varepsilon}(A)
\bigr|
\le
\frac{\Delta-\delta}{\rho_0}
\biggl(
    1-\frac{\gamma}{\Delta}
\biggr)^{2T_\varepsilon}
\le
\frac{\varepsilon}{2}.
\end{equation}
We next emulate the $T_\varepsilon$ steps of power iteration, i.e., Algorithm~\ref{alg:power_iteration}, by an $\mathcal{A}$NN.  Since $d$ is fixed, matrix-vector multiplication uses only finitely many addition and multiplication gates, hence contributes $\mathcal{O}(1)$ width and $\mathcal{O}(1)$ depth per step; cf.~Table~\ref{tab_thrm:gate_approximation}.  Since, along the exact iteration,
$
    \delta
\le
    \|Ax_k\|_2
\le
    \Delta
$ and $
    \delta
\le
    \lambda_k(A)
\le
    \Delta
$,
every inversion and radical is evaluated on a compact interval bounded away from $0$.  
Therefore, by the gate-emulation results in Table~\ref{tab_thrm:gate_approximation}, each step of the power iteration can be uniformly approximated to accuracy $\eta$ by an $\mathcal{A}$NN block of depth
$
\mathcal{O}\bigl(
    \log(1/\eta)\bigl(1+\log^{\circ 2}(1/\eta)\bigr)
\bigr)
$, 
width $\mathcal{O}(1)$, and with
$
\mathcal{O}\bigl(
    \log(1/\eta)\bigl(1+\log^{\circ 2}(1/\eta)\bigr)
\bigr)
$ non-zero parameters.

Set $\eta_\varepsilon>0$, chosen sufficiently small so that, after composing $T_\varepsilon$ such blocks and the final Rayleigh-quotient computation, the resulting $\mathcal{A}$NN
$
\widehat{\Lambda}_\varepsilon
$
satisfies
\begin{equation}
\label{eq:los_big_bound__dos}
    \sup_{A\in \mathrm{PD}_{\delta,\gamma,\Delta}^{d,\rho_0}(\mathbb{R})}
    \,
        \bigl|
            \widehat{\Lambda}_\varepsilon(A)-\lambda_{T_\varepsilon}(A)
        \bigr|
\le
    \frac{\varepsilon}{2}
.
\end{equation}
Since
$
T_\varepsilon=\mathcal{O}(\log(1/\varepsilon))
$
and
$
\log(1/\eta_\varepsilon)=\mathcal{O}(\log(1/\varepsilon))
$,
it follows that $\widehat{\Lambda}_\varepsilon$ has depth at most
$
\mathcal{O}\bigl(
    \log(1/\varepsilon)^2
    \bigl(
        1+\log^{\circ 2}(1/\varepsilon)
    \bigr)
\bigr)
$,
width at most
$
\mathcal{O}(1)
$,
and at most
$
\mathcal{O}\bigl(
    \log(1/\varepsilon)^2
    \bigl(
        1+\log^{\circ 2}(1/\varepsilon)
    \bigr)
\bigr)
$
non-zero parameters; and so
\[
\bigl|
    \widehat{\Lambda}_\varepsilon(A)-\lambda_{\max}(A)
\bigr|
\le
\bigl|
    \widehat{\Lambda}_\varepsilon(A)-\lambda_{T_\varepsilon}(A)
\bigr|
+
\bigl|
    \lambda_{T_\varepsilon}(A)-\lambda_{\max}(A)
\bigr|.
\]
Combining the bounds in~\eqref{eq:main_bound__pt__unoski} and~\eqref{eq:los_big_bound__dos} yields
\begin{equation}
\label{eq:pre_grand_finale}
        \bigl|
            \widehat{\Lambda}_\varepsilon(A)-\lambda_{\max}(A)
        \bigr|
    \le
        \frac{\varepsilon}{2}
    +
        \frac{\varepsilon}{2}
    =
        \varepsilon
.
\end{equation}
Since $A$ was arbitrary in $\mathrm{PD}_{\delta,\gamma,\Delta}^{d,\rho_0}(\mathbb{R})$ and since $\widehat{\Lambda}_{\varepsilon}$ did not depend on $A$, we may take suprema across~\eqref{eq:pre_grand_finale} to deduce our claim.
\end{proof}

\section{{Proof of Classical TCS Computation Result - Corollary~\ref{cor:AllpairsShortestpaths}}}
\label{a:Proofs__cor:AllpairsShortestpaths}

\begin{proof}[{Proof of Corollary~\ref{cor:AllpairsShortestpaths}}]
Let $n\eqdef |E_k|$.  We first observe that Algorithm~\ref{algo:BFM_12} computes $D^W_{1,2}$ exactly on $(0,\infty)^{E_k}$.  Indeed, after the initialization step, $d_j$ is the minimum weight of a path from $1$ to $j$ using at most one edge.  Suppose inductively that, after $r-1$ edge-relaxation rounds, $d_j$ is the minimum weight of a path from $1$ to $j$ using at most $r-1$ edges.  Then the update
\[
    \widetilde d_j
    =
    \min\Big\{
        d_j,\,
        \min_{q\in [k]_+\setminus\{1,j\}}
        \big(d_q+W_{\{q,j\}}\big)
    \Big\}
\]
takes the minimum over all paths from $1$ to $j$ using at most $r$ edges: either the path already uses at most $r-1$ edges, or its last edge is $\{q,j\}$ and its initial segment is a path from $1$ to $q$ using at most $r-1$ edges.  Thus, after the final round, $d_j$ is the minimum weight of all paths from $1$ to $j$ using at most $k-1$ edges.  Since all edge weights are positive, every shortest path is simple and therefore uses at most $k-1$ edges.  Hence the returned value is $d_2=D^W_{1,2}$.

We encode Algorithm~\ref{algo:BFM_12} as a $\mathbb{G}_{t\text{-}alg}^{2,1}$-circuit.  The additions $d_q+W_{\{q,j\}}$ are binary addition gates.  The binary minimum gate belongs to the algebro-tropical language since
$
    \min\{x,y\}=-\max\{-x,-y\}
$,
using only affine operations and the binary maximum gate.  Each inner minimum over $q\in [k]_+\setminus\{1,j\}$ is implemented by a balanced binary tree of depth $\mathcal{O}(\log k)$ and size $\mathcal{O}(k)$.  In each relaxation round there are $\mathcal{O}(k^2)$ addition and binary minimum gates, and there are $\mathcal{O}(k)$ rounds.  Therefore the resulting binary algebro-tropical circuit has
$
    \Delta_k=\mathcal{O}(k\log k)
$, $N_k=\mathcal{O}(k^3)$, and $
    \Upsilon_k=\mathcal{O}(k^2)$.  
Moreover, since every gate has arity at most two, the number of circuit edges satisfies $|E_{\mathcal{C}_k}|=\mathcal{O}(N_k)=\mathcal{O}(k^3)$.  Hence, in the notation of Theorem~\ref{thrm:concrete_surgery},
\[
    \bar{\Upsilon}_k
    =
    \max\{1,\Upsilon_k+|E_{\mathcal{C}_k}|\}
    =
    \mathcal{O}(k^3).
\]

Let $F_k:\mathbb{R}^{E_k}\to\mathbb{R}$ denote the function computed by this finite tropical circuit, i.e., the output obtained by running the same Bellman-Ford-Moore recursion for arbitrary real edge-weights.  The circuit computes $F_k$ exactly on $[-1,1]^{E_k}$, and the preceding paragraph shows that
$
    F_k(W)=D^W_{1,2}
$
for every $W\in [w_-,1]^{E_k}$.

Apply the pure min-plus specialization of the proof of Theorem~\ref{thrm:concrete_surgery} to the exact $\mathbb{G}_{t\text{-}alg}^{2,1}$-circuit computing $F_k$, with accuracy parameter $\eta\eqdef \varepsilon/3$.  Since the circuit computes $F_k$ exactly, the specialization produces an $\mathcal{A}$NN $\Phi_{\operatorname{dist:1,2}}:\mathbb{R}^{|E_k|}\to\mathbb{R}$ such that
\[
    \sup_{W\in [-1,1]^{E_k}}
        \big|
            F_k(W)
            -
            \Phi_{\operatorname{dist:1,2}}(W)
        \big|
    \le
        2\eta
    <
        \varepsilon.
\]
Restricting this estimate to $[w_-,1]^{E_k}$ and using $F_k(W)=D^W_{1,2}$ on this positive cube gives the required approximation bound.

It remains only to read off the complexity estimates.  We use here the pure min-plus specialization of the proof of Theorem~\ref{thrm:concrete_surgery}.  In this circuit, the only non-affine gates are binary minima, equivalently binary maxima after affine sign changes.  The additions are affine operations in the network wiring and therefore introduce no gate-emulation error.  Moreover, the Bellman-Ford-Moore recursion propagates errors additively: if, at the beginning of a relaxation round, all current labels $d_j$ are approximated within $\rho$, and if each binary-minimum tree is approximated within accuracy $\tau$ at each of its $\mathcal{O}(\log k)$ levels, then after the round all updated labels are approximated within $\rho+\mathcal{O}(\tau\log k)$.  Since there are $\mathcal{O}(k)$ rounds, choosing
$
    \tau
    \eqdef
    c_0\frac{\eta}{k\log k}
$
with $c_0>0$ small enough yields a total circuit-emulation error at most $\eta$.

By Proposition~\ref{prop:max_ANNs}, each binary maximum/minimum gate can be emulated to accuracy $\tau$ with depth and number of non-zero parameters
$
    \mathcal{O}\big(
        \log(1/\tau)\log^{\circ 2}(1/\tau)
    \big).
$
The binary-minimum trees have depth $\mathcal{O}(\log k)$, and the Bellman-Ford-Moore circuit has $\mathcal{O}(k)$ sequential relaxation rounds.  Hence the resulting $\mathcal{A}$NN has depth
\[
    \mathcal{O}\left(
        k\log k\,
        \log\Big(\frac{k\log k}{\eta}\Big)
        \log^{\circ 2}\Big(\frac{k\log k}{\eta}\Big)
    \right).
\]
Furthermore, the binary min-plus circuit has $\mathcal{O}(k^3)$ gates and $\mathcal{O}(k^3)$ circuit edges, so the parallelized surgery construction has effective width factor
$
    \bar{\Upsilon}_k=\mathcal{O}(k^3).
$
Consequently, the number of non-zero parameters is
\[
    \mathcal{O}\left(
        k^4\log k\,
        \log\Big(\frac{k\log k}{\eta}\Big)
        \log^{\circ 2}\Big(\frac{k\log k}{\eta}\Big)
    \right).
\]
Taking $\eta=\varepsilon/3$ gives the stated bounds.
\end{proof}

\section{{Proof of \texorpdfstring{Theorem~\ref{thrm:Universal_Computation}}{Main Qualitative Result}}}
\label{s:Proof_of_thrm:Universal_Computation}

Our proof of Theorem~\ref{thrm:Universal_Computation} is divided into three cases: affine non-linearities (where universal approximation fails), piecewise linear (which is reducible to the $\operatorname{ReLU}$-MLP case where universality is already well-established with minimax-rates), and the remaining definable case which we have already taken care of in Theorem~\ref{thrm:UAT__continouusLowreg} when $\mu$ is the Lebesgue measure on $[0,1]^d$.

\subsection{Case 1: The Affine Case -- When Universal Approximation Fails}
\label{s:PROOFOFUAT__continouusLowreg__ss:No_Universality__AffineCase}
\begin{lemma}
\label{lem:aff_bad}
Fix $1\le p<\infty$.
If $\mathcal{A}$ consists only of affine functions, then there exists some $\varepsilon>0$ and some $g_{\varepsilon}\in C([0,1]^d)$ such that: there exists no $\mathcal{A}$NN $\hat{f}:\mathbb{R}^d\to \mathbb{R}$ satisfying $\|g_{\varepsilon}-\hat{f}\|_{L^p([0,1]^d)}\le \varepsilon$.
\end{lemma}
\begin{proof}
Let
$
\mathrm{Aff}_d
\eqdef
\{x\mapsto a^\top x+b:\ a\in\mathbb{R}^d,\ b\in\mathbb{R}\}.
$
Since every element of $\mathcal{A}$ is affine in its input variable, it follows by induction over the layers that every realization of an $\mathcal{A}$NN is itself an affine map. Hence every $\mathcal{A}$NN $\hat f:\mathbb{R}^d\to\mathbb{R}$ belongs to $\mathrm{Aff}_d$.

Now view $\mathrm{Aff}_d$ as a subspace of $L^p([0,1]^d)$. Since $\dim(\mathrm{Aff}_d)=d+1<\infty$, this is a finite-dimensional linear subspace, and therefore it is closed in $L^p([0,1]^d)$. Hence $\mathrm{Aff}_d$ cannot be dense in $L^p([0,1]^d)$, since no finite-dimensional linear subspace of an infinite-dimensional Banach space can be dense therein. Therefore, there exists some $g_{\varepsilon}\in C([0,1]^d)$ such that $g_{\varepsilon}\notin \overline{\mathrm{Aff}_d}^{\,L^p([0,1]^d)}$. Setting
$
\varepsilon
\eqdef
\tfrac12 \inf_{h\in \mathrm{Aff}_d}\|g_{\varepsilon}-h\|_{L^p([0,1]^d)},
$
we have $\varepsilon>0$, and no $\mathcal{A}$NN $\hat f:\mathbb{R}^d\to\mathbb{R}$ can satisfy
$
\|g_{\varepsilon}-\hat f\|_{L^p([0,1]^d)}
\le \varepsilon
$.
\end{proof}

\subsection{Case 2: The Piecewise Linear Case}

As the next simple result shows, it is reducible to the $\operatorname{ReLU}$ MLP case where virtually all these approximation guarantees are known (up to a $\log(1/\varepsilon)$-factor slower for multiplication and faster for maximization).
\begin{proposition}[Reduction to $\operatorname{ReLU}$-MLP Setting]
\label{prop:ReLU_Embedding}
Let $f:\mathbb{R}^d\to \mathbb{R}$ be continuous and piecewise linear in the sense of Definition~\ref{defn:PWLlinear}.  
If $f$ is not affine, then there exists a non-constant affine map
$
E:\mathbb{R}\to \mathbb{R}^d
$
such that, for every $M>0$, there exists a $(f\circ E)$-MLP $\hat{f}_{\operatorname{ReLU},M}:\mathbb{R}\to \mathbb{R}$ of depth $1$ and width $2$ such that: for all $x\in [-M,M]$
\[
    \operatorname{ReLU}(x) = \hat{f}_{\operatorname{ReLU},M}(x)
.
\]
\end{proposition}
\begin{proof}[{Proof of Proposition~\ref{prop:ReLU_Embedding}}]
Let $E(t)=x_0+t v$ with $v\neq 0$.  
Since $f$ is piecewise linear, for any partition $\{\Pi^{(p)}\}_{p=1}^P$ as in Definition~\ref{defn:PWLlinear}, the sets
$
E^{-1}(\Pi^{(p)})\subseteq \mathbb{R}
$
form a definable partition of $\mathbb{R}$, and on each of them
$
(f\circ E)(t)=A^{(p)}(x_0+t v)+b^{(p)}
$
is affine in $t$. Hence $f\circ E$ is piecewise linear.
We now show that one can choose $E$ so that $f\circ E$ is not affine.  
Assume by contradiction that $f\circ E$ is affine for every affine map $E:\mathbb{R}\to\mathbb{R}^d$. Fix $x,y\in\mathbb{R}^d$ and define
$
E_{x,y}(t)\eqdef x+t(y-x)
$.
Then $t\mapsto f(E_{x,y}(t))$ is affine, so for every $\lambda\in\mathbb{R}$,
$
    f((1-\lambda)x+\lambda y)
=
    (1-\lambda)f(x)+\lambda f(y)
$.
Thus $f$ is affine on $\mathbb{R}^d$, contradicting the assumption; whence there is an affine map $E$ such that $f\circ E$ is not affine. Since $f\circ E$ is piecewise linear, it must then have at least two pieces.
\hfill\\
\noindent
Applying~\citep[Proposition 1 (b)]{yarotsky2017error} with $\rho\eqdef f\circ E$ and $D\eqdef [-M,M]$, we find that for every $M>0$ there exists a $(f\circ E)$-MLP $\hat{f}_{\operatorname{ReLU},M}:\mathbb{R}\to \mathbb{R}$ satisfying
$
\hat{f}_{\operatorname{ReLU},M}(x)=\operatorname{ReLU}(x)
$
for all $x\in [-M,M]$. Moreover, $\hat{f}_{\operatorname{ReLU},M}$ has depth $1$ and width $2$.
\end{proof}
The immediate consequence of Proposition~\ref{prop:ReLU_Embedding} and a (minimax-optimal) universal approximation theorem for $\operatorname{ReLU}$-MLPs, e.g., \cite{yarotsky2018optimal,petersen2018optimal,shen2022optimal,hong2024bridging}, is that $\mathcal{A}$NNs with non-affine (continuous) piecewise linear non-linearities $\mathcal{A}$ are universal approximators.
\begin{corollary}[Universality of $\mathcal{A}$NNs with Piecewise Linear Non-Linearities]
\label{cor:UAT_PWL}
If $\mathcal{A}$ contains a continuous piecewise linear but non-affine non-linearity
$\sigma:\mathbb{R}^m\to \mathbb{R}$,
then for every uniformly continuous function $g:[0,1]^d\to\mathbb{R}$, every modulus of regularity $\omega$ of $g$, every $1\leq p\leq \infty$, and every $\varepsilon>0$, there exists an $\mathcal{A}$NN $\hat{g}_{\varepsilon}:[0,1]^d\to\mathbb{R}$ satisfying
$$
\|g-\hat{g}_{\varepsilon}\|_{L^p([0,1]^d)}
\leq
\omega(\varepsilon).
$$
Moreover, the width and size of $\hat{g}_{\varepsilon}$ are $\mathcal{O}(\varepsilon^{-d})$, while its depth is $\mathcal{O}(1)$.
\end{corollary}
\begin{proof}[Proof of Corollary~\ref{cor:UAT_PWL}]
Let $\sigma\in \mathcal{A}$ be a continuous piecewise linear, non-affine non-linearity.
Fix a uniformly continuous function $g:[0,1]^d\to \mathbb{R}$, $1\leq p\leq \infty$, and $\varepsilon>0$. Let $\omega$ be a modulus of regularity of $g$ on $[0,1]^d$. Choose
$
n_{\varepsilon}
\eqdef
\max\{1,\lceil \tfrac{d}{2\varepsilon}\rceil\}
$.
Then
$
\tfrac{d}{2n_{\varepsilon}}
\leq
\varepsilon
$.
Hence, by \citep[Theorem 4.1]{hong2024bridging}, there exists a $\operatorname{ReLU}$-MLP $\Phi_{\varepsilon}$ on $[0,1]^d$ with width at most
$
8d(n_{\varepsilon}+1)^d
$,
depth at most
$
\lceil \log_2 d\rceil+4
$,
and at most
$
16d(n_{\varepsilon}+1)^d
$
nonzero parameters, such that
\begin{equation}
\label{eq:ReLU_MLP_approximation_bound}
\|g-\Phi_{\varepsilon}\|_{L^p([0,1]^d)}
\leq
\|g-\Phi_{\varepsilon}\|_{L^\infty([0,1]^d)}
\leq
\omega\Big(\frac{d}{2n_{\varepsilon}}\Big)
\leq
\omega(\varepsilon)
.
\end{equation}

Since $\Phi_{\varepsilon}$ is a finite network and $[0,1]^d$ is compact, there exists $M_{\varepsilon}>0$ such that every scalar pre-activation appearing in $\Phi_{\varepsilon}$ lies in $[-M_{\varepsilon},M_{\varepsilon}]$ whenever the input belongs to $[0,1]^d$.

Applying Proposition~\ref{prop:ReLU_Embedding} to the non-affine piecewise linear map $\sigma$, we obtain a non-constant affine map
$
E:\mathbb{R}\to \mathbb{R}^m
$
such that there exists an $(\sigma\circ E)$-MLP
$
\hat{f}_{\operatorname{ReLU},M_{\varepsilon}}:\mathbb{R}\to \mathbb{R}
$
of depth $1$ and width $2$ satisfying
$$
\hat{f}_{\operatorname{ReLU},M_{\varepsilon}}(x)
=
\operatorname{ReLU}(x)
$$
for all $x\in[-M_{\varepsilon},M_{\varepsilon}]$.
Since $\sigma\in\mathcal{A}$, this depth-$1$, width-$2$ realization is an $\mathcal{A}$NN block.

We now form an $\mathcal{A}$NN $\hat{g}_{\varepsilon}$ by replacing each $\operatorname{ReLU}$ activation in $\Phi_{\varepsilon}$ by a copy of this fixed $\mathcal{A}$NN block. By the choice of $M_{\varepsilon}$, every such replacement is exact on $[0,1]^d$, and hence
$$
\hat{g}_{\varepsilon}(x)=\Phi_{\varepsilon}(x)
$$
for all $x\in[0,1]^d$.
Therefore, by~\eqref{eq:ReLU_MLP_approximation_bound} we have 
$    \|g-\hat{g}_{\varepsilon}\|_{L^p([0,1]^d)}
=
    \|g-\Phi_{\varepsilon}\|_{L^p([0,1]^d)}
\le
    \omega(\varepsilon)
$.
Finally, since $\hat{f}_{\operatorname{ReLU},M_{\varepsilon}}$ has $\mathcal{O}(1)$ depth, width, and size, replacing each $\operatorname{ReLU}$ gate in $\Phi_{\varepsilon}$ by this block changes the width and size by at most a constant factor and the depth by at most a constant factor. Thus $\hat{g}_{\varepsilon}$ has width and size $\mathcal{O}((n_{\varepsilon}+1)^d)=\mathcal{O}(\varepsilon^{-d})$, while its depth remains $\mathcal{O}(1)$.
\end{proof}

\subsection{Case 3: The Non-Piecewise Definable Case}

\begin{proof}[{Proof of Theorem~\ref{thrm:Universal_Computation}}]
If $\mathcal{A}$ consists only of affine functions, then Lemma~\ref{lem:aff_bad} implies that there exists some $\varepsilon>0$ and some $g_{\varepsilon}\in C([0,1]^d)$ such that no $\mathcal{A}$NN $f_{\varepsilon}:\mathbb{R}^d\to \mathbb{R}$ satisfies
$
\|g_{\varepsilon}-f_{\varepsilon}\|_{L^p([0,1]^d)}
\le \varepsilon
$.
Since
$
\|h\|_{L^p([0,1]^d)}
\le
\|h\|_{L^\infty([0,1]^d)}
$
for every $h\in L^\infty([0,1]^d)$, it follows a fortiori that no such $\mathcal{A}$NN can satisfy~\eqref{eq:universality}. 

Conversely, if $\mathcal{A}$ has at least one non-affine function $f$ then there are two sub-cases to be considered.  Either (A) $f$ is non-affine and piecewise linear, in which case the conclusion follows from Corollary~\ref{cor:UAT_PWL}, or (B) $f$ is not piecewise linear, in which case the conclusion follows from Theorem~\ref{thrm:UAT__continouusLowreg}.
\end{proof}

\section{Notation}
\label{a:Notation_full}
The following is a full list of the notation used in our paper.


\begin{notation*}
We use the following notation throughout the paper.
\begin{enumerate}[label=(\roman*)]
    \item $\mathbb{N}_+$ denotes the positive integers, and $\mathbb{N}_0$ denotes the nonnegative integers.

    \item $\mathbb{Z}$, $\mathbb{Q}$, $\mathbb{R}$, and $\mathbb{C}$ denote the integers, rationals, reals, and complex numbers, respectively.

    \item For $N\in\mathbb{N}_+$, $[N]_+$ denotes the set $\{1,\dots,N\}$.

    \item For $m\in\mathbb{N}_+$, $\operatorname{Id}_{\mathbb{R}^m}$ denotes the identity map on $\mathbb{R}^m$.

    \item For $1\le p\le\infty$, $\operatorname{Ball}_{p}(x,r)$ denotes the closed $\ell^p$-ball centered at $x$ with radius $r>0$.

    \item $\mathcal{A}$ denotes a dictionary of nonlinearities, and an $\mathcal{A}$NN denotes a neural network built from nonlinearities in $\mathcal{A}$.

    \item If $f:\mathbb{R}^n\times\mathbb{R}^p\to\mathbb{R}^m$ belongs to $\mathcal{A}$, then $f_\theta\eqdef f(\cdot,\theta)$ denotes the corresponding parameter-frozen map.

    \item $\Theta=(\theta_l,A_l,b_l)_{l=1}^L$ denotes the full parameter vector of an $\mathcal{A}$NN, where $\theta_l$ are local nonlinear parameters, $A_l$ are affine connection matrices, and $b_l$ are bias vectors.

    \item $\mathcal{F}[f]\eqdef \{f_{\Theta}(\cdot):\Theta\in\mathbb{R}^P\}$ denotes the function class realized by the parametric network $f$.

    \item $f^{\|k}$ denotes the $k$-fold parallelization of a nonlinearity $f$.

    \item $f^{\circ k}$ denotes the $2\le k$-fold composition of $f$ with itself, for any self-map $f$.

    \item $\mathcal{S}$ denotes a fixed o-minimal structure over the real field, and ``definable'' always means definable in $\mathcal{S}$.

    \item $\mathbb{G}$ denotes a dictionary of gates, i.e., elementary operations allowed in a circuit.

    \item $[\mathbb{R}^n:\mathbb{R}]$ denotes the set of all functions from $\mathbb{R}^n$ to $\mathbb{R}$.

    \item $\mathbb{G}_{alg}^{k,c}$ denotes the algebraic gate dictionary, containing constants, identity gates, addition gates, and multiplication gates, with arity at most $k$ and constants bounded by $c$.

    \item $\mathbb{G}_{t\text{-}alg}^{k,c}$ denotes the algebro-tropical gate dictionary, obtained by adding absolute value and maximization gates to $\mathbb{G}_{alg}^{k,c}$.

    \item $\mathbb{G}_{rt\text{-}alg}^{k,c,r^\star}$ denotes the radical-algebro-tropical gate dictionary, obtained by adding integer powers and radicals of order at most $r^\star$.

    \item $\mathbb{G}_{Rat}^{k,c,r^\star}$ denotes the rational-tropical gate dictionary, obtained by adding the inversion gate $x\mapsto 1/x$.

    \item $D=(V,E)$ denotes a directed acyclic graph, with vertex set $V$ and directed edge set $E$.

    \item For a node $v\in V$, $\pa{v}$ denotes the set of parents of $v$, and $\ch{v}$ denotes the set of children of $v$.

    \item $V_{\operatorname{in}}$, $V_{\operatorname{out}}$, and $V_{\operatorname{comp}}$ denote the input nodes, output nodes, and computation nodes of a DAG, respectively.

    \item $\DAG[d,D]$ denotes the class of connected DAGs with $d$ input nodes, $D$ output nodes, and finitely many computation nodes.

    \item A $\mathbb{G}$-circuit is a triple $\mathcal{C}\eqdef (V,E,\mathcal{G})$, where $D=(V,E)$ is a DAG and $\mathcal{G}\eqdef \{g_v\}_{v\in V}$ is a family of gates decorating its vertices.

    \item $\Rep[\mathcal C]$ denotes the function represented, or computed, by the circuit $\mathcal C$.

    \item $x^{(v)}$ denotes the value carried by the node $v$ during the recursive evaluation of a circuit.

    \item $\Pi$ denotes a lifting channel, used to duplicate or route the input variables before circuit evaluation.

    \item $\Delta$, $\Upsilon$, and $N$ denote circuit depth, circuit width, and circuit size, respectively.

    \item $\operatorname{G}^{\Delta,\Upsilon}_{d_0}$ denotes the class of functions computable by $\mathbb{G}$-circuits with lifting dimension $d_0$, depth $\Delta$, and width $\Upsilon$.

    \item $\Phi_{\mathbb{G}}$ denotes the $\mathbb{G}$-propagator, which controls the growth of bounded boxes under the gates in $\mathbb{G}$.

    \item $\mathcal{R}_{\mathbb{G}}$ denotes the $\mathbb{G}$-regulator, which controls the worst-case modulus of continuity of gates in $\mathbb{G}$.

    \item $\mathcal{H}_{\mathbb{G},\mathcal{A}}$ denotes the $(\mathbb{G},\mathcal{A})$-hardness map, which records the cost of emulating gates from $\mathbb{G}$ by $\mathcal{A}$NNs.

    \item $\mathcal{H}_{\mathbb{G},\mathcal{A}}(M,\varepsilon)
    =(\Delta_{M,\varepsilon},\Upsilon_{M,\varepsilon},N_{M,\varepsilon})$ denotes the depth, width, and size required to approximate a gate on $[-M,M]^{d_g}$ to accuracy $\varepsilon$.

    \item $\delta_\cdot=(\delta_l)_{l=0}^{\Delta}$ denotes a sequence of layerwise gate-approximation errors.

    \item $\mathfrak{S}_j$ and $\mathfrak{E}_j$ denote the propagated size bound and propagated error bound after layer $j$, respectively.

    \item $B_{p,q}^s(\mathcal X)$ denotes a Besov space on $\mathcal X$ with smoothness $s$, integrability $p$, and summability $q$.

    \item $\mathcal{H}_{d,\gamma}$ denotes the holomorphic function class used in the high-regularity approximation results.

    \item $\mathcal{B}_{\varrho}$ denotes a Bernstein polyellipse, with anisotropy parameter $\varrho=(\varrho_i)_{i=1}^d$.

    \item $\Psi_\nu$ denotes a tensorized Legendre polynomial indexed by the multi-index $\nu\in\mathbb{N}_0^d$.

    \item For a multi-index $\nu\in\mathbb{N}_0^d$, $\operatorname{supp}(\nu)$ denotes its support.

    \item $\kappa\eqdef 1/\varepsilon$ denotes the reciprocal approximation error.
\end{enumerate}
\end{notation*}

\section{Funding and Acknowledgements}
\label{s:Ackn}
A.\ Kratsios acknowledges financial support from an NSERC Discovery Grant No.\ RGPIN-2023-04482 and No.\ DGECR-2023-00230.   and by the project Bando PRIN 2022 named ``Qnt4Green - Quantitative Approaches for Green Bond Market: Risk Assessment, Agency Problems and Policy Incentives'', codice 2022JRY7EF, CUP E53D23006330006, funded by European Union – NextGenerationEU, M4c2.  S.\ Brugiapaglia acknowledges support from NSERC through grant RGPIN-2020-06766 and FRQ - Nature et Technologies through grant 359708. B.J.\ Kim acknowledges support from JSPS KAKENHI Grant Number JP25K24362.
A.\ Kratsios also acknowledges that resources used in preparing this research were provided, in part, by the Province of Ontario, the Government of Canada through CIFAR, and companies sponsoring the Vector Institute.\footnote{\href{https://vectorinstitute.ai/partnerships/current-partners/}{https://vectorinstitute.ai/partnerships/current-partners/}}

\bibliographystyle{acm}
\bibliography{Bookkeeping/3_References}


\end{document}